\documentclass{article}

\usepackage{hyperref}

\PassOptionsToPackage{round}{natbib}
\usepackage[final]{neurips_2019}

\usepackage[utf8]{inputenc} 
\usepackage[T1]{fontenc}    
\usepackage[french,english]{babel}
\usepackage{hyperref}       
\usepackage{url}            
\usepackage{booktabs}       
\usepackage{amsfonts}       
\usepackage{nicefrac}       
\usepackage{microtype}      

\usepackage{natbib}

\usepackage{amsmath}
\usepackage{amsthm}
\usepackage{amssymb}

\usepackage{algorithm}
\usepackage{algorithmic}
\usepackage[titlenumbered,ruled,noend,algo2e]{algorithm2e}

\SetCommentSty{mycommfont}
\SetEndCharOfAlgoLine{}

\usepackage{stmaryrd} 
\usepackage{booktabs}  

\usepackage{subcaption}

\usepackage{color}
\definecolor{linkcolor}{RGB}{83,83,182}
\definecolor{citecolor}{RGB}{128,0,128}
\hypersetup{
    colorlinks=true,
    citecolor=citecolor,
    linkcolor=linkcolor,
    urlcolor=linkcolor
}

\usepackage[capitalise,nameinlink]{cleveref} 
\usepackage{quentin_shortcuts}

\title{Handling correlated and repeated measurements with the smoothed multivariate square-root Lasso}
\author{%
  Quentin Bertrand~$^*$ \\
  Universit\'e Paris Saclay, Inria, CEA\\
  Palaiseau, 91120, France\\
  \texttt{quentin.bertrand@inria.fr} \\
  \And
  Mathurin Massias~$^*$ \\
  Universit\'e Paris Saclay, Inria, CEA\\
  Palaiseau, 91120, France\\
  \texttt{mathurin.massias@inria.fr} \\
  \AND
  Alexandre Gramfort \\
  Universit\'e Paris Saclay, Inria, CEA\\
  Palaiseau, 91120, France\\
  \texttt{alexandre.gramfort@inria.fr} \\
  \And
  Joseph Salmon \\
  Univ. Montpellier, CNRS \\
  Montpellier, France \\
  \texttt{joseph.salmon@umontpellier.fr} \\
}

\begin{document}

\maketitle

\begin{abstract}
A limitation of  Lasso-type estimators is that the optimal regularization parameter depends on the unknown noise level. Estimators such as the concomitant Lasso address this dependence
by jointly estimating the noise level and the regression coefficients. Additionally, in many applications, the data is obtained by averaging multiple measurements: this reduces the noise variance, but it dramatically reduces sample sizes and prevents refined noise modeling.
In this work, we propose a concomitant estimator that can cope with complex noise structure by using non-averaged measurements, its data-fitting term arising as a smoothing of the nuclear norm.
The resulting optimization problem is convex and amenable, thanks to smoothing theory, to state-of-the-art optimization techniques that leverage the sparsity of the solutions. Practical benefits are demonstrated on toy datasets, realistic simulated data and real neuroimaging data.

\end{abstract}


\section{Introduction}
\label{sec:introduction}

In many statistical applications, the number of parameters $p$ is much larger than the number of observations $n$.
A popular approach to tackle linear regression problems in such scenarios is to consider convex $\ell_1$-type penalties, as popularized by \citet{Tibshirani96}.
The use of these penalties relies on a regularization parameter $\lambda$ trading data fidelity versus sparsity.
Unfortunately, \cite{Bickel_Ritov_Tsybakov09} showed that, in the case of white Gaussian noise, the optimal $\lambda$
depends linearly on the standard deviation of the noise -- referred to as
\emph{noise level}.
Because the latter is rarely known in practice, one can jointly estimate the noise level and the regression coefficients, following pioneering work on concomitant estimation \citep{Huber_Dutter74,Huber81}.
Adaptations to sparse regression \citep{Owen07} have been analyzed under the names of square-root Lasso \citep{Belloni_Chernozhukov_Wang11} or scaled Lasso \citep{Sun_Zhang12}.
Generalizations have been proposed in the multitask setting, the canonical estimator being \mtlfull \citep{Obozinski_Taskar_Jordan10}. \blfootnote{$^*$ These authors contributed equally.}

The latter estimators take their roots in a white Gaussian noise model.
However some real-world data (such as magneto-electroencephalographic data) are contaminated with strongly non-white Gaussian noise \citep{Engemann_Gramfort14}.
From a statistical point of view, the non-uniform noise level case has been widely explored: \citet{Daye_Chen_Li12,Wagener_Dette12,Kolar_Sharpnack12,Dalalyan_Hebiri_Meziani_Salmon13}.
In a more general case, with a correlated Gaussian noise model, estimators based on non-convex optimization problems were proposed \citep{Lee_Liu12} and analyzed for sub-Gaussian covariance matrices \citep{Chen_Banerjee17} through the lens of penalized Maximum Likelihood Estimation (MLE).
Other estimators \citep{Rothman_Levina_Zhu10,Rai_Kumar_Daume12} assume that the inverse of the covariance (the \emph{precision matrix}) is sparse, but the underlying optimization problems remain non-convex.
A convex approach to regression with correlated noise, the Smooth Generalized Concomitant Lasso (SGCL) was proposed by \citet{Massias_Fercoq_Gramfort_Salmon17}.
Relying on smoothing techniques \citep{Moreau65,Nesterov05,Beck_Teboulle12}, the SGCL jointly estimates the regression coefficients and the noise \emph{co-standard deviation matrix} (the square root of the noise covariance matrix).
However, in applications such as M/EEG, the number of parameters in the co-standard deviation matrix ($\approx 10^4$) is typically equal to the number of observations, making it statistically hard to estimate accurately.

In this article we consider applications to M/EEG data in the context of neuroscience.
M/EEG data consists in recordings of the electric and magnetic fields at the surface or close to the head.
Here we tackle the \emph{source localization} problem, which aims at estimating which regions of the brain are responsible for the observed electro-magnetic signals: this problem can be cast as a multitask high dimensional linear regression~\citep{Ndiaye_Fercoq_Gramfort_Salmon15}.
MEG and EEG data are obtained from heterogeneous types of sensors: magnetometers, gradiometers and electrodes, leading to samples contaminated with different noise distributions, and thus non-white Gaussian noise.
Moreover the additive noise in M/EEG data is correlated between sensors and rather strong: the noise variance is commonly even stronger that the signal power.
It is thus customary to make several repetitions of the same cognitive experiment, \eg showing 50 times the same image to a subject in order to record 50 times the electric activity of the visual cortex.
The multiple measurements are then classically averaged across the experiment's repetitions in order to increase the signal-to-noise ratio.
In other words, popular estimators for M/EEG usually discard the individual observations, and rely on Gaussian \iid noise models~\citep{Ou_Hamalainen_Golland2009,Gramfort_Strohmeier_Haueisen_Hamalainen_Kowalski13}.

In this work we propose \usfull (\us), an estimator that is
\begin{itemize}[itemsep=1pt,parsep=0pt,topsep=0pt,partopsep=0pt]
    \item designed to exploit all available measurements collected during repetitions of experiments,
    \item defined as the solution of a \emph{convex} minimization problem, handled efficiently by proximal block coordinate descent techniques,
      \item built thanks to an \emph{explicit} connection with nuclear norm smoothing\footnote{Other Schatten norms are treated in \Cref{subsec:schatten_norms}.}.
      This can also be viewed as a partial smoothing of the multivariate square-root Lasso \citep{vandeGeer_Stucky16},
      \item shown (through extensive benchmarks \wrt existing estimators) to leverage experimental repetitions to improve support identification,
      \item available as open source code to reproduce all the experiments.
\end{itemize}
%

In \Cref{sec:hetero_conco_estimation}, we recall the framework of concomitant estimation, and introduce \us.
In \Cref{sec:multi_epochs}, we detail the properties of \us, and derive an algorithm to solve it.
Finally, \Cref{sec:experiments} is dedicated to experimental results.


\section{Concomitant estimation with correlated noise}
\label{sec:hetero_conco_estimation}

\paragraph{Probabilistic model}
Let $r$ be the number of repetitions of the experiment.
The $r$ observation matrices are denoted $Y^{(1)},\dots, Y^{(r)} \in \bbR^{n \times q}$ with $n$ the number of sensors/samples and $q$ the number of tasks/time samples.
The mean over the repetitions of the observation matrices is written $\bar{Y} = \frac{1}{r} \sum_{l=1}^r Y^{(l)}$.
Let $X \in \bbR^{n \times p}$ be the design (or gain) matrix, with $p$ features stored column-wise: $X=[X_{:1} | \dots | X_{:p}]$, where for a matrix $A \in \bbR^{m \times n }$
its $\jth$ column (\resp row) is denoted $A_{:j} \in \bbR^{m \times 1}$ (\resp $A_{j:} \in \bbR^{1 \times n}$.
The matrix $\Beta^* \in \bbR^{p \times q }$ contains the coefficients of the linear regression model.
%
Each measurement (\ie repetition of the experiment) follows the model:
\begin{model}\label{eq:model_multiepochs}
   \forall l \in [r], \quad Y^{(l)} = X \Beta^* + \Snoise^* \Epsilon^{(l)} \enspace,
\end{model}
where the entries of $\Epsilon^{(l)}$ are \iid samples from standard normal distributions, the $\Epsilon^{(l)}$'s are independent, and $\Snoise^* \in \cS^n_{++}$ is the co-standard deviation matrix, and $\cS_{++}^n$ (\resp $\cS_{+}^n$) stands for the set of positive (\resp semi-definite positive) matrices.
Note that even if the observations $Y^{(1)}, \dots, Y^{(r)}$ differ because of the noise $\Epsilon^{(1)}, \dots, \Epsilon^{(r)}$,  $\Beta^*$ and the noise structure $\Snoise^*$ are shared across repetitions.
%
%
\paragraph{Notation}
We write $\norm{\cdot}$ (\resp $\langle \cdot, \cdot\rangle$) for the Euclidean norm (\resp inner product) on vectors and matrices, $\norm{\cdot}_{p}$ for the $\ell_{p}$ norm, for any $p \in [1,\infty)$.
For a matrix $\Beta \in \bbR^{p \times q}$, $\norm{\Beta}_{2, 1} = \sum_{j=1}^p \norm{\Beta_{j:}}$ (\resp $\norm{\Beta}_{2, \infty} = \max_{j \in [p]} \norm{\Beta_{j:}}$), and for any $p \in [1, \infty]$, we write
$\norm{\Beta}_{\mathscr{S},p}$ for the Schatten $p$-norm (\ie the $\ell_{p}$ norm of the singular values of $\Beta$).
The unit $\ell_{p}$ ball is written $\mathcal{B}_{p}$, $p \in [1,\infty)$.
For $\Snoise_1$ and $\Snoise_2 \in \cS_{+}^n$, $\Snoise_1 \succeq \Snoise_2$  if $\Snoise_1 - \Snoise_2 \in \cS^n_+$.
When we write $\Snoise_1 \succeq \Snoise_2$ we implicitly assume that both matrices belong to $\cS_{+}^n$.
For a square matrix $A\in \bbR^{n\times n}$,  $\Tr(A)$ represents the trace of $A$ and $\norm{A}_S = \sqrt{\Tr(A^{\top} S A)}$ is the Mahalanobis norm induced by $S \in \cS_{++}^n$.
For $a, b \in \bbR$, we denote $(a)_+ = \max (a, 0) $, $a \vee b = \max(a,b)$ and $a \wedge b = \min(a,b)$.
The block soft-thresholding operator at level $\tau >0$, is denoted $\BST(\cdot, \tau)$, and reads for any vector $x$, $\BST (x, \tau) = \left( 1 - {\tau}/{\norm{x}} \right)_+ x$.
The identity matrix of size $n \times n$ is denoted $\Id_n$, and $[r]$ is the set of integers from $1$ to $r$.

\subsection{The proposed CLaR estimator}
\label{sub:model_and_proposed_estimator}

To leverage the multiple repetitions while taking into account the noise structure, we introduce the \usfull (\us):
\begin{definition}
 \us estimates the parameters of \Cref{eq:model_multiepochs} by solving:
\begin{problem}\label{eq:clar}
      (\Betaopt^{\US}, \Sopt^{\US}) \in
       \argmin_{
           \substack{\Beta \in \bbR^{p \times q} \\
           \Snoise \succeq \sigmamin \Id_n }
           }
           f(\Beta, \Snoise) + \lambda  \norm{\Beta}_{2, 1}, \text{ with } f(\Beta, \Snoise)  \eqdef \sum_{l=1}^r \tfrac{\norm{Y^{(l)} - X \Beta}_{\Snoise^{-1}}^2}{2nqr} + \frac{\Tr(\Snoise)}{2n} ,
\end{problem}
where
 $\lambda > 0$ controls the sparsity of $\Betaopt^{\US}$ and $\sigmamin > 0$ controls the smallest eigenvalue of $\Sopt^{\US}$.
\end{definition}

\subsection{Connections with concomitant Lasso on averaged data}
\label{sub:previous_estimator}

In low SNR settings, a standard way to deal with strong noise is to use the averaged observation $\bar{Y} \in \bbR^{n\times q}$ instead of the raw observations.
The associated model reads:
\begin{model}\label{eq:averaged_model}
  \bar{Y} = X\Beta^*+ \tilde{S}^* \tilde{\Epsilon} \enspace,
\end{model}
with $\tilde{S}^* \eqdef \Snoise^*/\sqrt{r}$ and $\tilde{\Epsilon}$ has $\iid$ entries drawn from a standard normal distribution.
The SNR\footnote{See the definition we consider in \cref{eq:def_SNR}.} is multiplied by $\sqrt{r}$, yet the number of
samples goes from $rnq$ to $nq$, making it statistically difficult to estimate the $\cO (n^2)$ parameters of $\Snoise^*$.
\us generalizes the \sgclfull \citep{Massias_Fercoq_Gramfort_Salmon17}, which has the drawback of only targeting averaged observations:
\begin{definition}[\sgcl, \citealt{Massias_Fercoq_Gramfort_Salmon17}]
  \sgcl estimates the parameters of \Cref{eq:averaged_model}, by solving:
  \begin{problem}\label{eq:sgcl}
      (\Betaopt^{\SGCL}, \Sopt^{\SGCL}) \in
       \argmin_{
           \substack{\Beta \in \bbR^{p \times q} \\
           \tilde{\Snoise} \succeq {\sigmamin} / {\sqrt{r}\Id_n}} %
           }
           \tilde{f}(\Beta, \tilde{\Snoise}) + \lambda \norm{\Beta}_{2, 1},
             \text{ with } \tilde{f}(\Beta, \tilde{\Snoise}) \eqdef
           \frac{\displaystyle\normin{\bar{Y} - X \Beta}_{\tilde{\Snoise}^{-1}}^2}{2nq}
           + \frac{\Tr(\tilde{\Snoise})}{2n} .
  \end{problem}
\end{definition}

\begin{remark}\label{rem:clar_gives_sgcl}
  Note that $\Sopt^{\US}$ estimates $\Snoise^*$, while $\Sopt^{\SGCL}$ estimates $\tilde{\Snoise}^* = \Snoise^* / \sqrt{r}$.
  Since we impose the constraint $\Sopt^{\US} \succeq \sigmamin \Id_n$, we rescale the constraint so that $\Sopt^{\SGCL} \succeq {\sigmamin}/{\sqrt{r}\Id_n}$ in \eqref{eq:sgcl} for future comparisons.
  Also note that \us and \sgcl are the same when $r = 1$ and $Y^{(1)} = \bar{Y}$.
\end{remark}

The justification for \us is the following: if the quadratic loss $\norm{Y-X\Beta}^2$ were used, the parameters of \Cref{eq:model_multiepochs} could be estimated by
using either $\normin{\bar{Y}-X\Beta}^2$ or $\tfrac{1}{r}  \sum \normin{Y^{(l)}-X\Beta}^2$ as a data-fitting term.
Yet, both alternatives yield the same solutions as the two terms are equal up to constants.
Hence, the quadratic loss does not leverage the multiple repetitions and ignores the noise structure.
On the contrary, the more refined data-fitting term of \us allows to take into account the individual repetitions, leading to improved performance in applications.

\section{Results and properties of \us}
\label{sec:multi_epochs}

We start this part by introducing some elements of smoothing theory \citep{Moreau65,Nesterov05,Beck_Teboulle12} that sheds some light on the origin of the data-fitting term introduced earlier.

\subsection{Smoothing of the nuclear norm}
\label{sub:smoothing_nuclear_norm}

Let us analyze the data-fitting term of \us, by connecting it to the Schatten 1-norm.
We derive a formula for the smoothing of the this norm (\cref{prop:smoothed_conco_nuc}), which paves the way
for a more general smoothing theory for matrix variables (see \cref{app_sec:smoothing}).
Let us define the following smoothing function:
\begin{align}\label{eq:smoothing_function}
    \omega_{\sigmamin}(\cdot) \eqdef \frac{1}{2} \left( \norm{\cdot}^2 +  n \right) \sigmamin \enspace,
\end{align}
and the inf-convolution of functions $f_1$ and $f_2$, $f_1 \infconv f_2 (y) \eqdef \inf_{x} f_1(x)+f_2(y-x)$.
The name ``smoothing'' used in this paper comes from the following fact: if $f_1$ is a closed proper convex function, then $f_1^* + \frac{1}{2} \normin{\cdot}^2$ is strongly convex, and thus its Fenchel transform $(f_1^* + \frac{1}{2} \normin{\cdot}^2)^* = (f_1^* + (\frac{1}{2} \normin{\cdot}^2)^*)^* = (f_1 \infconv \frac{1}{2} \normin{\cdot}^2)^{**} = f_1 \infconv \frac{1}{2} \normin{\cdot}^2$ is smooth (see \Cref{app_sub:basic_infconv_prop} for a detailed proof).

The next propositions are key to our framework and show the connection between the \sgcl, \us and the Schatten 1-norm:
\begin{restatable}[Proof in \Cref{app_sub:nuclear_norm_case}]{proposition}{smoothedconconuc}
\label{prop:smoothed_conco_nuc}
The $\omega_{\sigmamin}$-smoothing of the Schatten-1 norm, \ie the function $\nucnorm{\cdot} \infconv \omega_{\sigmamin}: \bbR^{n \times q} \mapsto \bbR$, is the solution of the following smooth optimization problem:
    \begin{align}\label{eq:lemma_nuclear_norm_smoothed}
        (\nucnorm{\cdot} \infconv \omega_{\sigmamin})(Z) =
        \min_{S \succeq \sigmamin\Id_{n}} \tfrac{1}{2} \norm{Z}_{\Snoise^{-1}}^2 + \tfrac{1}{2} \Tr(S)
        \enspace.
    \end{align}
    Moreover $(\nucnorm{\cdot} \infconv \omega_{\sigmamin})$ is a $\sigmamin$-smooth $\frac{n}{2} \sigmamin$-approximation of $\nucnorm{\cdot}$.
\end{restatable}

\begin{definition}[Clipped Square Root]\label{def:spcl}
    For $\Sigma \in \cS^n_{+}$ with spectral decomposition $\Sigma = U \diag(\gamma_1, \dots, \gamma_{n})U^\top$ ($U$ is orthogonal), let us define the \emph{Clipped Square Root} operator:
      \begin{align}\label{eq:SpCl}
        \SpCl(\Sigma, \sigmamin) = U \diag(\sqrt{\gamma_1}\vee\sigmamin, \dots, \sqrt{\gamma_n}\vee\sigmamin)U^\top\enspace.
      \end{align}
\end{definition}
\begin{restatable}[Proof in \Cref{sub:proof_of_smoothed_sgcl_us}]{proposition}{convexsmoothing}
\label{prop:smoothed_sglc_us}
Any solution of the \us \Cref{eq:clar}, $(\Betaopt, \Snoiseopt) =(\Betaopt^{\US}, \Snoiseopt^{\US})$ is also a solution of:
  \begin{align*}
    \Betaopt & = \argmin_{\Beta \in \bbR^{p \times q}} 
   \left(\nucnorm{\cdot} \infconv \omega_{\sigmamin}\right)(Z)
   + \lambda n \norm{\Beta}_{2, 1}
   \\
  \Snoiseopt &= \SpCl\big(\tfrac{1}{rq}  \hat R \hat R^\top, \sigmamin\big) \enspace,
  \text{ where } \hat R=[Y^{(1)} - X \hat \Beta |\dots| Y^{(r)} - X \hat \Beta ]  \enspace .
  \end{align*}
\end{restatable}

Properties similar to \cref{prop:smoothed_sglc_us} can be traced back to {\citet[Sec 2.2]{vandeGeer_Stucky16}, who introduced the multivariate square-root Lasso:
\begin{equation}\label{eq:multivar_sqrt}
  \Betaopt \in \argmin_{\Beta \in \bbR^{p\times q}}
  \frac{1}{n\sqrt{q}}\normin{\bar{Y} - X \Beta}_{\mathscr{S}, 1}
  + \lambda \norm{\Beta}_{2,1} \enspace,
\end{equation}
and showed that if $(\bar{Y} - X \hat \Beta)(\bar{Y} - X \hat \Beta)^{\top} \succ 0$, the latter optimization problem admits a variational\footnote{also called \emph{concomitant} formulation since minimization is performed over an additional variable \citep{Owen07,Ndiaye_Fercoq_Gramfort_Leclere_Salmon16}.} formulation:
  \begin{equation}\label{eq:conco_multivar_sqrt}
      (\Betaopt, \Sopt) \in
       \argmin_{
           \substack{\Beta \in \bbR^{p \times q}, \\
           \tilde{\Snoise} \succ 0} }
           \frac{1}{2nq} \normin{\bar{Y} - X \Beta}_{\Snoise^{-1}}^2
           + \frac{\Tr(\Snoise)}{2n}
           + \lambda \norm{\Beta}_{2, 1}.
  \end{equation}
In other words \cref{prop:smoothed_sglc_us} generalizes \citet[Lemma 3.4]{vandeGeer16} for all matrices $\bar{Y} - X \hat \Beta$, getting rid of the condition $(\bar{Y} - X \hat \Beta)(\bar{Y} - X \hat \Beta)^{\top} \succ 0$.
In the present contribution, the problem formulation in \cref{prop:smoothed_conco_nuc} is motivated by computational aspects, as it helps to address the combined non-smoothness
of the data-fitting term $\nucnorm{\cdot}$ and the penalty term $\norm{\cdot}_{2, 1}$.
Note that another smoothing of the nuclear norm was proposed in \citet[Sec. 5.2]{Argyriou_Evgeniou_Pontil08,Bach_Jenatton_Mairal_Obozinski12}:
\begin{equation}
  \label{eq:smoothing_Bach_Pontil}
    Z \mapsto \min_{\substack{S \succ 0 }}
    \frac{1}{2} \Tr [Z^\top S^{-1} Z]
    + \frac{1}{2} \Tr(S)
    + \frac{\sigmamin^2}{2} \Tr(S^{-1})
\enspace ,
\end{equation}
which is a  $\sigmamin$-smooth $n \sigmamin$-approximation of $\nucnorm{\cdot}$ (see \Cref{app_sub:prop_other_smoothing}), therefore less precise than ours.

Other alternatives to exploit the multiple repetitions without simply averaging them, would consist in investigating other Schatten $p$-norms:
\begin{equation}\label{eq:schatten_as_datafit_repet}
  \argmin_{\Beta \in \bbR^{p\times q}}
  \tfrac{1}{\sqrt{rq}}\normin{[Y^{(1)} - X \Beta | \dots | Y^{(r)} - X \Beta] }_{\mathscr{S}, p}
  + \lambda n \norm{\Beta}_{2,1} \enspace .
\end{equation}
Without smoothing, problems of the form given in \Cref{eq:schatten_as_datafit_repet} present the drawback of having two non-smooth terms, and calling for primal-dual algorithms \citep{Chambolle_Pock11}
with costly proximal operators.  
Even if the non-smooth Schatten 1-norm is replaced by the formula in \Cref{eq:lemma_nuclear_norm_smoothed},
numerical challenges remain: $\Snoise$ can approach 0 arbitrarily, hence, the gradient \wrt $\Snoise$ of the data-fitting term is not Lipschitz over the optimization domain.
Recently, \citet{Molstad19} proposed two algorithms to directly solve \Cref{eq:schatten_as_datafit_repet}: a prox-linear ADMM, and accelerated proximal gradient descent, the latter lacking convergence guarantees since the composite objective has two non-smooth terms.
Before that, \citet{vandeGeer_Stucky16}
devised a fixed point method, lacking descent guarantees.
A similar problem was raised for the concomitant Lasso by \citet{Ndiaye_Fercoq_Gramfort_Leclere_Salmon16} who used smoothing techniques to address it.
Here we replaced the nuclear norm ($p=1$) by its smoothed version $\norm{\cdot}_{\mathscr{S}, 1} \infconv \omega_{\sigmamin}$. 
Similar results for the Schatten $2$-norm and Schatten $\infty$-norm are provided in the Appendix  (\Cref{prop:smoothing_schatten_2,prop:smoothing_schatten_inf}).

\subsection{Algorithmic details: convexity, (block) coordinate descent, parameters influence}
\label{sub:algorithmic_details_convexity_and_}

We detail the principal results needed to solve \Cref{eq:clar} numerically, leading to the implementation proposed in \cref{alg:sgcl_and_clar}.
We first recall useful results for alternate minimization of convex composite problems.

{\fontsize{4}{4}\selectfont
\begin{algorithm}[t]
\SetKwInOut{Input}{input}
\SetKwInOut{Init}{init}
\SetKwInOut{Parameter}{param}
\caption{\textsc{Alternate minimization for \us}}
\Input{\quad $X, \bar{Y}, \sigmamin, \lambda, \freq, T$}
\Init{
  \quad $\Beta = 0_{p, q}$, $\Snoise^{-1}=\sigmamin^{-1} \Id_n$, $\bar{R}=\bar{Y}$, $\text{cov}_Y = \frac{1}{r} \sum_{l=1}^r Y^{(l)} Y^{(l)\top} $ \tcp{precomputed}
     }
    \For{$t =1,\dots,T$
        }
        {
        \If(\tcp*[h]{noise update}){$t = 1 \quad (\mathrm{mod} \, \freq)$}
            {
               $RR^{\top} = \text{RRT}(\text{cov}_Y, Y, X, \Beta)$ \tcp*[h]{Eq.~\eqref{eq:RRT}}

                $\Snoise \leftarrow  \SpCl(\tfrac{1}{qr} RR^\top, \sigmamin)$ \tcp*[h]{Eq.~\eqref{eq:Sigma_update_multi}}

                \lFor{$j = 1, \ldots, p $}{$L_j = X_{:j}^\top \Snoise^{-1}X_{:j}$}
            }

        \For(\tcp*[h]{coef. update}){
                $j = 1, \ldots, p $
            }
            {

                    $\bar{R} \leftarrow \bar{R} + X_{:j}\Beta_{j:}$ 
\quad;\quad
                $\displaystyle \Beta_{j:} \leftarrow \mathrm{BST}\Big(\tfrac{X_{:j}^\top \Snoise^{-1} \bar{R} }{L_j}, \tfrac{\lambda n q}{L_j}\Big)$
\quad;\quad
                    $\bar{R} \leftarrow \bar{R} - X_{:j}\Beta_{j:}$ 
            }

       }
\Return{$\Beta, \Snoise$}
\label{alg:sgcl_and_clar}
\end{algorithm}
}

\begin{restatable}[Proof in \cref{app_sub:proof_joint_convexity}]{proposition}{jointconv}
  \label{prop:joint_convexity}
  \us is jointly convex in $(\Beta, \Snoise)$. Moreover, $f$ is convex and smooth on the feasible set, and $\norm{\cdot}_{2, 1}$ is convex and separable in $\Beta_{j:} $'s, thus minimizing the objective alternatively in $\Snoise$ and in $\Beta_{j:} $'s (see \cref{alg:sgcl_and_clar}) converges to a global minimum.
\end{restatable}
Hence, for our alternate minimization implemenation, we only need to consider solving problems with $\Beta$ or $S$ fixed, which we detail in the next propositions.

\begin{restatable}[Minimization in $S$; proof in \cref{app_sub:proof_updae_S_clar}]{proposition}{updateS}
\label{prop:update_S_multiepoch}
  Let $\Beta \in \bbR^{n \times q}$ be fixed. The minimization of $f(\Beta, S)$ \wrt $S$ with the constraint $S\succeq \sigmamin\Id_n$ admits the closed-form solution:
  \begin{equation}\label{eq:Sigma_update_multi}
    \Snoise = \SpCl \bigg(\frac{1}{rq} \sum_{l=1}^r (Y^{(l)} - X\Beta)(Y^{(l)} - X\Beta)^\top, \sigmamin \bigg)
    \enspace .
  \end{equation}
\end{restatable}

\begin{restatable}[Proof in \cref{app_sub:proof_updae_B_j_clar}]{proposition}{updateB}
  \label{prop:update_Bj_clar}
  For a fixed $S\in \cS^n_{++}$, each step of the block minimization of $f(\cdot, S) + \lambda \norm{\cdot}_{2,1}$ in the $\jth$ line of $\Beta$ admits a closed-form solution:
    \begin{equation}\label{eq:update_Bj_clar}
         \Beta_{j:} = \BST \left(\Beta_{j:} + \tfrac{X_{:j}^{\top} \Snoise^{-1} (\bar Y - X\Beta)}{ \norm{X_{:j}}_{ \Snoise^{-1}}^2 } ,
                              \tfrac{\lambda n q}{\norm{X_{:j}}_{ \Snoise^{-1}}^2 } \right)
         \enspace.
    \end{equation}
\end{restatable}

As for other Lasso-type estimators, there exists $\lambda_{\max} \geq 0$ such that whenever $\lambda \geq \lambda_{\max}$, the estimated coefficients vanish.
This $\lambda_{\max}$ helps calibrating roughly $\lambda$ in practice by choosing it as a fraction of $\lambda_{\max}$.
\begin{proposition}[Critical regularization parameter; proof in \cref{app_sub:proof_lambda_max_clar}.]\label{prop:lambda_max_clar}
    For the \sgclme estimator we have:
    with $\Snoise_{\max} \eqdef \SpCl\big(\tfrac{1}{qr} \sum_{l=1}^r Y^{(l)} Y^{(l)\top}, \sigmamin\big)$,
    \begin{align}
        \forall \lambda \geq \lambda_{\max}
        \eqdef \tfrac{1}{nq} \norm{X^{\top} \Snoise_{\max}^{-1} \bar{Y}}_{2, \infty}, \quad \Betaopt^{\US} = 0
        \enspace.
    \end{align}
\end{proposition}
\textit{Convex formulation benefits.}
Thanks to the convex formulation, convergence of \cref{alg:sgcl_and_clar} can be ensured using the duality gap as a stopping criterion (as it guarantees a targeted sub-optimality level).
To compute the duality gap, we derive the dual of \Cref{eq:clar} in \cref{prop:dual_clar}.
In addition, convexity allows to leverage acceleration methods such as working sets strategies \citep{Fan_Lv2008,Tibshirani_Bien_Friedman_Hastie_Simon_Tibshirani12,Johnson_Guestrin15,Massias_Gramfort_Salmon18} or safe screening rules \citep{ElGhaoui_Viallon_Rabbani12,Fercoq_Gramfort_Salmon15} while retaining theoretical convergence guarantees.
Such techniques are trickier to adapt in the non-convex case (see \cref{app_sec:competitors}), as they could change the local minima reached.

\textit{Choice of $\sigmamin$.}
Although $\sigmamin$ has a smoothing interpretation, from a practical point of view it remains an hyperparameter to set.
As in \cite{Massias_Fercoq_Gramfort_Salmon17}, $\sigmamin$ is always chosen as follows: $\sigmamin = \norm{Y} / (1000 \times nq)$.
In practice, the experimental results were little affected by the choice of $\sigmamin$.

\begin{restatable}{remark}{costsgclclar}\label{rem:cost_sgcl_clar}
  Once $\text{cov}_Y \eqdef \tfrac{1}{r} \sum_1^r Y^{(l)} Y^{(l)\top}$ is pre-computed, the cost of updating $\Snoise$ 
  does not depend on $r$, \ie is the same as working with averaged data.
  Indeed, with $R = [Y^{(1)} - X\Beta| \dots| Y^{(r)} - X\Beta]$, the following computation can be done in $\cO(qn^2)$ (details are in \cref{app_sub:proof_cost_sgcl_clar}).
  \begin{align}\label{eq:RRT}
      RR^{\top} = \text{RRT}(\text{cov}_Y, Y, X, \Beta)
    \eqdef r\text{cov}_Y +
  r(X \Beta) (X \Beta)^\top
  - r \bar{Y}^{\top} (X \Beta)
  - r (X\Beta)^\top \bar{Y} \enspace.
  \end{align}
\end{restatable}

Statistical properties showing the advantages of using \us (over \sgcl) can be found in \cref{app_sub:stat_comparison}.
As one could expect, using $r$ times more observations improves the covariance estimation.


\section{Experiments}
\label{sec:experiments}




Our Python code (with Numba compilation, \citealt{Lam_Pitrou_Seibert15}) is released as an open source package: \url{https://github.com/QB3/CLaR}.
We compare \sgclme to other estimators: \sgcl \citep{Massias_Fercoq_Gramfort_Salmon17}, an $\ell_{2, 1}$ version of MLE \citep{Chen_Banerjee17,Lee_Liu12} (\mle), a version of the \mle with multiple repetitions (\mler), an $\ell_{2,1}$ penalized version of \mrceo \citep{Rothman_Levina_Zhu10} with repetitions (\mrcer) and the Multi-Task Lasso (\mtl, \citealt{Obozinski_Taskar_Jordan10}).
The cost of an epoch of block coordinate descent is summarized in \Cref{tab:algorithms_costs} in \Cref{app_sub:summary_algos} for each algorithm\footnote{The cost of computing the duality gap is also provided whenever available.}.
All competitors are detailed in \Cref{app_sec:competitors}.

\paragraph{Synthetic data}
Here we demonstrate the ability of our estimator to recover the support \ie the ability to identify the predictive features.
There are $n=150$ observations, $p=500$ features, $q=100$ tasks.
The design $X$ is random with Toeplitz-correlated features with parameter $\rho_X = 0.6$ (correlation between $X_{:i}$ and $X_{:j}$ is $\rho_X^{|i - j|}$), and its columns have unit Euclidean norm.
The true coefficient $\Beta^*$ has $30$ non-zeros rows whose entries are independent and normally centered distributed.
$\Snoise^*$ is a Toeplitz matrix with parameter $\rho_S$. The SNR is fixed and constant across all repetitions
\begin{align}
  \label{eq:def_SNR}
\text{SNR} \eqdef \normin{X\Beta^*} / \sqrt{r} \normin{X\Beta^*- \bar Y}\enspace.
\end{align}
For \Cref{fig:roc_curves_influ_SNR,fig:roc_curves_influ_rho_S,fig:roc_curves_influ_n_epochs}, the figure of merit is the ROC curve, \ie the true positive rate (TPR) against the false positive rate (FPR).
For each estimator, the ROC curve is obtained by varying the value of the regularization parameter $\lambda$ on a geometric grid of $160$ points, from $\lambda_{\max}$  (specific to each algorithm) to $\lambda_{\min}$, the latter also being estimator specific and chosen to obtain a FPR larger than $0.4$.

\def \figsize {0.97}
\def \figprop {0.73}
\def \colprop {0.26}

\begin{figure}[tb]
  \begin{minipage}{\columnwidth}
    \begin{minipage}{\figprop\columnwidth}
        \includegraphics[width=\figsize\linewidth]{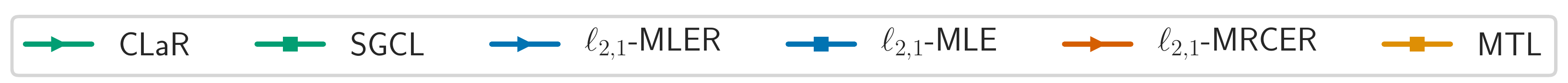}
        \centering
         \includegraphics[width=\figsize\linewidth]{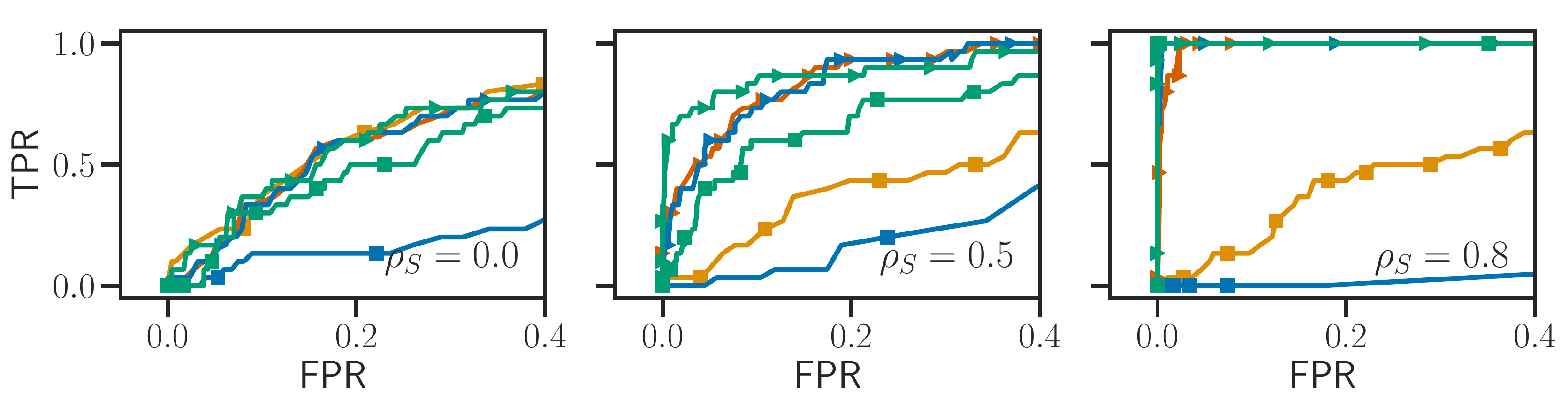}
    \end{minipage}
    \hfill
    \begin{minipage}{\colprop \columnwidth}
        \caption{\textit{Influence of noise structure.} ROC curves of support recovery ($\rho_X=0.6$, $\SNR=0.03$, $r=20$) for different $\rho_{\Snoise}$ values.}
        \label{fig:roc_curves_influ_rho_S}
    \end{minipage}
\end{minipage}  
  \begin{minipage}{\columnwidth}
      \begin{minipage}{\figprop \columnwidth}
          \centering
           \includegraphics[width=\figsize\linewidth]{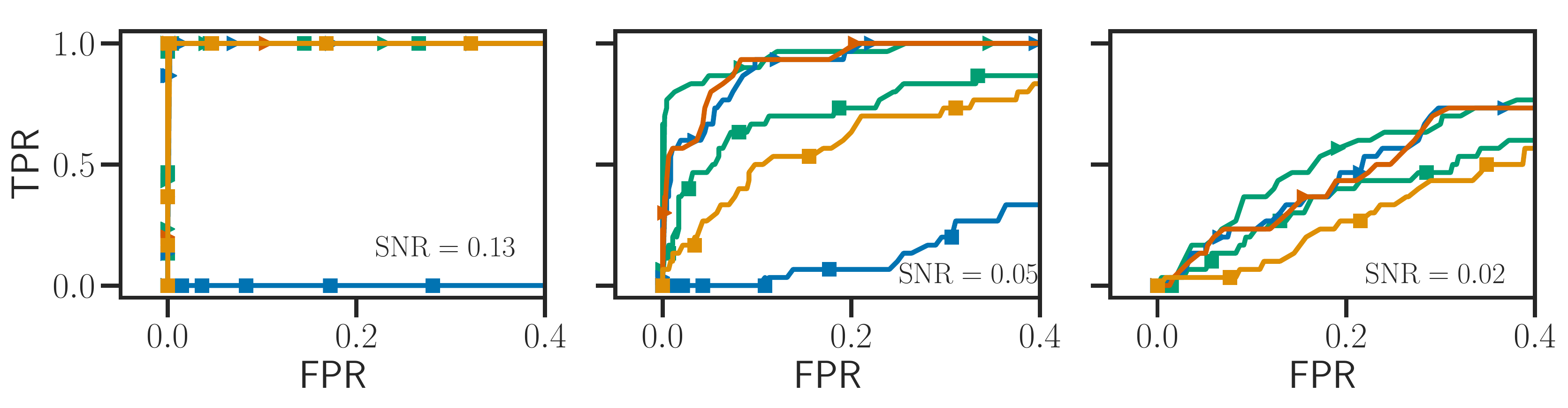}
      \end{minipage}
      \hfill
      \begin{minipage}{\colprop \columnwidth}
          \caption{\textit{Influence of SNR.} ROC curves of support recovery ($\rho_X=0.6$, $\rho_{\Snoise}=0.4$, $r=20$) for different $\SNR$ values.}
          \label{fig:roc_curves_influ_SNR}
      \end{minipage}
  \end{minipage}

\begin{minipage}{\columnwidth}
  \begin{minipage}{\figprop\columnwidth}
      \centering
       \includegraphics[width=\figsize\linewidth]{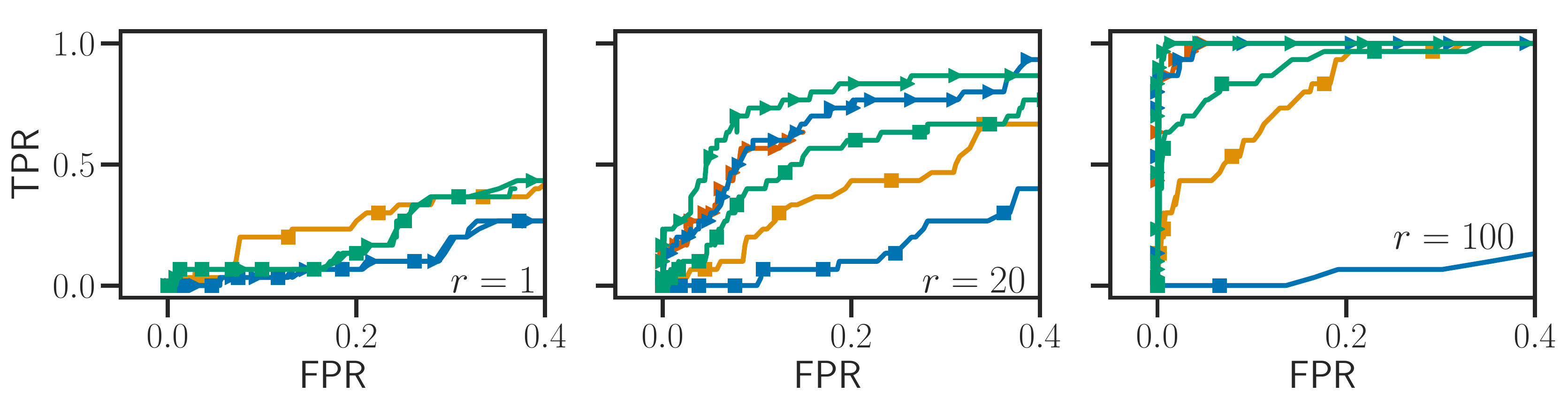}
  \end{minipage}
  \hfill
  \begin{minipage}{\colprop \columnwidth}
      \caption{\textit{Influence of the number of repetitions.} ROC curves of support recovery ($\rho_X=0.6$, $\SNR=0.03$, $\rho_{\Snoise}=0.4$) for different $r$ values.}
      \label{fig:roc_curves_influ_n_epochs}
  \end{minipage}
\end{minipage}

\end{figure}

\textit{Influence of noise structure.}
\Cref{fig:roc_curves_influ_rho_S} represents the ROC curves for different values of $\rho_{\Snoise}$.
As $\rho_\Snoise$ increases, the noise becomes more and more correlated.
From left to right, the performance of \us, \sgcl, \mrcer, \mrce, and \mler increases as they are designed to exploit correlations in the noise, while the performance of \mtl decreases, as its \iid Gaussian noise model becomes less and less valid.

\textit{Influence of SNR.}
On \Cref{fig:roc_curves_influ_SNR} we can see that when the SNR is high (left), all estimators (except \mle) reach the (0, 1) point.
This means that for each algorithm (except \mle), there exists a $\lambda$ such that the estimated support is exactly the true one.
However, when the SNR decreases (middle), the performance of \sgcl and \mtl starts to drop, while that of \us, \mler and \mrcer remains stable (\us performing better), highlighting their capacity to leverage multiple repetitions of measurements to handle the noise structure.
Finally, when the SNR is too low (right), all algorithms perform poorly, but \us, \mler and \mrcer still performs better.

\textit{Influence of the number of repetitions.}
\Cref{fig:roc_curves_influ_n_epochs} shows ROC curves of all compared approaches for different $r$, starting from $r=1$ (left) to $100$ (right).
Even with $r=20$ (middle) \us outperforms the other estimators, and when $r=100$ \us can better leverage the large number of repetitions.

\begin{figure}[tb]
  \begin{minipage}{\columnwidth}
    \begin{minipage}{\figprop\columnwidth}
        \includegraphics[width=0.97\linewidth]{roc_curves_influ_rho_S_legend}
        \centering
         \includegraphics[width=0.97\linewidth]{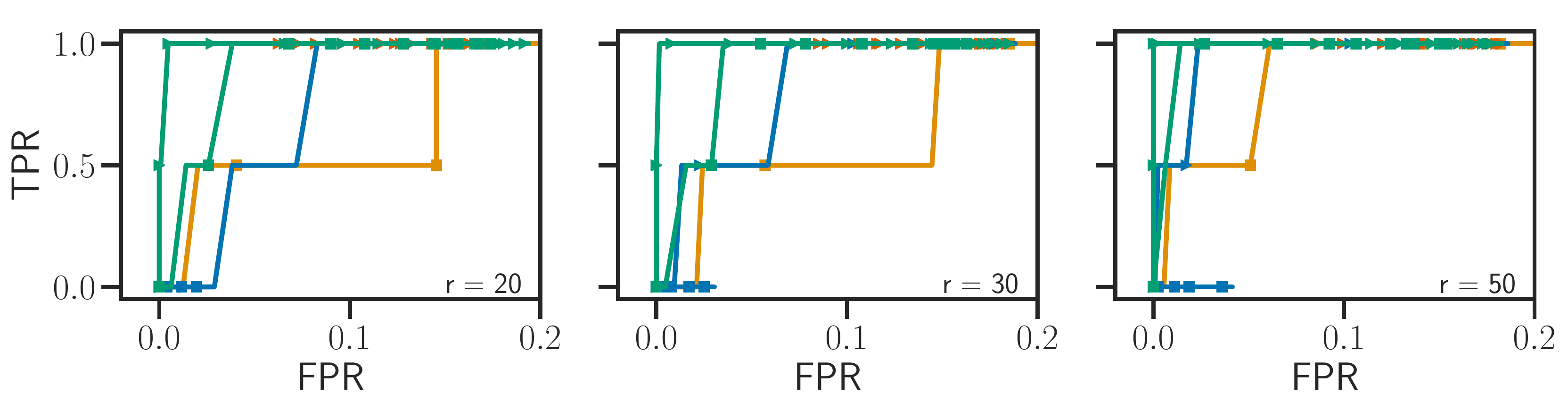}
    \end{minipage}
    \hfill
    \begin{minipage}{\colprop \columnwidth}
        \caption{\textit{Influence of the number of repetitions.} ROC curves with empirical $X$ and $\Snoise$ and simulated $\Beta^*$ ($\mathrm{amp} =2 \, \mathrm{nA.m}$), for different number of repetitions.}
        \label{fig:roc_curves_semi_real_influ_n_epochs}
    \end{minipage}
\end{minipage}
\begin{minipage}{\columnwidth}
    \begin{minipage}{\figprop \columnwidth}
          \centering
           \includegraphics[width=0.97\linewidth]{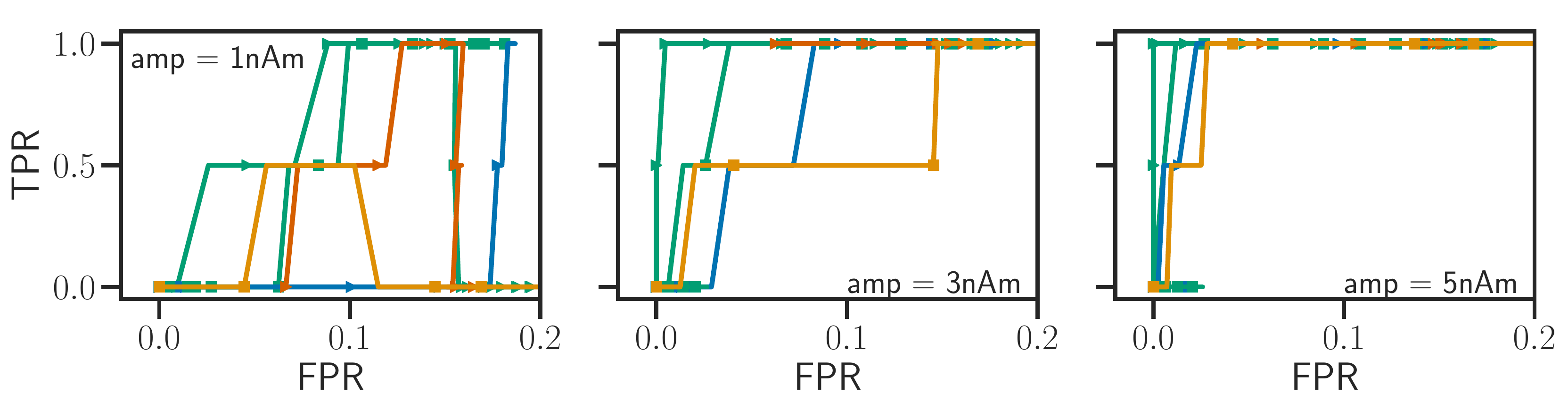}
      \end{minipage}
      \hfill
      \begin{minipage}{\colprop \columnwidth}
          \caption{\textit{Amplitude influence.} ROC curves with empirical $X$ and $\Snoise$ and simulated $\Beta^*$ ($r=50$), for different amplitudes of the signal.}
          \label{fig:roc_curves_semi_real_influ_amp}
      \end{minipage}
  \end{minipage}
\end{figure}
\paragraph{Realistic data}
We now evaluate the estimators on realistic magneto- and electroencephalography (M/EEG) data.
The M/EEG recordings measure the electrical potential and magnetic fields induced by the active neurons.
Data are time series of length $q$ with $n$ sensors and $p$ sources mapping to locations in the brain.
Because the propagation of the electromagnetic fields is driven by the linear Maxwell equations, one can assume that the relation between the measurements $Y^{(1)}, \dots, Y^{(r)}$ and the amplitudes of sources in the brain $\Beta^*$ is linear.

The M/EEG inverse problem consists in identifying $\Beta^*$.
Because of the limited number of sensors (a few hundreds in practice), as well as the physics of the problem, the M/EEG inverse problem is severely ill-posed and needs to be regularized.
Moreover, the experiments being usually short (less than 1 s.) and focused on specific cognitive functions, the number of active sources is expected to be small, \ie $\Beta^*$ is assumed to be row-sparse.
This plausible biological assumption motivates the framework of \Cref{sec:hetero_conco_estimation} \citep{Ou_Hamalainen_Golland2009}.

\textit{Dataset.}
We use the \emph{sample} dataset
\footnote{publicly available real M/EEG data recorded after auditory or visual stimulations.}
from the MNE software \citep{mne}.
The experimental conditions here are auditory stimulations in the right or left ear, leading to two main foci of activations in bilateral auditory cortices (\ie 2 non-zeros rows for $\Beta^*$).
For this experiment, we keep only the gradiometer magnetic channels. After removing
one channel corrupted by artifacts, this leads to $n=203$ signals.
The length of the temporal series is $q=100$, and the data contains $r=50$ repetitions.
We choose a source space of size $p=1281$ which corresponds to about 1\,cm distance between neighboring sources.
The orientation is fixed, and normal to the cortical mantle.

\textit{Realistic MEG data simulations.}
We use here true empirical values for $X$ and $S$ by solving Maxwell equations and taking an empirical co-standard deviation matrix.
To generate realistic MEG data we simulate neural responses $\Beta^*$ with 2 non-zeros rows corresponding to areas known to be related to auditory processing (Brodmann area 22).
Each non-zero row of $\Beta^*$ is chosen as a sinusoidal signal with realistic frequency (5\,Hz) and amplitude ($\text{amp} \sim 1-10$ nAm).
We finally simulate $r$ MEG signals $Y^{(l)}=X\Beta^* + S^*\Epsilon^{(l)}$, $\Epsilon^{(l)}$ being matrices with \iid normal entries.

The signals being contaminated with correlated noise, if one wants to use homoscedastic solvers it is necessary to whiten the data first (and thus to have an estimation of the covariance matrix, the later often being unknown).
In this experiment we demonstrate that without this whitening process, the homoscedastic solver \mtl fails, as well as solvers which does not take in account the repetitions: \sgcl and \mle.
In this scenario \us, \mler and \mrcer do succeed in recovering the sources, \us leading to the best results.
As for the synthetic data, \Cref{fig:roc_curves_semi_real_influ_amp,fig:roc_curves_semi_real_influ_n_epochs} are obtained by varying the estimator-specific regularization parameter $\lambda$ from $\lambda_{\max}$ to $\lambda_{\min}$ on a geometric grid.


\textit{Amplitude influence.} \Cref{fig:roc_curves_semi_real_influ_amp} shows ROC curves
for different values of the amplitude of the signal.
When the amplitude is high (right), all the algorithms perform well, however when the amplitude decreases (middle) only \us leads to good results, almost hitting the $(0, 1)$ corner.
When the amplitude gets lower (left) all algorithms perform worse, \us still yielding the best results.

\textit{Influence of the number of repetitions.} \Cref{fig:roc_curves_semi_real_influ_n_epochs} shows ROC curves
for different number of repetitions $r$.
When the number of repetitions is high (right, $r=50$), the algorithms taking into account all the repetitions (\us, \mler, \mrcer) perform best, almost hitting the $(0, 1)$ corner, whereas the algorithms which do not take into account all the repetitions (\mle, \mtl, \sgcl) perform poorly.
As soon as the number of repetitions decreases (middle and left) the performances of all the algorithms except \us start dropping severely.
\us is once again the algorithm taking the most advantage of the number of repetitions.

\paragraph{Real data}
As before, we use the \emph{sample} dataset, keeping only the magnetometer magnetic channels ($n=102$ signals).
We choose a source space of size $p=7498$ (about 5\,mm between neighboring sources).
The orientation is fixed, and normal to the cortical mantle.
As for realistic data, $X$ is the empirical design matrix, but this time we use the empirical measurements $Y^{(1)}, \dots, Y^{(r)}$.
The experiment are left or right auditory stimulations, extensive results for right auditory stimulations (\resp visual stimulations) can be found in \Cref{app_sub:real_right_audi} (\resp \Cref{app_sub:real_left_visual,app_sub:real_right_visual}).
As two sources are expected (one in each hemisphere, in bilateral auditory cortices), we vary $\lambda$ by dichotomy between $\lambda_{\max}$ (returning 0 sources) and a $\lambda_{\min}$ (returning more than 2 sources), until finding a $\lambda$ giving exactly 2 sources.
Results are provided in \Cref{fig:real_left_audi,fig:real_data_right_audi_r_34_main}.
Running times of each algorithm are of the same order of magnitude and can be found in \Cref{app_sub:time}.

\textit{Comments on \Cref{fig:real_left_audi}, left auditory stimulations.}
Sources found by the algorithms are represented by red spheres.
\sgcl, \mle and \mrcer completely fail, finding sources that are not in the auditory cortices at all (\sgcl sources are deep, thus not in the auditory cortices, and cannot be seen).
\mtl and \mler do find sources in auditory cortices, but only in one hemisphere (left for \mtl and right for \mler).
\us is the only one that finds one source in each hemisphere in the auditory cortices as expected.

\textit{Comments on \Cref{fig:real_data_right_audi_r_34_main}, right auditory stimulations.}
In this experiment we only keep $r=33$ repetitions (out of $65$ available) and it can be seen that only \us finds correct sources, \mtl finds sources only in one hemisphere and all the other algorithms do find sources that are not in the auditory cortices.
This highlights the robustness of \us, even with a limited number of repetitions, confirming previous experiments (see \Cref{fig:roc_curves_influ_n_epochs}).

\def \figsize {0.1674}
\def \spacefig {-0.3}

\begin{figure}[t]
  \centering 

  \begin{subfigure}{\figsize \textwidth}
  \includegraphics[width=68px,height=47px]{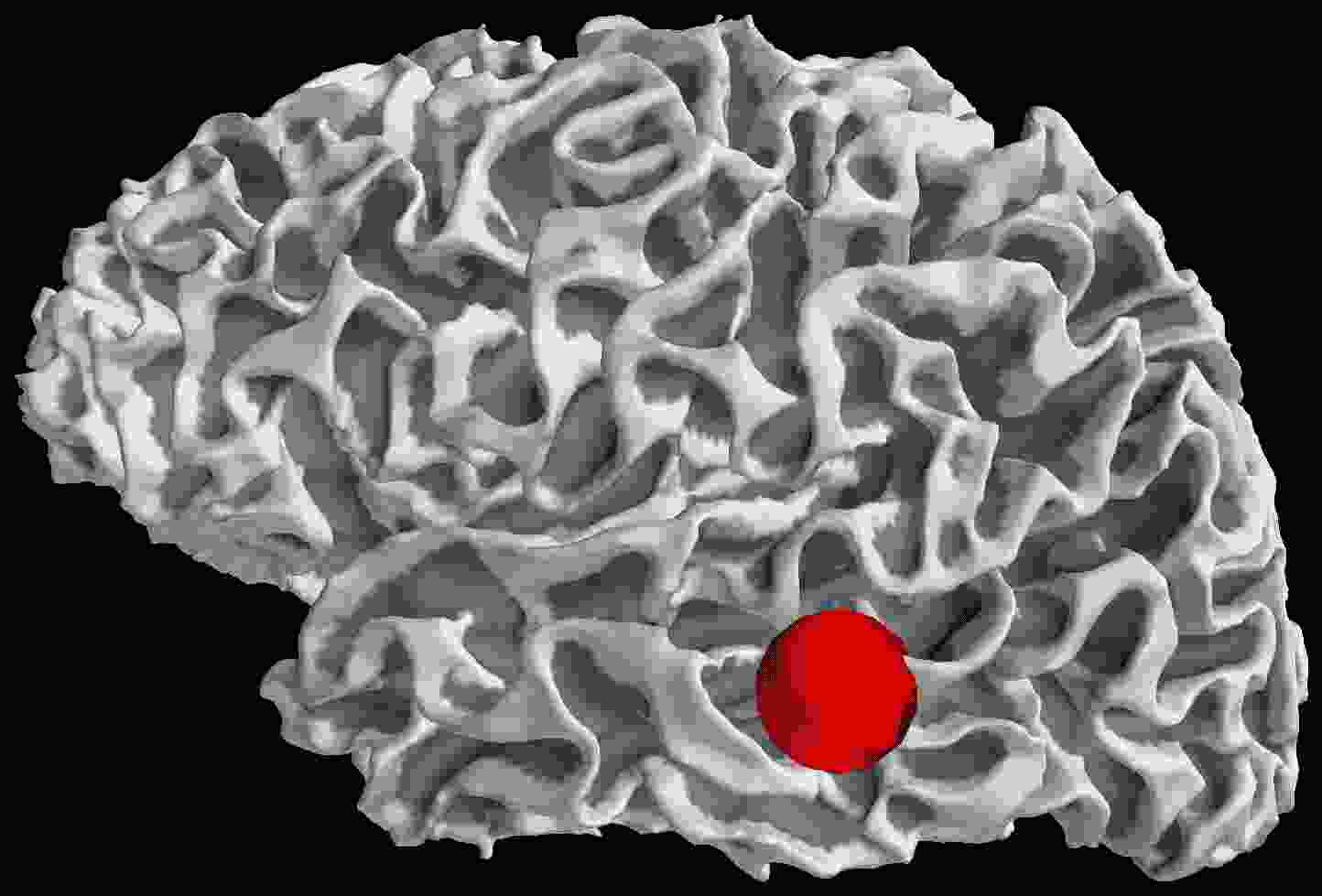}
  \end{subfigure}\hspace{\spacefig em}
  \begin{subfigure}{\figsize \textwidth}
  \includegraphics[width=68px,height=47px]{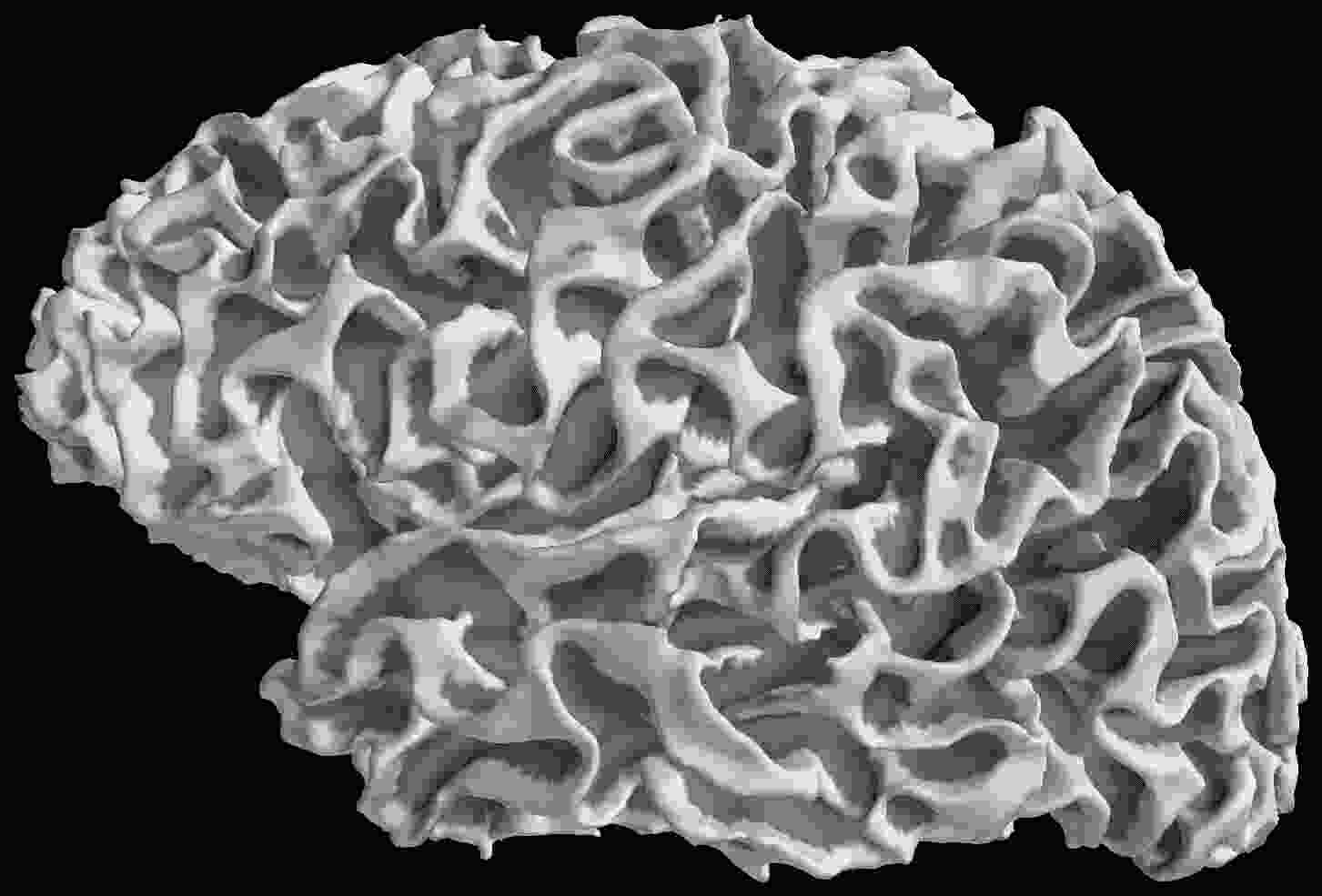}
  \end{subfigure}\hspace{\spacefig em}
  \begin{subfigure}{\figsize\textwidth}
    \includegraphics[width=68px,height=47px]{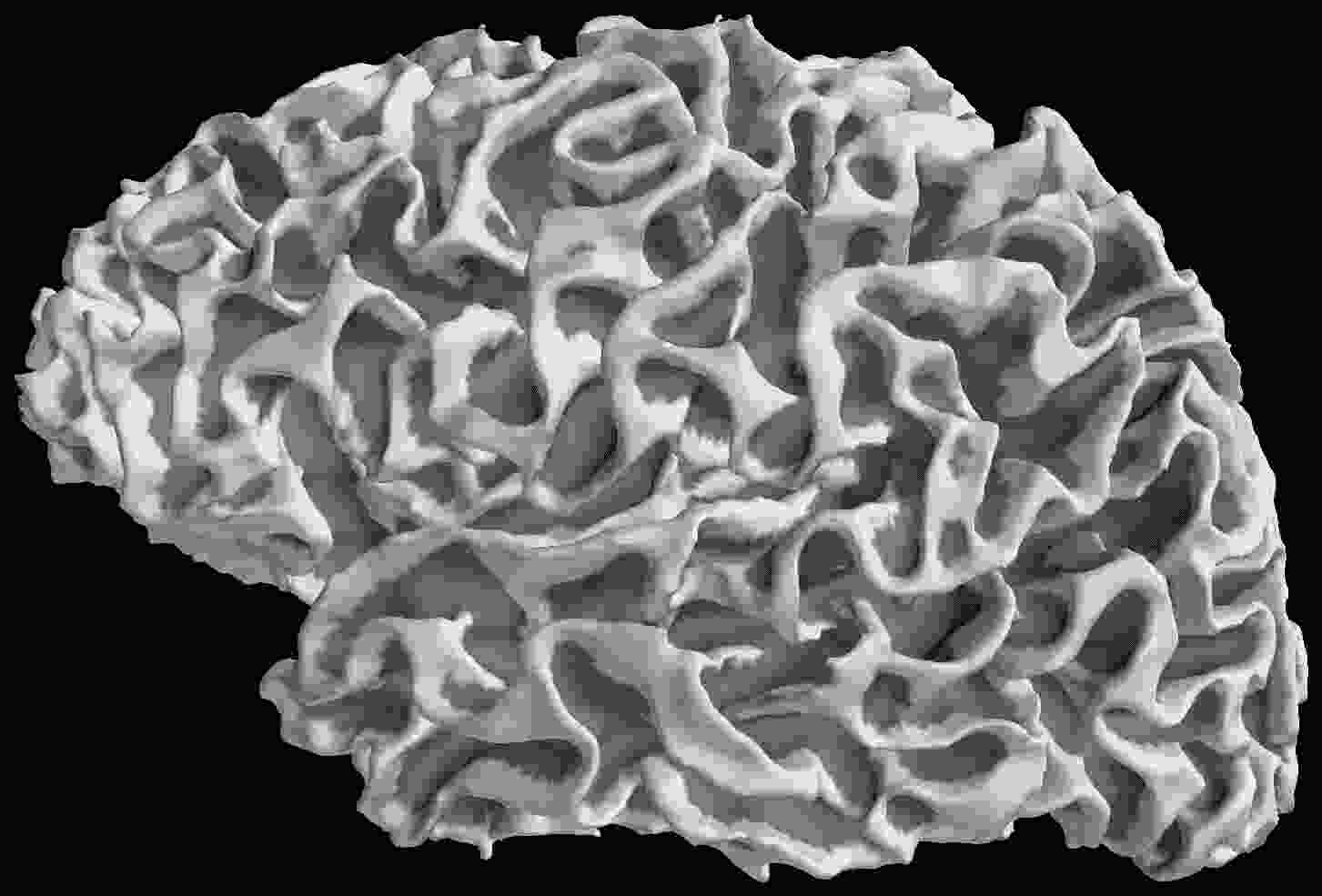}
    \end{subfigure}\hspace{\spacefig em}
  \begin{subfigure}{\figsize\textwidth}
  \includegraphics[width=68px,height=47px]{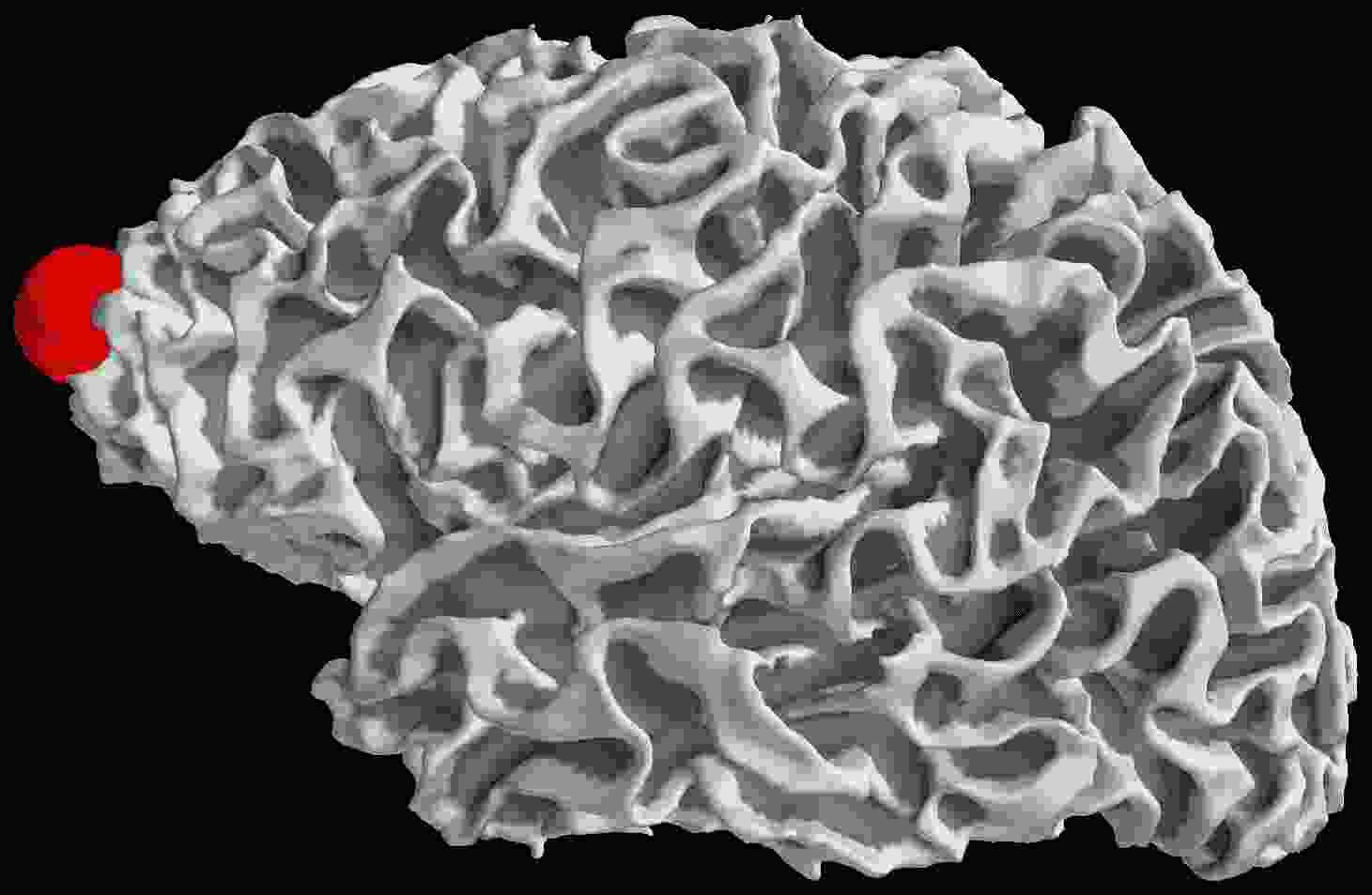}
  \end{subfigure}\hspace{\spacefig em}
  \begin{subfigure}{\figsize\textwidth}
  \includegraphics[width=68px,height=47px]{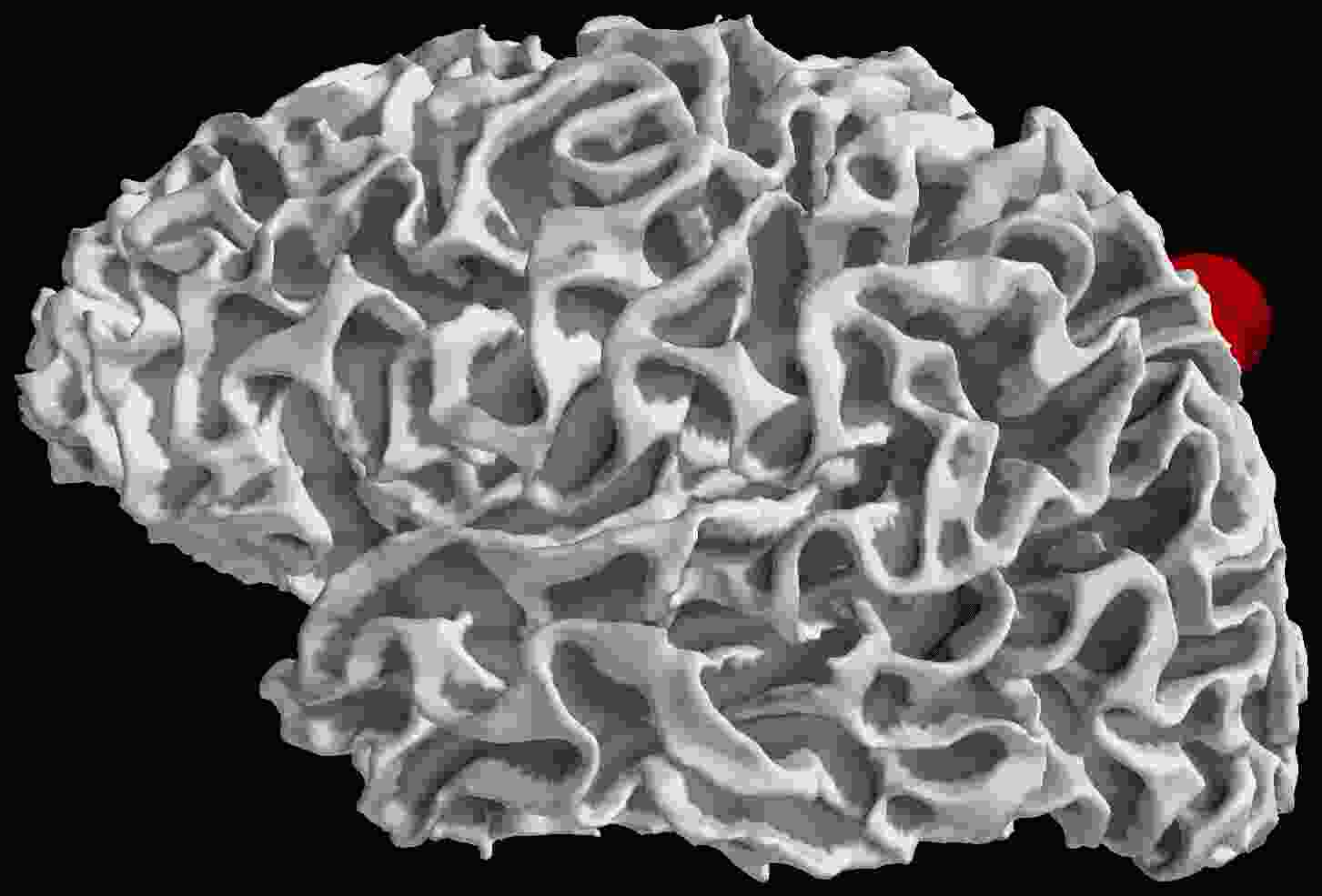}
  \end{subfigure}\hspace{\spacefig em}
  \begin{subfigure}{\figsize \textwidth}
    \includegraphics[width=68px,height=47px]{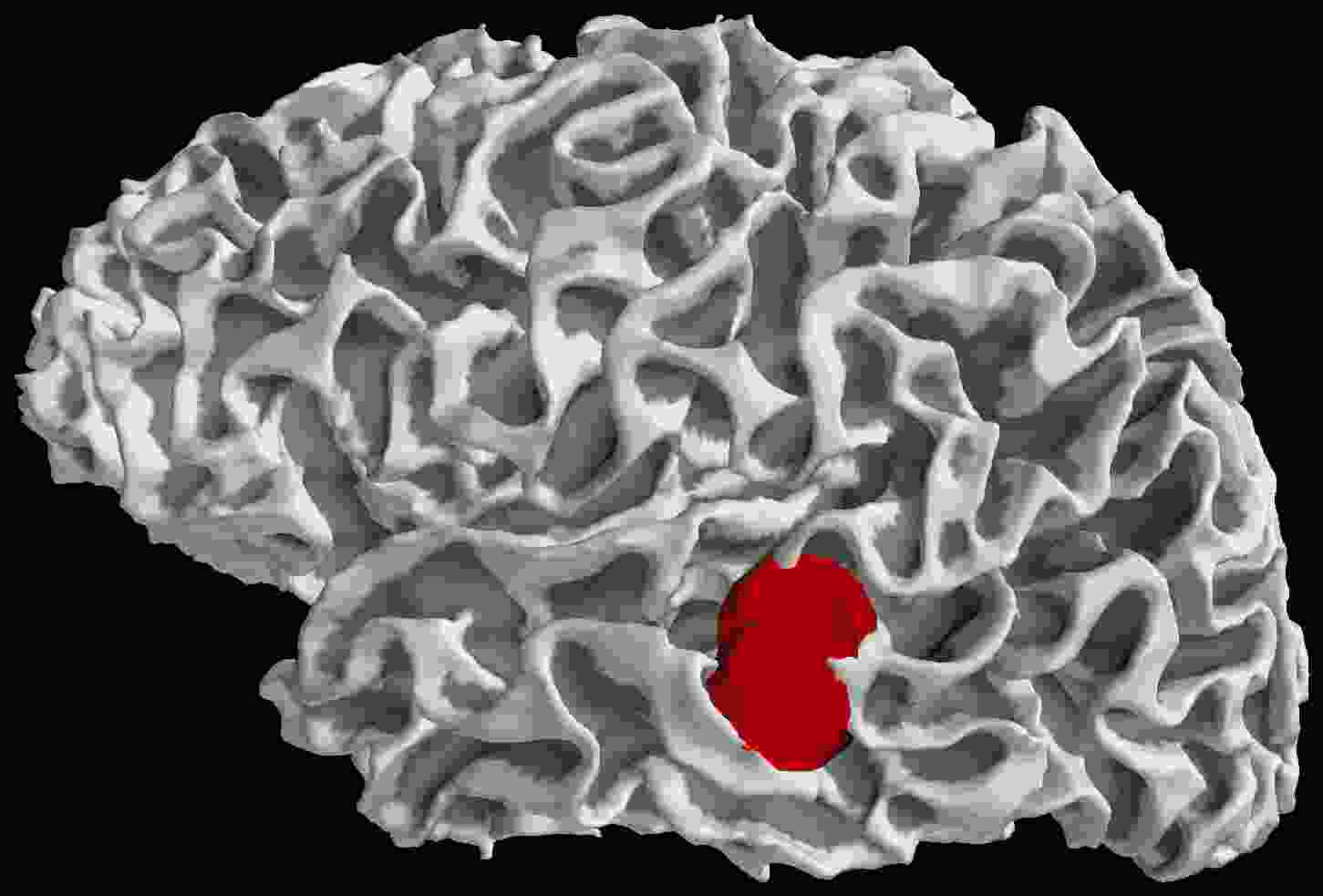}
    \end{subfigure}\hspace{\spacefig em}

  \vspace{-0.1em}
  \begin{subfigure}{\figsize\textwidth}
    \includegraphics[width=68px,height=47px]{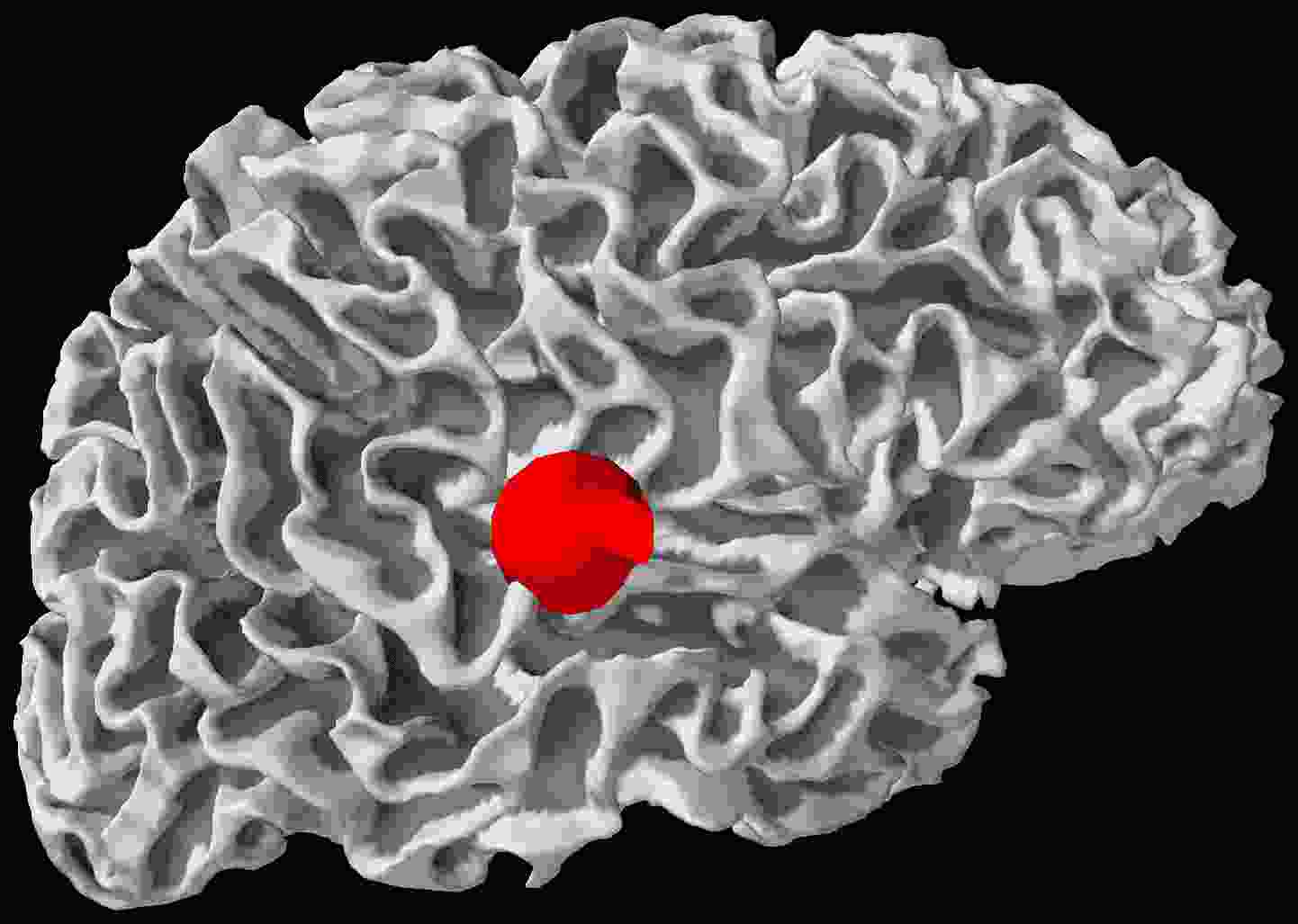}
    \caption{\us}
  \end{subfigure}\hspace{\spacefig em}
  \begin{subfigure}{\figsize\textwidth}
    \includegraphics[width=68px,height=47px]{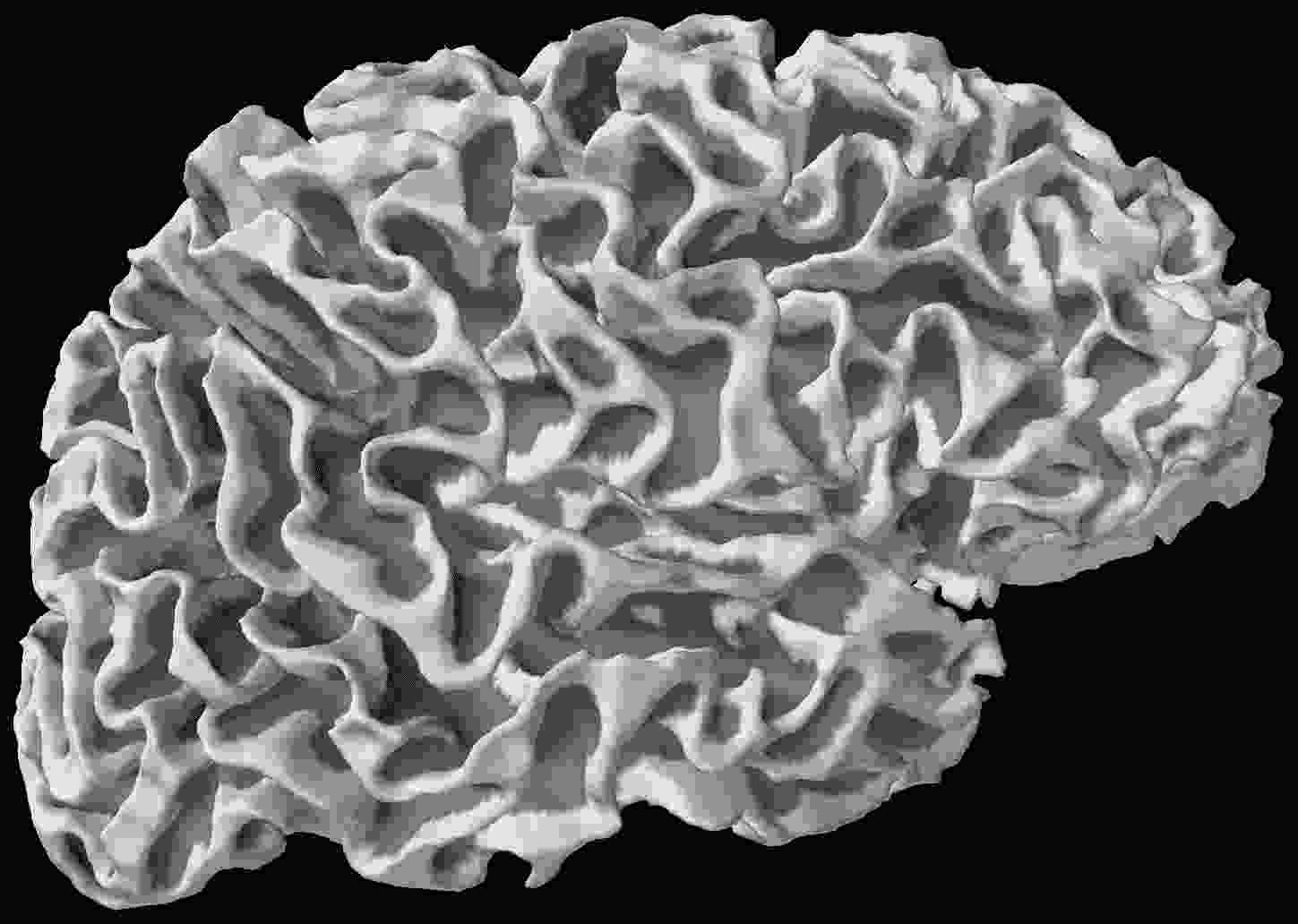}
    \caption{\sgcl}
  \end{subfigure}\hspace{\spacefig em}
  \begin{subfigure}{\figsize\textwidth}
    \includegraphics[width=68px,height=47px]{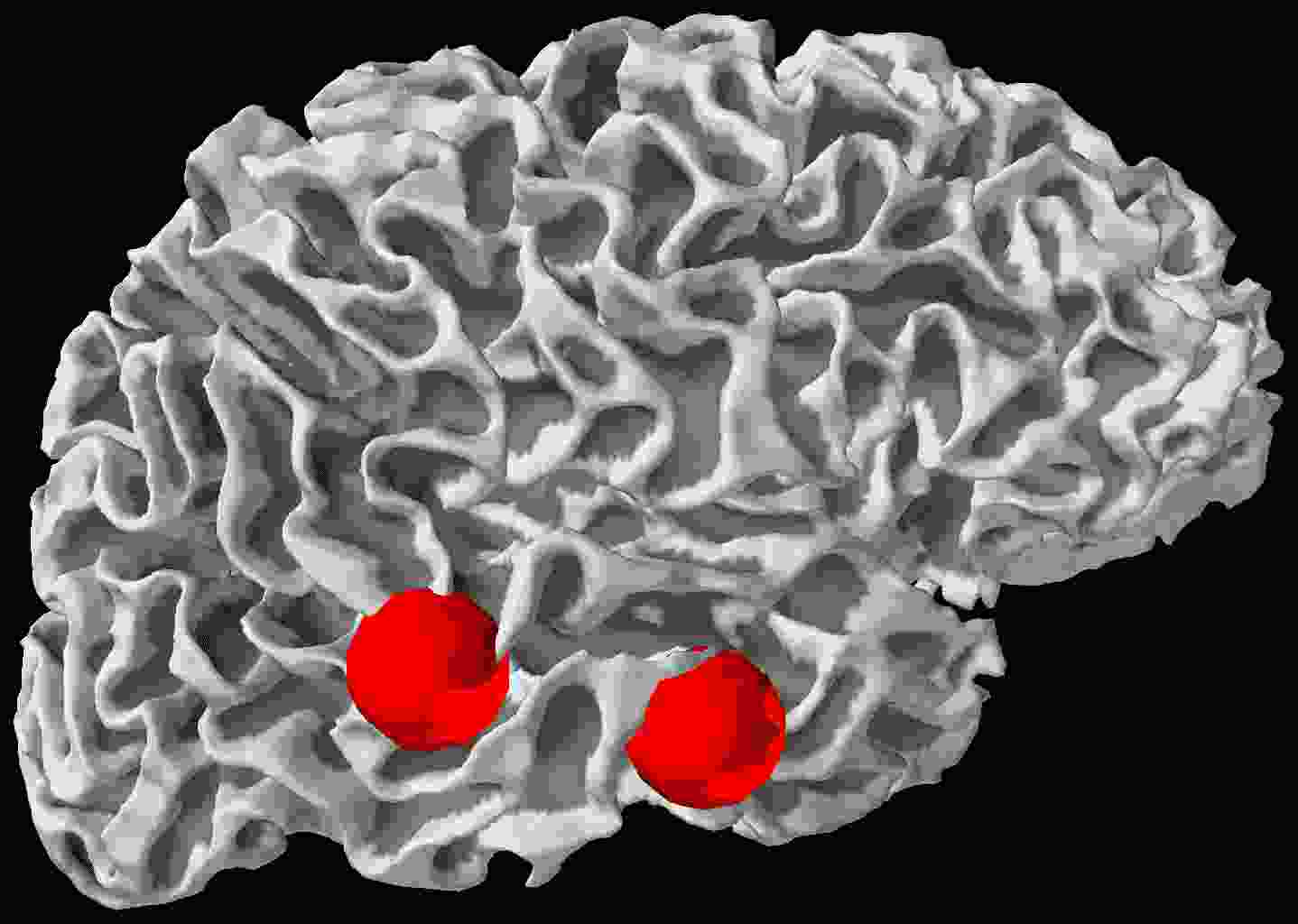}
    \caption{\mler}
    \end{subfigure}\hspace{\spacefig em}
  \begin{subfigure}{\figsize\textwidth}
    \includegraphics[width=68px,height=47px]{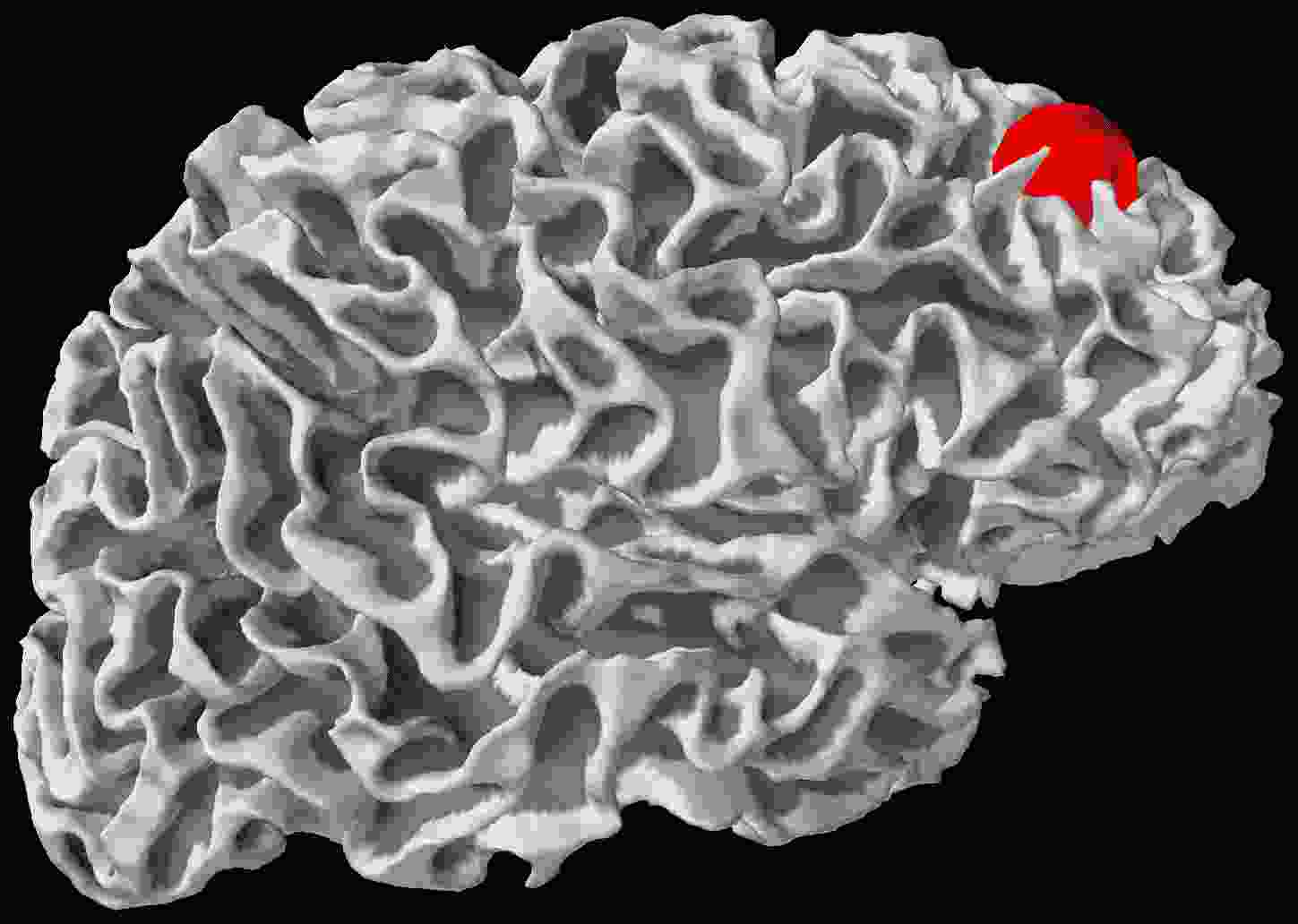}
    \caption{\mle}
    \end{subfigure}\hspace{\spacefig em}
  \begin{subfigure}{\figsize\textwidth}
    \includegraphics[width=68px,height=47px]{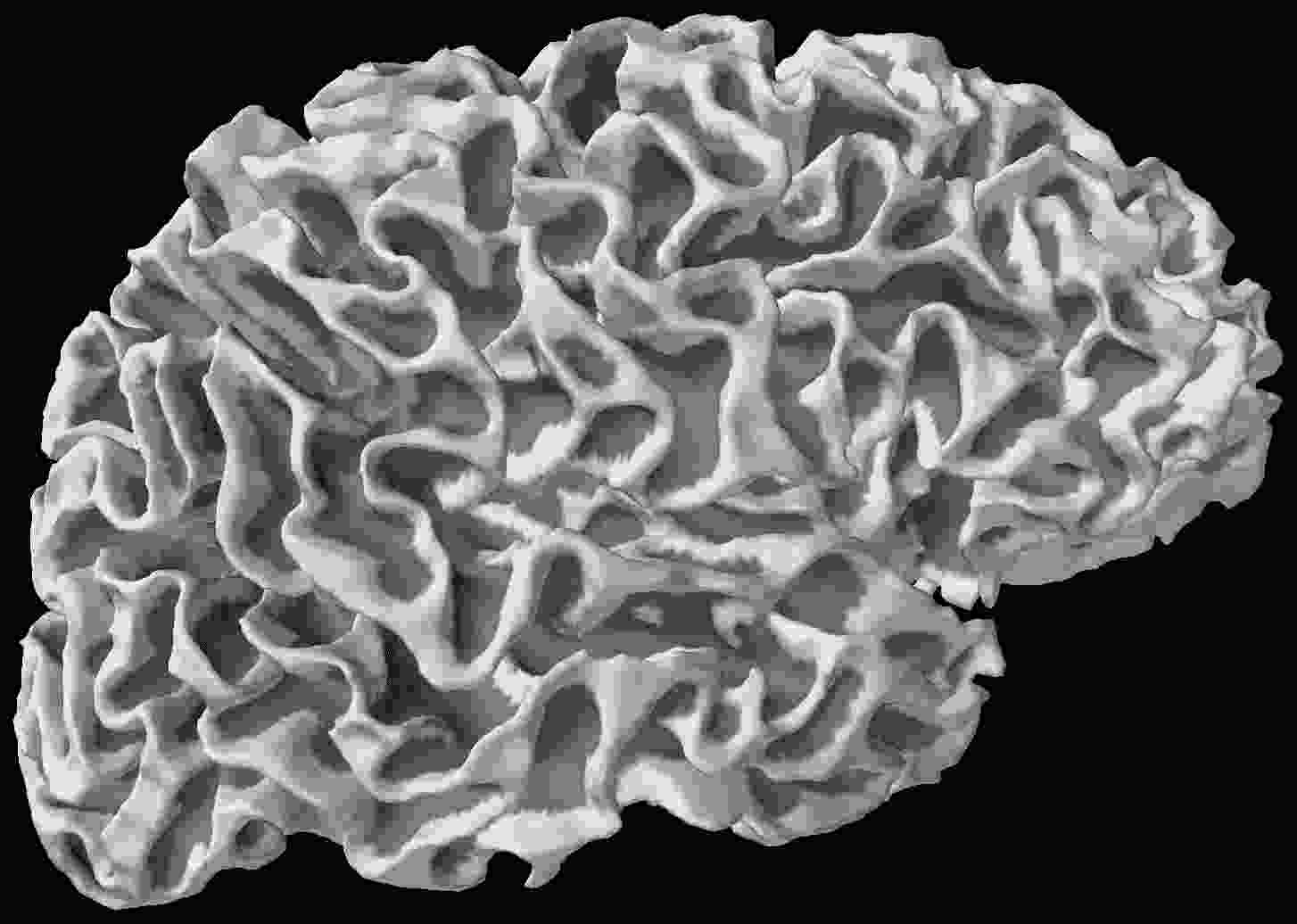}
    \caption{\mrcer}
    \end{subfigure}\hspace{\spacefig em}
  \begin{subfigure}{\figsize\textwidth}
    \includegraphics[width=68px,height=47px]{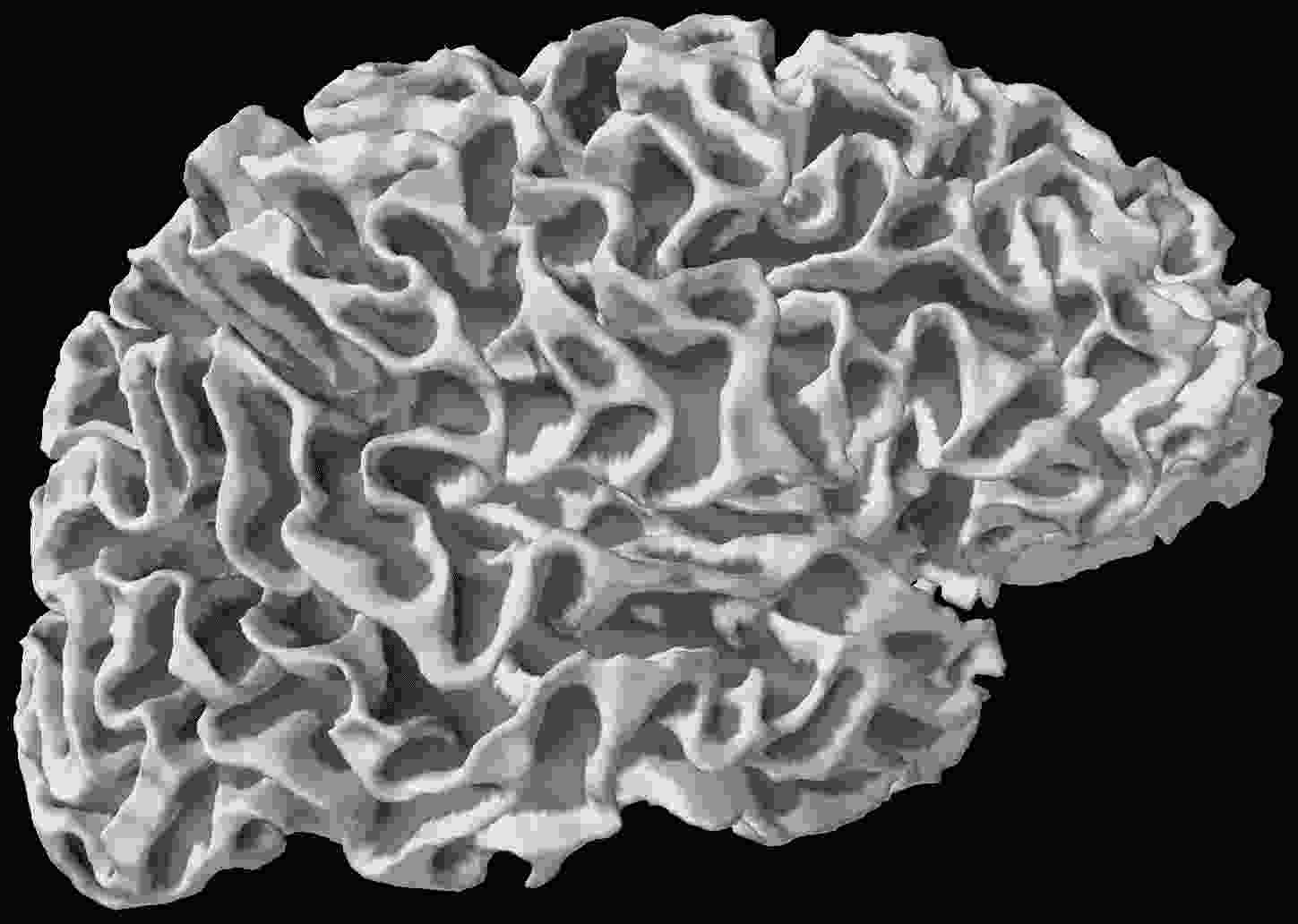}
    \caption{\mtl}
  \end{subfigure}\hspace{\spacefig em}
  \caption{\textit{Real data, left auditory stimulations ($n=102$, $p=7498$, $q=76$, $r=63$)} Sources found in the left hemisphere (top) and the right hemisphere (bottom) after left auditory stimulations.}
  \label{fig:real_left_audi}
\end{figure}



\begin{figure}[t]
  \centering 

  \begin{subfigure}{\figsize \textwidth}
  \includegraphics[width=68px,height=47px]{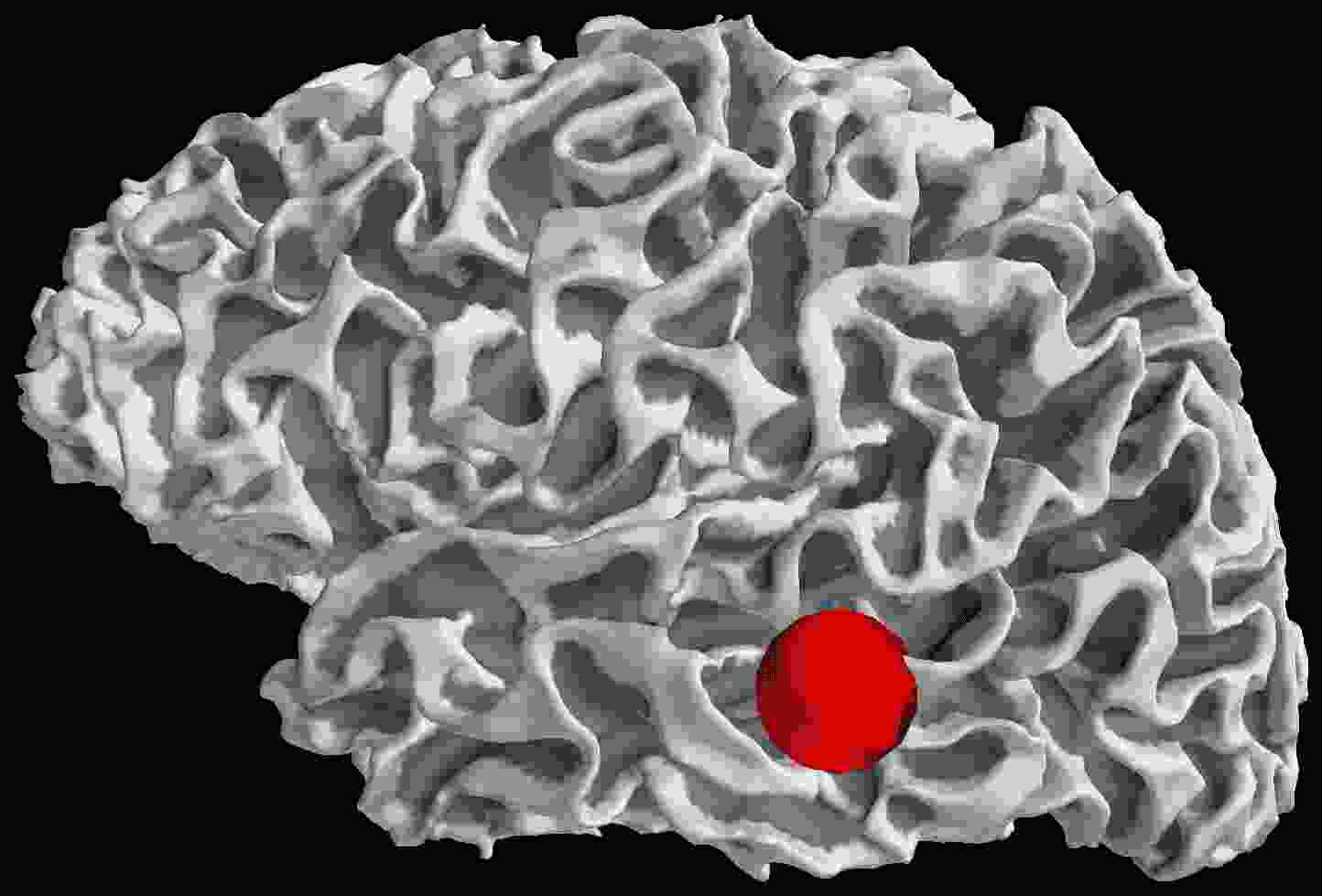}
  \end{subfigure}\hspace{\spacefig em}
  \begin{subfigure}{\figsize \textwidth}
  \includegraphics[width=68px,height=47px]{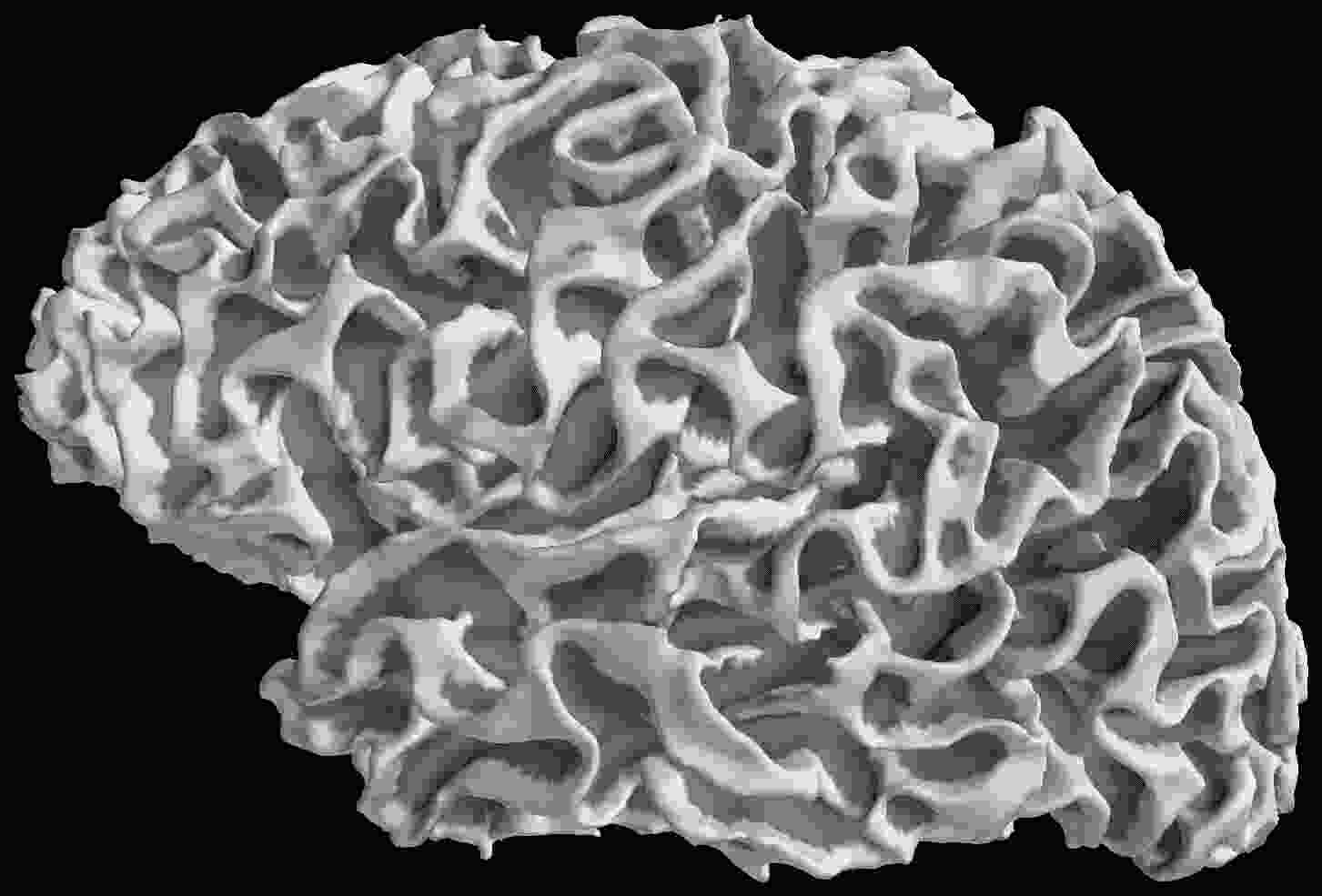}
  \end{subfigure}\hspace{\spacefig em}
  \begin{subfigure}{\figsize\textwidth}
    \includegraphics[width=68px,height=47px]{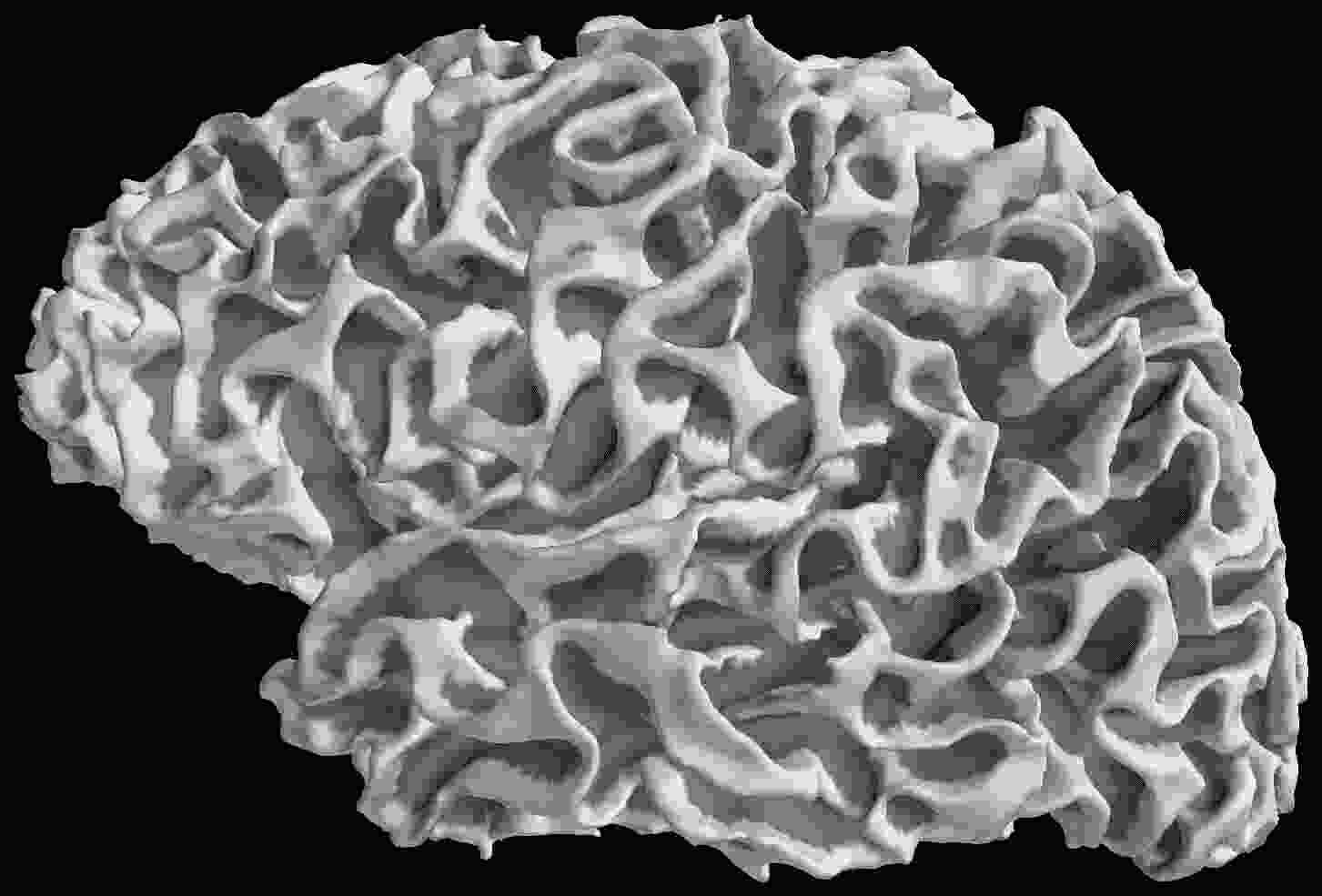}
    \end{subfigure}\hspace{\spacefig em}
  \begin{subfigure}{\figsize\textwidth}
  \includegraphics[width=68px,height=47px]{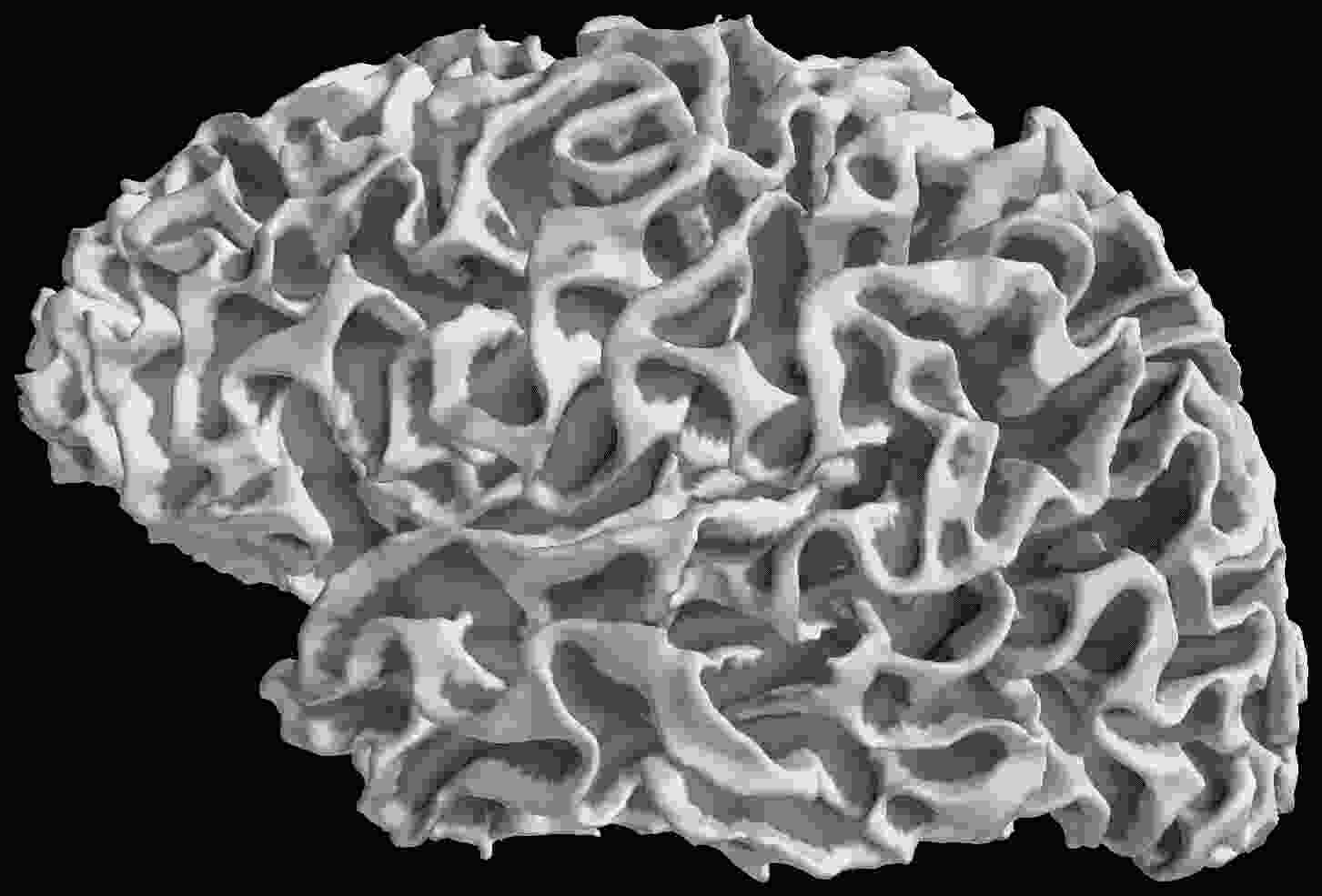}
  \end{subfigure}\hspace{\spacefig em}
  \begin{subfigure}{\figsize\textwidth}
  \includegraphics[width=68px,height=47px]{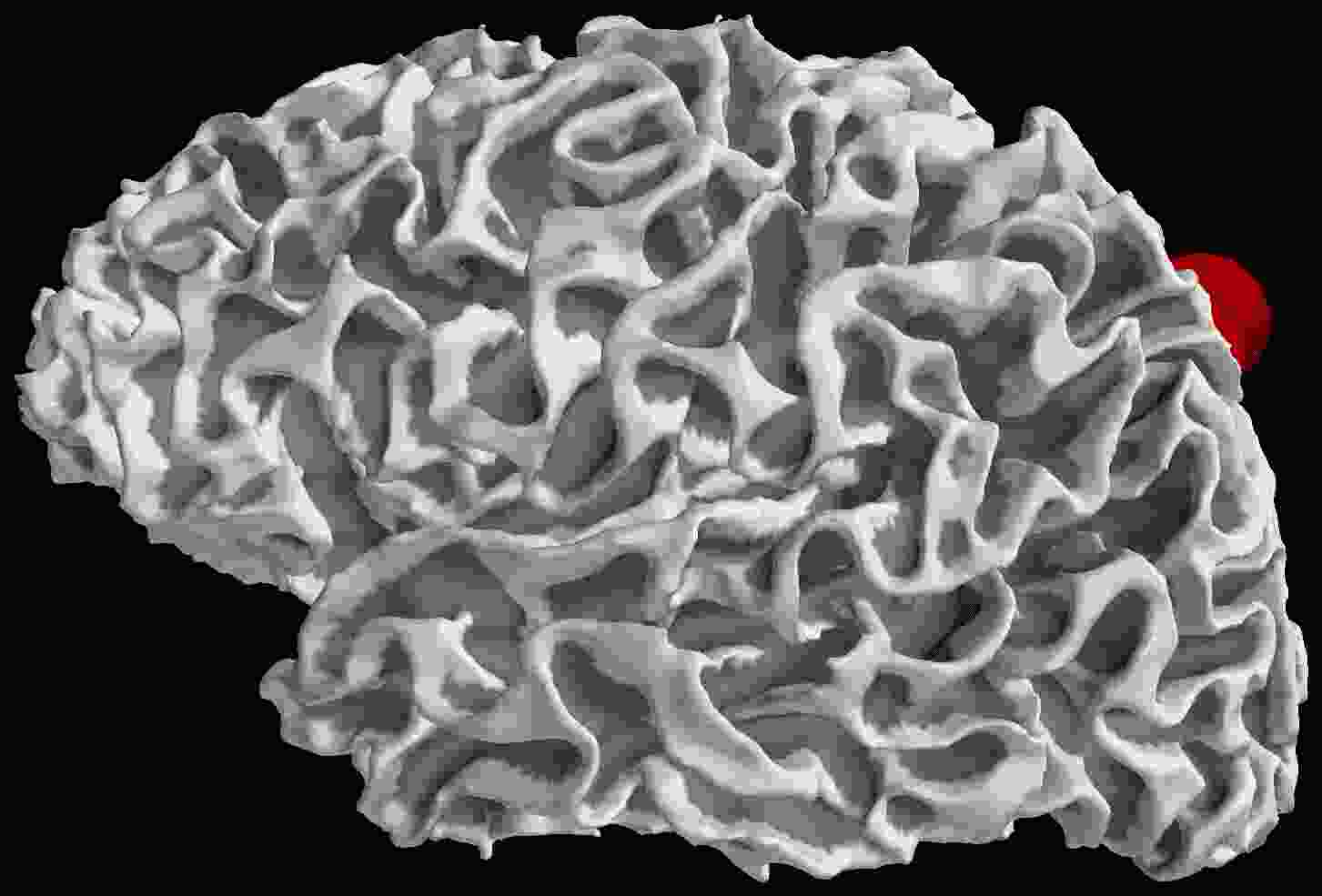}
  \end{subfigure}\hspace{\spacefig em}
  \begin{subfigure}{\figsize \textwidth}
    \includegraphics[width=68px,height=47px]{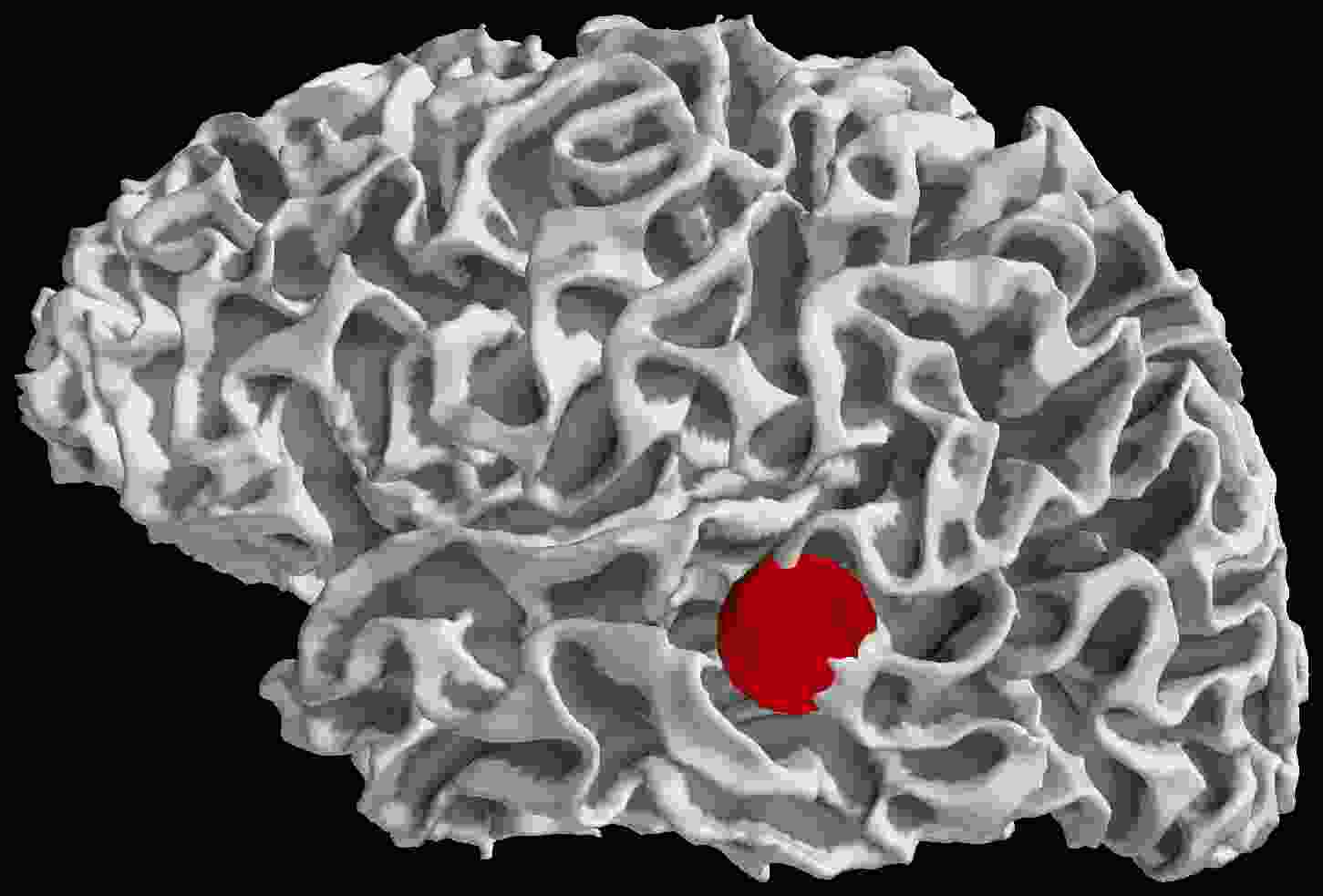}
    \end{subfigure}\hspace{\spacefig em}

  \vspace{-0.1em}
  \begin{subfigure}{\figsize\textwidth}
    \includegraphics[width=68px,height=47px]{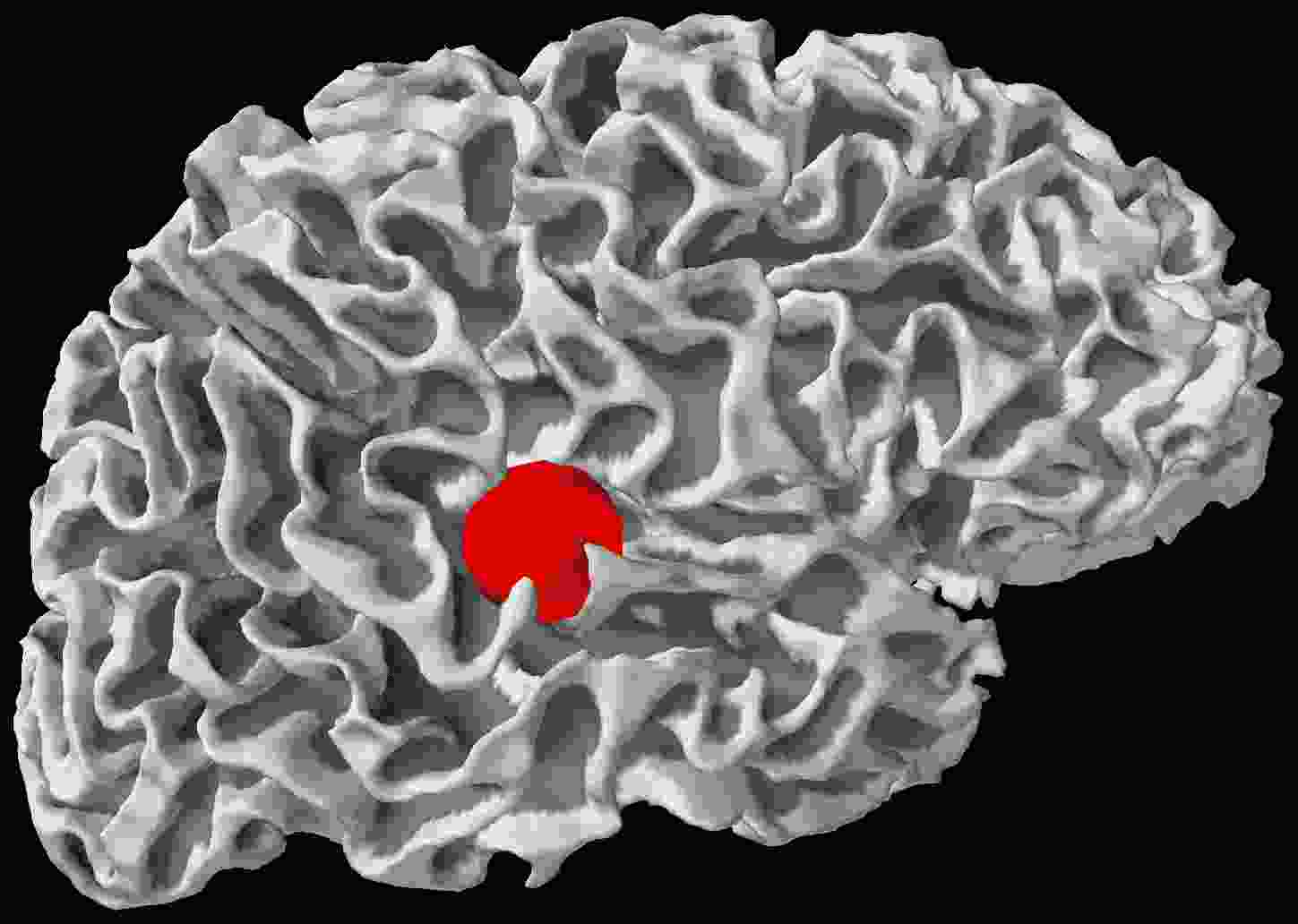}
    \caption{\us}
  \end{subfigure}\hspace{\spacefig em}
  \begin{subfigure}{\figsize\textwidth}
    \includegraphics[width=68px,height=47px]{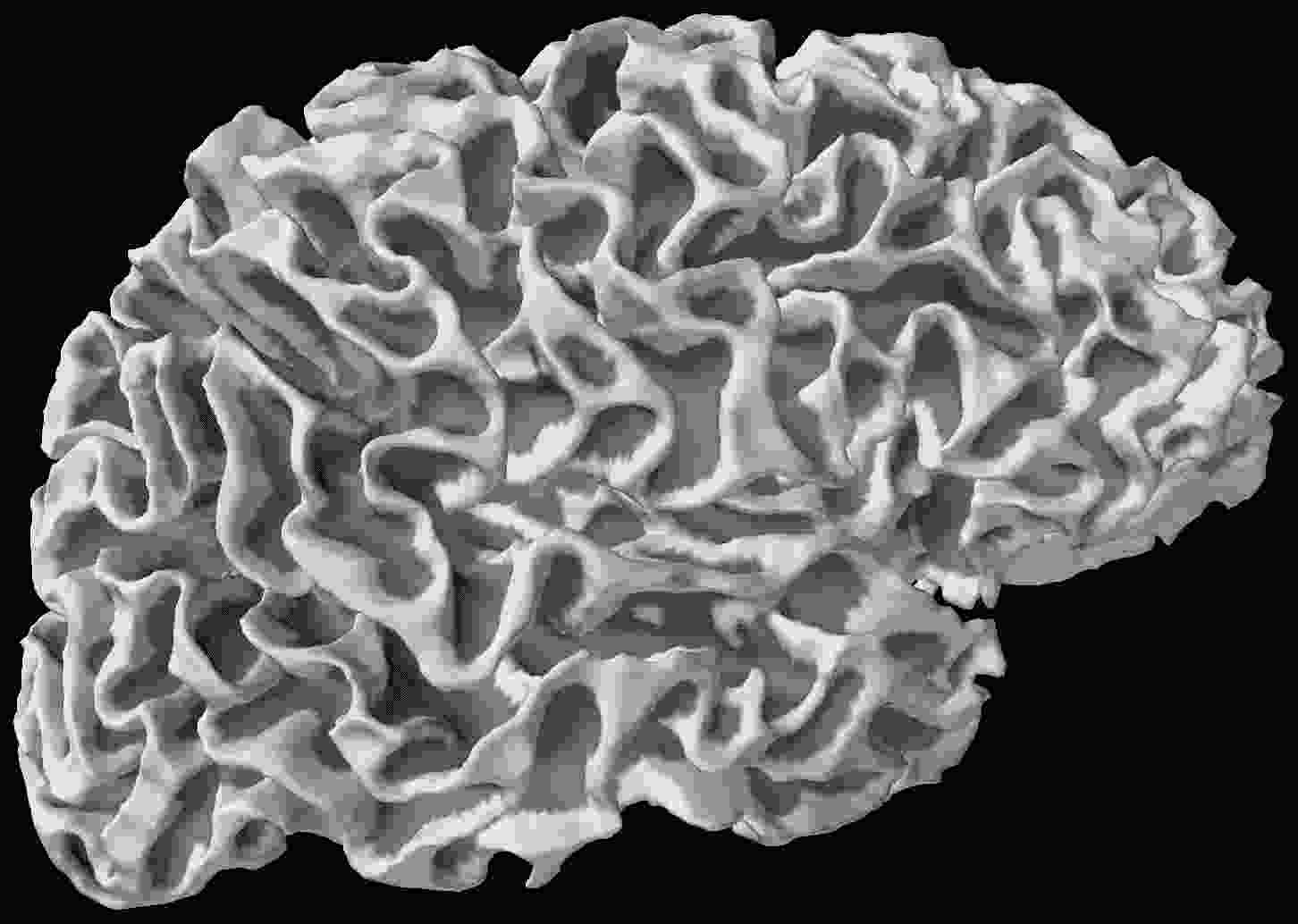}
    \caption{\sgcl}
  \end{subfigure}\hspace{\spacefig em}
  \begin{subfigure}{\figsize\textwidth}
    \includegraphics[width=68px,height=47px]{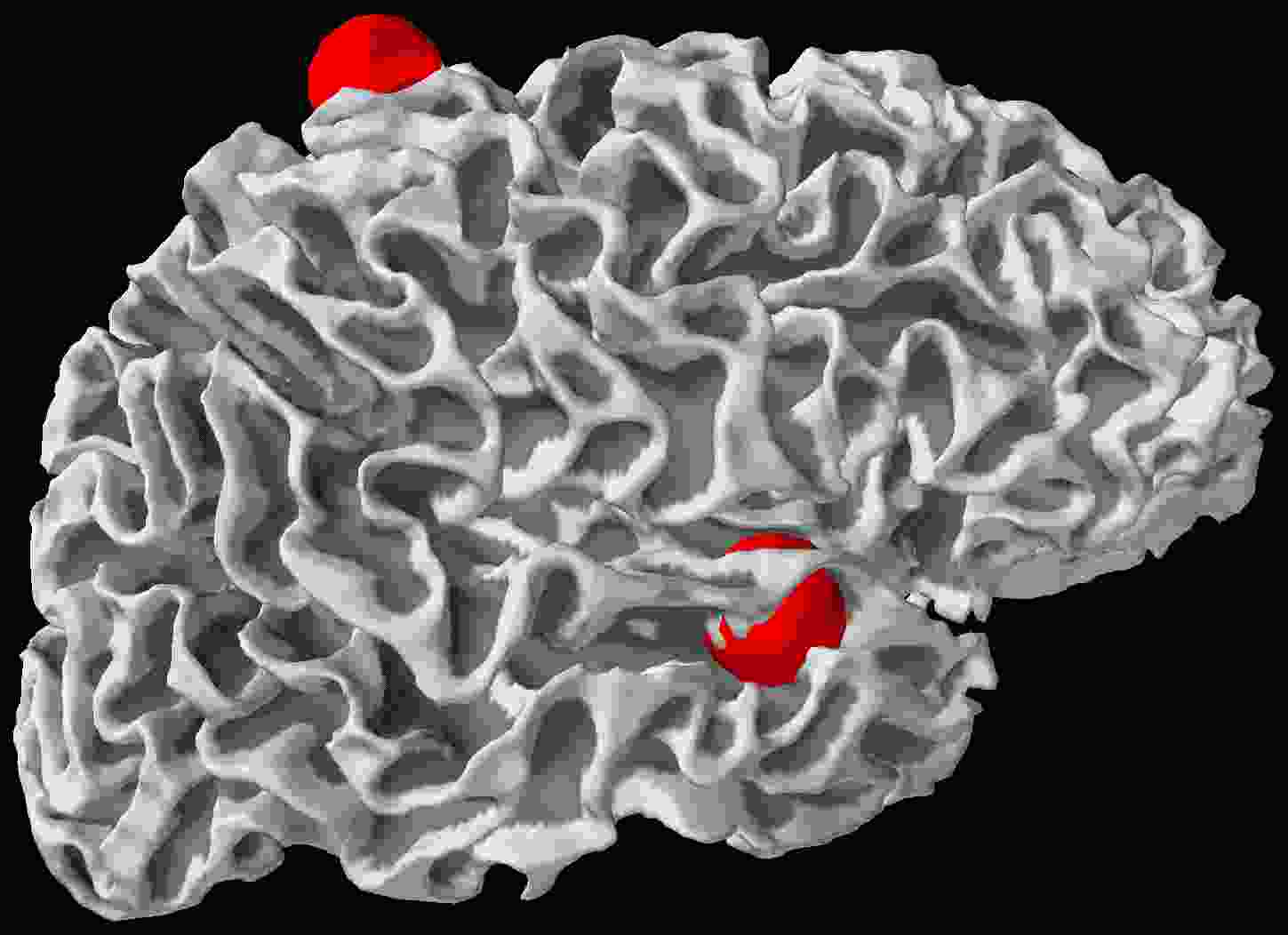}
    \caption{\mler}
    \end{subfigure}\hspace{\spacefig em}
  \begin{subfigure}{\figsize\textwidth}
    \includegraphics[width=68px,height=47px]{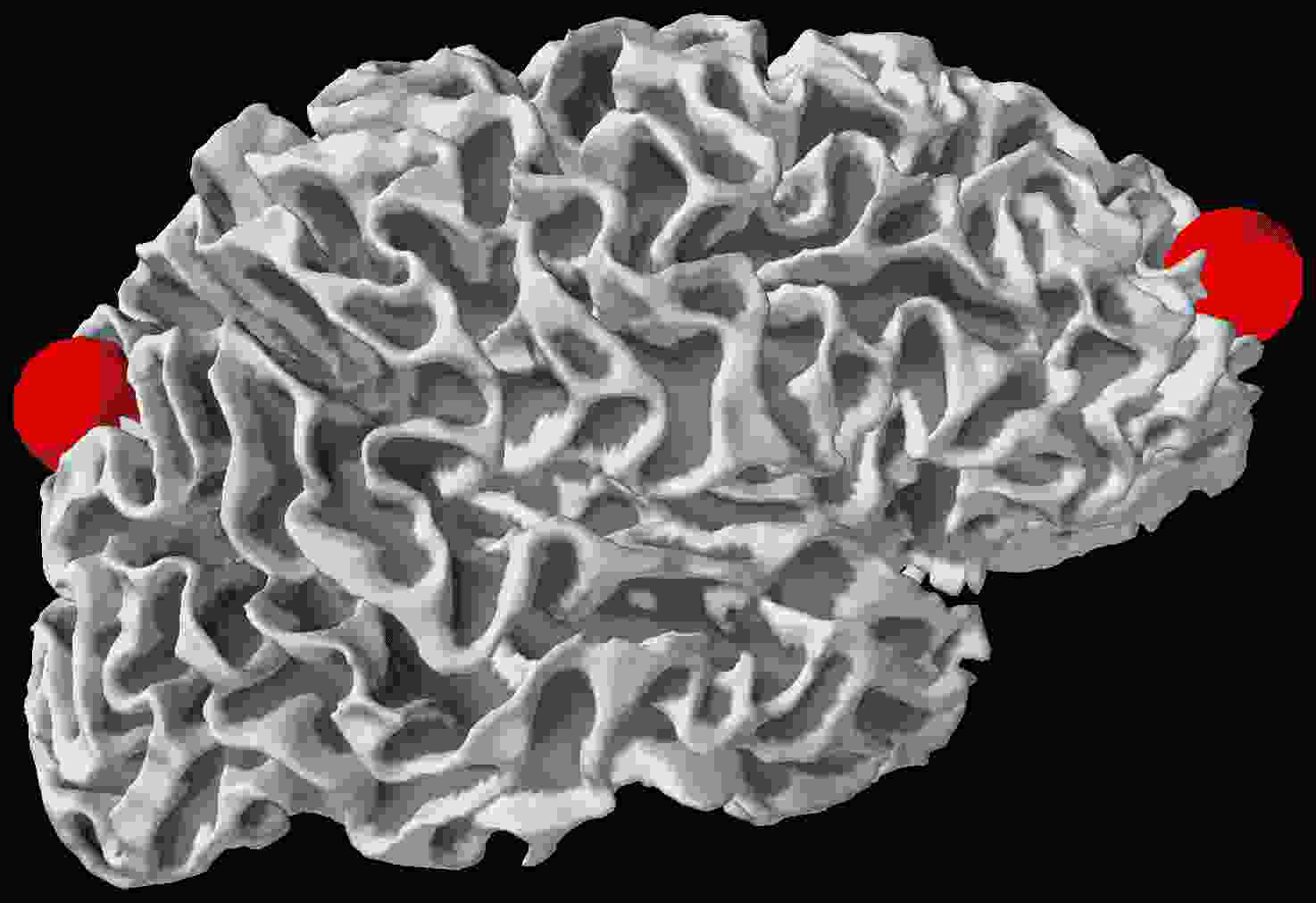}
    \caption{\mle}
    \end{subfigure}\hspace{\spacefig em}
  \begin{subfigure}{\figsize\textwidth}
    \includegraphics[width=68px,height=47px]{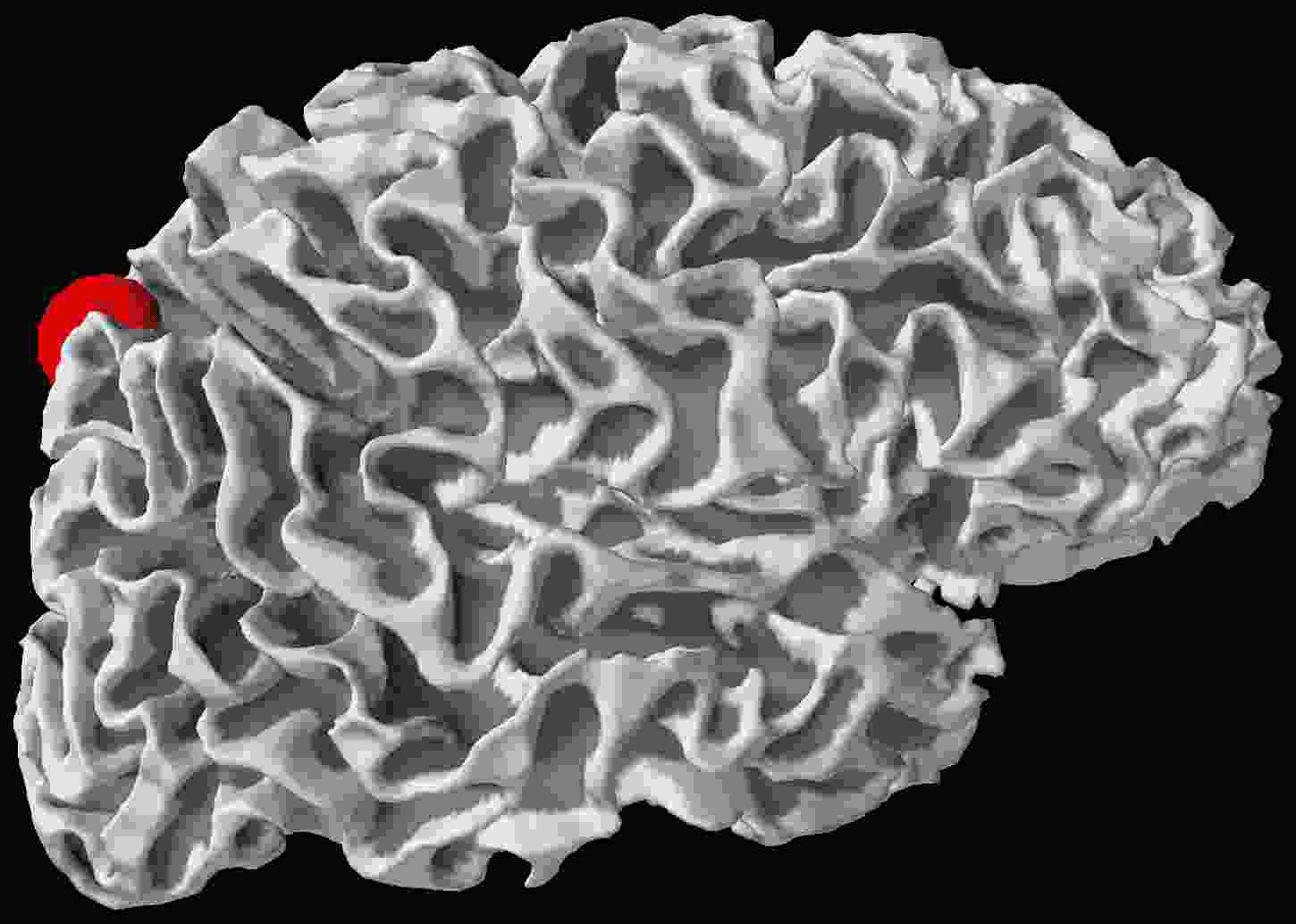}
    \caption{\mrcer}
    \end{subfigure}\hspace{\spacefig em}
  \begin{subfigure}{\figsize\textwidth}
    \includegraphics[width=68px,height=47px]{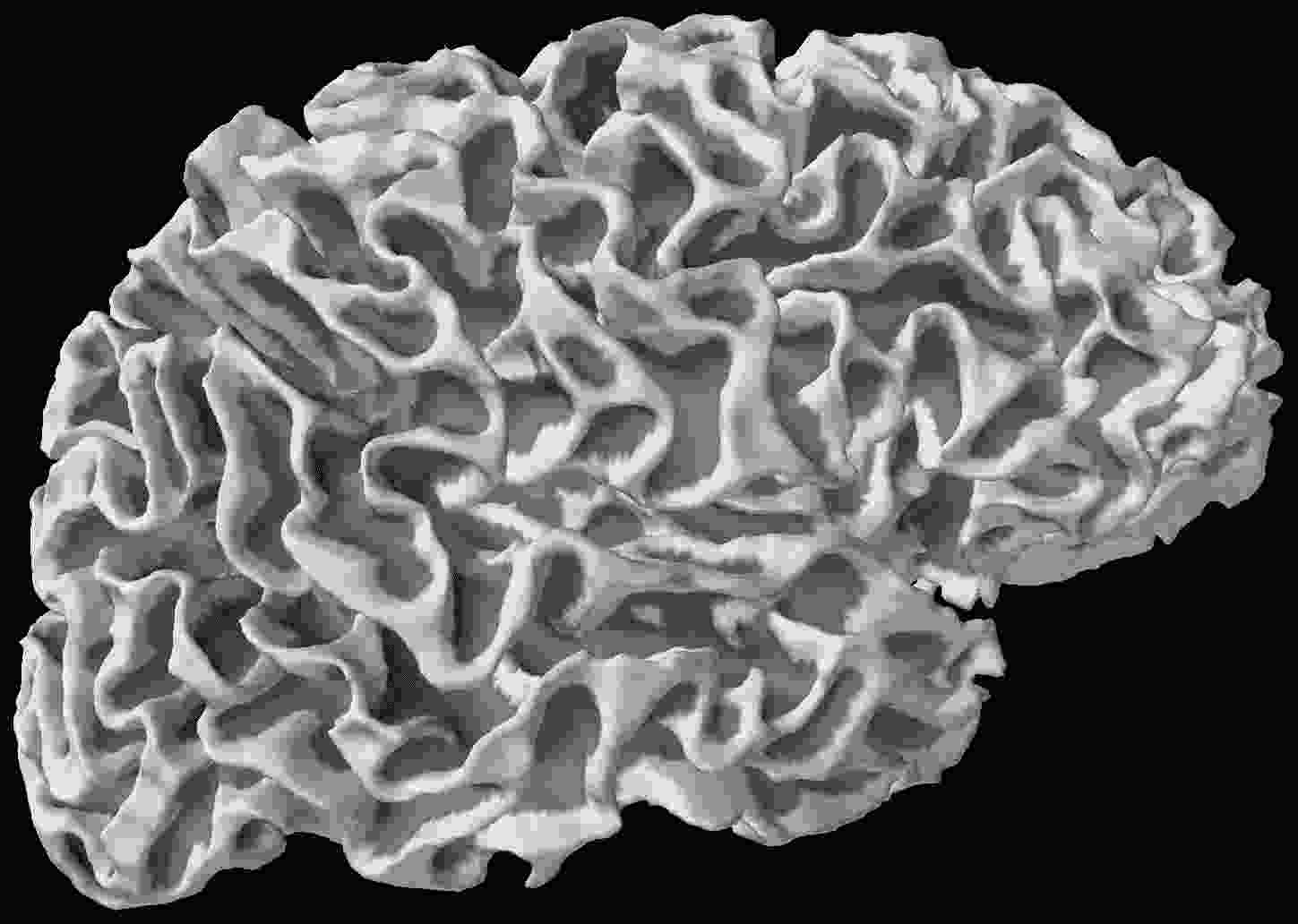}
    \caption{\mtl}
  \end{subfigure}\hspace{\spacefig em}
  \caption{\textit{Real data, right auditory stimulations ($n=102$, $q=7498$, $q=76$, $r=33$)} Sources found in the left hemisphere (top) and the right hemisphere (bottom)  after right auditory stimulations.}
  \label{fig:real_data_right_audi_r_34_main}
\end{figure}

\paragraph{Conclusion}
This work introduces CLaR, a sparse estimator for multitask regression.
It is designed to handle correlated Gaussian noise in the context of repeated observations, a standard framework in applied sciences such as neuroimaging.
The resulting optimization problem can be solved efficiently with state-of-the-art convex solvers, and the algorithmic cost is the same as for single repetition data. The theory of smoothing connects \us to the Schatten 1-Lasso in a principled manner, which opens the way to the use of more sophisticated datafitting terms.
The benefits of \us for support recovery in the presence of non-white Gaussian noise were extensively evaluated against a large number of competitors, both on simulations and on empirical MEG data.

\paragraph{Acknowledgments}
This work was funded by ERC Starting Grant SLAB ERC-YStG-676943.

\clearpage
\bibliographystyle{plainnat} 
\bibliography{references_all} 

\clearpage

\appendix
\onecolumn

\section{Smoothing theory for convex optimization}
\label{app_sec:smoothing}
\paragraph{Notation}
Let $d\in \bbN$, and let $\cC$ be a closed and convex subset of $\bbR^{d}$.
We write $\iota_{\cC}$ for the indicator function of the set $\cC$, \ie $\iota_{\cC}(x) = 0$ if $x \in \cC$ and $\iota_{\cC}(x) = + \infty$ otherwise, and $\Pi_{\cC}$ for the Euclidean projection on $\cC$.
The Fenchel conjugate of a function $h:\bbR^{d} \to \bbR$ is written $h^*$ and is defined for any $y \in \bbR^d,$ by $h^*(y) = \sup_{x\in \bbR^d} \langle x, y \rangle - h(x)$.
For $p \in \left[1, +\infty\right]$, let us write $\mathcal{B}_{\mathscr{S}, p}$ for the Schatten-$p$ unit ball, and $\norm{\cdot}_{p}$ for the standard $\ell_{p}$-norm in $\bbR^d$.

\subsection{Basic properties of inf-convolution}
\label{app_sub:basic_infconv_prop}

\begin{proposition}\label{prop:reminder-infconv}
  Let $g: \bbR^d\to \bbR$, $h: \bbR^d\to \bbR$ be closed proper convex functions.
  Then, the following holds (see \citealt[p. 136]{Parikh_Boyd_Chu_Peleato_Eckstein13}):
    \begin{align}
    h^{**}
      &= h \enspace,\label{eq:fenchel-involution}\\
    (h \infconv g)^*
      &= h^* + g^* \enspace,\label{eq:inf-conv-morphism}\\
    \left( \sigmamin g \left(\tfrac{\cdot}{\sigmamin}\right) \right)^*
      &=\sigmamin g^* \enspace,\label{eq:inf-conv-scaling}\\
    \norm{\cdot}_{p}^*
      &= \iota_{\cB_{p^*}} , \text{ where } \frac {1}{p} + \frac{1}{p^*} = 1 \enspace,\label{eq:fenchel-norm}\\
    (h+\delta)^*
      &=  h^*-\delta, \quad \forall \delta \in \bbR \enspace,\label{eq:fenchel-translation}\\
    \left(\tfrac{1}{2} \norm{\cdot}^2 \right)^*
      &= \tfrac{1}{2} \norm{\cdot}^2 \label{eq:fenchel-norm2} \enspace.
  \end{align}
\end{proposition}


\subsection{Smoothing of Schatten norms}
\label{subsec:schatten_norms}

In all this section, the variable is a matrix $Z \in \bbR^{n \times q}$, and the function $\omega_{\sigmamin}$ is defined as $\sigmamin \omega \left(\tfrac{\cdot}{\sigmamin}\right)$.
\begin{lemma}\label{lemma:conco_p_q}
Let $c \in \bbR,  \in [1, \infty]$. Let $p^* \in [1, \infty]$ be the H\"{o}lder conjugate of $p$, $\frac{1}{p} + \frac{1}{p^*} = 1$. For the choice $\omega(\cdot) = \frac{1}{2}\norm{\cdot}^2 + c$,  the following holds true:
    \begin{align*}
        \left(\norm{\cdot}_{\mathscr{S}, p} \infconv \omega_{\sigmamin}\right)(Z)
        &= \frac{1}{2\sigmamin}\norm{Z}^2
          + c \sigmamin
          - \frac{\sigmamin}{2} \norm{\Pi_{\cB_{\mathscr{S}, p^*}}\left(\tfrac{Z}{\sigmamin}\right) -\tfrac{Z}{\sigmamin}}^2 \enspace.
    \end{align*}
\end{lemma}

\begin{proof}
  \begin{align}
    \left(\norm{\cdot}_{\mathscr{S}, p} \infconv \omega_{\sigmamin} \right)(Z)
    & = \left(\norm{\cdot}_{\mathscr{S}, p} \infconv \omega_{\sigmamin}\right)^{**}(Z) & \text{(using \cref{eq:fenchel-involution})}  \nonumber\\
    & = \left(\norm{\cdot}_{\mathscr{S}, p}^* + \omega_{\sigmamin}^*\right)^{*}(Z)
      & \text{(using \cref{eq:inf-conv-morphism})}\nonumber  \\
    & = \left(\iota_{\cB_{\mathscr{S}, p^*}}
      + \tfrac{\sigmamin}{2} \norm{\cdot}^2
      - c \sigmamin \right)^*(Z) \nonumber
      &\text{(using \cref{eq:fenchel-norm})} \\
    & = \left(\tfrac{\sigmamin}{2} \norm{\cdot}^2
      + \iota_{\cB_{\mathscr{S}, p^*}}\right)^*(Z)
      + c \sigmamin
      & \text{(using \cref{eq:fenchel-translation})}\enspace. \label{eq:first_lp_case}
  \end{align}
  We can now compute the last Fenchel transform remaining:
  \begin{align}
      \left(\frac{\sigmamin}{2} \norm{\cdot}^2 + \iota_{\cB_{\mathscr{S},p^*}}\right)^*(Z)
      &= \sup_{U \in \bbR^{n \times q}} \left( \langle U, Z \rangle - \frac{\sigmamin}{2} \norm{U}^2 - \iota_{\cB_{\mathscr{S},p^*}}(U) \right) \nonumber \\
      &= \sup_{U \in \cB_{\mathscr{S}, p^*} } \left(\langle U, Z \rangle - \frac{\sigmamin}{2} \norm{U}^2 \right)\nonumber \\
      &= -\inf_{U \in \cB_{\mathscr{S},p^*} } \left(\frac{\sigmamin}{2} \norm{U}^2 - \langle U, Z \rangle \right)\nonumber \\
      &= - \sigmamin \cdot \inf_{U \in \cB_{\mathscr{S}, p^*} } \left(\tfrac{1}{2} \norm{U}^2 - \left\langle U, \tfrac{Z}{\sigmamin} \right\rangle \right)\nonumber \\
      &= - \sigmamin \cdot \inf_{U \in \cB_{\mathscr{S}, p^*} } \left(\tfrac{1}{2} \norm{U -\tfrac{Z}{\sigmamin}}^2 - \tfrac{1}{2\sigmamin^2}\norm{Z}^2\right)  \nonumber \\
      &= \tfrac{1}{2\sigmamin}\norm{Z}^2 - \tfrac{\sigmamin}{2}  \cdot \inf_{U \in \cB_{\mathscr{S}, p^*} } \left(\norm{U -\tfrac{Z}{\sigmamin}}^2
      \right) \nonumber \\
      &= \tfrac{1}{2\sigmamin}\norm{Z}^2 -
          \tfrac{\sigmamin}{2}  \norm{\Pi_{\cB_{\mathscr{S}, p^*}}\left(\tfrac{Z}{\sigmamin}\right) -\tfrac{Z}{\sigmamin}}^2
      \label{eq:in_proof_projection_lq} \enspace.
  \end{align}
The result follows by combining \cref{eq:first_lp_case,eq:in_proof_projection_lq}.

\end{proof}

\subsection{Schatten 1-norm (nuclear/trace norm), proof of \Cref{prop:smoothed_conco_nuc} }
\label{app_sub:nuclear_norm_case}
%
\subsubsection{Preliminary lemmas}
First we need the formula of the projection of a matrix onto the Schatten infinity ball:
\begin{lemma}[{Projection onto $\cB_{\mathscr{S}, \infty}$, \citealt[Example 7.31]{Beck17}}]
  Let $Z \in \bbR^{n \times q}$, let $Z = V \diag (\gamma_1, \dots, \gamma_{n \wedge q}) W^\top$ be a singular value decomposition of $Z$, then:
  \begin{equation}
    \Pi_{\cB_{\mathscr{S}, \infty}} (Z) = V \diag( \gamma_1 \wedge 1, \dots, \gamma_{n \wedge q} \wedge 1) W^\top \enspace .
  \end{equation}
\end{lemma}
Then we need to link the value of the primal to the singular values of $ZZ^\top$:
\begin{lemma}[Value of the primal] \label{lem:value_optimal}
  Let $\gamma_1, \dots, \gamma_{n \wedge q}$ be a singular value decomposition of $Z$, then:
  \begin{lemmaenum}
    \item \label{lem:value_optimal_n} 
    $
      \min_{\substack{S\succeq\sigmamin \Id_{n} }}
    \frac{1}{2} \Tr [Z^\top S^{-1} Z] + \frac{1}{2} \Tr(S)
    =
    \frac{1}{2} \sum_{i=1}^{n \wedge q} \frac{\gamma_i^2}{\gamma_i \vee \sigmamin}
    +
    \frac{1}{2} \sum_{i=1}^{n \wedge q} \gamma_i \vee \sigmamin
    + \frac{1}{2} (n - n \wedge q) \sigmamin
    \enspace ,$
    \item   
    $
      \min_{\substack{S\succeq\sigmamin \Id_{q} }}
    \frac{1}{2} \Tr [Z S^{-1} Z^\top]
    + \frac{1}{2} \Tr(S)
    =
    \frac{1}{2} \sum_{i=1}^{n \wedge q} \frac{\gamma_i^2}{\gamma_i \vee \sigmamin}
    +
    \frac{1}{2} \sum_{i=1}^{n \wedge q} \gamma_i \vee \sigmamin
    + \frac{1}{2} (q - n \wedge q) \sigmamin
    \enspace .
    $ \label{lem:value_optimal_q}
\end{lemmaenum}
\end{lemma}
\begin{proof}[ Proof of \Cref{lem:value_optimal_n}]

  The minimum in the left hand side is attained at  $\hat S = U \diag(\gamma_1 \vee \sigmamin, \dots, \gamma_{n \wedge q} \vee \sigmamin, \sigmamin, \dots, \sigmamin) U^\top $ (see \citealt[Prop. 2]{Massias_Fercoq_Gramfort_Salmon17}).
  \begin{align}
    \min_{\substack{S\succeq\sigmamin \Id_{n} }}
    \frac{1}{2} \Tr [Z^\top S^{-1} Z] + \frac{1}{2} \Tr(S)
    &=
    \frac{1}{2} \Tr [Z^\top \hat S^{-1} Z] + \frac{1}{2} \Tr(\hat S)
    \nonumber\\
    &=
    \frac{1}{2} \Tr [\hat S^{-1} Z Z^\top ] + \frac{1}{2} \Tr(\hat S)
     \nonumber\\
    &=
    \frac{1}{2} \Tr [U \diag(\gamma_1^2 / (\gamma_1 \vee \sigmamin), \dots, \gamma_{n \wedge q}^2 / (\gamma_{n \wedge q} \vee \sigmamin), 0, \dots, 0) U^\top] \nonumber \\
    & + \frac{1}{2} \Tr[U \diag(\gamma_1 \vee \sigmamin, \dots, \gamma_{n \wedge q} \vee \sigmamin, \sigmamin, \dots, \sigmamin) U^\top]
    \nonumber\\
    &=
      \frac{1}{2} \sum_{i=1}^{n \wedge q} \frac{\gamma_i^2}{\gamma_i \vee \sigmamin}
      +
      \frac{1}{2} \sum_{i=1}^{n \wedge q} \gamma_i \vee \sigmamin
      + \frac{1}{2} \sum_{n \wedge q + 1}^{n} \sigmamin
    \nonumber\\
    &=
      \frac{1}{2} \sum_{i=1}^{n \wedge q} \frac{\gamma_i^2}{\gamma_i \vee \sigmamin}
      +
      \frac{1}{2} \sum_{i=1}^{n \wedge q} \gamma_i \vee \sigmamin
      + \frac{1}{2} (n - n \wedge q) \sigmamin
     \enspace .
  \end{align}
  This completes the proof of \Cref{lem:value_optimal_n}.
  \Cref{lem:value_optimal_q} is obtained by symmetry.
\end{proof}
%
\subsubsection{Main result: an explicit variational formula for the inf-convolution smoothing of the nuclear norm}
We now recall the main result that we claim to prove:
\smoothedconconuc*

\begin{proof}
  Let $V \diag(\gamma_1, \dots , \gamma_{n \wedge q}) W^{\top}$ be a singular value decomposition of $Z$.
  We remind that $\Pi_{\cB_{\mathscr{S},{\infty}}}$, the projection over $\cB_{{\mathscr{S},{\infty}}}$, is given by (see \citealt[Example 7.31]{Beck17}):
  \begin{align}
    \Pi_{\cB_{\mathscr{S},{\infty}}}\left(\tfrac{Z}{\sigmamin}\right)
    &= V \diag \left(\Pi_{\cB_{\mathscr{S},\infty}}\left(\tfrac{\gamma_1}{\sigmamin}, \dots, \tfrac{\gamma_{n \wedge q}}{\sigmamin}\right)\right) W^{\top} \nonumber\\
    &= V \diag \left(\tfrac{\gamma_1}{\sigmamin} \wedge 1, \dots, \tfrac{\gamma_{n \wedge q }}{\sigmamin}  \wedge 1 \right) W^\top\enspace, \label{eq:proj_s_inf_ball}
  \end{align}
where we used that the (vectorial) projection over $\cB_{\infty}$ is given coordinate-wise by $\left( \Pi_{\cB_{\infty}}(\gamma_i) \right)_i = \left( \gamma_i \wedge 1 \right)_i$ on the positive orthant.
Then we have,
\begin{align}
     \norm{ \Pi_{\cB_{\mathscr{S},{\infty}}} \left(\tfrac{Z}{\sigmamin} \right) -\tfrac{Z}{\sigmamin}}^2
     &=   \norm{ V \diag \left(\tfrac{\gamma_1}{\sigmamin} \wedge 1 - \tfrac{\gamma_1}{\sigmamin}, \dots, \tfrac{\gamma_{n \wedge q}}{\sigmamin} \wedge 1 - \tfrac{\gamma_{n \wedge q}}{\sigmamin}\right) W^\top}^2 &
       \nonumber\\
     &= \sum_{i=1}^{n \wedge q} \left( \tfrac{\gamma_i}{\sigmamin} \wedge 1 - \tfrac{\gamma_i}{\sigmamin} \right)^2 \nonumber\\
     &= \frac{1}{ \sigmamin^2} \sum_{i=1}^{n \wedge q} \left({\gamma_i} \wedge \sigmamin - {\gamma_i} \right)^2 \nonumber\\
     &= \frac{1}{ \sigmamin^2} \sum_{\gamma_i > \sigmamin} \left({\gamma_i} \wedge \sigmamin - {\gamma_i} \right)^2  \nonumber\\
     &= \frac{1}{ \sigmamin^2} \sum_{\gamma_i > \sigmamin} \left( \sigmamin - {\gamma_i} \right)^2 \nonumber\\
     &= \frac{1}{ \sigmamin^2} \sum_{\gamma_i > \sigmamin} \left( \sigmamin^2  + \gamma_i^2 - 2 \sigmamin \gamma_i \right) \nonumber\\
     &=  \sum_{\gamma_i > \sigmamin} 1
     +  \frac{1}{ \sigmamin^2} \sum_{\gamma_i > \sigmamin} \gamma_i^2
     - 2  \frac{1}{ \sigmamin} \sum_{\gamma_i > \sigmamin} \gamma_i \\
    - \frac{\sigmamin}{2}\norm{ \Pi_{\cB_{\mathscr{S},{\infty}}} \left(\tfrac{Z}{\sigmamin} \right) -\tfrac{Z}{\sigmamin}}^2
    & = - \frac{\sigmamin}{2}\sum_{\gamma_i > \sigmamin} 1
    -  \frac{1}{ 2 \sigmamin} \sum_{\gamma_i > \sigmamin} \gamma_i^2
    + \sum_{\gamma_i > \sigmamin} \gamma_i
     \label{eq:norm_proj_p_1} \enspace .
\end{align}
By combining \Cref{lemma:conco_p_q,eq:norm_proj_p_1} with $p^*=\infty, c \in \bbR$,  it follows:
\begin{align}
  \left(\nucnorm{\cdot} \infconv \omega_{\sigmamin}\right)(Z)
  &=
  \frac{1}{2 \sigmamin} \sum_{i=1}^n \gamma_i^2
  + c \sigmamin
  - \frac{\sigmamin}{2}\sum_{\gamma_i > \sigmamin} 1
  -  \frac{1}{ 2 \sigmamin} \sum_{\gamma_i > \sigmamin} \gamma_i^2
  + \sum_{\gamma_i > \sigmamin} \gamma_i \nonumber\\
  &=
  \frac{1}{2 \sigmamin} \sum_{\gamma_i \leq \sigmamin}^n \gamma_i^2
  + c \sigmamin
  - \frac{\sigmamin}{2}\sum_{\gamma_i > \sigmamin} 1
  + \sum_{\gamma_i > \sigmamin} \gamma_i & \text{by grouping the $\gamma_i$ terms}\nonumber\\
  &=
  \frac{1}{2 \sigmamin} \sum_{\gamma_i \leq \sigmamin}^n \gamma_i^2
  + \sum_{\gamma_i > \sigmamin} \gamma_i
  - \frac{\sigmamin}{2}\sum_{\gamma_i > \sigmamin} 1
  + c \sigmamin
   & \text{by reordering.}
  \label{eq:inf_conv_spectral}
\end{align}
The goal is now to link the optimization problem to the right-hand side of \Cref{eq:inf_conv_spectral}.
Let $Z Z^\top = U^\top \diag(\underbrace{\gamma_1, \dots, \gamma_{n \wedge q}, 0, \dots, 0}_{\in \bbR^n}) U $ be an eigenvalue decomposition of $ZZ^\top$.
\begin{align}
  \min_{\substack{S\succeq\sigmamin \Id_{n} }}
  \frac{1}{2} \Tr [Z^\top S^{-1} Z] + \frac{1}{2} \Tr(S)
  &=
    \frac{1}{2} \sum_{i=1}^{n \wedge q} \frac{\gamma_i^2}{\gamma_i \vee \sigmamin}
    +
    \frac{1}{2} \sum_{i=1}^{n \wedge q} \gamma_i \vee \sigmamin
    + \frac{1}{2}(n - n \wedge q) \sigmamin \quad \text{(using \Cref{lem:value_optimal})}
    \nonumber\\
  &=
    \frac{1}{2\sigmamin} \sum_{\gamma_i \leq \sigmamin} \gamma_i^2
    +
    \frac{1}{2} \sum_{\gamma_i > \sigmamin} \gamma_i
    +
    \frac{1}{2} \sum_{\gamma_i \leq \sigmamin} \sigmamin
    +
    \frac{1}{2} \sum_{\gamma_i > \sigmamin} \gamma_i
    + \frac{1}{2} (n - n \wedge q) \sigmamin
    \nonumber\\
  &=
    \frac{1}{2\sigmamin} \sum_{\gamma_i \leq \sigmamin} \gamma_i^2
    +
    \sum_{\gamma_i > \sigmamin} \gamma_i
    +
    \frac{\sigmamin}{2} \sum_{\gamma_i \leq \sigmamin} 1
    + \frac{1}{2} (n - n \wedge q) \sigmamin
    \nonumber\\
  &=
    \frac{1}{2\sigmamin} \sum_{\gamma_i \leq \sigmamin} \gamma_i^2
    +
    \sum_{\gamma_i > \sigmamin} \gamma_i
    + \frac{\sigmamin}{2} (n \wedge q - \sum_{\gamma_i > \sigmamin} 1)
    +
    \frac{1}{2} (n - n \wedge q) \sigmamin
    \nonumber\\
  &=
    \frac{1}{2\sigmamin} \sum_{\gamma_i \leq \sigmamin} \gamma_i^2
    + \sum_{\gamma_i > \sigmamin} \gamma_i
    - \frac{\sigmamin}{2} \sum_{\gamma_i > \sigmamin} 1
    + \underbrace{\frac{\sigmamin}{2} n \wedge q
    + \frac{1}{2} (n - n \wedge q) \sigmamin}_{\frac{\sigma}{2} n} \nonumber\\
  &=
    \frac{1}{2\sigmamin} \sum_{\gamma_i \leq \sigmamin} \gamma_i^2
    + \sum_{\gamma_i > \sigmamin} \gamma_i
    - \frac{\sigmamin}{2} \sum_{\gamma_i > \sigmamin} 1
    + \frac{\sigmamin}{2} n
    \label{eq:details_spectral_sgcl}\enspace,
\end{align}
and identifying \Cref{eq:details_spectral_sgcl,eq:inf_conv_spectral} leads to the result for $c=\frac{n}{2}$.
\end{proof}
%
\subsection{Properties of the proposed smoothing for the nuclear norm}
\label{app_sub:prop_our_smoothing}
First let us recall the definition of a smoothable function and a $\mu$-smooth approximation of
\citet[Def. 2.1]{Beck_Teboulle12}:
\begin{definition}[Smoothable function, $\mu$-smooth approximation]
  \label{def:smoothable_function}
  Let $g: \bbE \to \left]- \infty, + \infty \right]$ be a closed and proper convex function, and let $E\subseteq \dom(g)$ be a closed convex set.
  The function $g$ is called \emph{$(\alpha, \delta, K)$-smoothable} on $E$
  if there exists $\delta_1, \delta_2$ satisfying $\delta_1 + \delta_2 = \delta > 0$ such that for every $\mu$ there exists a continuously differentiable convex function $g_\mu : \bbE \to \left]- \infty, + \infty \right[$ such that the following holds:
  \begin{lemmaenum}
    \item
        $g(x) - \delta_1 \mu
        \leq g_\mu(x)
        \leq g(x) + \delta_2 \mu$
        for every $x \in E$ .
    \item The function $\nabla g_\mu$ has a Lipschitz constant which is less than or equal to $K + \frac{\alpha}{\mu}$:
    \begin{equation}
      \norm{\nabla g_\mu (x) - \nabla g_\mu(y)}
      \leq \left ( K + \frac{\alpha}{\mu} \right) \norm{x - y} \text{ for every } x, y \in E
      \enspace .
    \end{equation}
\end{lemmaenum}
  The function $g$ is called a \emph{$\mu$-smooth} approximation of $g$ with parameters $(\alpha, \delta, K)$.
\end{definition}
The nuclear norm $\nucnorm{\cdot}$ is non-smooth (and not even differentiable at 0), but one can construct a smooth approximation of the nuclear norm based on the following variational formula, if $Z Z^\top \succ 0$:
\begin{equation}
  \nucnorm{Z} = \min_{\substack{S\succ 0 }}
  \frac{1}{2} \Tr [Z^\top S^{-1} Z] + \frac{1}{2} \Tr(S)
  \enspace,
\end{equation}
see \citet[Lemma 3.4]{vandeGeer16}.
When $Z Z^\top \nsucc 0$, one can approximate $\nucnorm{\cdot}$ with
\begin{equation}
  \min_{\substack{S\succeq \sigmamin \Id }}
  \frac{1}{2} \Tr [Z^\top S^{-1} Z]
  + \frac{1}{2} \Tr(S) = \nucnorm{\cdot} \infconv \omega_{\sigmamin}
\enspace ,
\end{equation}
as shown in \Cref{app_sub:nuclear_norm_case}.
It can be shown that this approximation stays close to the nuclear norm.
%
\begin{proposition}
    $\nucnorm{\cdot} \infconv \omega_{\sigmamin}$
    is a $\sigmamin$-smooth approximation of $\nucnorm{\cdot}$ with parameters
    $(1,  \tfrac{n }{2}, 0)$.
    More precisely: $\nucnorm{\cdot}\infconv \omega_{\sigmamin}$ has a $\sigmamin$-Lipschitz gradient and
    \begin{equation}
      0
      \leq \nucnorm{\cdot} \infconv \omega_{\sigmamin} - \nucnorm{\cdot}
      = \frac{\sigmamin}{2} \sum_{\gamma_i < \sigmamin}\left(1 - \frac{\gamma_i}{\sigmamin} \right)^2
      \leq \frac{\sigmamin}{2} n
      \enspace .
    \end{equation}
\end{proposition}

\begin{proof}
  Since $\omega$ is 1-smooth, \citet[Thm. 4.1]{Beck_Teboulle12} shows that $\nucnorm{\cdot} \infconv \omega_{\sigmamin}$ is $\sigmamin$-smooth.

  Let $Z \in \bbR^{n \times q}$ and let $\gamma_1, \dots, \gamma_{n \wedge q}$ be its singular value decomposition:
  \begin{align}
    \left(\nucnorm{\cdot} \infconv \omega_{\sigmamin}\right)(Z) - \nucnorm{Z}
      &= \frac{1}{2} \sum_{i=1}^{n \wedge q} \frac{\gamma_i^2}{\gamma_i \vee \sigmamin}
      + \frac{1}{2} \sum_{i=1}^{n \wedge q} \gamma_i \vee \sigmamin
      + \frac{1}{2} \sum_{n \wedge q + 1}^{n} \sigmamin
      - \sum_{i=1}^{n \wedge q} \gamma_i \nonumber\\
      &= \frac{1}{2} \sum_{i=1}^{n \wedge q}
      \left(
        \frac{\gamma_i^2}{\gamma_i \vee \sigmamin}
        +  \gamma_i \vee \sigmamin
        - 2 \gamma_i
      \right)
      + \frac{1}{2} \sum_{n \wedge q + 1}^{n} \sigmamin
      \nonumber\\
      &= \frac{1}{2} \sum_{\gamma_i \leq \sigmamin}
      \left(
        \frac{\gamma_i^2}{\gamma_i \vee \sigmamin}
        +  \gamma_i \vee \sigmamin
        - 2 \gamma_i
      \right)
      + \frac{1}{2} \sum_{n \wedge q + 1}^{n} \sigmamin
      \nonumber\\
      &= \frac{1}{2} \sum_{\gamma_i \leq \sigmamin}
      \left(
        \frac{\gamma_i^2}{\sigmamin}
        +  \sigmamin
        - 2 \gamma_i
      \right)
      + \frac{1}{2} \sum_{n \wedge q + 1}^{n} \sigmamin
      \nonumber\\
      &= \frac{1}{2} \sum_{\gamma_i \leq \sigmamin}
        \frac{(\gamma_i - \sigmamin)^2}{\sigmamin}
        + \frac{1}{2} ( n - n \wedge q) \sigmamin \enspace.
      \label{eq:error_our_smoothing}
  \end{align}
  Hence,
  \begin{align}
            0 \leq \left(\nucnorm{\cdot} \infconv \omega_{\sigmamin}\right)(Z) - \nucnorm{Z} &= \frac{1}{2} \sum_{\gamma_i \leq \sigmamin}
            \frac{(\gamma_i - \sigmamin)^2}{\sigmamin}
            + \frac{1}{2} ( n - n \wedge q) \sigmamin
            \leq \frac{\sigmamin}{2}  n \enspace.
  \end{align}

  Moreover this bound is attained when $Z = 0$.
\end{proof}
%
\subsection{Comparison with another smoothing of the nuclear norm}
\label{app_sub:prop_other_smoothing}

Another regularization was proposed in \citet[p. 62]{Argyriou_Evgeniou_Pontil08,Bach_Jenatton_Mairal_Obozinski12}:
%
\begin{equation}
  \label{eq:smoothing_Bach}
    \min_{\substack{S \succ 0 }}
    \underbrace{\frac{1}{2} \Tr [Z^\top S^{-1} Z]
    + \frac{1}{2} \Tr(S)
    + \frac{\sigmamin^2}{2} \Tr(S^{-1})}_{h(S^{-1})}
\enspace .
\end{equation}
By putting the gradient of the objective function in \Cref{eq:smoothing_Bach} to zero it follows that:
\begin{equation}
  0 = \nabla h (\hat S^{-1}) =
    Z Z^\top - \hat S^{2} + \sigmamin^2 \Id
  \enspace ,
\end{equation}
leading to :
\begin{equation}
  \hat S = (ZZ^\top + \sigmamin^2 \Id )^{\frac{1}{2}}
  \enspace.
\end{equation}
Let $\gamma_1, \dots, \gamma_{n \wedge q}$ be the singular values of $Z$:
\begin{align}
  \frac{1}{2} \Tr [Z^\top \hat S^{-1} Z]
  + \frac{1}{2} \Tr(\hat S)
  + \frac{\sigmamin^2}{2} \Tr(\hat S^{-1})
    &=\frac{1}{2} \sum_{i=1}^{n}
    \left(
      \frac{\gamma_i^2}{\sqrt{\gamma_i^2 + \sigmamin^2}}
      + \sqrt{\gamma_i^2 + \sigmamin^2}
      +  \frac{\sigmamin^2}{\sqrt{\gamma_i^2 + \sigmamin^2}}
    \right) \nonumber \\
    &= \frac{1}{2} \sum_{i=1}^{n}
    \left(
      \frac{\gamma_i^2 + \gamma_i^2 + \sigmamin^2 + \sigmamin^2}{\sqrt{\gamma_i^2 + \sigmamin^2}}
    \right) \nonumber \\
    &= \sum_{i=1}^{n} \sqrt{\gamma_i^2 + \sigmamin^2}
    \enspace .
\end{align}
\begin{proposition}
  $Z \mapsto \min_{\substack{S \succ 0 }}
  \frac{1}{2} \Tr [Z^\top S^{-1} Z]
  + \frac{1}{2} \Tr(S)
  + \frac{\sigmamin^2}{2} \Tr(S^{-1})$
  is a $\sigmamin$-smooth approximation of $\nucnorm{\cdot}$ with parameters $(1, n, 0)$.
  More precicely: $Z \mapsto \min_{\substack{S \succ 0 }}
  \frac{1}{2} \Tr [Z^\top S^{-1} Z]
  + \frac{1}{2} \Tr(S)
  + \frac{\sigmamin^2}{2} \Tr(S^{-1})$ has a gradient $\sigmamin$-Lipschitz and
  \begin{equation}
    0
    \leq \min_{\substack{S \succ 0 }}
      \frac{1}{2} \Tr [Z^\top S^{-1} Z]
      + \frac{1}{2} \Tr(S)
      + \frac{\sigmamin^2}{2} \Tr(S^{-1})
      - \nucnorm{Z}
    = \sigmamin \sum_{i} \frac{1}{\sqrt{1 + \frac{\gamma_i^2}{\sigmamin^2}} + \frac{\gamma_i}{\sigmamin}}
    \leq \sigmamin n
    \enspace .
  \end{equation}
\end{proposition}
\begin{proof}
  $\sum_{i=1}^{n \wedge q} \sqrt{\gamma_i^2 + \sigmamin^2}$ is a $\sigmamin$-smooth approximation of $\sum_{i=1}^{n \wedge q} \sqrt{\gamma_i^2} = \nucnorm{Z}$, see \citet[Example 4.6]{Beck_Teboulle12}.

  \begin{align}
    \min_{\substack{S \succ 0 }}
      \frac{1}{2} \Tr [Z^\top S^{-1} Z]
      + \frac{1}{2} \Tr(S)
      + \frac{\sigmamin^2}{2} \Tr(S^{-1})
      - \nucnorm{Z}
      &= \sum_{i=1}^{n}
      \left(
        \sqrt{\gamma_i^2 + \sigmamin^2}
        - \gamma_i
      \right) \nonumber \\
      &= \sum_{i=1}^{n}
        \frac{\sigmamin^2}{\sqrt{\gamma_i^2 + \sigmamin^2}
          + \gamma_i} \nonumber \\
      &= \sigmamin \sum_{i=1}^{n}
        \frac{1}{\sqrt{ 1 + \frac{\gamma_i^2}{\sigmamin^2}}
          + \frac{\gamma_i}{\sigmamin}}
          \label{eq:error_Bach_smoothing} \\
      & \leq \sigmamin n
      \enspace .
  \end{align}
  Moreover this bound is attained when $Z = 0$.
\end{proof}
It can be shown that with a fixed Lipschitz constant, the proposed smoothing is (at least) a twice better approximation.
This can be quantified even more precisely:
\begin{proposition}
  \begin{equation}
    0
    \leq \underbrace{\left(\nucnorm{\cdot} \infconv \omega_{\sigmamin} \right)(Z) - \nucnorm{Z}}_{\Err_1(Z)}
    \leq \frac{1}{2} \left(\underbrace {
      \min_{\substack{S \succ 0 }}
      \frac{1}{2} \Tr [Z^\top S^{-1} Z]
      + \frac{1}{2} \Tr(S)
      + \frac{\sigmamin^2}{2} \Tr(S^{-1})
      - \nucnorm{Z} }_{\Err_2(Z)}
      \right)
    \enspace .
  \end{equation}
  More precisely
  \begin{align}
    \frac{1}{2} \Err_2(Z) - \Err_1(Z)= \frac{\sigmamin}{2} \sum_{\gamma_i \geq \sigma}
    \underbrace{\left( \sqrt{ 1 + \frac{\gamma_i^2}{\sigmamin^2} }  - \frac{\gamma_i}{\sigmamin}
    \right)}_{\geq 0}
    + \frac{\sigmamin}{2}
    \sum_{\gamma_i < \sigma} \underbrace{\left(
    \frac{1}{\sqrt{ 1 + \frac{\gamma_i^2}{\sigmamin^2}}
      + \frac{\gamma_i}{\sigmamin}}
    - (1 + \frac{\gamma_i}{\sigmamin})^2
    \right)}_{ \geq 0}
    \enspace , \label{eq:diff_both_errors}
  \end{align}
  which means that for a fixed smoothing constant $\sigmamin$, our smoothing is at least twice uniformly better.
  Moreover the proposed smoothing can be much better, in particular when a lot a singular values are around $\sigmamin$.
\end{proposition}
\begin{proof}
  Using the formulas of $\Err_1$ (\Cref{eq:error_our_smoothing}) and $\Err_2$ (\Cref{eq:error_Bach_smoothing}), \Cref{eq:diff_both_errors} is direct.
  In \Cref{eq:diff_both_errors} the positivity of the first sum is trivial, the positivity of the second can be obtained with an easy function study.
\end{proof}

\subsection{Schatten 1-norm (nuclear/trace norm) with repetitions}
\label{app_sub:nuclear_norm_with_repet_case}
Let $Z^{(1)}, \dots, Z^{(r)}$ be matrices in $ \bbR^{n \times q}$, then we define $Z \in \bbR^{n \times qr}$ by $Z = [Z^{(1)}| \dots | Z^{(r)}]$.

\begin{proposition}\label{prop:conco_nuc_repet}
For the choice $\omega(Z) = \frac{1}{2}\norm{Z}^2 + \frac{n \wedge qr}{2}$, then the following holds true:
  \begin{align}
      \left(\nucnorm{\cdot} \infconv \omega_{\sigmamin}(\cdot)\right)( Z)
      = \min_{S \succeq \sigmamin \Id_n} \frac{1}{2}
      \sum_{l=1}^r \Tr \left(Z^{(l)\top} S^{-1} Z^{(l)}\right) + \frac{1}{2} \Tr(S)
      \enspace.
  \end{align}
\end{proposition}

\begin{proof}
  The result is a direct application of \Cref{prop:smoothed_conco_nuc}, with
  $Z = [Z^{(1)}| \dots | Z^{(r)}]$.
  It suffices to notice that $\Tr Z^\top S^{-1} Z =\sum_{l=1}^r \Tr \left(Z^{(l)\top} S^{-1} Z^{(l)}\right)$.
\end{proof}

\subsection{Schatten 2-norm (Frobenius norm)}
\label{app_sub:frobenius_norm}

\begin{proposition}\label{prop:smoothing_schatten_2}
  For the choice $\omega(\cdot) = \frac{1}{2}\norm{\cdot}^2 + \frac{1}{2}$, and for $Z\in \bbR^{n \times q}$ then the following holds true:
  \begin{align}
    \left(\norm{\cdot} \infconv \omega_{\sigmamin}\right)(Z)
    = \min_{\sigma \geq \sigmamin} \left(\tfrac{1}{2 \sigma} \norm{Z}^2 + \tfrac{\sigma}{2}\right)
    & =
    \begin{cases}
      \frac{\norm{Z}^2}{2 \sigmamin} + \frac{\sigmamin}{2} \enspace,
      &  \text{if } \norm{Z} \leq \sigmamin \enspace, \\
      \norm{Z} \enspace, & \text{if } \norm{Z} > \sigmamin \enspace.
    \end{cases}
  \end{align}
\end{proposition}

\begin{proof} Let us recall that $\norm{\cdot}=\norm{\cdot}_{\mathscr{S},2}$.
  Therefore
  \begin{align}
 \Pi_{\cB_{\mathscr{S}, 2}}\left(\tfrac{Z}{\sigmamin}\right)
    & =
    \begin{cases}
      0 \enspace,
      & \text{if } \norm{Z} \leq \sigmamin \enspace, \\
      \frac{Z}{\norm{Z}} \enspace, & \text{if } \norm{Z} > \sigmamin \enspace .
    \end{cases} \label{eq:proj_frob_ball}
  \end{align}
By combining \Cref{eq:proj_frob_ball,lemma:conco_p_q} with $p^*=2$, and $ c = \frac{1}{2}$, the later yields
  \begin{align*}
      \left(\norm{\cdot} \infconv \omega_{\sigmamin}\right)(Z)
      &=
      \begin{cases}
        \frac{1}{2\sigmamin}\norm{Z}^2 + \frac{\sigmamin}{2} \enspace,
        & \text{if } \norm{Z} \leq \sigmamin \enspace,\\
        \norm{Z} \enspace, & \text{if } \norm{Z} > \sigmamin \enspace.
    \end{cases}
  \end{align*}
\end{proof}

\subsection{Schatten infinity-norm (spectral norm)}
\label{app_sub:spectral_norm}

\begin{proposition}\label{prop:smoothing_schatten_inf}
  For the choice $\omega(\cdot) = \frac{1}{2}\norm{\cdot}^2 + \frac{1}{2}$ and for $Z\in \bbR^{n \times q}$, then the following holds true:
  \begin{align*}
    \left(\norm{\cdot}_{\mathscr{S},\infty} \infconv \omega_{\sigmamin}\right)(Z)
    &=
    \begin{cases}
      \frac{1}{2\sigmamin}\norm{Z}^2 + \frac{\sigmamin}{2} \enspace,
      & \text{if } \nucnorm{Z} \leq \sigmamin \enspace, \\
       \frac{\sigmamin}{2} \sum_{i = 1}^{n \wedge q} \big (\frac{\gamma_i^2}{\sigmamin^2} -  \nu^2 \big )_+
      + \frac{\sigmamin}{2} \enspace,
      & \text{if } \nucnorm{Z} > \sigmamin \enspace,
    \end{cases}
  \end{align*}
 where $\nu \geq 0$ is defined by the implicit equation
 \begin{align}\label{eq:gamma_implicit}
      \norm{ \left(  \ST\left(\tfrac{\gamma_1}{\sigmamin}, \nu \right),\dots, \ST \left(\tfrac{\gamma_{n\wedge q}}{\sigmamin}, \nu \right)\right)}_{1} = 1
      \enspace.
  \end{align}
\end{proposition}

\begin{proof}

  We remind that $\Pi_{\cB_{\mathscr{S},{1}}}$, the projection over $\cB_{\mathscr{S},{1}}$, is given by \citet[Example 7.31]{Beck17}:
  \begin{align}\label{eq:proj_l1_ball}
    \Pi_{\cB_{\mathscr{S},1}} \left(\frac{Z}{\sigmamin}\right)
    &= \begin{cases}
        \tfrac{Z}{\sigmamin} \enspace,
    & \text{if } \nucnorm{Z} \leq \sigmamin \enspace, \\
        V \diag(\ST(\frac{\gamma_i}{\sigmamin}, \nu )) W^{\top} \enspace, & \text{if } \nucnorm{Z} > \sigmamin \enspace, \\
        \end{cases}
  \end{align}
 $\gamma$ being defined by the implicit equation
 \begin{align}
      \norm{ \left(\ST\Big(\tfrac{\gamma_1}{\sigmamin}, \nu \Big),\dots, \ST\Big(\tfrac{\gamma_{n\wedge q}}{\sigmamin},\nu \Big) \right)}_{1} = 1 \enspace.
  \end{align}
By combining \Cref{eq:proj_l1_ball} and \Cref{lemma:conco_p_q} (with $p^*=1, c = \frac{1}{2}$) it follows that
  \begin{align}\label{eq:proj_schatten_1}
    (\norm{\cdot} \infconv \omega_{\sigmamin})(Z)
    &=
      \begin{cases}
      \frac{1}{2\sigmamin}\norm{Z}^2 + \frac{\sigmamin}{2} \enspace,
      & \text{if } \nucnorm{Z} \leq \sigmamin \enspace, \\
    \frac{1}{2\sigmamin}\norm{Z}^2
      + \frac{\sigmamin}{2}
      - \frac{\sigmamin}{2}\norm{\Pi_{\cB_{\mathscr{S}, 1}}\left(\frac{Z}{\sigmamin}\right) -\frac{Z}{\sigmamin}}^2 \enspace,
      & \text{if } \nucnorm{Z} > \sigmamin   \enspace .
    \end{cases}
  \end{align}

Let us compute $\norm{\Pi_{\cB_{\mathscr{S}, 1}}\left(\frac{Z}{\sigmamin}\right) -\frac{Z}{\sigmamin}}^2$.
If $\nucnorm{Z} > \sigmamin$ we have
\begin{align}
    \norm{\Pi_{\cB_{\mathscr{S}, 1}}\left(\tfrac{Z}{\sigmamin}\right) -\tfrac{Z}{\sigmamin}}^2
    &= \frac{1}{\sigmamin^2} \norm{V \diag((\gamma_i - \nu\sigmamin)_+ -\gamma_i) W^\top}^2 \quad \quad \text{(using \Cref{eq:proj_l1_ball})} \nonumber \\
    &= \frac{1}{\sigmamin^2}  \sum_{i=1}^{n \wedge q} \big( (\gamma_i - \nu\sigmamin)_+ -\gamma_i \big)^2  \nonumber\\
    &= \frac{1}{\sigmamin^2}  \big (\sum_{\gamma_i \geq \nu\sigmamin} \nu^2 \sigmamin^2
      + \sum_{\gamma_i < \nu\sigmamin} \gamma_i^2 \big)  \label{eq:calcul_int_proj_l_1_ball} \enspace .
\end{align}

By plugging \Cref{eq:calcul_int_proj_l_1_ball} into \Cref{eq:proj_schatten_1} it follows, that if $\nucnorm{Z} > \sigmamin$:
\begin{align}
    \left(\norm{\cdot}_{\mathscr{S},\infty} \infconv \omega_{\sigmamin}\right)(Z)
    &= \frac{1}{2\sigmamin} \sum_{i=1}^{n \wedge q} {\gamma_i^2}
      + \frac{\sigmamin}{2}
      - \frac{1}{2\sigmamin} \sum_{\gamma_i \geq \nu\sigmamin}^{n \wedge q}  \nu^2 \sigmamin^2
      - \frac{1}{2\sigmamin} \sum_{\gamma_i < \nu\sigmamin}^{n \wedge q}  \gamma_i^2 \nonumber \\
    &= \frac{1}{2\sigmamin} \sum_{\gamma_i \geq \nu\sigmamin}^{n \wedge q} \big ( \gamma_i^2 -  \nu^2 \sigmamin^2 \big )
      + \frac{\sigmamin}{2} \nonumber \\
    &= \frac{\sigmamin}{2} \sum_{i = 1}^{n \wedge q} \left(\frac{\gamma_i^2}{\sigmamin^2} -  \nu^2 \right)_+
      + \frac{\sigmamin}{2} \label{eq:calcul_int_2_proj_l_1_ball} \enspace .
\end{align}

\Cref{prop:smoothing_schatten_inf} follows by plugging \Cref{eq:calcul_int_2_proj_l_1_ball} for the case $\nucnorm{Z} > \sigmamin$, and the fact that when $\nucnorm{Z} \leq \sigmamin$ the result is straightforward.
\end{proof}

\begin{remark}
  Since $\nu\mapsto \norm{ \left(  \ST\left(\tfrac{\gamma_1}{\sigmamin}, \nu \right),\dots, \ST \left(\tfrac{\gamma_{n\wedge q}}{\sigmamin}, \nu \right)\right)}_{1}$ is decreasing and piecewise linear, the solution of \Cref{eq:gamma_implicit} can be computed exactly in $\cO(n \wedge q \; \log (n \wedge q))$ operations.
\end{remark}


\section{Proofs \us}
\label{app_sec:proof_clar}

\subsection{Proof of \Cref{prop:smoothed_sglc_us}} 
\label{sub:proof_of_smoothed_sgcl_us}

\convexsmoothing*

\begin{proof}
\Cref{prop:smoothed_sglc_us} follows from \Cref{app_sub:nuclear_norm_with_repet_case} by choosing $Z=\tfrac{1}{\sqrt{rq}} [Y^{(1)}-X\Beta, \dots, Y^{(r)}-X\Beta]$ and by taking the $\argmin$ over $\Beta$.
\end{proof}

\subsection{Proof of \Cref{prop:joint_convexity}}
\label{app_sub:proof_joint_convexity}
\jointconv*

\begin{proof}
    \begin{align*}
      f(\Beta, \Snoise)
      &= \frac{1}{2nqr}  \sum_{1}^r \norm{Y^{(l)} - X \Beta}_{\Snoise^{-1}}^2 + \frac{1}{2n}\Tr(\Snoise) = \Tr(Z^T \Snoise^{-1} Z) + \frac{1}{2n}\Tr(\Snoise)\enspace,
    \end{align*}
    with $Z = \tfrac{1}{\sqrt{2nqr}} [Y^{(1)} -X\Beta| \dots | Y^{(r)} -X\Beta]$~.

    First note that the (joint) function $(Z, \Sigma) \mapsto \Tr Z^\top \Sigma^{-1} Z$ is jointly convex over $\bbR^{n \times q} \times \cS_{++}^n$, see \citet[Example~3.4]{Boyd_Vandenberghe04}.
    This means that $f$ is jointly  convex in $(Z, \Snoise)$, moreover $\Beta \mapsto \tfrac{1}{\sqrt{2nqr}} [Y^{(1)} -X\Beta| \dots | Y^{(r)} -X\Beta]$ is linear in $\Beta$, thus $f$ is jointly convex in $(\Beta, \Snoise)$, meaning that $(\Beta, \Snoise) \to f + \lambda \norm{\cdot}_{2, 1}$ is jointly convex in $(\Beta, \Snoise)$~.
    Moreover the constraint set is convex and thus solving \us is a convex problem.

    The function $f$ is convex and smooth on the feasible set and $\norm{\cdot}_{2, 1}$ is convex in $\Beta$ and separable in $\Beta_{j:} $'s, thus (see \citealt{Tseng01,Tseng_Yun09}) $f + \lambda \norm{\cdot}_{2, 1}$ can be minimized through coordinate descent in $\Snoise$ and the $\Beta_{j:} $'s (on the feasible set).
\end{proof}

\subsection{Proof of \Cref{prop:update_S_multiepoch}} 
\label{app_sub:proof_updae_S_clar}

\updateS*

\begin{proof}
  Minimizing $f(\Beta, \cdot)$ amounts to solving
  \begin{equation}
    \argmin_{S\succeq\sigmamin\Id_n}
    \tfrac{1}{2} \norm{Z}_{S^{-1}}^2 + \tfrac{1}{2} \Tr(S) \enspace, \enspace \text{ with } Z = \frac{1}{\sqrt{r}}[Z^{(1)} | \dots | Z^{(l)}] \enspace.
  \end{equation}
  The solution is $ \SpCl \Big(ZZ^\top , \sigmamin \Big)$ (see \citealt[Appendix A2]{Massias_Fercoq_Gramfort_Salmon17}), and
   $Z Z^\top = \frac{1}{r} \sum_{l=1}^r Z^{(l)} Z^{(l)\top}$.
\end{proof}

\subsection{Proof of \Cref{prop:update_Bj_clar}} 
\label{app_sub:proof_updae_B_j_clar}
\updateB*

\begin{proof}
    The function to minimize is the sum of a smooth term $f(\cdot, S)$ and a non-smooth but separable term, $\normin{\cdot}_{2,1}$, whose proximal operator \footnote{As a reminder, for a scalar $t > 0$, the proximal operator of a function $h:\bbR^{d} \to \bbR$ can be defined for any $x_0 \in \bbR^d$ by $\prox_{t,h}(x_0) = \argmin_{x \in \bbR^d} \tfrac{1}{2t} \norm{x-x_0}^2 + h(x)~.$} can be computed:
    \begin{itemize}
      \item $f$ is $\norm{X_{:j}}_{ \Snoise^{-1}}^2 / nq $-smooth with respect to $\Beta_{j:}$, with partial gradient
        $\nabla_j f(\cdot, S) = - \frac{1}{nq} X_{:j}^{\top} \Snoise^{-1}(\bar{Y} - X\Beta)$,
      \item $\normin{\Beta}_{2,1} = \sum_{j=1}^p \normin{\Beta_{j:}}$ is row-wise separable over $\Beta$, with     $\prox_{{\lambda n q }/{\norm{X_{:j}}_{ \Snoise^{-1}}^2}, \normin{\cdot}}(\cdot) =
        \BST \left(\cdot, \frac{\lambda n q}{\norm{X_{:j}}_{ \Snoise^{-1}}^2} \right)$.
    \end{itemize}
    Hence, proximal block-coordinate descent converges  \citep{Tseng_Yun09}, and the update are given by \Cref{eq:update_Bj_clar}.
    The closed-form formula arises since the smooth part of the objective is quadratic and isotropic \wrt $\Beta_{j:}$~.
  \end{proof}

\subsection{Proof of $\lambda_{\max}$ \us} 
\label{app_sub:proof_lambda_max_clar}

\begin{proof}
    First notice that if $\Betaopt = 0$, then $\Snoiseopt = \SpCl\left(\tfrac{1}{qr} \sum_{l=1}^r Y^{(l)} Y^{(l)\top}, \sigmamin\right) \eqdef \Snoise_{\max}$~.

    Fermat's rules states that
    \begin{align}
      \Betaopt = 0 &\Leftrightarrow  0 \in \partial \big( f(\cdot, \Snoise_{\max}) + \lambda \normin{\cdot}_{2, 1} \big) (0) \nonumber \\
      &\Leftrightarrow - \nabla f(\cdot, \Snoise_{\max}) \in \lambda \cB_{\normin{\cdot}_{2, \infty}} \nonumber\\
      &\Leftrightarrow \frac{1}{nq} \norm{X^\top \Snoise_{\max}^{-1} \bar{Y}}_{2, \infty} \eqdef \lambda_{\max} \leq \lambda \enspace.
    \end{align}
  \end{proof}

\subsection{Proof of dual formulation}
\label{app_sub:proof_dual_clar}

\begin{proposition}\label{prop:dual_clar}
  With $\Thetaopt = (\Thetaopt^{(1)}, \dots, \Thetaopt^{(r)})$, the dual formulation of \Cref{eq:clar} is
  \begin{align}\nonumber
    \Thetaopt
    &= \argmax_{\substack{(\Theta^{(1)}, \dots, \Theta^{(r)}) \in \Delta_{X, \lambda}}}
      \frac{\sigmamin}{2}
    \left (
        1 - \frac{q n \lambda^2}{r}  \sum_{l=1}^r \Tr \Theta^{(l)} {\Theta^{(l)}}^{\top}
      \right )
        + \dfrac{\lambda}{r} \sum_{l=1}^r \left \langle \Theta^{(l)}, Y^{(l)} \right \rangle \enspace,
  \end{align}
  with  $\bar{\Theta} = \frac{1}{r} \sum_{1}^r \Theta^{(l)}$ and
\begin{align}\label{eq:dual_feasible_clar}
  \nonumber &\Delta_{X, \lambda} = \Big\{
                    (\Theta^{(1)}, \dots, \Theta^{(r)}) \in (\bbR^{n \times q})^r
                    : \norm{X^{\top} \bar{\Theta} }_{2, \infty} \leq 1,
                    \Big \Vert\sum_{l=1}^r \Theta^{(l)} {\Theta^{(l)}}^{\top} \Big \Vert_2 \leq
                    \frac{r}{\lambda^2 n^2 q}
                    \Big\} \enspace.
\end{align}
\end{proposition}

In \Cref{alg:sgcl_and_clar} the dual point $\Theta$ at iteration $t$ is obtained through a residual rescaling similar to the way the dual point is created,\ie $\Theta^{(l)} = \frac{1}{nq\lambda} (Y^{(l)} -X\Beta)$ (with $\Beta$ the current primal iterate); then the dual point hence created is projected on $\Delta_{X, \lambda}$~.

\begin{proof}
  Let the primal optimum be
  \begin{equation}
      p^* \eqdef \min_{
            \substack{\Beta \in \bbR^{p \times q} \\ \Snoise \succeq \sigmamin\Id_n}
            }
            \frac{1}{2nqr}  \sum_{l=1}^{r} \normin{ Y^{(l)} - X \Beta}_{\Snoise^{-1}}^2
            + \frac{1}{2n}\Tr(\Snoise)
            + \lambda \norm{\Beta}_{2, 1}
            \nonumber
    \end{equation}

    Then
    \begin{align}
          p^* & = \nonumber \min_{
              \substack{
                  \Beta \in \bbR^{p \times q}\\
                  R^{(l)} = Y^{(l)} - X \Beta, \; \forall l \in [r] \\
                  \Snoise  \succeq \sigmamin \Id_n
              }
          }
          \frac{1}{2nqr}  \sum_{l=1}^{r} \normin{R^{(l)}}_{\Snoise^{-1}}^2
          + \frac{1}{2n}\Tr(\Snoise)
          + \lambda \norm{\Beta}_{2, 1}
          \enspace\\
          & =        \nonumber \min_{
              \substack{
                  \Beta \in \bbR^{p \times q}\\
                  R^{(1)}, \dots, R^{(r)} \in \bbR^{n \times q} \\
                  \Snoise  \succeq \sigmamin \Id_n
              }
          }
          \max_{\substack{\Theta^{(1)}, \dots, \Theta^{(r)}} \in \bbR^{n \times q} }
          \frac{1}{2nqr}  \sum_{l=1}^{r} \normin{R^{(l)}}_{\Snoise^{-1}}^2
          + \frac{1}{2n}\Tr(\Snoise)\\
          & \phantom{
           = \nonumber \min_{
              \substack{
                  \Beta \in \bbR^{p \times q}\\
                  R^{(1)}, \dots, R^{(r)} \in \bbR^{n \times q} \\
                  \Snoise  \succeq \sigmamin \Id_n
              }
          }
          \max_{\substack{\Theta^{(1)}, \dots, \Theta^{(r)}} \in \bbR^{n \times q} }
          }
          + \lambda \normin{\Beta}_{2, 1}
          + \frac{\lambda}{r} \sum_{l=1}^{r} \Big \langle \Theta^{(l)}, Y^{(l)} - X \Beta - R^{(l)} \Big \rangle
          \enspace.
    \end{align}

    Since Slater's conditions are met $\min$ and $\max$ can be inverted:

    \begin{align}
       p^* &= \max_{\substack{\Theta^{(1)}, \dots, \Theta^{(r)}} \in \bbR^{n \times q} }
           \min_{
              \substack{
                  \Beta \in \bbR^{p \times q}\\
                  R^{(1)}, \dots, R^{(r)} \in \bbR^{n \times q} \\
                  \Snoise  \succeq \sigmamin \Id_n
              }
          }
          \frac{1}{2nqr}  \sum_{l=1}^{r} \normin{R^{(l)}}_{\Snoise^{-1}}^2
          + \frac{1}{2n}\Tr(\Snoise)\\
          & \phantom{
           = \nonumber \min_{
              \substack{
                  \Beta \in \bbR^{p \times q}\\
                  R^{(1)}, \dots, R^{(r)} \in \bbR^{n \times q} \\
                  \Snoise  \succeq \sigmamin \Id_n
              }
          }
          \max_{\substack{\Theta^{(1)}, \dots, \Theta^{(r)}} \in \bbR^{n \times q} }
          }
          + \lambda \norm{\Beta}_{2, 1}
          + \frac{\lambda}{r} \sum_{l=1}^{r} \Big \langle \Theta^{(l)}, Y^{(l)} - X \Beta - R^{(l)} \Big \rangle\\
      &=
          \max_{\substack{\Theta^{(1)}, \dots, \Theta^{(r)} \in \bbR^{n \times q} }}
          \left(
              \min_{\substack{\Snoise  \succeq \sigmamin \Id_n}}
              \frac{1}{r}  \sum_{l=1}^{r}
              \min_{R^{(l)} \in \bbR^{n \times q}
                  }
              \left(
                \frac{\norm{R^{(l)}}_{\Snoise^{-1}}^2}{2nq}
                - \Big \langle  \Theta^{(l)}, R^{(l)} \Big \rangle
              \right)
              +\frac{1}{2n}\Tr(\Snoise)
              \right. \nonumber\\
              &
              \left.
              \phantom{=\max_{\substack{\Theta^{(1)}, \dots, \Theta^{(r)} \in \bbR^{n \times q} }}}
              + \lambda  \min_{
                  \substack{
                       \Beta \in \bbR^{ p \times q}
                  }
              } \left(
              \norm{\Beta}_{2, 1}
              - \Big \langle \bar{\Theta}, X \Beta \Big \rangle
              \right)
              + \frac{\lambda}{r} \sum_{l=1}^{r} \Big \langle \Theta^{(l)}, Y^{(l)} \Big \rangle \right) \enspace .
      \end{align}

    Morover we have
    \[\min_{\substack{R^{(l)} \in \bbR^{n \times q}}}
          \left (
            \frac{\norm{R^{(l)}}_{\Snoise^{-1}}^2}{2nq}
            - \Big \langle  \Theta^{(l)}, R^{(l)}  \Big\rangle
          \right )
          = - \frac{nq \lambda^2}{2} \left \langle \Theta^{(l)} {\Theta^{(l)}}^{\top}, \Snoise \right \rangle \]  and

    \[ \min_{
              \substack{
                   \Beta \in \bbR^{ p \times q}
              }
          } \left(
          \norm{\Beta}_{2, 1}
          - \langle \bar{\Theta}, X \Beta \rangle
          \right)
          = - \max \left (\langle X^{\top} \bar{\Theta}, \Beta \rangle - \norm{\Beta}_{2, 1} \right)
          = - \iota_{\cB_{2, \infty}}(X^{\top}  \bar{\Theta}) \enspace.\]

    This leads to:
    \begin{align}
          d^* &=
          \max_{\substack{\Theta^{(1)}, \dots, \Theta^{(r)}\in \bbR^{n \times q} }}
          \min_{\substack{\Snoise  \succeq \sigmamin \Id_n}}
          - \frac{1}{r}  \sum_{l=1}^r
           \tfrac{nq \lambda^2}{2} \left \langle \Theta^{(l)} {\Theta^{(l)}}^{\top}, \Snoise \right \rangle
           -\lambda \iota_{\cB_{2, \infty}}(X^{\top} \bar{\Theta})
           + \frac{\Tr(\Snoise)}{2n}
           \nonumber \\
           & \phantom{= \nonumber
          \max_{\substack{\Theta^{(1)}, \dots, \Theta^{(r)}\in \bbR^{n \times q} }}
          \min_{\substack{\Snoise  \succeq \sigmamin \Id_n}}
          }
          + \frac{\lambda}{r} \sum_{l=1}^{r} \langle \Theta^{(l)}, Y^{(l)}\rangle
         \nonumber \\
          &=
          \max_{\substack{\Theta^{(1)}, \dots, \Theta^{(r)}\in \bbR^{n \times q} }}
          \frac{1}{2n}
          \min_{\substack{\Snoise  \succeq \sigmamin \Id_n}}
          \left (
            \Big \langle \Id_n, \Snoise \Big \rangle
            - \frac{q n^2 \lambda^2}{r}  \sum_{l=1}^{r}
            \left \langle \Theta^{(l)} {\Theta^{(l)}}^{\top}, \Snoise \right \rangle
          \right )
           -\lambda \iota_{\cB_{2, \infty}}(X^{\top} \bar{\Theta})
            \nonumber \\
          &\phantom{=\max_{\substack{\Theta^{(1)}, \dots, \Theta^{(r)}\in \bbR^{n \times q} }}}
          + \frac{\lambda}{r} \sum_{l=1}^r \langle \Theta^{(l)}, Y^{(l)}\rangle
             \nonumber\\
          &= \nonumber\max_{\substack{\Theta^{(1)}, \dots, \Theta^{(r)} \in \bbR^{n \times q} }}
            \frac{1}{2n}
            \min_{\substack{\Snoise  \succeq \sigmamin \Id_n}}
            \left \langle
                \Id_n - \tfrac{q n^2 \lambda^2}{r}  \sum_{l=1}^{r}
                \Theta^{(l)} {\Theta^{(l)}}^{\top}, \Snoise
            \right \rangle
           -\lambda \iota_{\cB_{2, \infty}}(X^{\top} \bar{\Theta})\nonumber\\
          &\phantom{=\max_{\substack{\Theta^{(1)}, \dots, \Theta^{(r)}\in \bbR^{n \times q} }}}
          + \frac{\lambda}{r} \sum_{l=1}^r \langle \Theta^{(l)}, Y^{(l)}\rangle
          \enspace.
    \end{align}

    \begin{align}
        \nonumber &\min_{\substack{\Snoise  \succeq \sigmamin \Id_n}}
            \left \langle
                \Id_n - \tfrac{q n^2 \lambda^2}{r}  \sum_{l=1}^{r} \Theta^{(l)} {\Theta^{(l)}}^{\top},
                \Snoise
            \right \rangle
       \\& \quad \quad \quad \quad =
       \begin{cases}
            \left \langle
                \Id_n - \tfrac{q n^2 \lambda^2}{r}  \sum_{l=1}^{r}
                \Theta^{(l)} {\Theta^{(l)}}^{\top}, \sigmamin
            \right \rangle \enspace,
        & \text{if } \Id_n \succeq \frac{q n^2 \lambda^2}{r}  \sum_{l=1}^{r} \Theta^{(l)} {\Theta^{(l)}}^{\top} \enspace ,\\
        - \infty \enspace, & \text{otherwise.}
       \end{cases}
    \end{align}
It follows that the dual problem of \us is
\begin{equation}
    \max_{\substack{(\Theta^{(1)}, \dots, \Theta^{(r)}) \in \Delta_{X, \lambda}}}
    \frac{\sigmamin}{2}
  \left (
      1 - \frac{q n \lambda^2}{r}  \sum_{l=1}^{r} \Tr \Theta^{(l)} {\Theta^{(l)}}^{\top}
    \right )
    + \dfrac{\lambda}{r} \sum_{l=1}^{r} \left \langle \Theta^{(l)}, Y^{(l)} \right \rangle
    \enspace,
\end{equation}
where $\Delta_{X, \lambda} \eqdef
  \Big \{
    (\Theta^{(1)}, \dots, \Theta^{(r)}) \in \bbR^{n \times q \times r}:
    \normin{X^{\top} \bar{\Theta} }_{2, \infty} \leq 1,
    \normin{ \sum_{l=1}^r \Theta^{(l)} {\Theta^{(l)}}^{\top} }_2 \leq
    \frac{r}{\lambda^2 n^2 q}
  \Big \}$~.
\end{proof}

\subsection{Proof of \Cref{rem:cost_sgcl_clar}} 
\label{app_sub:proof_cost_sgcl_clar}

\costsgclclar*

\begin{proof}
  \begin{align}
   RR^\top
  & = \sum_{l=1}^r R^{(l)} R^{(l)\top} \nonumber\\
  & = \sum_{l=1}^r (Y^{(l)} - X \Beta) (Y^{(l)} - X \Beta)^\top \nonumber\\
  & = \sum_{l=1}^r  Y^{(l)} Y^{(l)\top}
      -  \sum_1^r Y^{(l)} (X \Beta)^\top -  \sum_1^r X \Beta Y^{(l)\top}
      + r X \Beta (X \Beta)^{\top} \nonumber \\
  & = r \text{cov}_Y
      - r \bar{Y}^{\top} X \Beta
      - r (X\Beta)^\top \bar{Y} +  r X \Beta (X \Beta)^{\top}   \label{eq:computation_sigma_me}
  \end{align}
\end{proof}


\subsection{Statistical comparison}
\label{app_sub:stat_comparison}

In this subsection, we show the statistical interest of using all repetitions of the experiments instead of using a mere averaging as \sgcl would do (remind that the later is equivalent to \us with $r=1$ and $Y^{(1)} = \bar{Y}$, see \Cref{rem:clar_gives_sgcl}).

Let us introduce $\Sigma^*$, the true covariance matrix of the noise (\ie $\Sigma^{*} = \Snoise^{*2} $ with our notation).
In \sgcl and \us alternate minimization consists in a succession of estimations of $\Beta^{*}$ and $\Sigma^{*}$ (more precisely $\Snoise=\SpCl(\Sigma, \sigmamin)$ is estimated along the process).
In this section we explain why the estimation of $\Sigma^*$ provided by \us has better statistical properties than that of \sgcl.
For that, we can compare the estimates of $\Sigma^*$ one would obtain provided that the true parameter $\Beta^*$ is known by both \sgcl and \us.
In such ``ideal'' scenario, the associated estimators of $\Sigma^*$ could be written:
\begin{align}
  \hat\Sigma^{\US} &\eqdef \frac{1}{qr} \sum_{l=1}^r (Y^{(l)}-X\hat\Beta) (Y^{(l)}-X\hat\Beta)^\top\enspace ,\label{eq:def_sig_us}\\
  \hat\Sigma^{\SGCL} &\eqdef \frac{1}{qr} \Big(\sum_{l=1}^r Y^{(l)}-X\hat\Beta \Big) \Big(\sum_{l=1}^r Y^{(l)}-X\hat\Beta \Big)^\top,\label{eq:def_sig_sgcl}
\end{align}
with $\hat\Beta=\Beta^*$, and satisfy the following properties:
\begin{proposition}\label{prop:expectation_covariance_sgcl_clar}
  Provided that the true signal is known, and that the covariance estimator $\hat\Sigma^{\US}$ and $\hat\Sigma^{\SGCL}$ are defined thanks to
  \Cref{eq:def_sig_us,eq:def_sig_sgcl}, then one can check that
  \begin{align}
    \bbE(\hat\Sigma^{\US}) &= \bbE(\hat\Sigma^{\SGCL}) = \Sigma^* \label{eq:expectation_sgcl_clar}
    \enspace , \\
    \cov(\hat\Sigma^{\US}) &= \frac{1}{r} \cov(\hat\Sigma^{\SGCL}) \label{eq:covariance_sgcl_clar}
    \enspace .
  \end{align}
\end{proposition}

\Cref{prop:expectation_covariance_sgcl_clar} states that $\hat\Sigma^{\US}$ and $\hat\Sigma^{\SGCL}$ are unbiased estimators of $\Sigma^*$ 
but our newly introduced \us, improves the estimation of the covariance structure by a factor $r$, the number of repetitions performed.

Empirically\footnote{In that case we plug $\hat\Beta=\hat\Beta^{\US}$ (resp. $\hat\Beta=\hat\Beta^{\US}$) in \Cref{prop:expectation_covariance_sgcl_clar}.}, we have also observed that $\hat\Sigma^{\US}$ has larger eigenvalues than $\hat\Sigma^{\SGCL}$, leading to a less biased estimation of $\Snoise^*$ after clipping the singular values.

Let us recall that
\begin{align}
  \Sigma^{\SGCL} = \frac{1}{qr} \left(\sum_{l=1}^r R^{(l)}\right) \left(\sum_{l=1}^r R^{(l)}\right)^{\top} \quad \quad  \mathrm{and} \quad\quad
  \Sigma^{\US} = \frac{1}{qr} \sum_{l=1}^r R^{(l)} R^{(l)\top} \enspace.
\end{align}

\paragraph{Proof of \Cref{eq:expectation_sgcl_clar}} 

\begin{proof}
  If $\Beta = \Beta^*$, $R^{(l)} = \Snoise^* \Epsilon^{(l)} $, where $\Epsilon^{(l)}$ are random matrices with normal \iid entries, and the result trivially follows.
\end{proof}

\paragraph{Proof of \Cref{eq:covariance_sgcl_clar}} 

\begin{proof}
If $\hat\Beta = \Beta^*$, $Y^{(l)}-X\hat\Beta = \Snoise^* \Epsilon^{(l)} $,  where the $\Epsilon^{(l)}$'s are random matrices with normal \iid entries.

Now, on the one hand :
\begin{align*}
   \hat\Sigma^{\SGCL}
  & = \frac{1}{qr} \left(\sum_{l=1}^{r} \Snoise^* \Epsilon^{(l)}\right) \left(\sum_{l=1}^{r} \Snoise^* \Epsilon^{(l)} \right)^\top \enspace .
\end{align*}
Since $\frac{1}{\sqrt{r}} \sum_{l=1}^{r} \Snoise^* \Epsilon^{(l)} \underset{law}{\sim} \Snoise^* \Epsilon $ it follows that
\begin{align*}
  \hat\Sigma^{\SGCL}
  & \underset{law}{\sim} \frac{1}{q}\Snoise^* \Epsilon(\Snoise^* \Epsilon )^\top, \\
  \cov(\hat\Sigma^{\SGCL})
  &= \frac{1}{q^2}\cov(\Snoise^* \Epsilon(\Snoise^* \Epsilon )^\top)
  \enspace .
\end{align*}

On the other hand:
\begin{align*}
  \hat\Sigma^{\US}
  &= \frac{1}{qr} \sum_{l=1}^{r} \Snoise^* \Epsilon^{(l)} (\Snoise^* \Epsilon^{(l)})^\top \enspace .
\end{align*}
Since the $\Epsilon^{(l)}$'s are independent it follows that

\begin{align*}
  \cov(\hat\Sigma^{\US})
  &= \frac{1}{r^2 q^2} \sum_{l=1}^{r} \cov\left(\Snoise^* \Epsilon^{(l)} (\Snoise^* \Epsilon^{(l)})^\top\right)
   = \frac{1}{r^2 q^2} \sum_{l=1}^{r} \cov\left(\Snoise^* \Epsilon (\Snoise^* \Epsilon)^\top\right) \\
  &= \frac{1}{r q^2} \cov\left(\Snoise^* \Epsilon (\Snoise^* \Epsilon)^\top\right)
   = \frac{1}{r} \cov\left(\hat\Sigma^{\SGCL}\right) \enspace.
\end{align*}
\end{proof}


\section{Alternative estimators}
\label{app_sec:competitors}

We compare \us to several estimators: \sgcl \citep{Massias_Fercoq_Gramfort_Salmon17}, the (smoothed) \mlefull (\mle), and a version of the \mle with multiple repetitions (\mler), an $\ell_{2,1}$ penalized version of the \mrcefull \citep{Rothman_Levina_Zhu10} (\mrce), an $\ell_{2,1}$ penalized version of \mrce with repetitions (\mrcer) and the Multi-Task Lasso (\citealt{Obozinski_Taskar_Jordan10}, \mtl).
The cost of an epoch of block coordinate descent and the cost of computing the duality gap for each algorithm are summarized in \Cref{tab:algorithms_costs}.
The updates of each algorithms are summarized in \Cref{tab:algorithms_updates}.

\us solves \Cref{eq:clar} and \sgcl solves \Cref{eq:sgcl}, let us introduce the definition of the alternative estimation procedures.

\subsection{\mtlfull (\mtl)}
\label{app_sub:mtl}
The \mtl \citep{Obozinski_Taskar_Jordan10} is the classical estimator used when the additive noise is supposed to be homoscedastic (with no correlation).
\mtl is obtained by solving:
\begin{align}\label{eq:multi_task_lasso}
      \Betaopt^{\MTL} \in
       \argmin_{\substack{\Beta \in \bbR^{p \times q}}}
       \frac{1}{2 n q} \norm{\bar{Y} - X \Beta}^2 + \lambda  \norm{\Beta}_{2, 1} \enspace .
\end{align}

\begin{remark}
  It can be seen that trying to use all the repetitions in the \mtl leads to \mtl itself because $\norm{\bar{Y} - X \Beta}^2 = \frac{1}{r} \sum_{l} \norm{Y^{(l)} - X \Beta}^2$.
\end{remark}

\subsection{\mlefull (\mle)}
\label{app_sub:mle}

Here we study a penalized Maximum Likelihood Estimator \citep{Chen_Banerjee17} (\mle).
When minimizing \mlefull the natural parameters of the problem are the regression coefficients $\Beta$ and the precision matrix $\Sigma^{-1}$.
Since real M/EEG covariance matrices are not full rank, one has to be algorithmically careful when $\Sigma$ becomes singular.
To avoid such numerical errors and to be consistent with the smoothed estimator proposed in the paper (\us), let us define the (smoothed) \mle as following:
\begin{problem}\label{eq:mle}
  (\Betaopt^{\MLE}, \Sigmaopt^{\MLE}) \in
   \argmin_{
       \substack{\Beta \in \bbR^{p \times q} \\
       \Sigma \succeq \sigmamin^2 / r ^2}
       }
       \normin{\bar Y - X \Beta}_{\Sigma^{-1}}^2
       - \log \det (\Sigma^{-1})
       + \lambda  \norm{\Beta}_{2, 1} \enspace,
\end{problem}
and its repetitions version (\mler):
\begin{problem}\label{eq:mler}
  (\Betaopt^{\MLER}, \Sigmaopt^{\MLER}) \in
   \argmin_{
       \substack{\Beta \in \bbR^{p \times q} \\
       \Sigma \succeq \sigmamin^2}
       }
       \sum_{1}^r \normin{Y^{(l)} - X \Beta}_{\Sigma^{-1}}^2
       - \log \det (\Sigma^{-1})
       + \lambda  \norm{\Beta}_{2, 1} \enspace.
\end{problem}

\Cref{eq:mle,eq:mler} are not convex because the objective functions are not convex in $(\Beta, \Sigma^{-1})$, however they are biconvex, \ie convex in $\Beta$ and convex in $\Sigma^{-1}$.
Alternate minimization can be used to solve \Cref{eq:mle,eq:mler}, but without  guarantees to converge toward a global minimum.

\paragraph{Minimization in $\Beta_{j:} $}
As for \us and \sgcl the updates in $\Beta_{j:} $'s for \mle and \mler clearly read:
\begin{align}\label{eq:update_Bj_mle}
  \Beta_{j:}
  = \BST \left(\Beta_{j:}
  + \frac{X_{:j}^{\top} \Sigma^{-1} (\bar Y - X\Beta)}{ \norm{X_{:j}}_{ \Sigma^{-1}}^2 } ,
  \frac{\lambda n q}{\norm{X_{:j}}_{ \Sigma^{-1}}^2 } \right)
  \enspace.
\end{align}
\paragraph{Minimization in $\Sigma^{-1}$:} for \mle (\resp for \mler) the update in $\Sigma$ reads
\begin{align}
  \Sigma = \Cl(\Sigma^{\EMP}, \sigmamin^2) \quad(\text{\resp } \Sigma = \Cl(\Sigma^{\EMP, r}, \sigmamin^2)) \enspace,
\end{align}
with $\Sigma^{\EMP} \eqdef \frac{1}{q} (\bar Y - X \Beta) (\bar Y - X \Beta)^{\top}$ (\resp $\Sigma^{\EMP, r} \eqdef \frac{1}{rq} \sum_{l=1}^r (Y^{(l)} - X \Beta) (Y^{(l)} - X \Beta)^{\top})$

Let us prove the last result.
Minimizing \Cref{eq:mle} in $\Sigma^{-1}$ amounts to solving
\begin{problem}\label{eq:min_Sigma_1_reprop}
    \Sigmaopt^{-1} \in
     \argmin_{
         \substack{0 \prec \Sigma^{-1} \preceq 1 / \sigmamin^2 \quad \,\,}
         }
         \left \langle\Sigma^{\EMP},  \Sigma^{-1} \right \rangle
         - \log\det(\Sigma^{-1}) \enspace.
\end{problem}

\begin{theorem} \label{thm:min_in_Sigma_1}
    Let $\Sigma^{\EMP} = U \diag(\sigma_i^2) U^{\top}$ be an eigenvalue decomposition of $\Sigma^{\EMP}$, a solution to  \Cref{eq:min_Sigma_1_reprop} is given by:
    \begin{align}
        \Sigmaopt^{ -1} =
        U \diag \left ({\frac{1}{\sigma_i^2 \vee \sigmamin^2}} \right ) U^{\top}
    \end{align}
\end{theorem}

\Cref{thm:min_in_Sigma_1} is very intuitive, the solution of the smoothed optimization problem \eqref{eq:min_Sigma_1_reprop} is the solution of the non-smoothed problem, where the eigenvalues of the solution have been clipped to satisfy the constraint. Let us proove this result.

\begin{proof}\label{prf:thm_min_in_Sigma_1}
    The KKT conditions of \Cref{eq:min_Sigma_1_reprop} for conic programming (see \citealt[p. 267]{Boyd_Vandenberghe04}) state that the optimum in the primal $\Sigmaopt^{-1}$ and the optimum in the dual $\Gammaopt$ should satisfy:
    \begin{align*}
        \Sigma^{\EMP} - \Sigmaopt + \Gammaopt &= 0 \enspace,
        &\Gammaopt^{\top} (\Sigmaopt^{-1} - \frac{1}{\sigmamin^2}\Id_n) = 0 \enspace, \\
        \Gammaopt &\in \cS_{+}^n \enspace,
        &0 \prec \Sigmaopt^{-1} \preceq \frac{1}{\sigmamin^2} \enspace.
    \end{align*}

    Since \Cref{eq:min_Sigma_1_reprop} is convex these conditions are also sufficient. Let us propose a primal-dual point $(\Sigmaopt^{-1}, \Gammaopt)$ satisfying the KKT conditions.
    Let $\Sigma^{\EMP} = U \diag(\sigma_i^2) U^{\top}$ be an eigenvalue decomposition of $\Sigma^{\EMP}$, one can check that
    \begin{align*}
        \Sigmaopt^{-1}
        &= U \diag({\frac{1}{\sigma_i^2 \vee \sigmamin^2}}) U^{\top} \enspace,\\
        \Gammaopt
        &= U \diag(\sigma_i^2 \vee \sigmamin^2 - \sigma_i^2) U^{\top} \enspace.
    \end{align*}
    verify the KKT conditions, leading to the desired result.
\end{proof}


\subsection{\mrcefull (\mrceo)}
\label{app_sub:mrce}

\mrceo \citep{Rothman_Levina_Zhu10} jointly estimates the regression coefficients (assumed to be sparse) and the precision matrix (\ie the inverse of the covariance matrix), which is supposed to be sparse as well.
Originally in \cite{Rothman_Levina_Zhu10} the sparsity enforcing term on the regression coefficients was an $\ell_1$-norm, which is not well suited for our problem, that is why in \Cref{app_sub_sub:mrce21}we introduce an $\ell_{2, 1}$ penalized version of \mrceo: \mrce.

\subsubsection{\mrcefull}
\label{app_sub_sub:mrce}

\mrce if defined as the solution of the following optimization problem:
\begin{problem}\label{eq:mrce}
  (\Betaopt^{\MRCEO}, \Sigmaopt^{\MRCEO}) \in
   \argmin_{
       \substack{\Beta \in \bbR^{p \times q} \\
       \Sigma^{-1} \succ 0 \quad \,\,}
       }
       \norm{\bar Y - X \Beta}_{\Sigma^{-1}}^2
       - \log \det (\Sigma^{-1})
       + \lambda  \norm{\Beta}_{1}
       + \mu \norm{\Sigma^{-1}}_{1} \enspace.
\end{problem}
\Cref{eq:mrce} is not convex, but can be solved heuristically (see \citealt{Rothman_Levina_Zhu10} for details) by coordinate descent doing soft-tresholdings for the udpdates in $\Beta_{j:} $'s and solving a Graphical Lasso \citep{Friedman_Hastie_Tibshirani08} for the update in $\Sigma^{-1}$.
The $\ell_{1}$-norm being not well suited for our problem, we introduce an $\ell_{2, 1}$ version of \mrceo.

\subsubsection{\mrcefull with $l_{2,1}$-norm (\mrce)}
\label{app_sub_sub:mrce21}

The $\ell_1$-norm penalization on the regression penalization $\Beta$ being not well suited for our problem, one can think to an $\ell_{2,1}$-penalized version of \mrceo defined as follow:
\begin{problem}\label{eq:mrce21}
  (\Betaopt^{\MRCE}, \Sigmaopt^{\MRCE}) \in
   \argmin_{
       \substack{\Beta \in \bbR^{p \times q} \\
       \Sigma^{-1} \succ 0 \quad \,\,}
       }
       \norm{\bar Y - X \Beta}_{\Sigma^{-1}}^2
       - \log \det (\Sigma^{-1})
       + \lambda  \norm{\Beta}_{2, 1}
       + \mu \norm{\Sigma^{-1}}_{1} \enspace.
\end{problem}

In order to combine \mrce to take take advantage of all the repetitions, one can think of the following estimator:
\begin{align}\label{eq:mrcer}
  (\Betaopt^{\MRCER}, \Sigmaopt^{\MRCER}) \in
   \argmin_{
       \substack{\Beta \in \bbR^{p \times q} \\
       \Sigma^{-1} \succeq  0 \quad \,\,}
       }
       \sum_{1}^r \norm{Y^{(l)} - X \Beta}_{\Sigma^{-1}}^2
       - \log \det (\Sigma^{-1})
       + \lambda  \norm{\Beta}_{2, 1}
       + \mu \norm{\Sigma^{-1}}_1 \enspace .
\end{align}
As for \Cref{app_sub_sub:mrce}, \Cref{eq:mrce21} (\resp \eqref{eq:mrcer}) can be heuristically solved through coordinate descent.

\paragraph{Update in $\Beta_{j:} $}
It is the same as \mle and \mler:
\begin{align}\label{eq:update_Bj_mle_bis}
  \Beta_{j:}
  = \BST \left(\Beta_{j:}
  + \frac{X_{:j}^{\top} \Sigma^{-1} (\bar Y - X\Beta)}{ \norm{X_{:j}}_{ \Sigma^{-1}}^2 } ,
  \frac{\lambda n q}{\norm{X_{:j}}_{ \Sigma^{-1}}^2 } \right)
  \enspace.
\end{align}

\paragraph{Update in $\Sigma^{-1}$}
Minimizing \eqref{eq:mrce21} in $\Sigma^{-1}$ amounts to solve:
\begin{align}\label{eq:glasso}
  \textrm{glasso}(\Sigma, \mu) \eqdef
   \argmin_{
       \substack{
       \Sigma^{-1} \succ 0 \quad \,\,}
       }
       \langle \Sigma^{\EMP}, \Sigma^{-1} \rangle
       - \log \det (\Sigma^{-1})
       + \mu \norm{\Sigma^{-1}}_{1} \enspace.
\end{align}

This is a well known and well studied problem \citep{Friedman_Hastie_Tibshirani08} that can be solved through coordinate descent.
For ourselves we used the \texttt{scikit-learn} \citep{Pedregosa_etal11} implementation of the Graphical Lasso.
Note that applying the Graphical Lasso on very ill conditioned empirical covariance matrix such as $\Sigma^{\EMP}$ is very long.
We thus only considered \mrcer were the Graphical Lasso is applied on $\Sigma^{\EMP, r}$.

\subsection{Algorithms summary}
\label{app_sub:summary_algos}

Each estimator, proposed or compared to is based on an optimization problem to solve.
Each optimization problem is solve with block coordinate descent, whether there is theoretical guarantees for it to converge toward a global minimum (for convex formulations, \us, \sgcl and \mtl), or not (for non-convex formulations, \mle, \mler, \mrcer).
The cost for the updates for each algorithm can be found in \Cref{tab:algorithms_costs}.
The formula for the updates in $\Beta_{j:} $'s and $\Snoise / \Sigma$ for each algorithm can be found in \Cref{tab:algorithms_updates}.

Let $\freq$ be the number of updates of $\Beta$ for one update of $\Snoise$ or $\Sigma$.

{\centering
\begin{table}[H]
  \caption{Algorithms cost in time summary}
  \label{tab:algorithms_costs}
  \centering
  \begin{tabular}{lccc}
    \toprule
      & CD epoch cost & convex & dual gap cost \\
    \midrule
     \sgclme
     & $\cO(\frac{n^3 + qn^2}{\freq} + pn^2 + pnq) $
     & yes &$\cO(rnq + p)$ \\
    \sgcl
    & $\cO(\frac{n^3 + qn^2}{\freq} + pn^2 + pnq) $
    & yes
    & $\cO(nq + p)$ \\
    \mler
    & $\cO(\frac{n^3 + qn^2}{\freq} + pn^2 + pnq) $
    & no
    & not convex \\
    \mle
    & $\cO(\frac{n^3 + qn^2}{\freq} + pn^2 + pnq) $
    & no
    & not convex  \\
    \mrcer
    & $\cO(\frac{\cO(\glso)}{\freq} + pn^2 + pnq) $
    & no
    & not convex \\
    \mtl  & $\cO(n p q)$ & yes & $\cO(nq + p)$ \\
    \bottomrule
  \end{tabular}
\end{table}
}

Recalling that $\Sigma^{\EMP} \eqdef \frac{1}{q} (\bar Y - X \Beta) (\bar Y - X \Beta)^{\top}$ and $\Sigma^{\EMP, r} \eqdef \frac{1}{rq} \sum_{l=1}^r (Y^{(l)} - X \Beta) (Y^{(l)} - X \Beta)^{\top}$, a summary of the updates in $\Snoise / \Sigma$ and $\Beta_{j:} $'s for each algorithm is given in \Cref{tab:algorithms_updates}.

\paragraph{Comments on \Cref{tab:algorithms_updates}}
The updates in $\Snoise / \Sigma$ and $\Beta_{j:} $'s are given in \Cref{tab:algorithms_updates}.
Although the updates may look similar, all the algorithms can lead to very different results, see \Cref{fig:real_left_audi,fig:real_data_right_audi,fig:real_data_left_visu,fig:real_data_right_visu}.

\begin{table}[H]
  \caption{Algorithms updates summary}
  \label{tab:algorithms_updates}
  \begin{tabular}{lcc}
    \toprule
      & update in $\Beta_{j:} $  & update in $\Snoise / \Sigma$ \\
    \midrule
     \us  & $  \Beta_{j:} = \BST \left(\Beta_{j:} + \frac{X_{:j}^{\top} \Snoise^{-1} (\bar Y - X\Beta)}{ \norm{X_{:j}}_{ \Snoise^{-1}}^2 } ,
     \frac{\lambda n q}{\norm{X_{:j}}_{ \Snoise^{-1}}^2 } \right)$  &$\Snoise = \SpCl(\Sigma^{\EMP, r}, \sigmamin)$\\
    \sgcl & $  \Beta_{j:} = \BST \left(\Beta_{j:} + \frac{X_{:j}^{\top} \Snoise^{-1} (\bar Y - X\Beta)}{ \norm{X_{:j}}_{ \Snoise^{-1}}^2 } ,
    \frac{\lambda n q}{\norm{X_{:j}}_{ \Snoise^{-1}}^2 } \right)$ & $\Snoise = \SpCl(\Sigma^{\EMP}, \sigmamin)$\\
    \mler
    & $  \Beta_{j:} = \BST \left(\Beta_{j:} + \frac{X_{:j}^{\top} \Sigma^{-1} (\bar Y - X\Beta)}{ \norm{X_{:j}}_{ \Sigma^{-1}}^2 } ,
    \frac{\lambda n q}{\norm{X_{:j}}_{ \Sigma^{-1}}^2 } \right)$ &
    $\Sigma = \Cl(\Sigma^{\EMP, r}, \sigmamin^2)$\\
    \mle
    & $  \Beta_{j:} = \BST \left(\Beta_{j:} + \frac{X_{:j}^{\top} \Sigma^{-1} (\bar Y - X\Beta)}{ \norm{X_{:j}}_{ \Sigma^{-1}}^2 } ,
    \frac{\lambda n q}{\norm{X_{:j}}_{ \Sigma^{-1}}^2 } \right)$
    & $\Sigma = \Cl(\Sigma^{\EMP}, \sigmamin^2)$ \\
    \mrcer
    & $  \Beta_{j:} = \BST \left(\Beta_{j:} + \frac{X_{:j}^{\top} \Sigma^{-1} (\bar Y - X\Beta)}{ \norm{X_{:j}}_{ \Sigma^{-1}}^2 } ,
    \frac{\lambda n q}{\norm{X_{:j}}_{ \Sigma^{-1}}^2 } \right)$
    & $\Sigma = \glso(\Sigma^{\EMP, r}, \mu)$ \\
    \mtl
    & $\Beta_{j:} = \BST \left(\Beta_{j:} + \frac{X_{:j}^{\top} (\bar Y - X\Beta)}{ \norm{X_{:j}}^2 } ,
    \frac{\lambda n q}{\norm{X_{:j}}^2 } \right)$
    & no update in $\Snoise / \Sigma$\\
    \bottomrule
  \end{tabular}
\end{table}


\section{Supplementary experiments}
\label{app_sec:supplementary_experiments}

In this section we describe the preprocessing pipeline used for the realistic and real data (see \Cref{app_sub:preprocessing}).
We then propose time comparison for all the algorithms (see \Cref{app_sub:time}).
And finally we expose supplementary experiments on real data (see \Cref{app_sub:real_right_audi,app_sub:real_left_visual,app_sub:real_right_audi}).

\subsection{Preprocessing steps for realistic and real data}
\label{app_sub:preprocessing}

When using multi-modal data without whitening, one has to rescale properly data, indeed data needs to have the same order of magnitude, otherwise some mode (for example EEG data) could be (almost) completely ignored by the optimization algorithm.
The preprocessing pipeline used to rescale realistic data (\Cref{fig:roc_curves_semi_real_influ_amp,fig:roc_curves_semi_real_influ_n_epochs}) and real data (\Cref{fig:real_left_audi,fig:real_data_right_audi,fig:real_data_left_visu,fig:real_data_right_visu}) is described in \Cref{alg:rescaling_process}.

{\fontsize{4}{4}\selectfont
\begin{algorithm}[H]
\SetKwInOut{Input}{input}
\SetKwInOut{Init}{init}
\SetKwInOut{Parameter}{param}
\caption{\textsc{Preprocessing steps for realistic and real data}}
\Input{\quad $X, Y^{(1)}, \dots, Y^{(r)}$}
    \tcp{rescale each line of $X$}
    \For{$i =1,\dots, n$}{
      \For{$l =1,\dots, r$}{
        $Y_{i : }^{(l)} \leftarrow Y_{i : }^{(l)} / \norm{X_{i :}}$
      }
      $X_{i : } \leftarrow X_{i : } / \norm{X_{i :}}$
      }
    \tcp{rescale each column of $X$}
    \For{$j =1,\dots, q$}{
      $X_{: j} \leftarrow X_{: j} / \norm{X_{: j}}$
      }

\Return{$X, Y^{(1)}, \dots, Y^{(r)}$}
\label{alg:rescaling_process}
\end{algorithm}
}

\subsection{Time comparison}
\label{app_sub:time}

The goal of this experiment is to show that our algorithm (\us) is as costly as a \mtlfull or other competitors (in the M/EEG context, \ie $n$ not too large).
The time taken by each algorithm to produce \Cref{fig:real_left_audi} (real data, left auditory stimulations) is given in \Cref{fig:time_comp}.
In this experiment the tolerance is set to $\text{tol=}10^{-3}$, the safe stopping criterion is $\text{duality gap} < \text{tol}$ (only available for convex optimization problems).
The heuristic stopping criterion is "if the objective do not decrease enough anymore then stop" \ie $\text{ if objective}(\Beta^{(t)}, \Sigma^{(t)}) - \text{objective}(\Beta^{(t+1)}, \Sigma^{(t+1)}) < \text{tol} / 10 \text{ then stop}$.
The safe stopping criterion is only available for \us, \sgcl and \mtl (it takes too much time \ie more than 10min for \sgcl to have a duality gap under the fixed tol, so we remove it).
\begin{figure}[H]
  \centering 
 \includegraphics[width=0.97\linewidth]{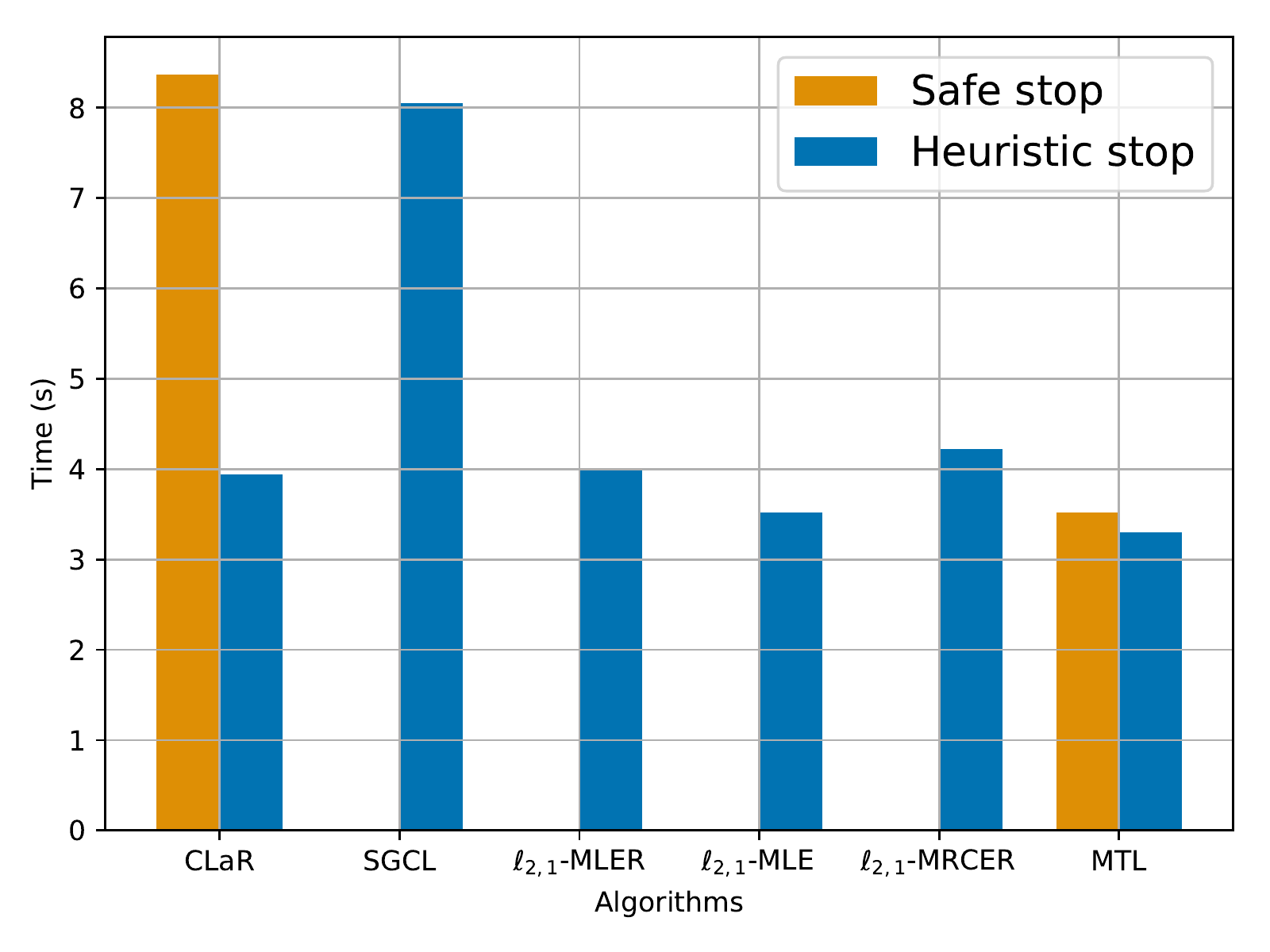}
 \caption{\textit{Time comparison, real data, $n=102$, $p=7498$, $q=54$, $r=56$} Time for each algorithm to produce \Cref{fig:real_left_audi}.}
  \label{fig:time_comp}
\end{figure}

\paragraph{Comment on \Cref{fig:time_comp}}
\Cref{fig:time_comp} shows that if we use the heuristic stopping criterion, \us  is as fast the other algorithm.
In addition \us has a safe stopping criterion which only take 2 to 3 more time than the heuristic one (less than 10sec).

\subsection{Supplementary experiments on real data: right auditory stimulations}
\label{app_sub:real_right_audi}

\begin{figure}[H]
  \centering 

  \begin{subfigure}{\figsize \textwidth}
  \includegraphics[width=68px,height=47px]{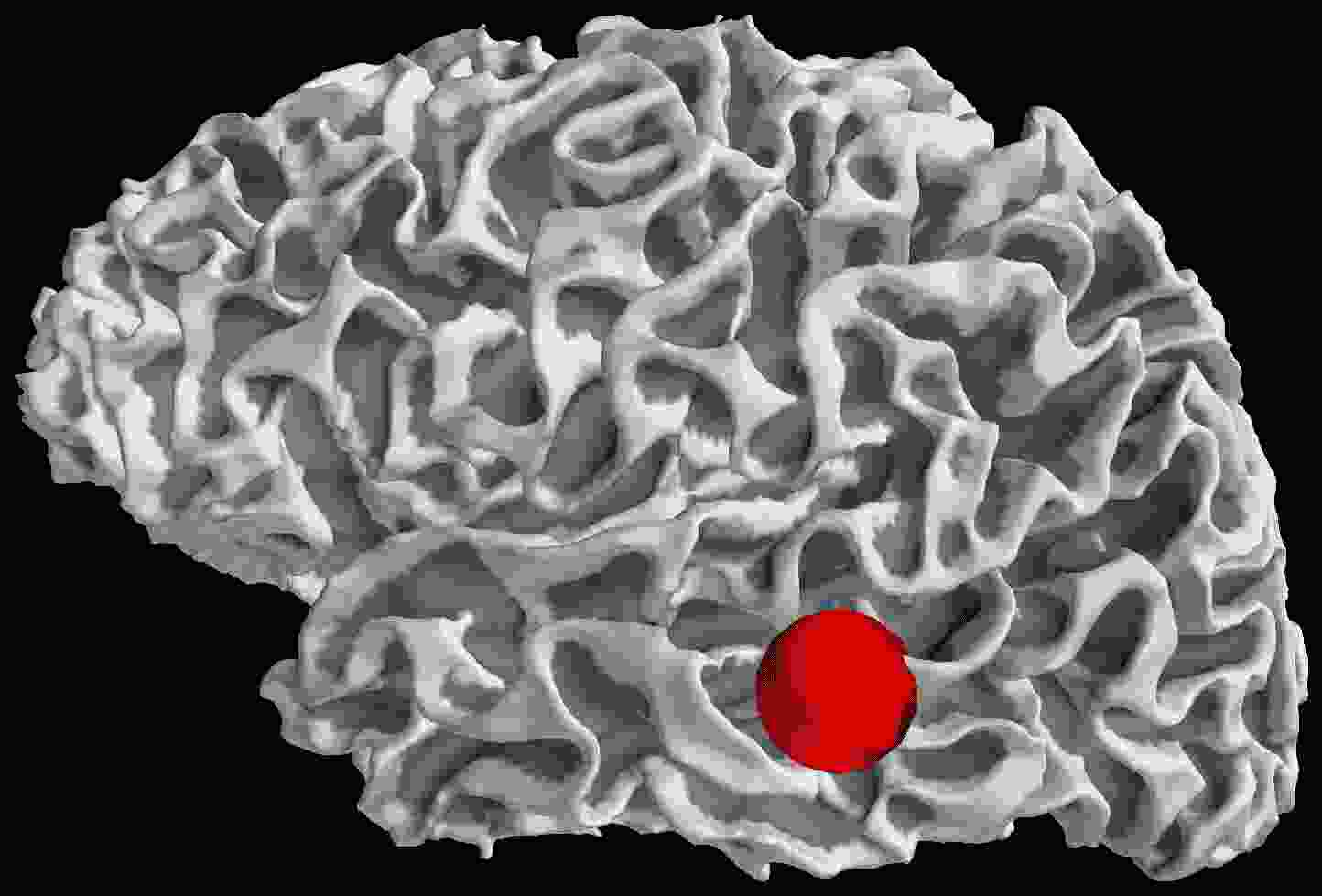}
  \end{subfigure}\hspace{\spacefig em}
  \begin{subfigure}{\figsize \textwidth}
  \includegraphics[width=68px,height=47px]{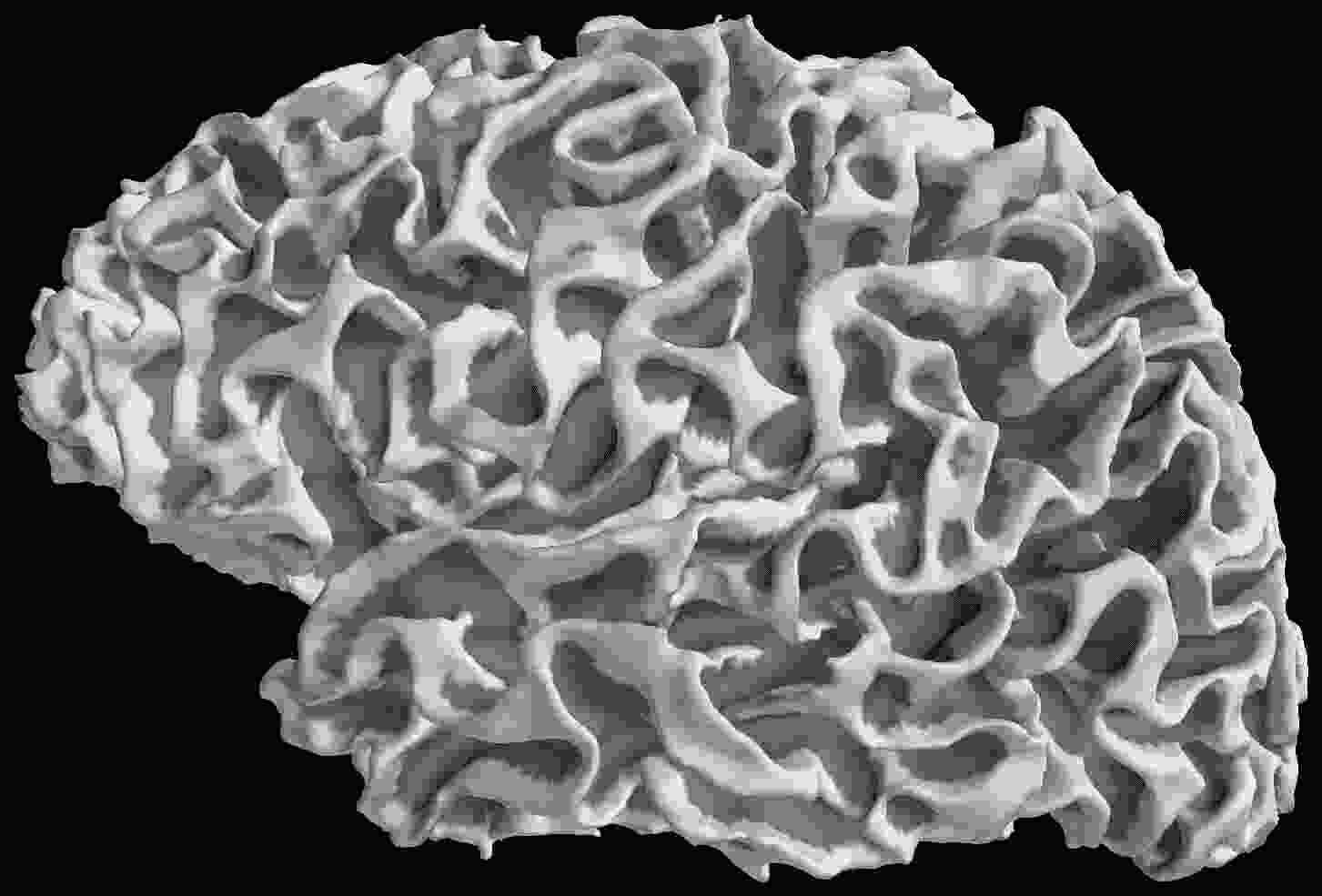}
  \end{subfigure}\hspace{\spacefig em}
  \begin{subfigure}{\figsize\textwidth}
    \includegraphics[width=68px,height=47px]{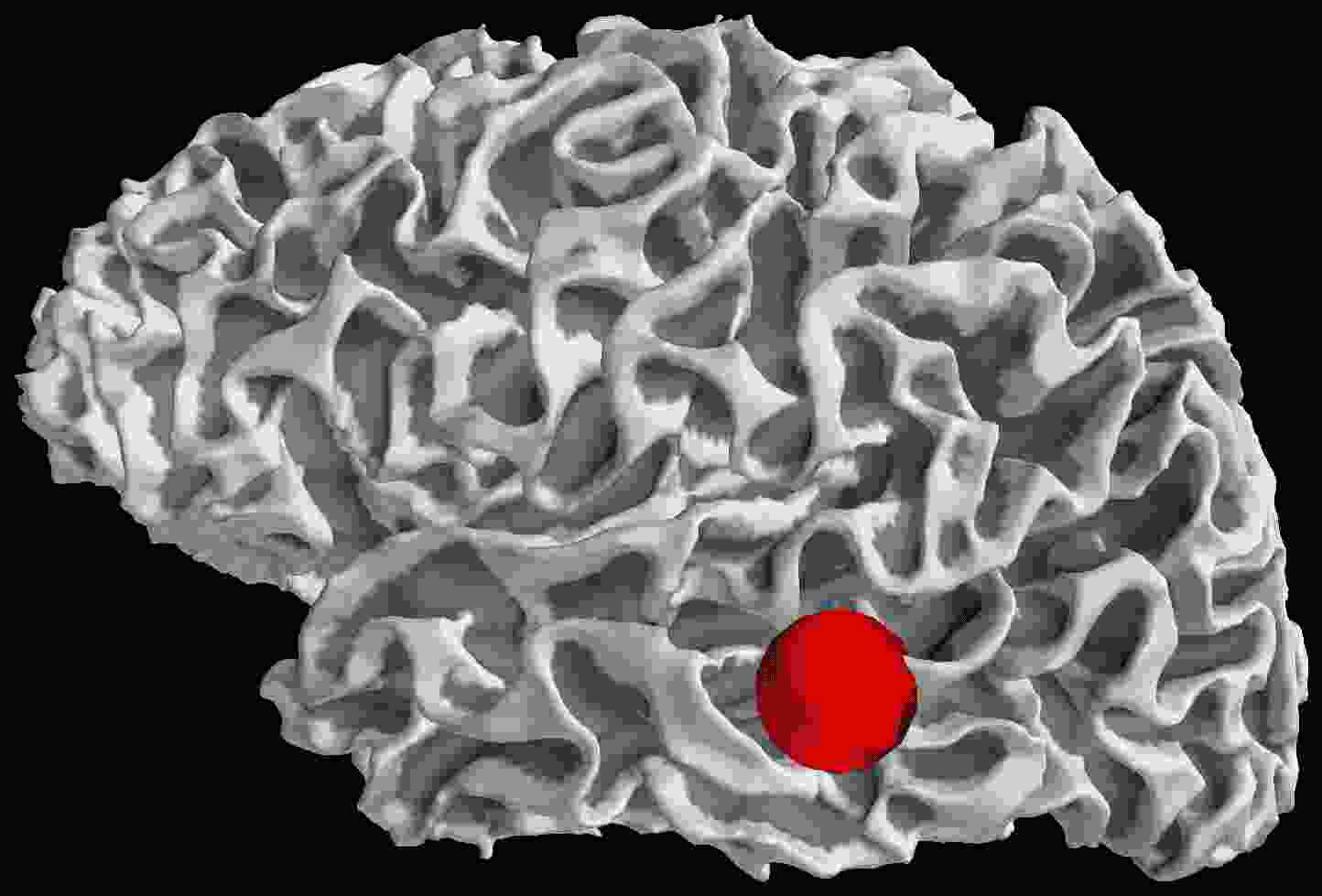}
    \end{subfigure}\hspace{\spacefig em}
  \begin{subfigure}{\figsize\textwidth}
  \includegraphics[width=68px,height=47px]{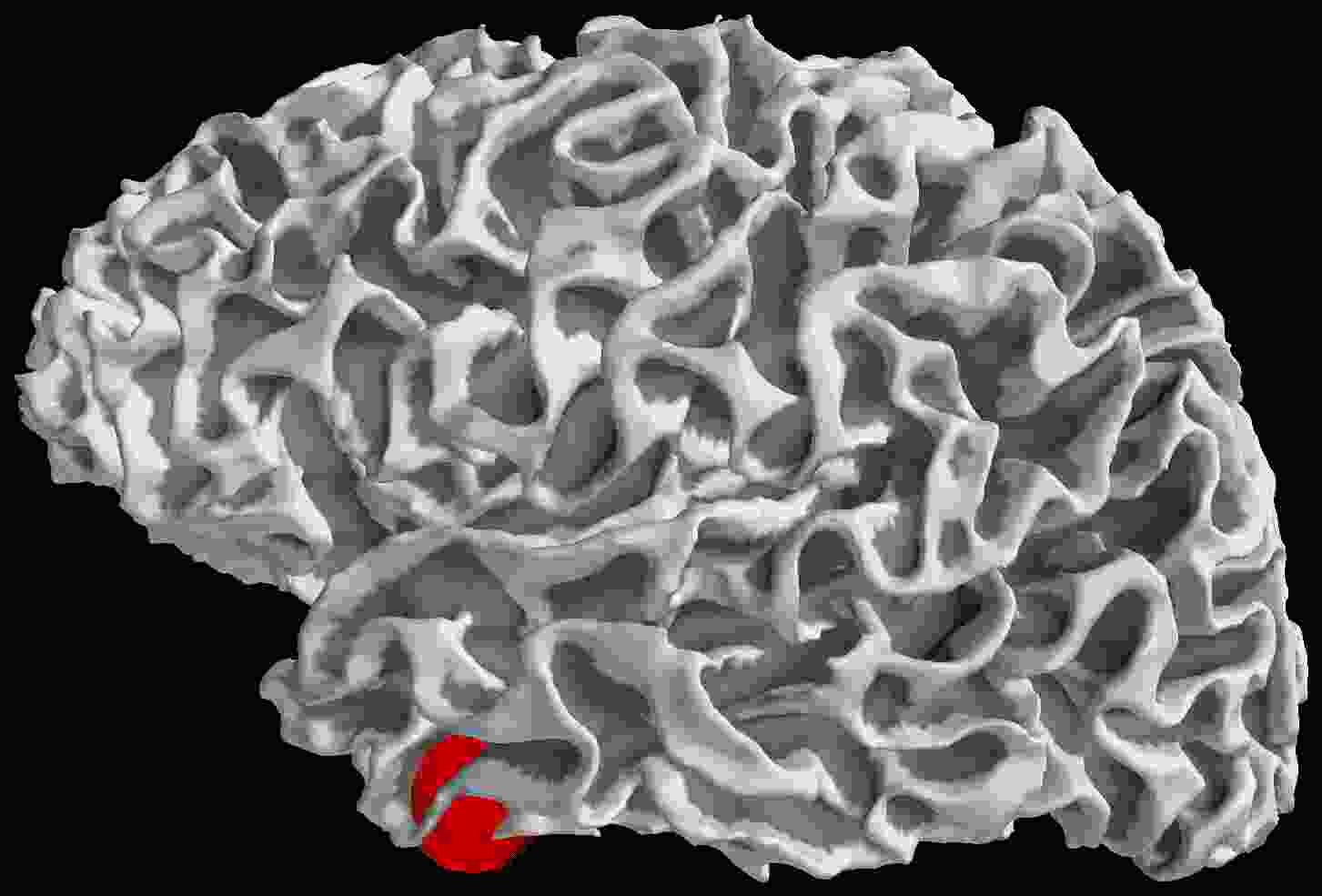}
  \end{subfigure}\hspace{\spacefig em}
  \begin{subfigure}{\figsize\textwidth}
  \includegraphics[width=68px,height=47px]{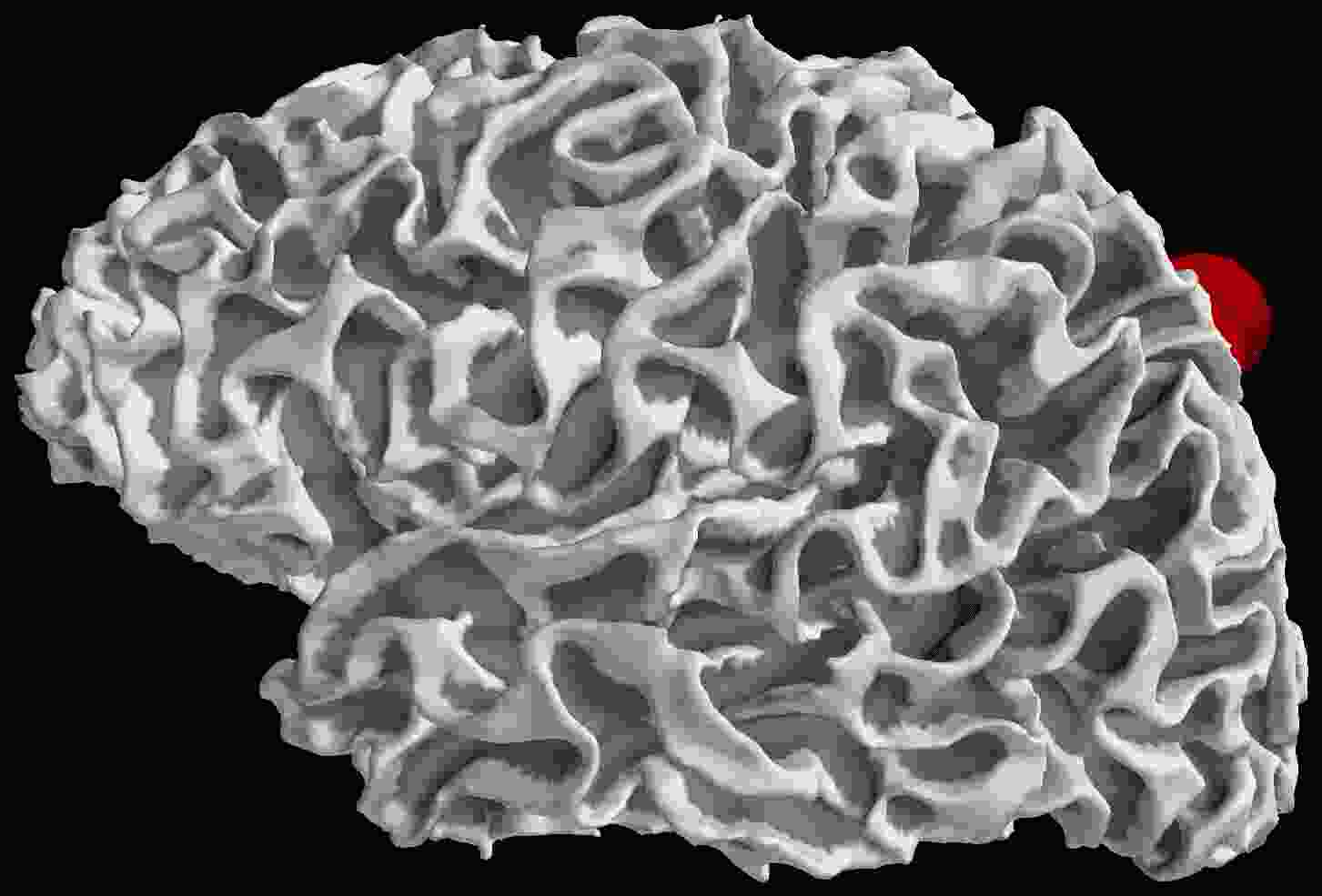}
  \end{subfigure}\hspace{\spacefig em}
  \begin{subfigure}{\figsize \textwidth}
    \includegraphics[width=68px,height=47px]{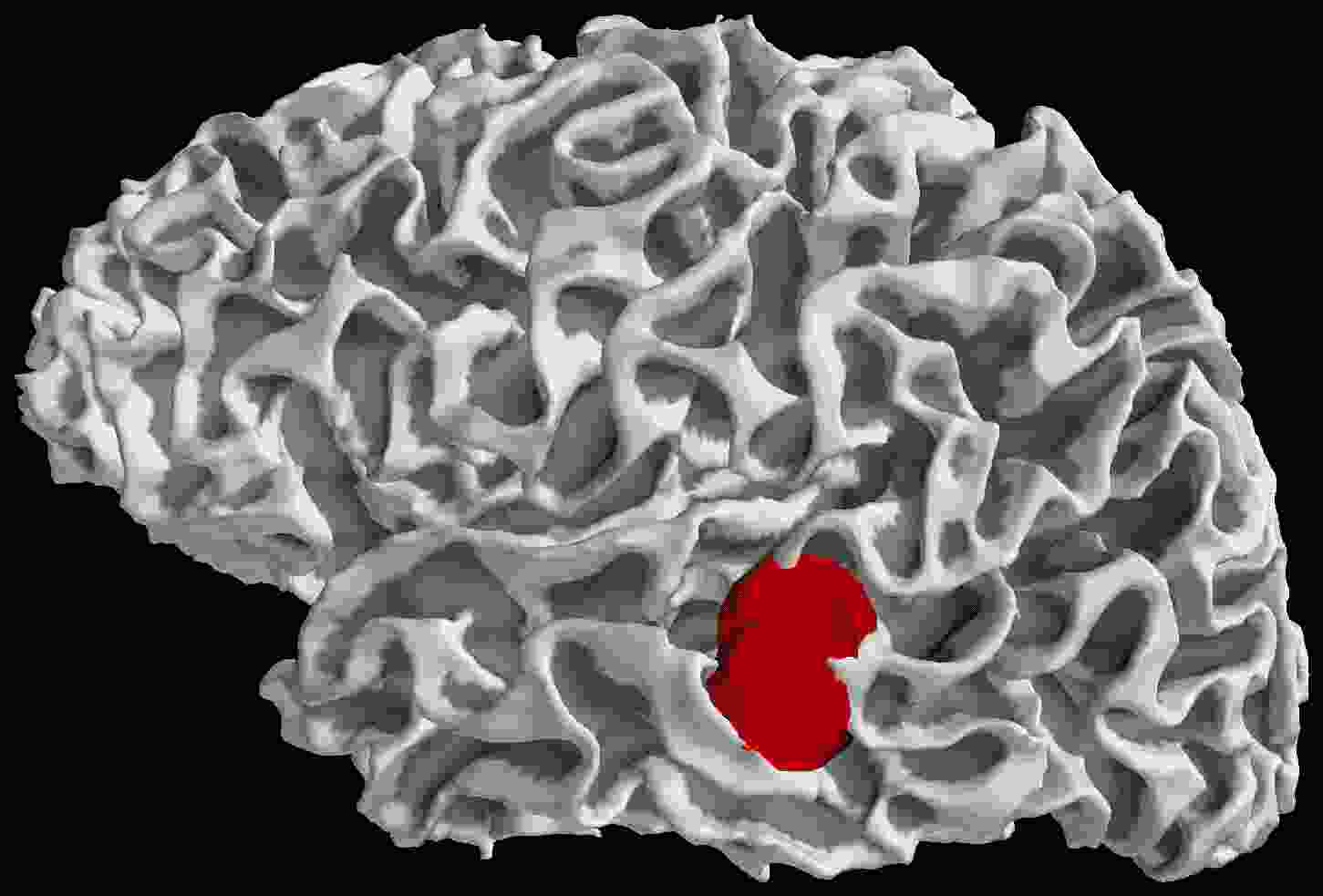}
    \end{subfigure}\hspace{\spacefig em}

  \vspace{-0.1em}
  \begin{subfigure}{\figsize\textwidth}
    \includegraphics[width=68px,height=47px]{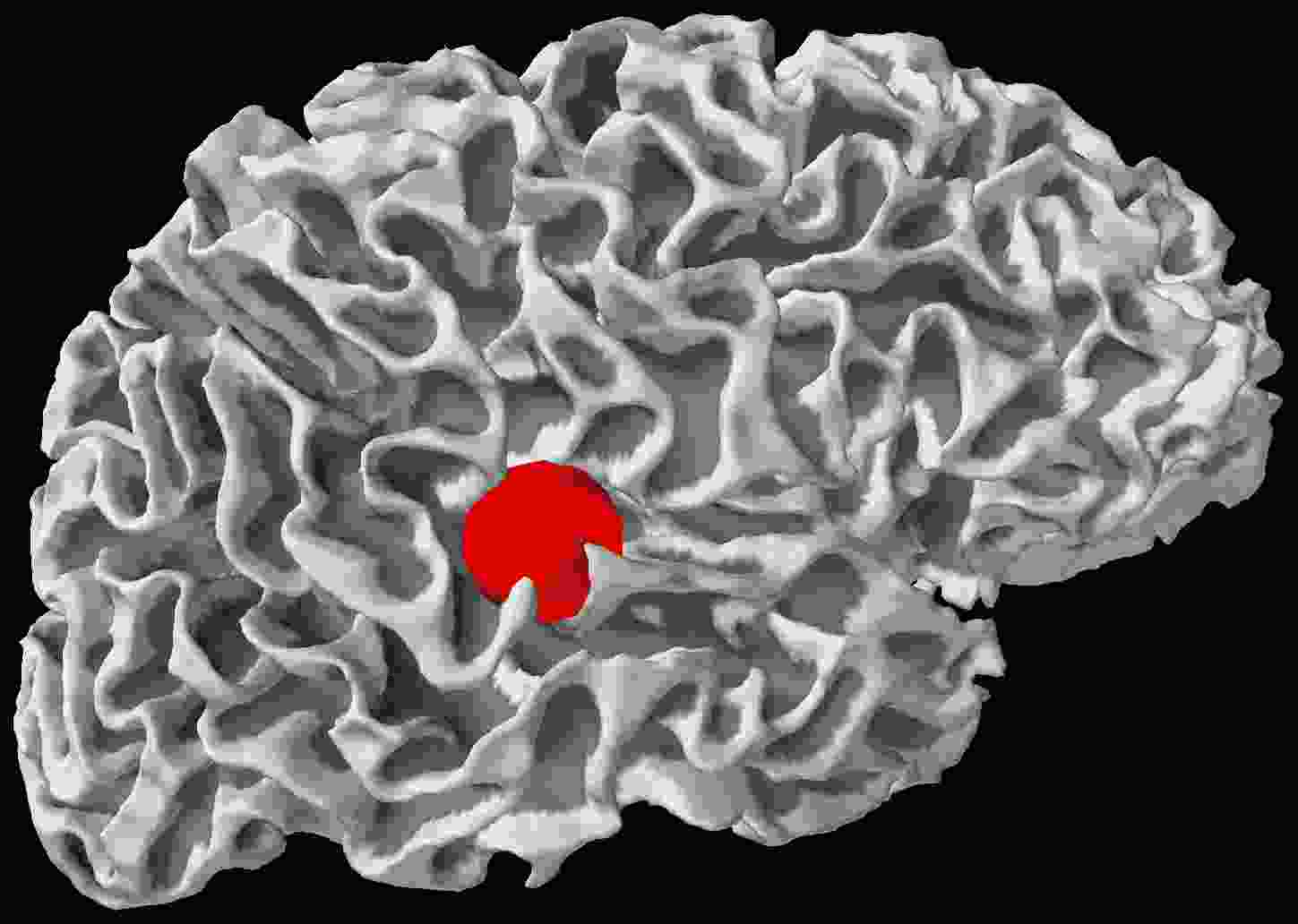}
    \caption{\us}
  \end{subfigure}\hspace{\spacefig em}
  \begin{subfigure}{\figsize\textwidth}
    \includegraphics[width=68px,height=47px]{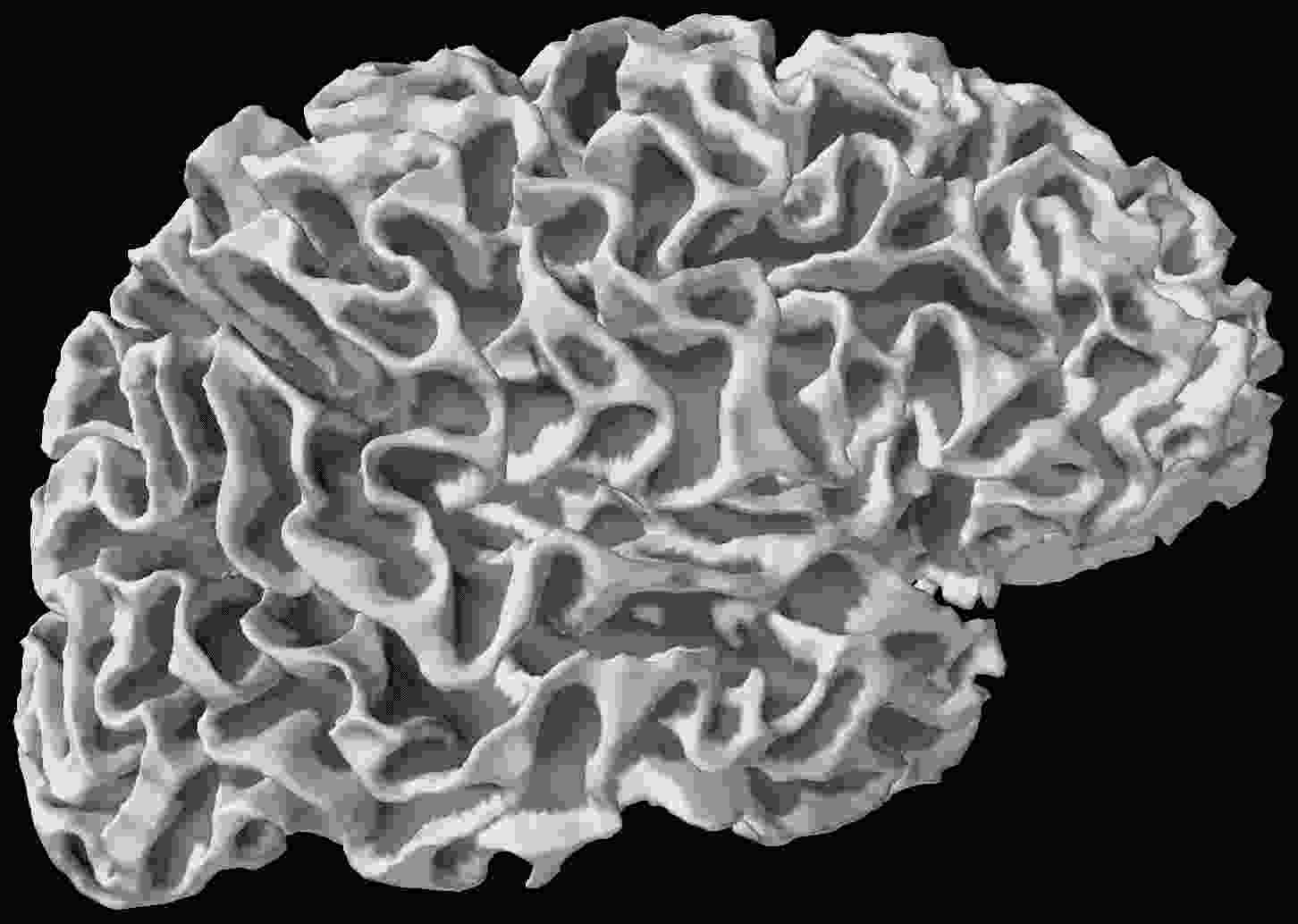}
    \caption{\sgcl}
  \end{subfigure}\hspace{\spacefig em}
  \begin{subfigure}{\figsize\textwidth}
    \includegraphics[width=68px,height=47px]{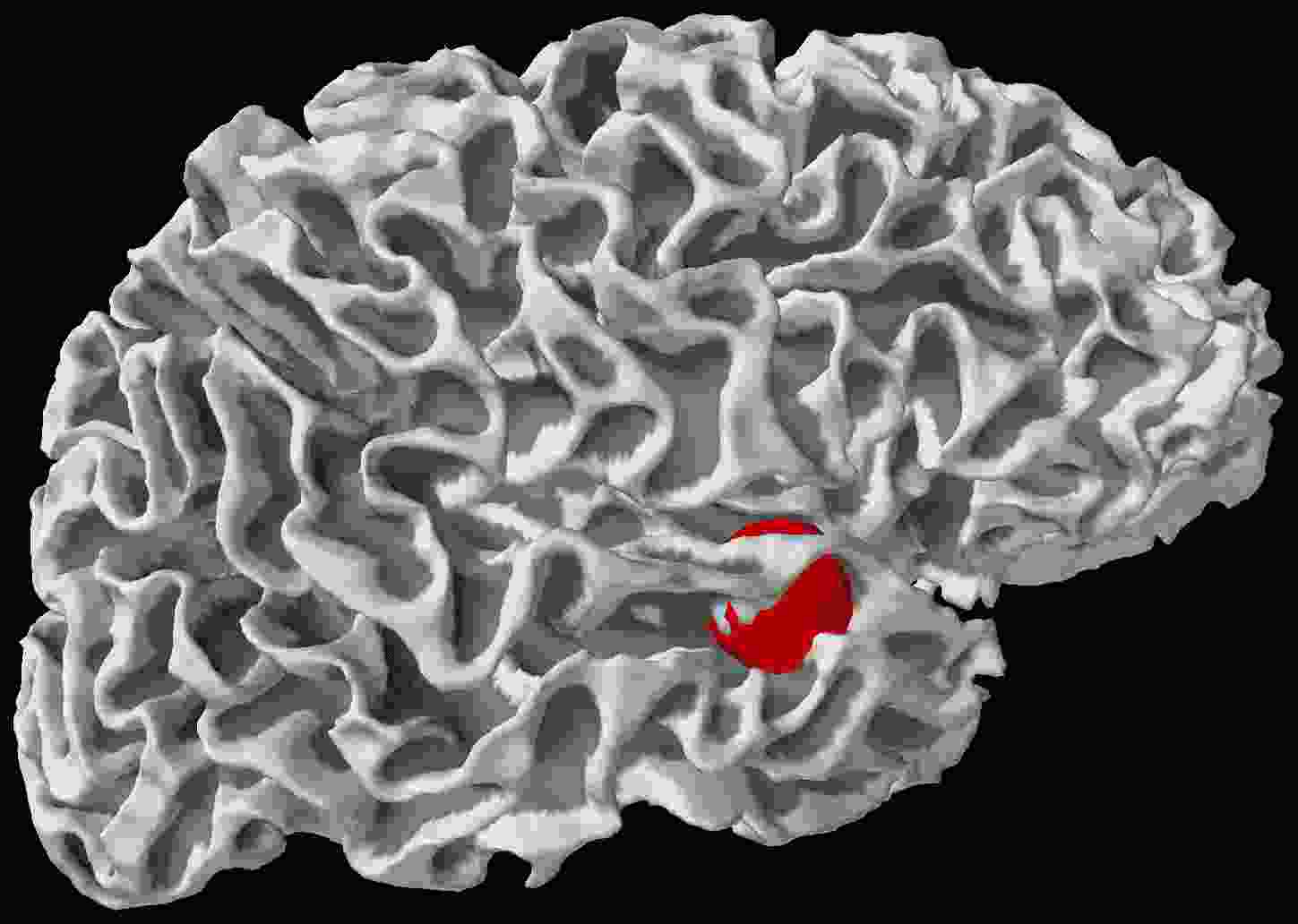}
    \caption{\mler}
    \end{subfigure}\hspace{\spacefig em}
  \begin{subfigure}{\figsize\textwidth}
    \includegraphics[width=68px,height=47px]{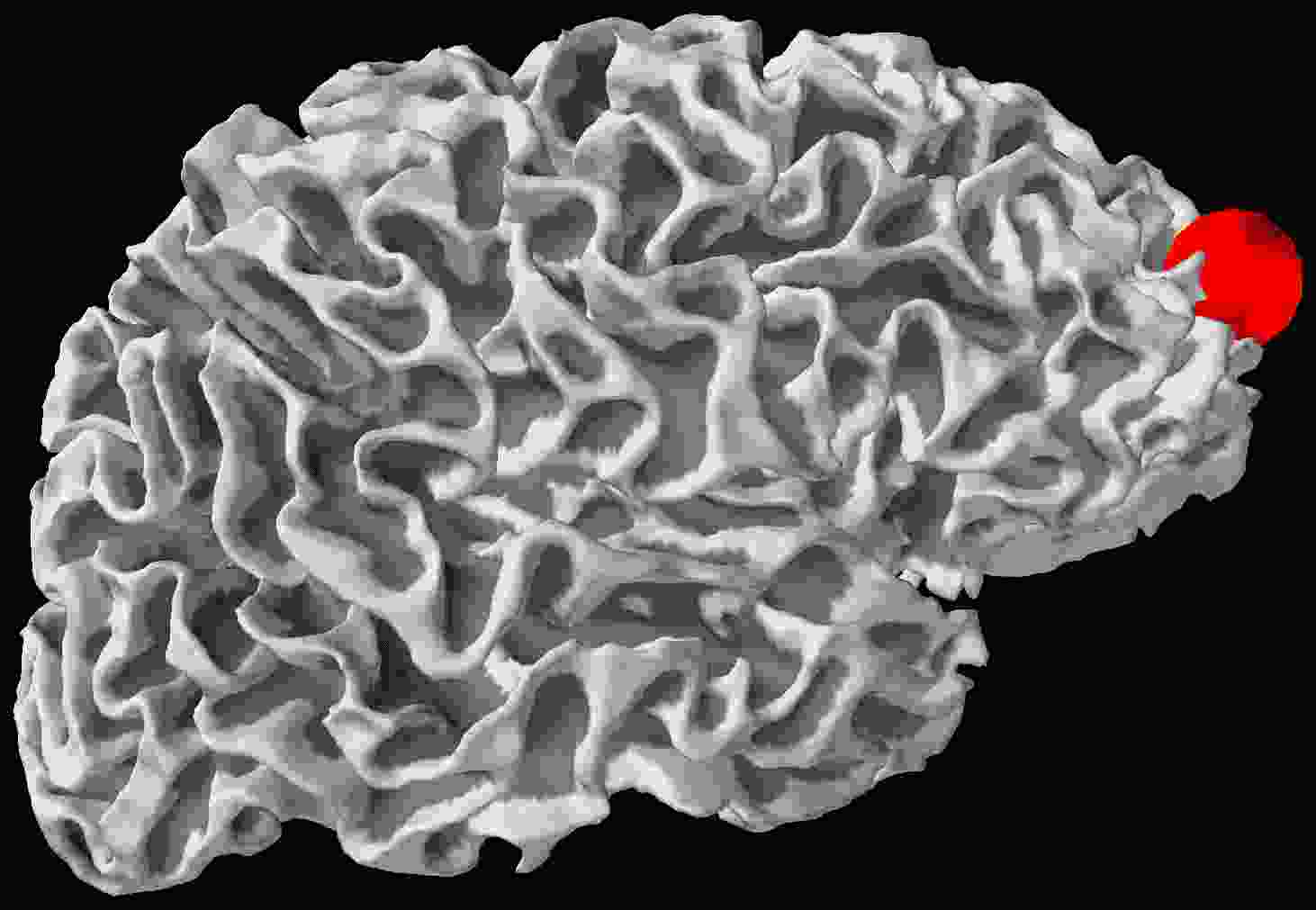}
    \caption{\mle}
    \end{subfigure}\hspace{\spacefig em}
  \begin{subfigure}{\figsize\textwidth}
    \includegraphics[width=68px,height=47px]{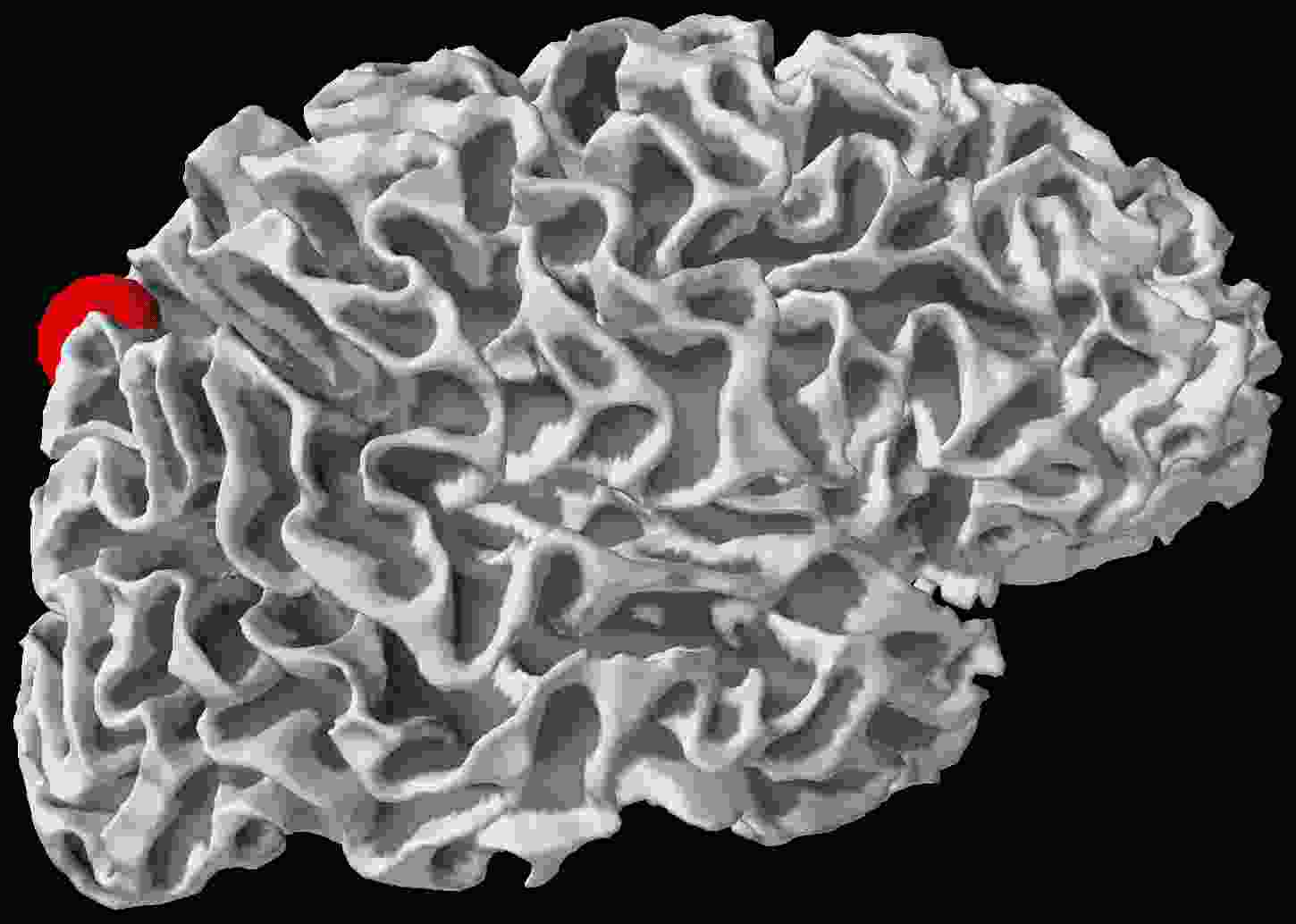}
    \caption{\mrcer}
    \end{subfigure}\hspace{\spacefig em}
  \begin{subfigure}{\figsize\textwidth}
    \includegraphics[width=68px,height=47px]{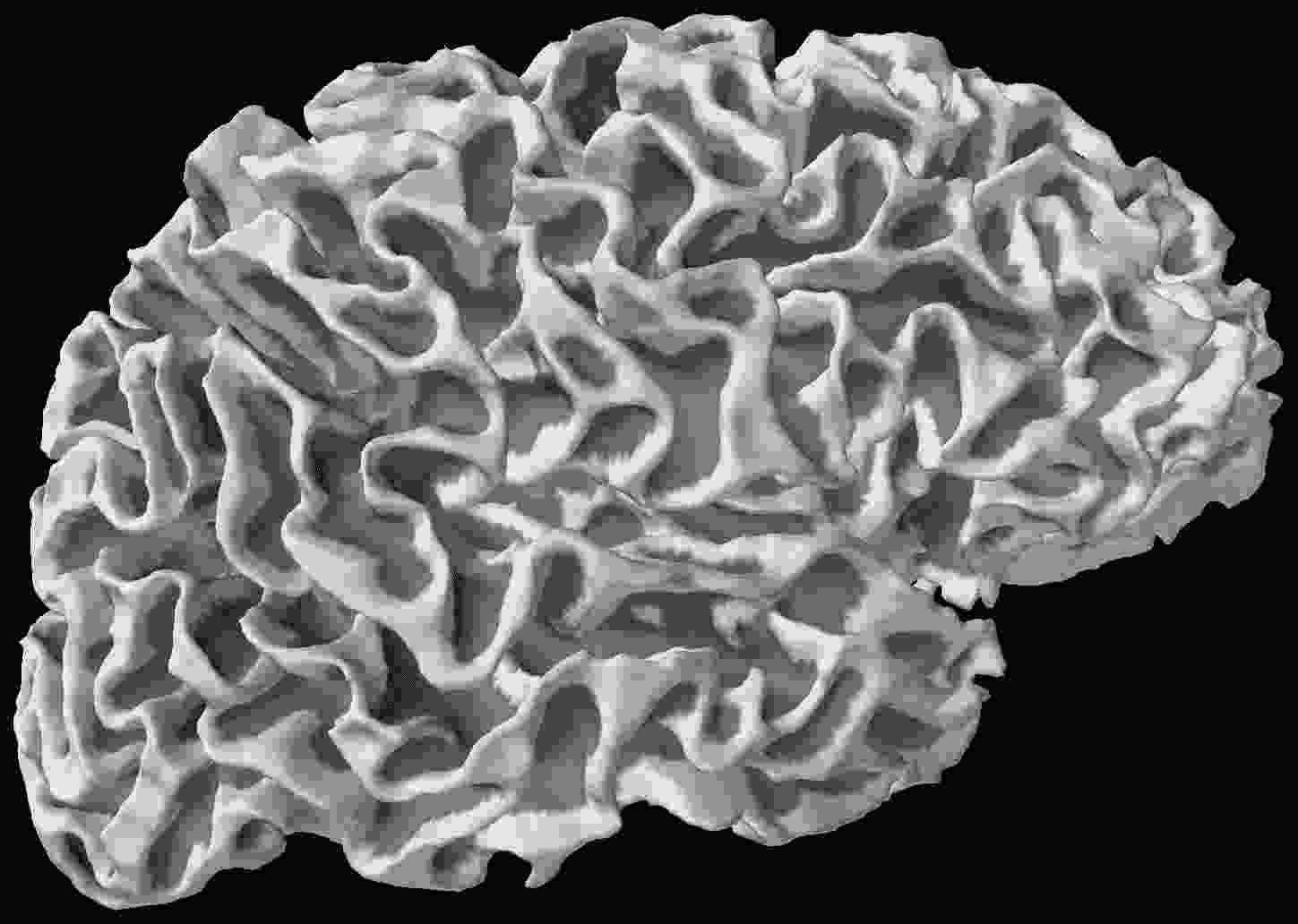}
    \caption{\mtl}
  \end{subfigure}\hspace{\spacefig em}
  \caption{\textit{Real data ($n=102$, $q=7498$, $q=76$, $r=65$)} Sources found in the left hemisphere (top) and the right hemisphere (bottom) after right auditory stimulations.}
  \label{fig:real_data_right_audi}
\end{figure}

\Cref{fig:real_data_right_audi,fig:real_data_right_audi_r_34} show the solution given by each algorithm on real data after right auditory stimulations.
As two sources are expected (one in each hemisphere, in bilateral auditory cortices), we vary $\lambda$ by dichotomy between $\lambda_{\max}$ (returning 0 sources) and a $\lambda_{\min}$ (returning more than 2 sources), until finding a lambda giving exactly 2 sources.
\Cref{fig:real_data_right_audi} (\resp \Cref{fig:real_data_right_audi_r_34}) shows the solution given by the algorithms taking in account all the repetitions (\resp only half of the repetitions).
When the number of repetitions is high (\Cref{fig:real_data_right_audi}) only \us and \mler find one source in each auditory cortex, \mtl does find sources only in one hemisphere, all the other algorithms fail by finding sources not in the auditory cortices at all.
Moreover when the number of repetitions is decreasing (\Cref{fig:real_data_right_audi_r_34}) \mler fails and only \us does find 2 sources, one in each hemisphere.
Once again \us is more robust and performs better, even when the number of repetitions is low.

\begin{figure}[H]
  \centering 

  \begin{subfigure}{\figsize \textwidth}
  \includegraphics[width=68px,height=47px]{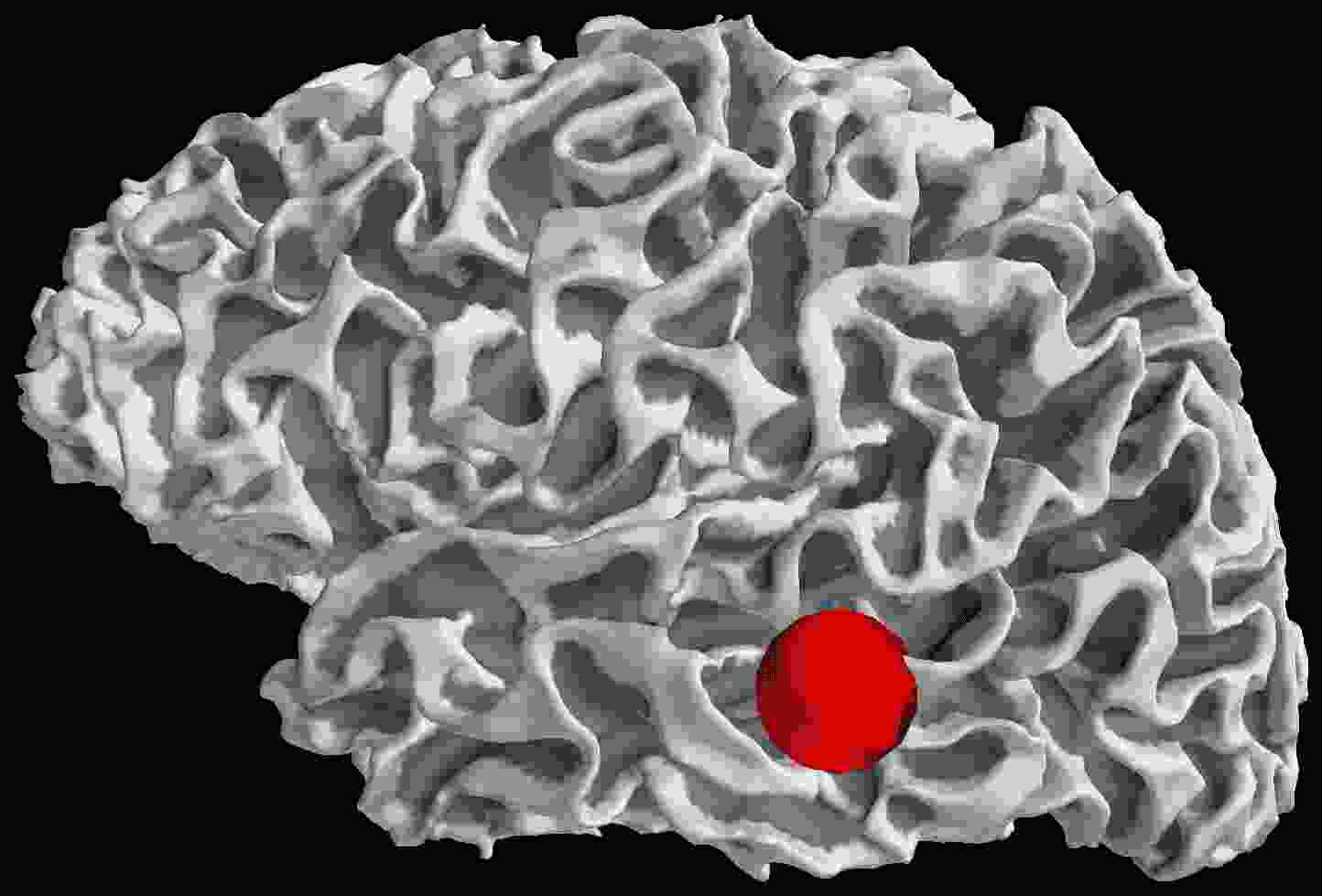}
  \end{subfigure}\hspace{\spacefig em}
  \begin{subfigure}{\figsize \textwidth}
  \includegraphics[width=68px,height=47px]{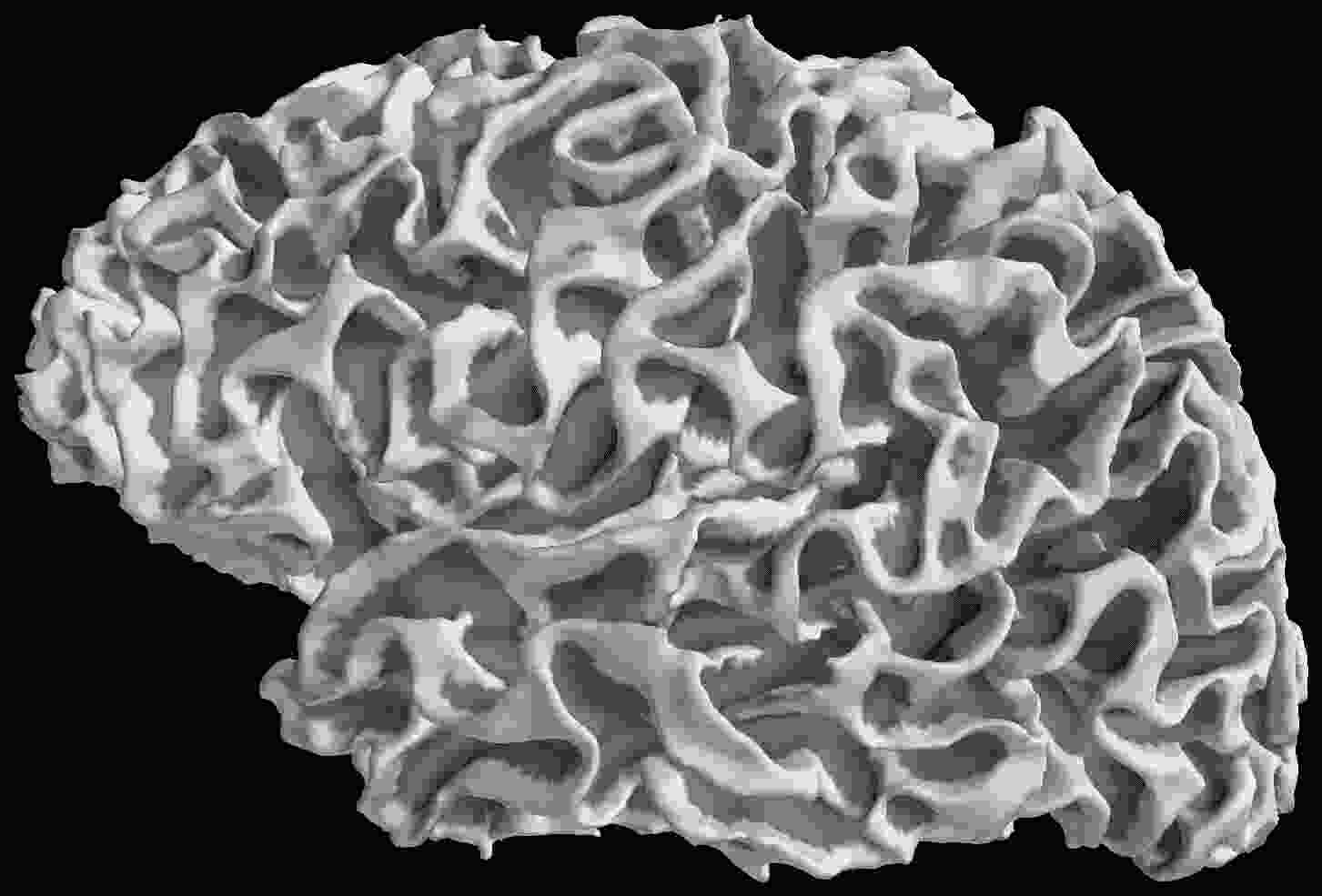}
  \end{subfigure}\hspace{\spacefig em}
  \begin{subfigure}{\figsize\textwidth}
    \includegraphics[width=68px,height=47px]{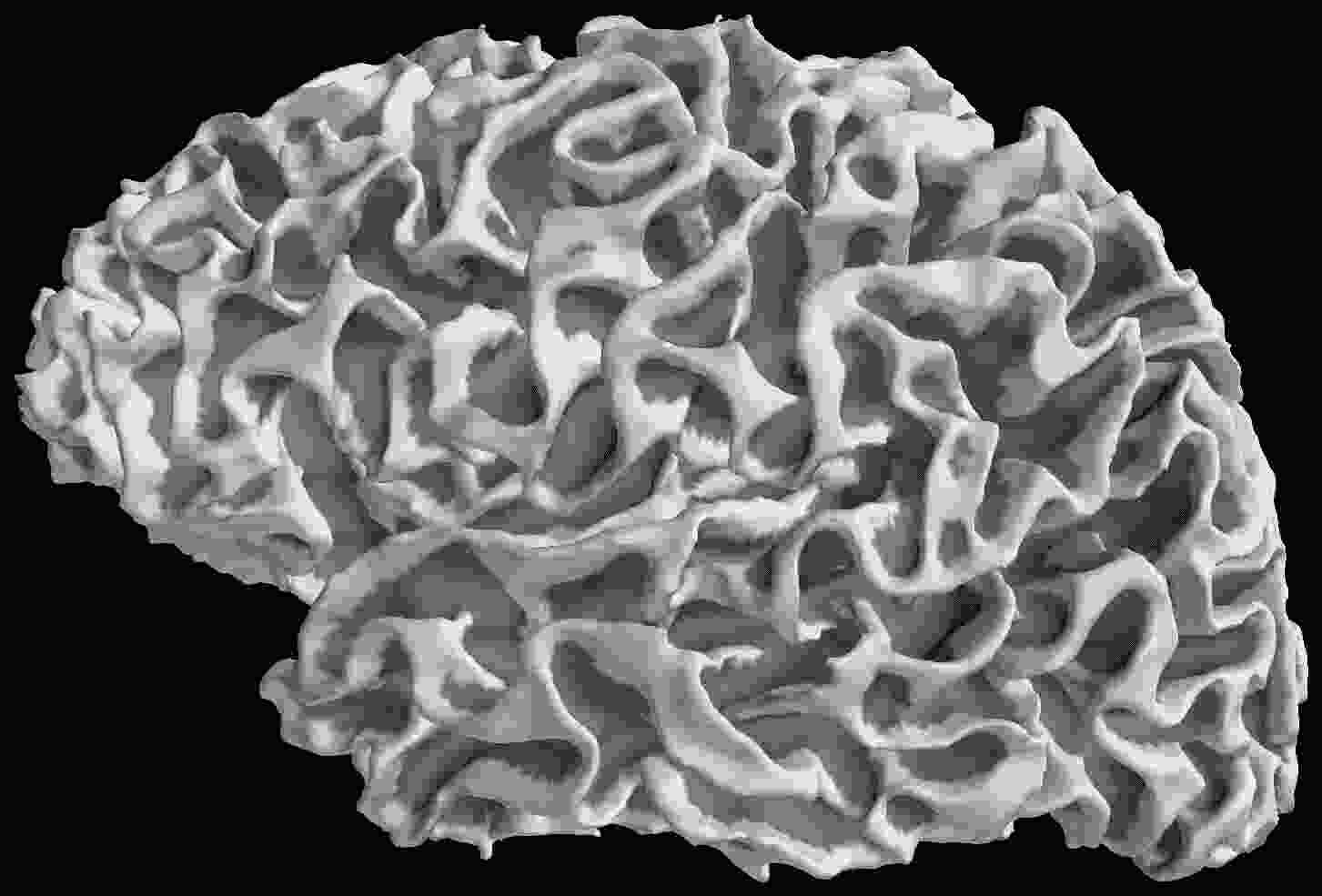}
    \end{subfigure}\hspace{\spacefig em}
  \begin{subfigure}{\figsize\textwidth}
  \includegraphics[width=68px,height=47px]{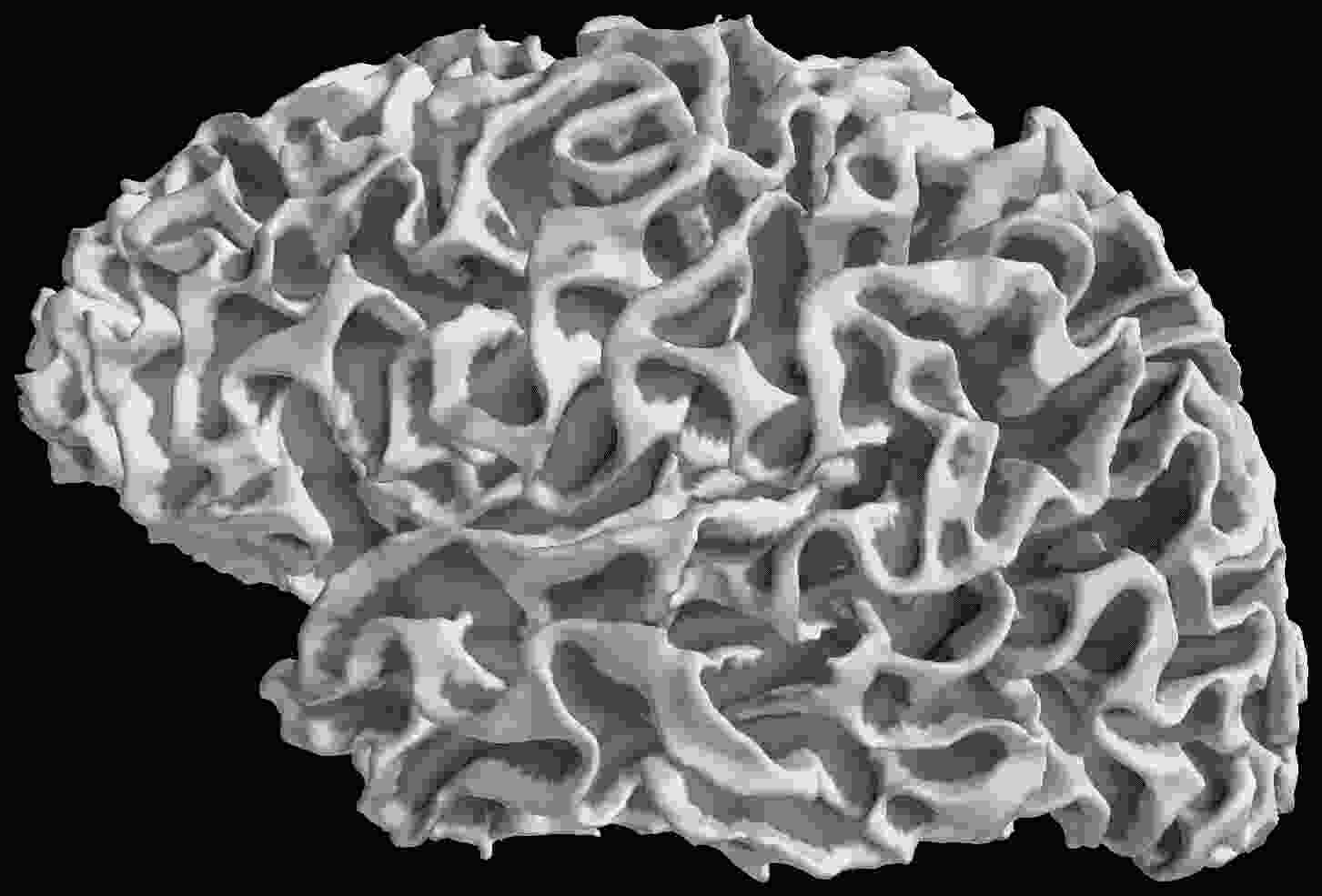}
  \end{subfigure}\hspace{\spacefig em}
  \begin{subfigure}{\figsize\textwidth}
  \includegraphics[width=68px,height=47px]{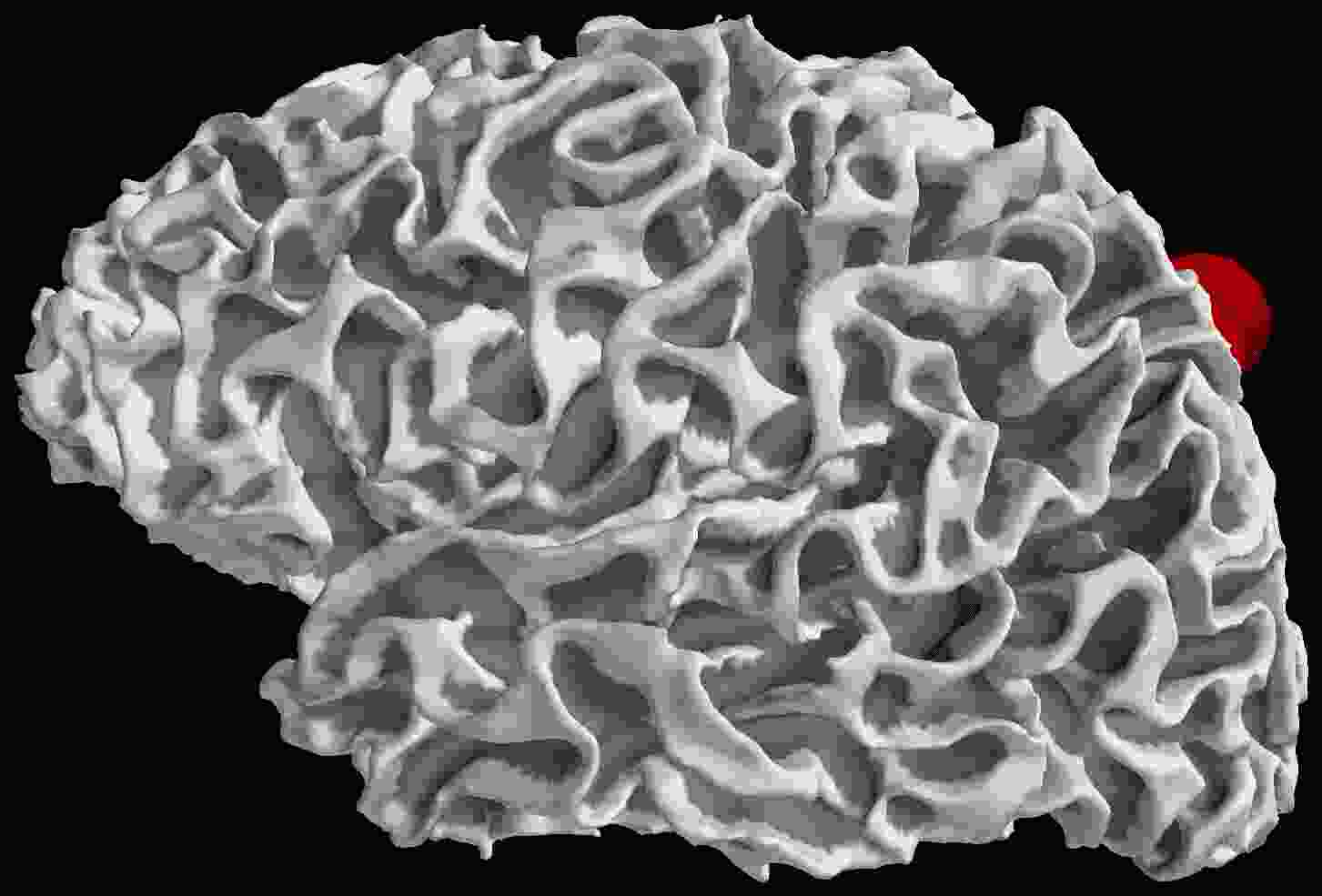}
  \end{subfigure}\hspace{\spacefig em}
  \begin{subfigure}{\figsize \textwidth}
    \includegraphics[width=68px,height=47px]{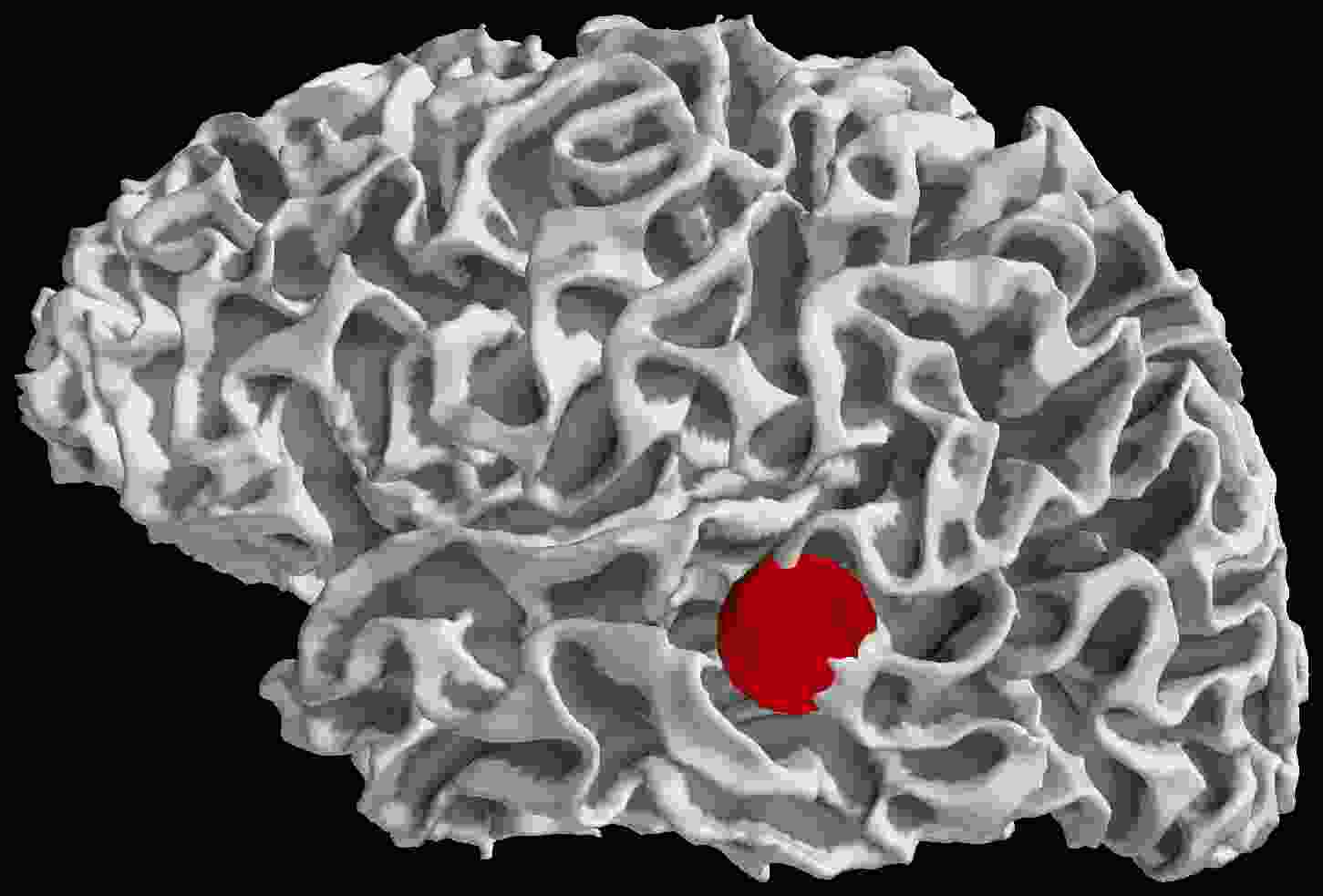}
    \end{subfigure}\hspace{\spacefig em}

  \vspace{-0.1em}
  \begin{subfigure}{\figsize\textwidth}
    \includegraphics[width=68px,height=47px]{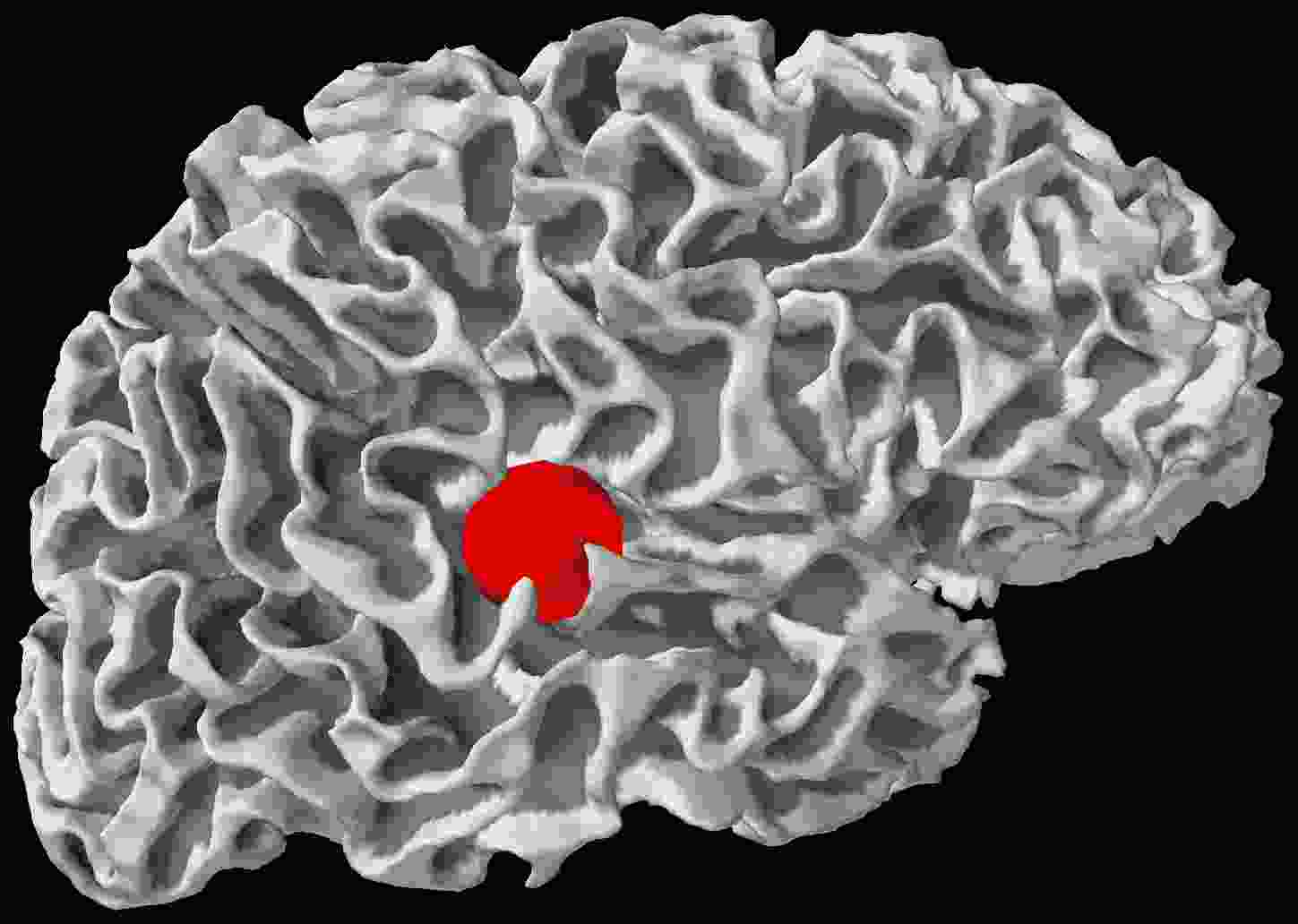}
    \caption{\us}
  \end{subfigure}\hspace{\spacefig em}
  \begin{subfigure}{\figsize\textwidth}
    \includegraphics[width=68px,height=47px]{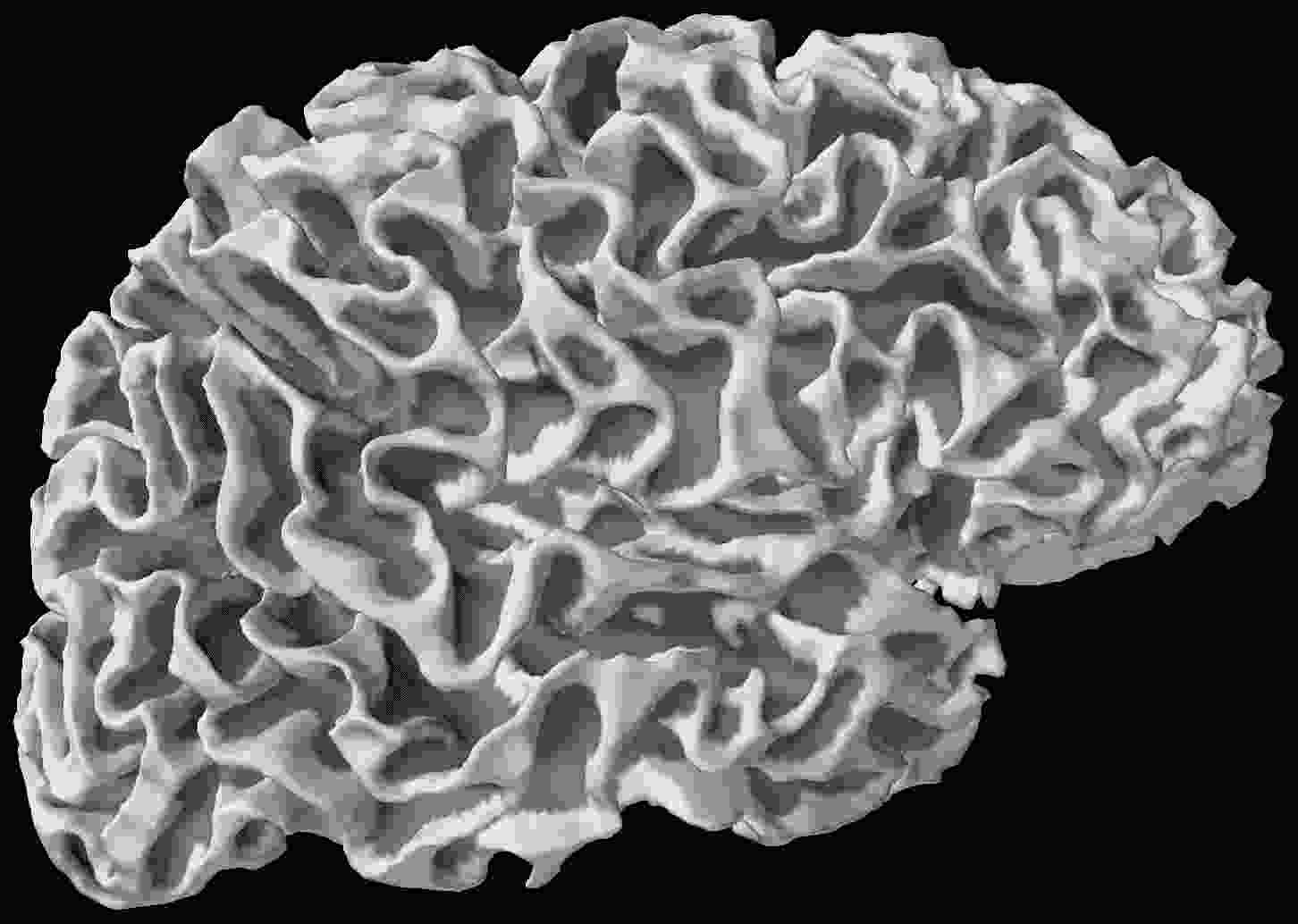}
    \caption{\sgcl}
  \end{subfigure}\hspace{\spacefig em}
  \begin{subfigure}{\figsize\textwidth}
    \includegraphics[width=68px,height=47px]{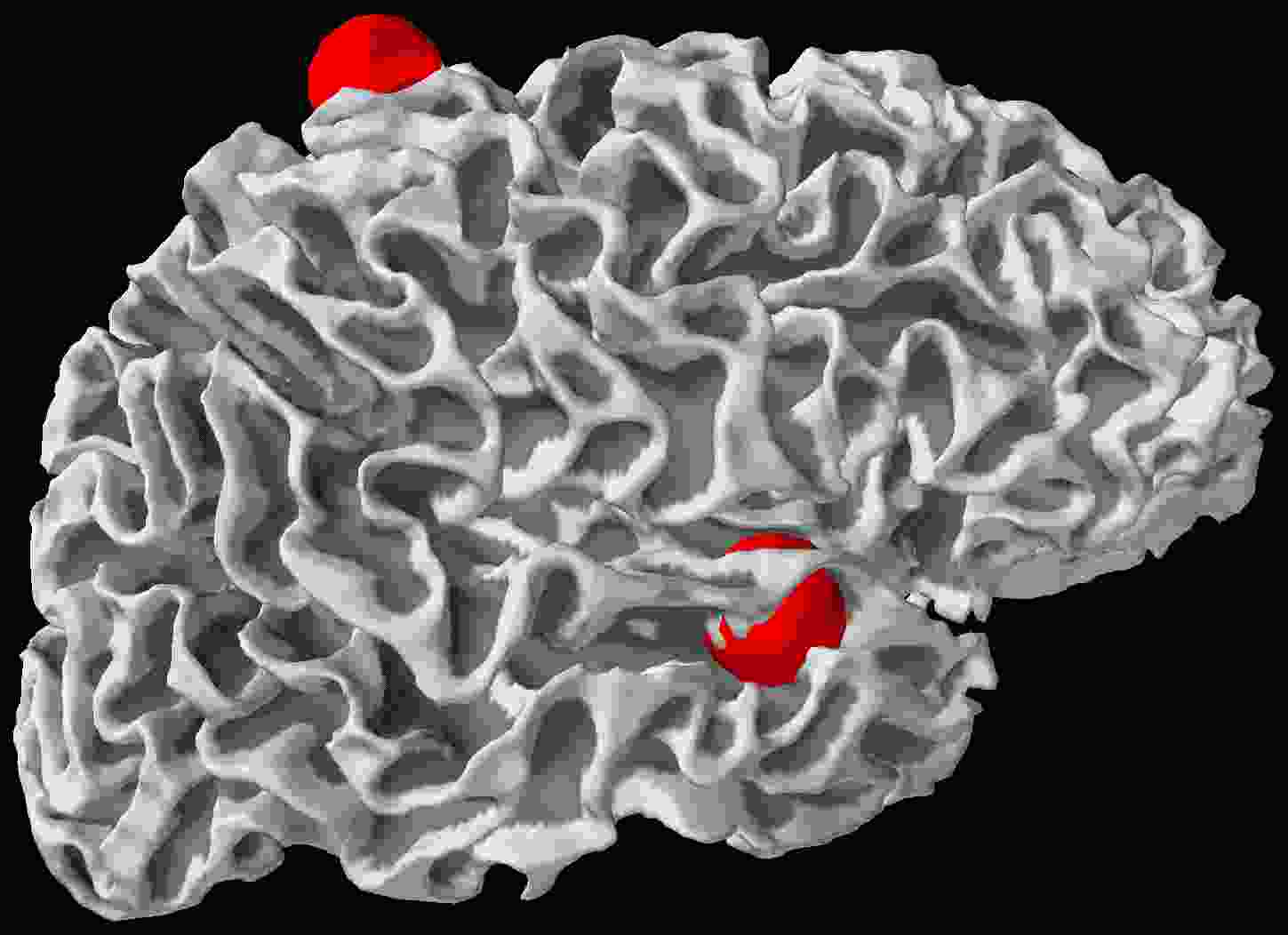}
    \caption{\mler}
    \end{subfigure}\hspace{\spacefig em}
  \begin{subfigure}{\figsize\textwidth}
    \includegraphics[width=68px,height=47px]{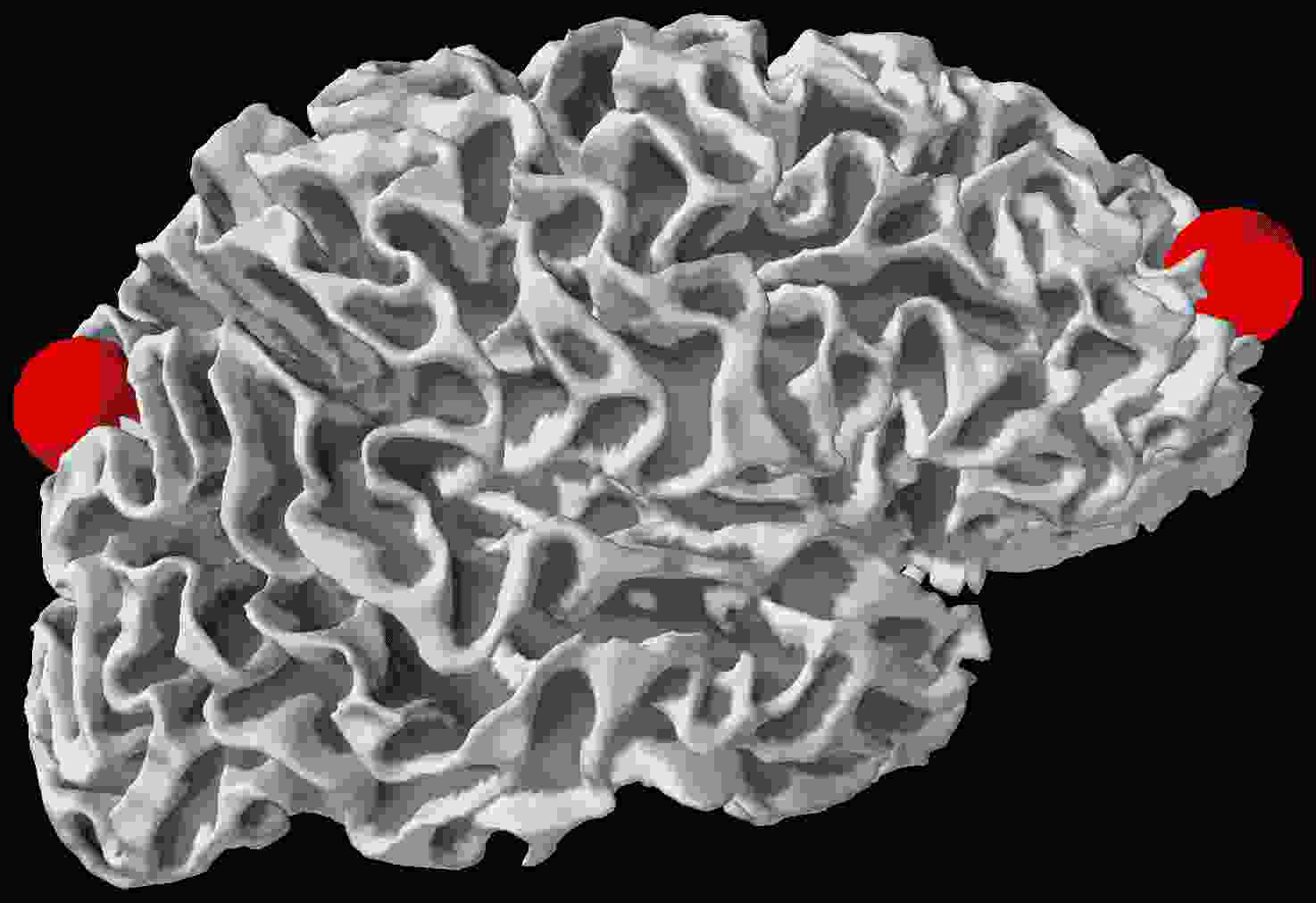}
    \caption{\mle}
    \end{subfigure}\hspace{\spacefig em}
  \begin{subfigure}{\figsize\textwidth}
    \includegraphics[width=68px,height=47px]{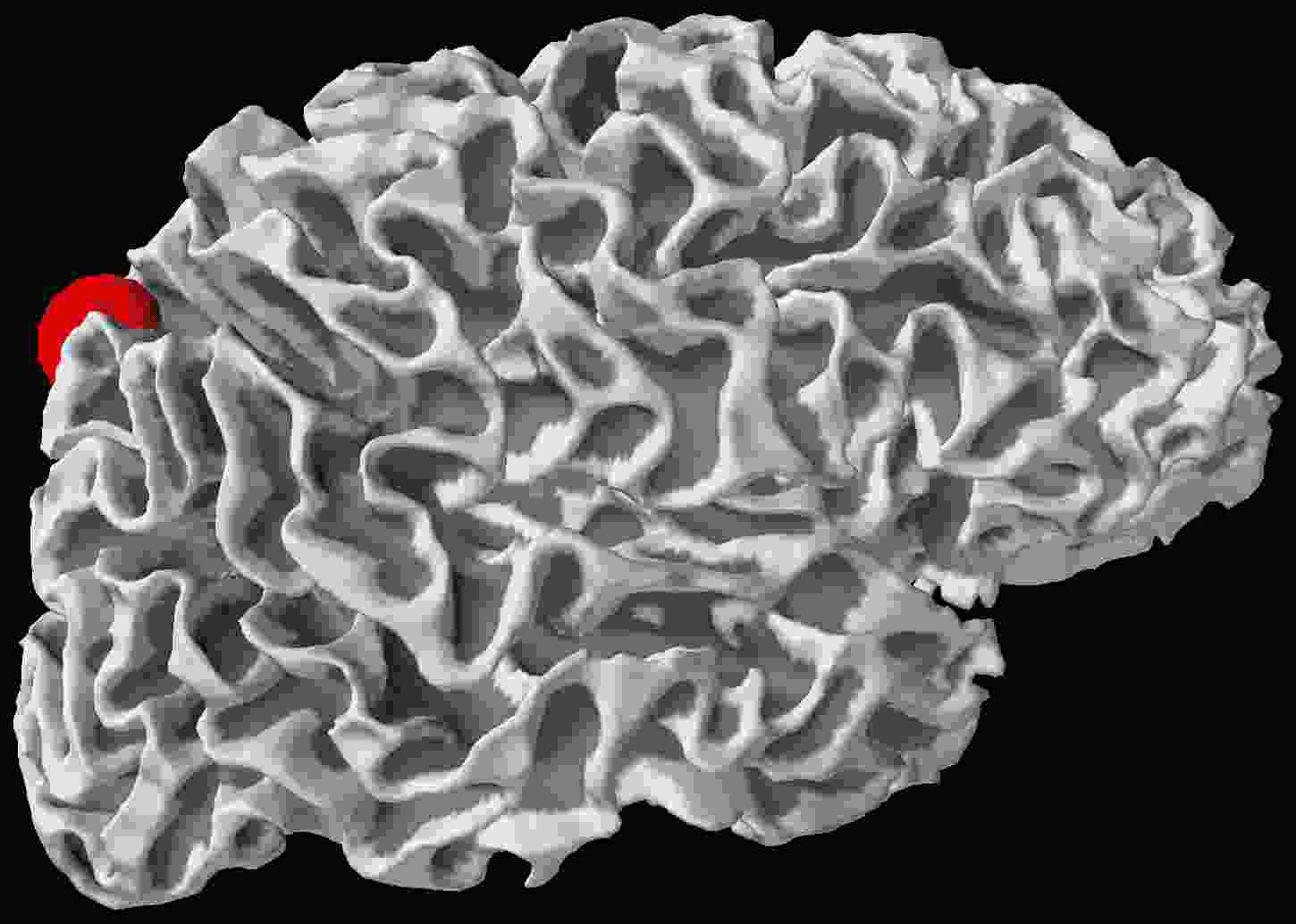}
    \caption{\mrcer}
    \end{subfigure}\hspace{\spacefig em}
  \begin{subfigure}{\figsize\textwidth}
    \includegraphics[width=68px,height=47px]{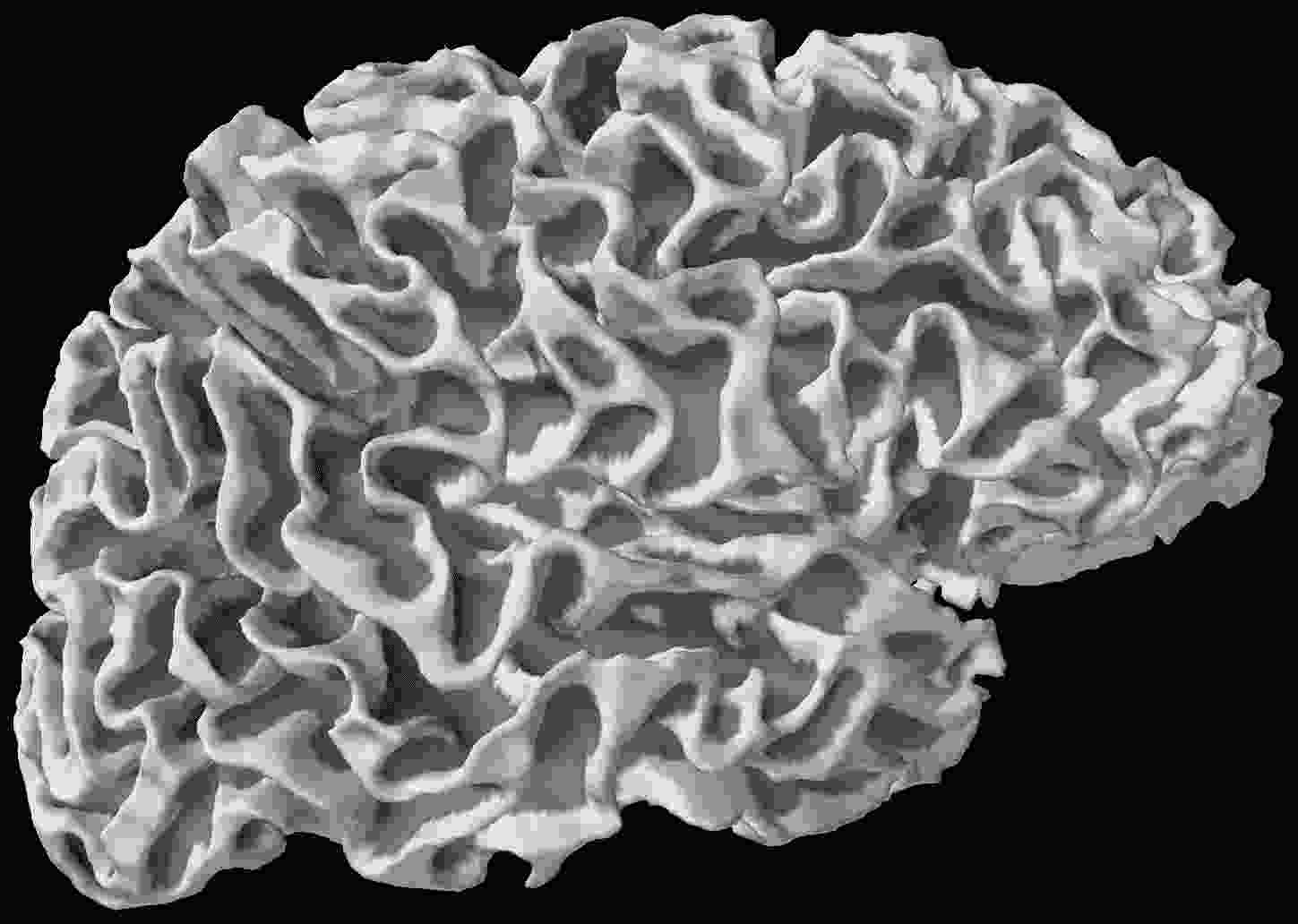}
    \caption{\mtl}
  \end{subfigure}\hspace{\spacefig em}
  \caption{\textit{Real data ($n=102$, $q=7498$, $q=76$, $r=33$)} Sources found in the left hemisphere (top) and the right hemisphere (bottom) after right auditory stimulations.}
  \label{fig:real_data_right_audi_r_34}
\end{figure}


\subsection{Supplementary experiments on real data: left visual stimulations}
\label{app_sub:real_left_visual}

\begin{figure}[H]
  \centering 

  \begin{subfigure}{\figsize \textwidth}
  \includegraphics[width=68px,height=47px]{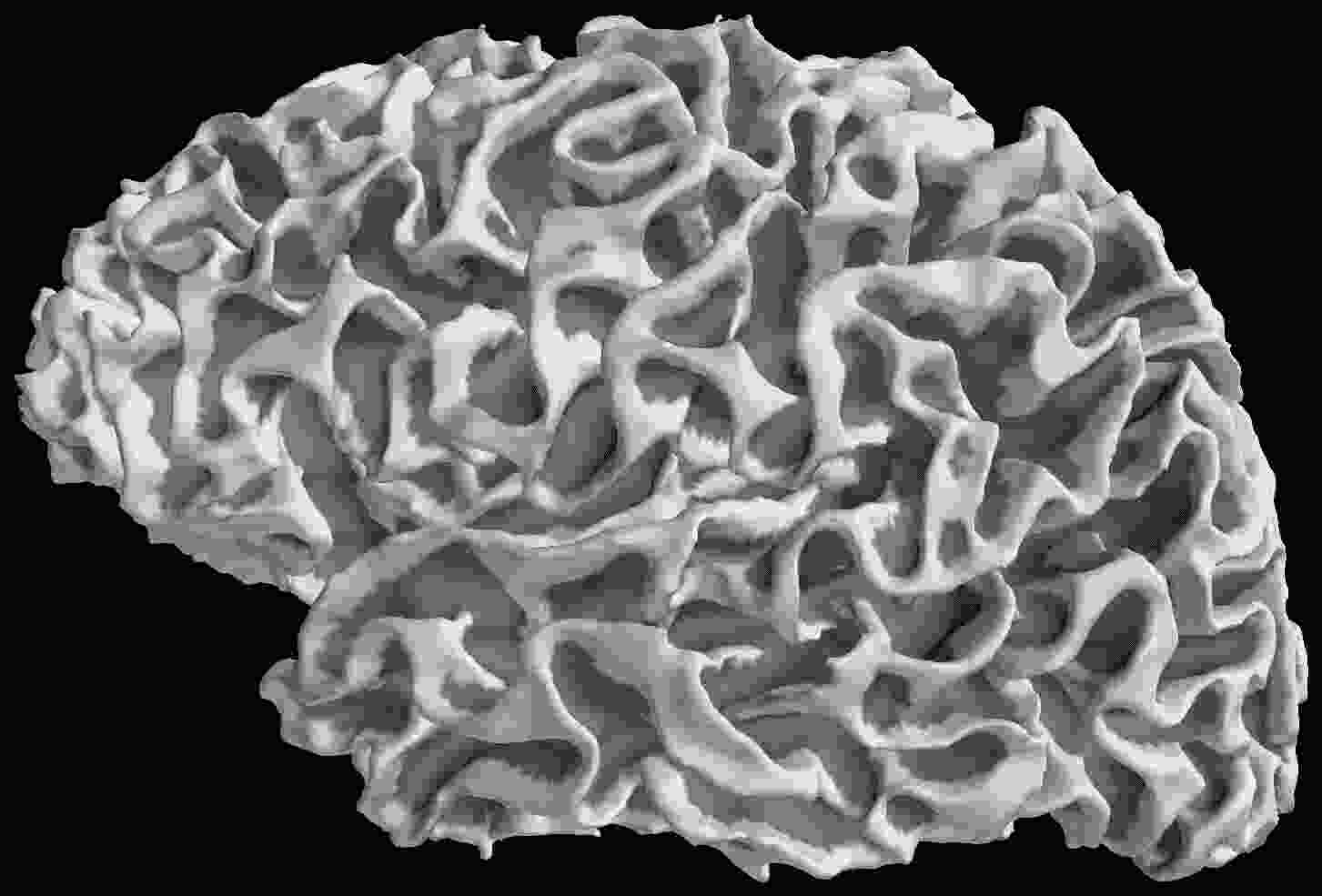}
  \end{subfigure}\hspace{\spacefig em}
  \begin{subfigure}{\figsize \textwidth}
  \includegraphics[width=68px,height=47px]{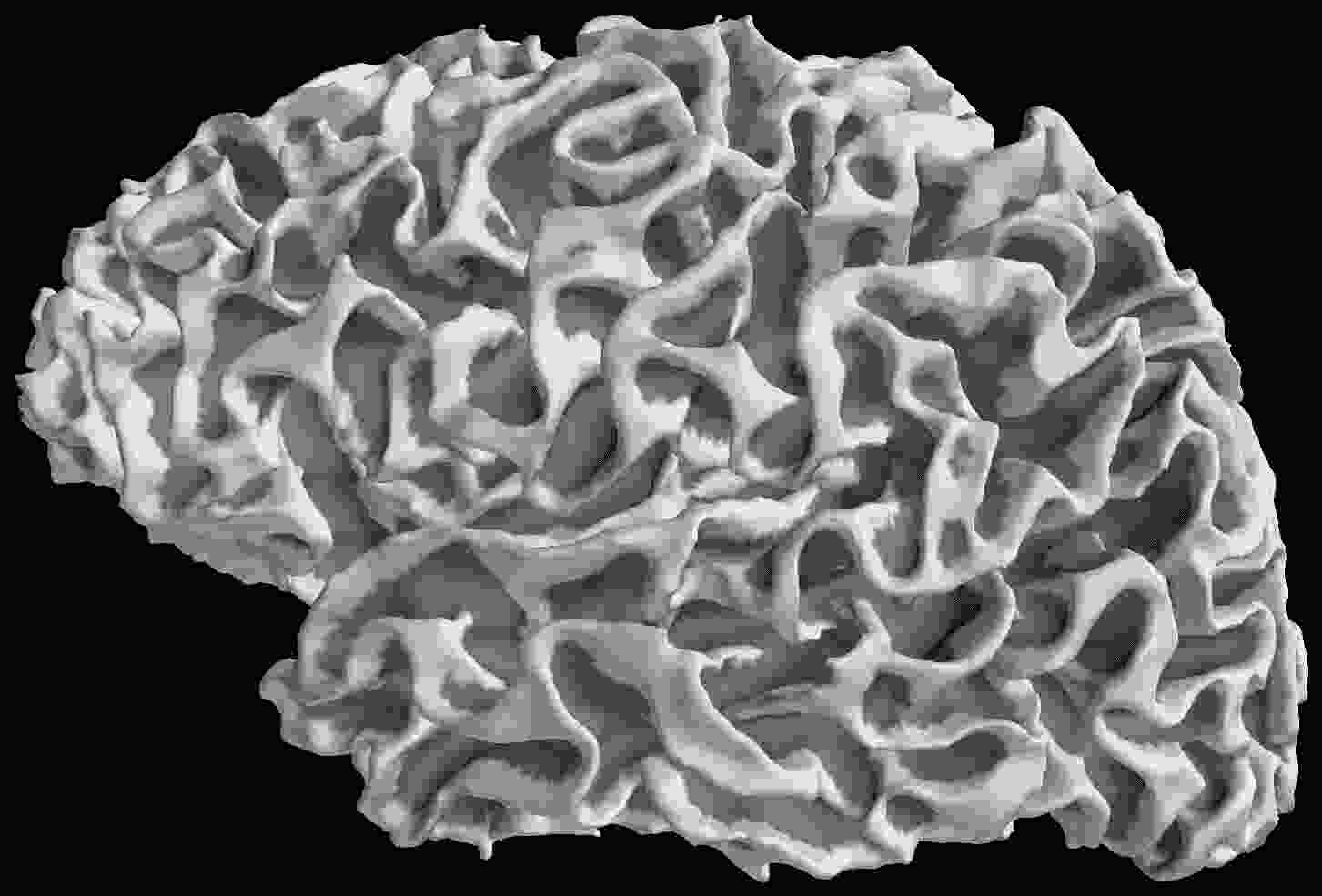}
  \end{subfigure}\hspace{\spacefig em}
  \begin{subfigure}{\figsize\textwidth}
    \includegraphics[width=68px,height=47px]{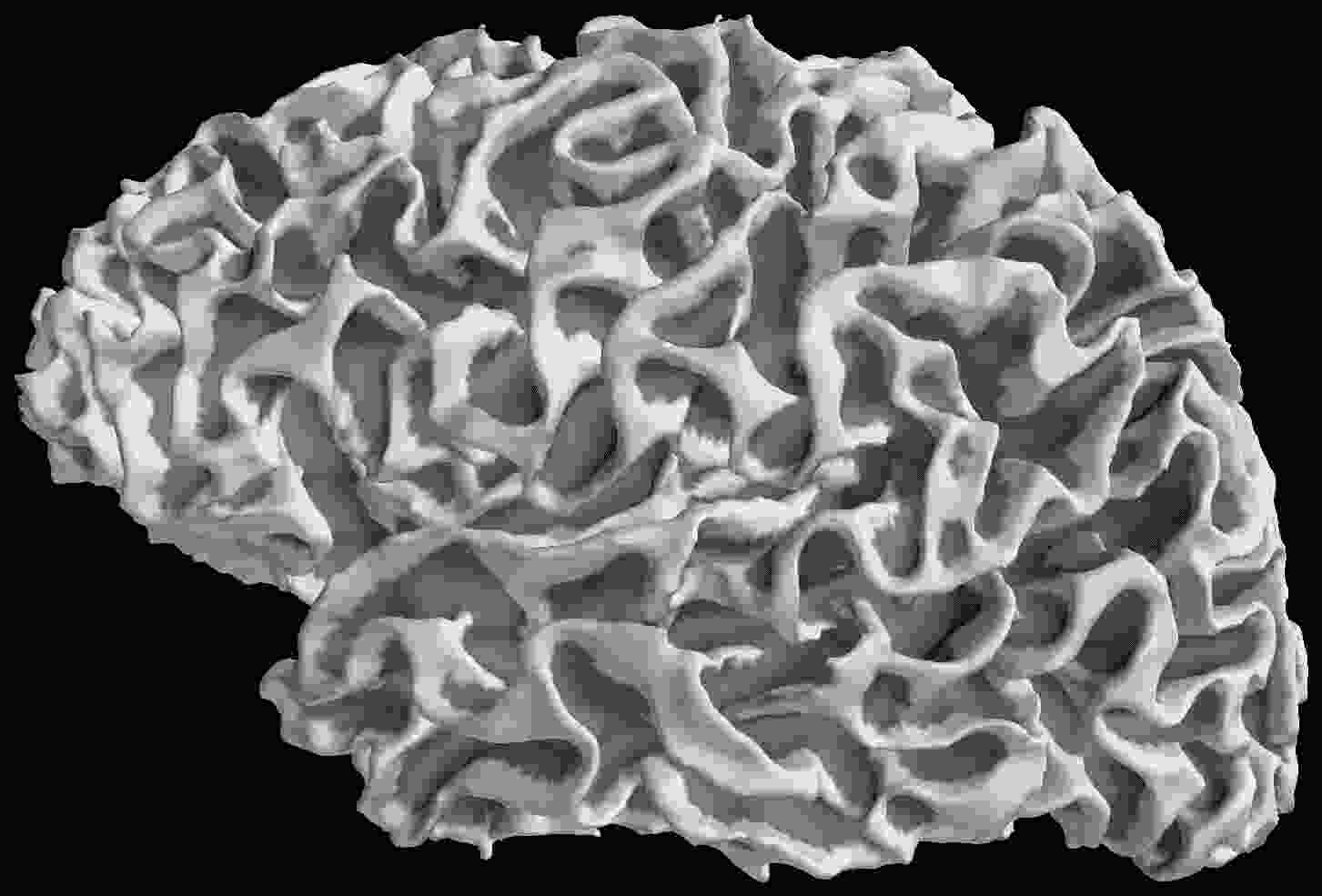}
    \end{subfigure}\hspace{\spacefig em}
  \begin{subfigure}{\figsize\textwidth}
  \includegraphics[width=68px,height=47px]{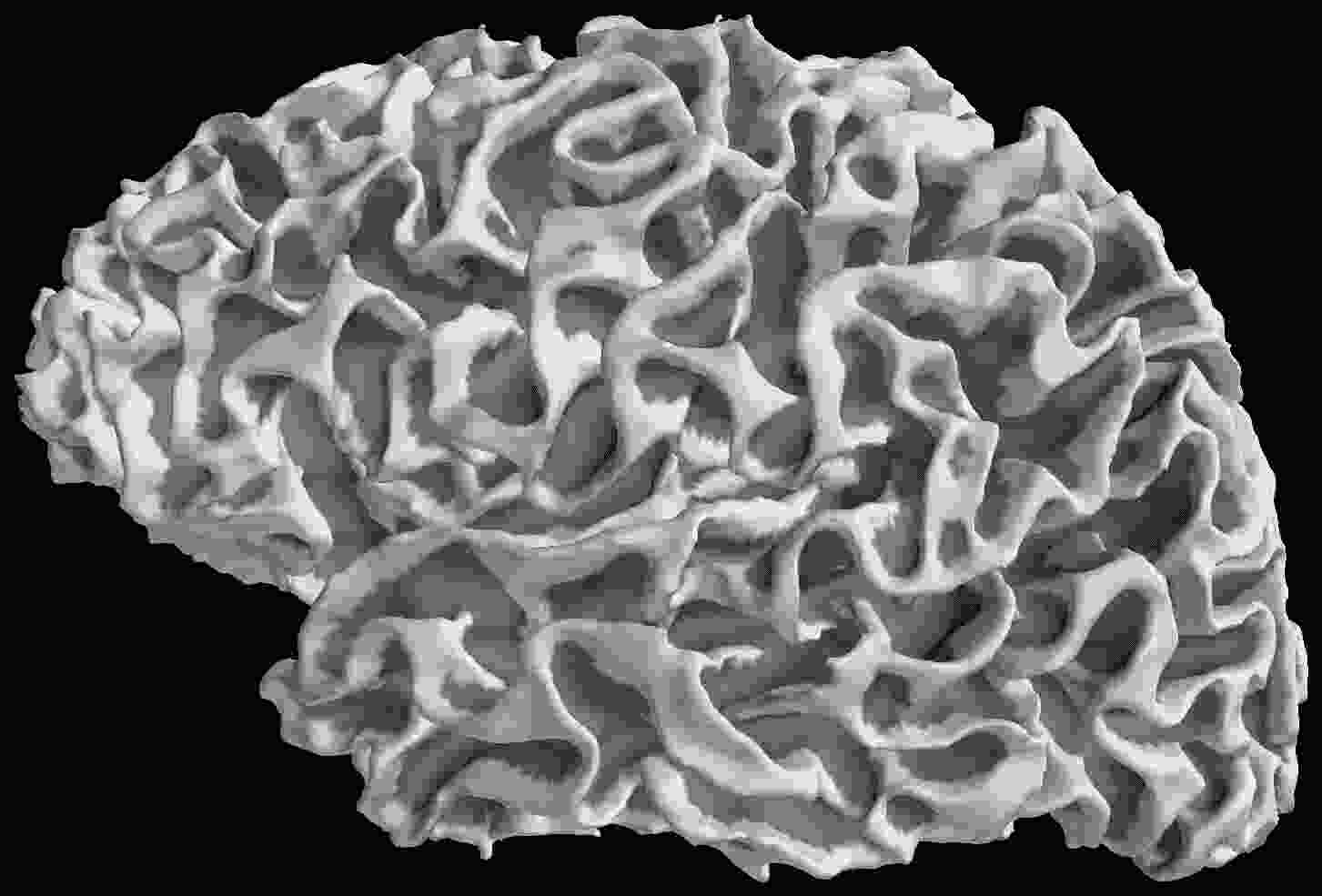}
  \end{subfigure}\hspace{\spacefig em}
  \begin{subfigure}{\figsize\textwidth}
  \includegraphics[width=68px,height=47px]{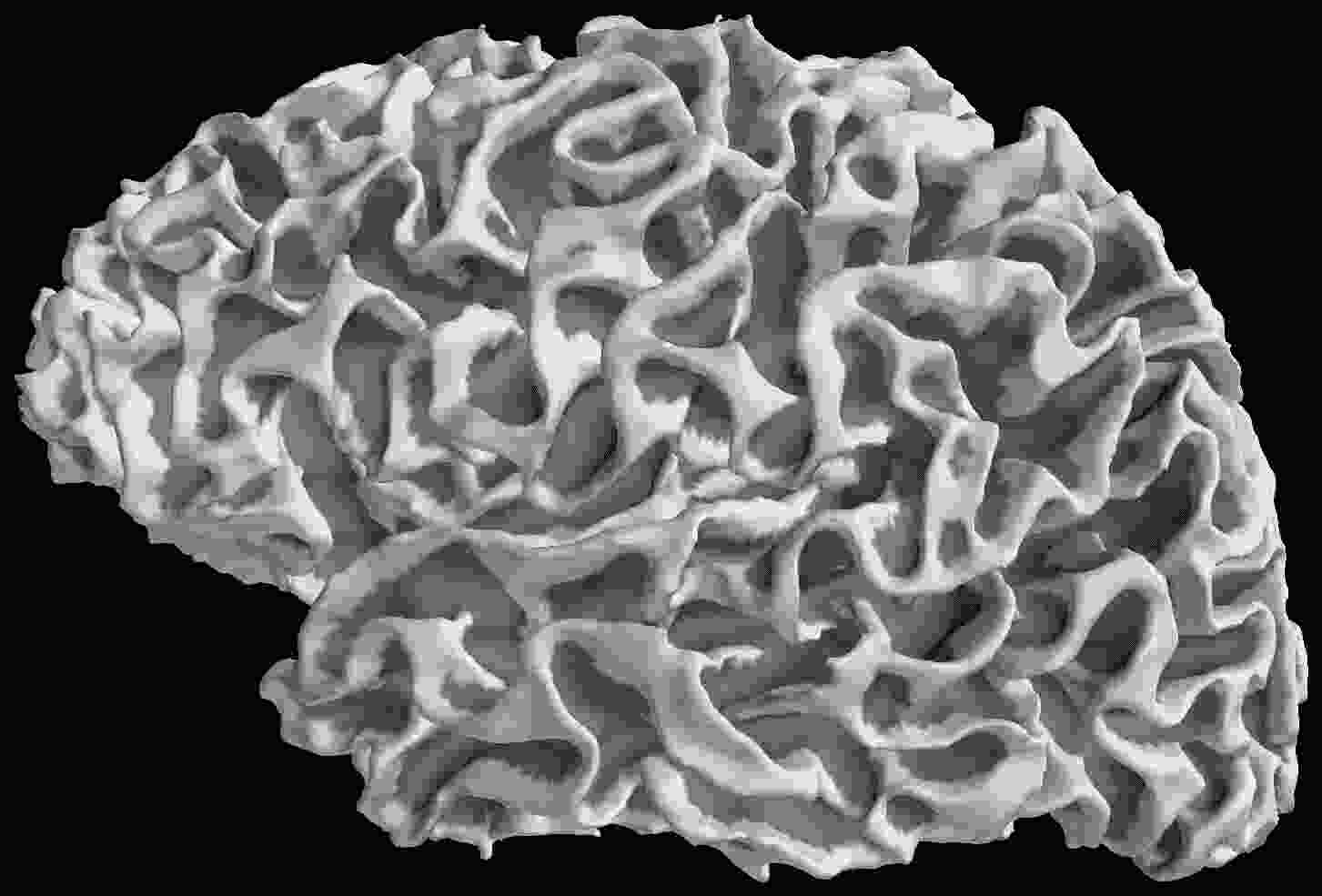}
  \end{subfigure}\hspace{\spacefig em}
  \begin{subfigure}{\figsize \textwidth}
    \includegraphics[width=68px,height=47px]{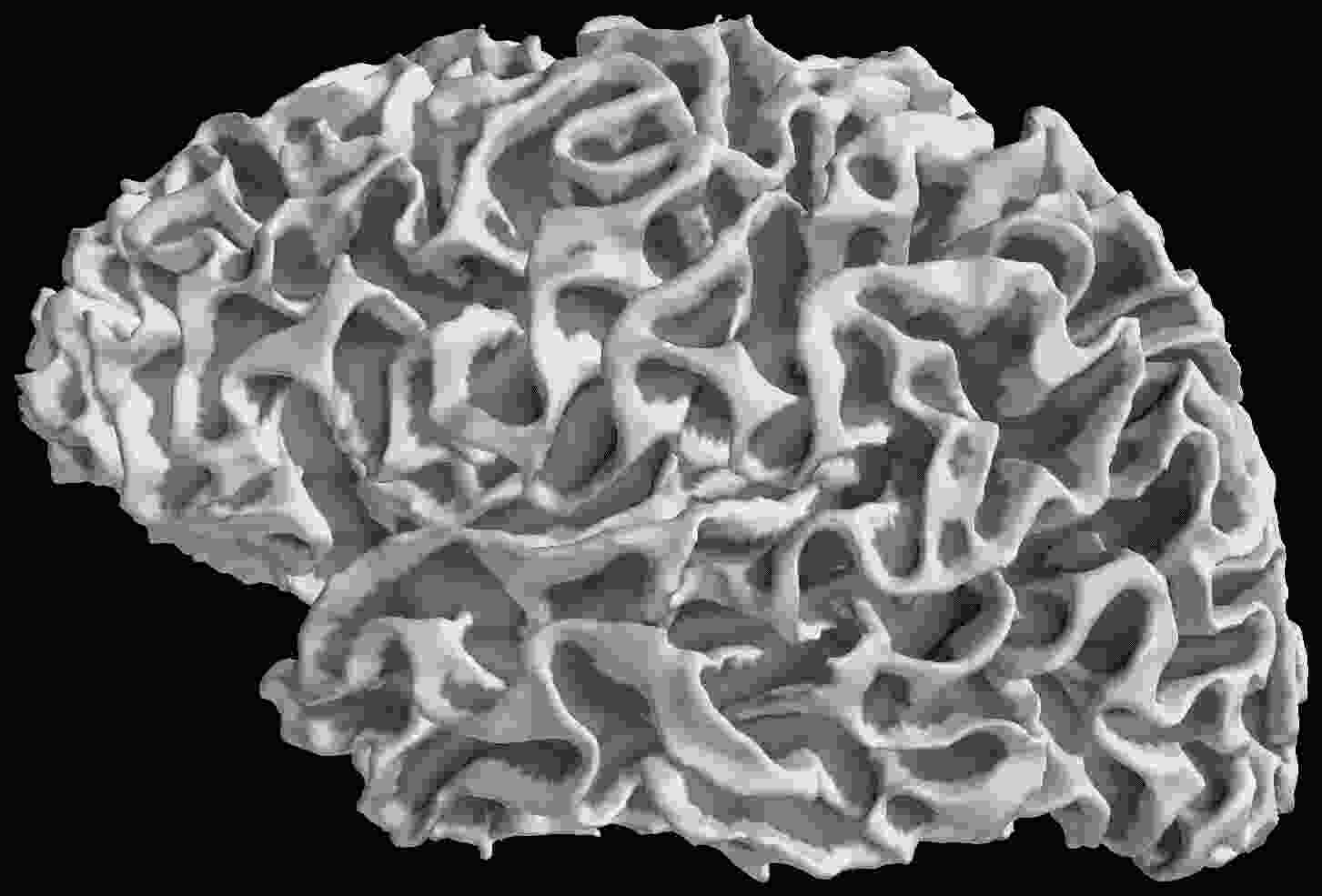}
    \end{subfigure}\hspace{\spacefig em}

  \vspace{-0.1em}
  \begin{subfigure}{\figsize\textwidth}
    \includegraphics[width=68px,height=47px]{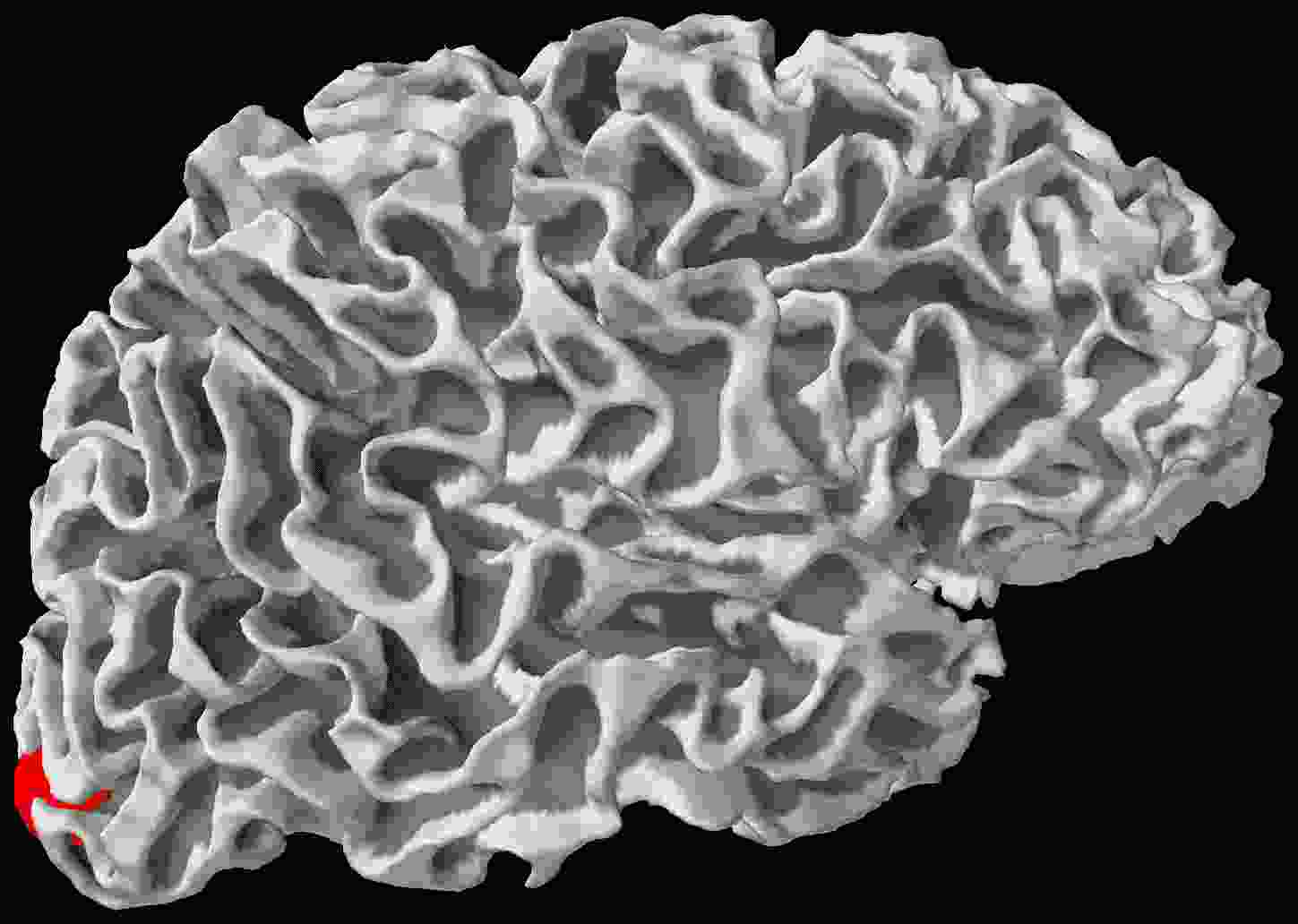}
    \caption{\us}
  \end{subfigure}\hspace{\spacefig em}
  \begin{subfigure}{\figsize\textwidth}
    \includegraphics[width=68px,height=47px]{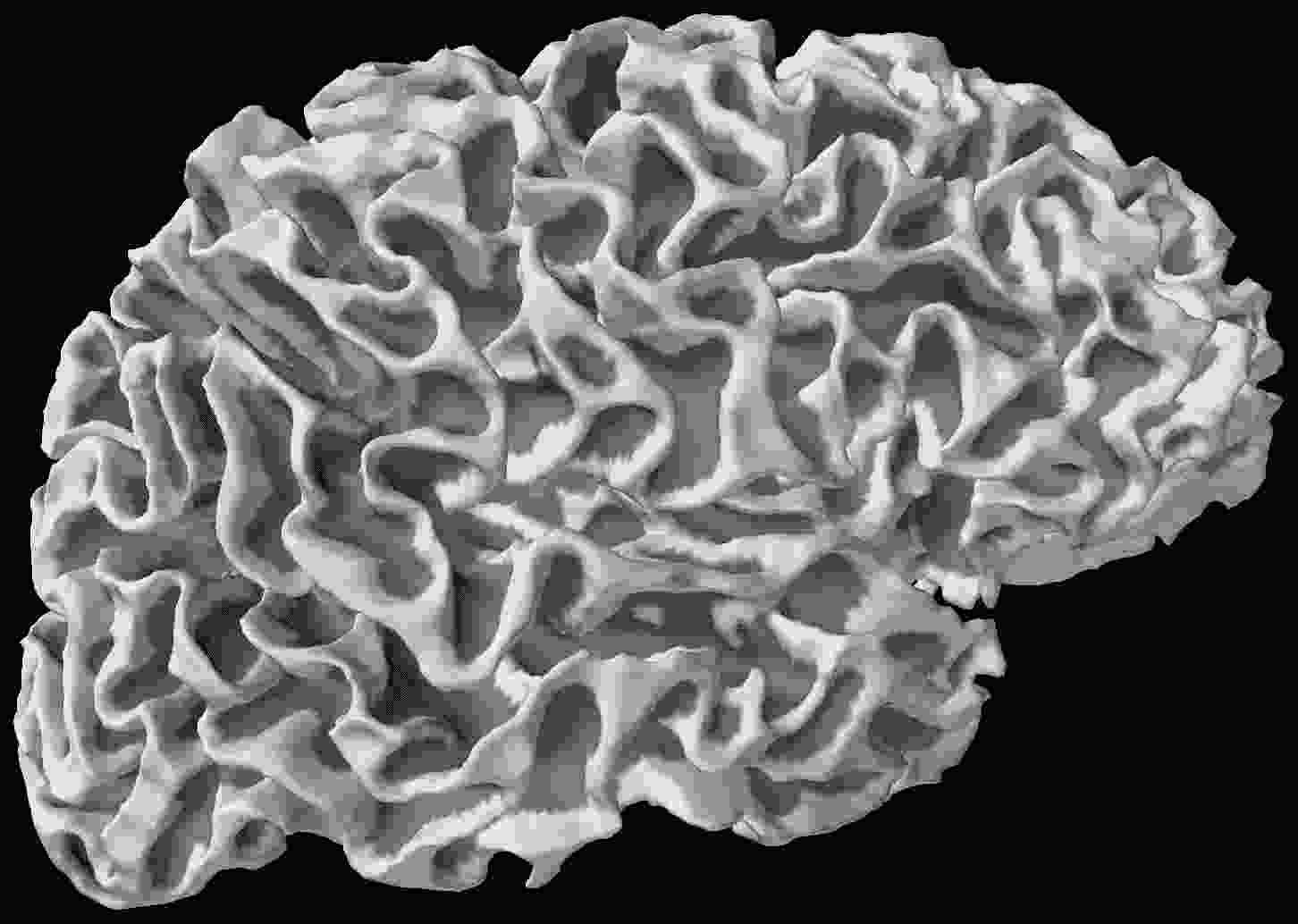}
    \caption{\sgcl}
  \end{subfigure}\hspace{\spacefig em}
  \begin{subfigure}{\figsize\textwidth}
    \includegraphics[width=68px,height=47px]{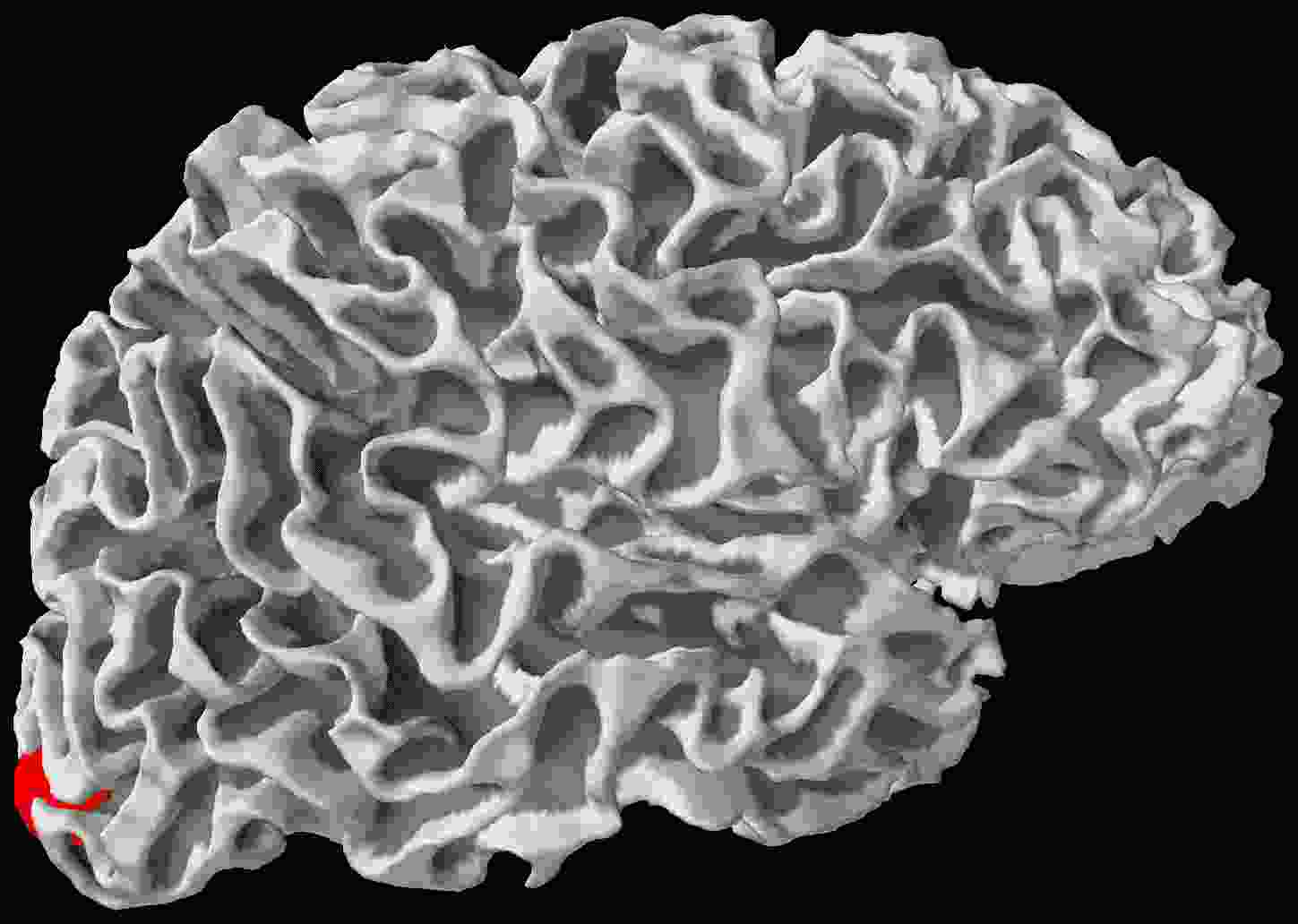}
    \caption{\mler}
    \end{subfigure}\hspace{\spacefig em}
  \begin{subfigure}{\figsize\textwidth}
    \includegraphics[width=68px,height=47px]{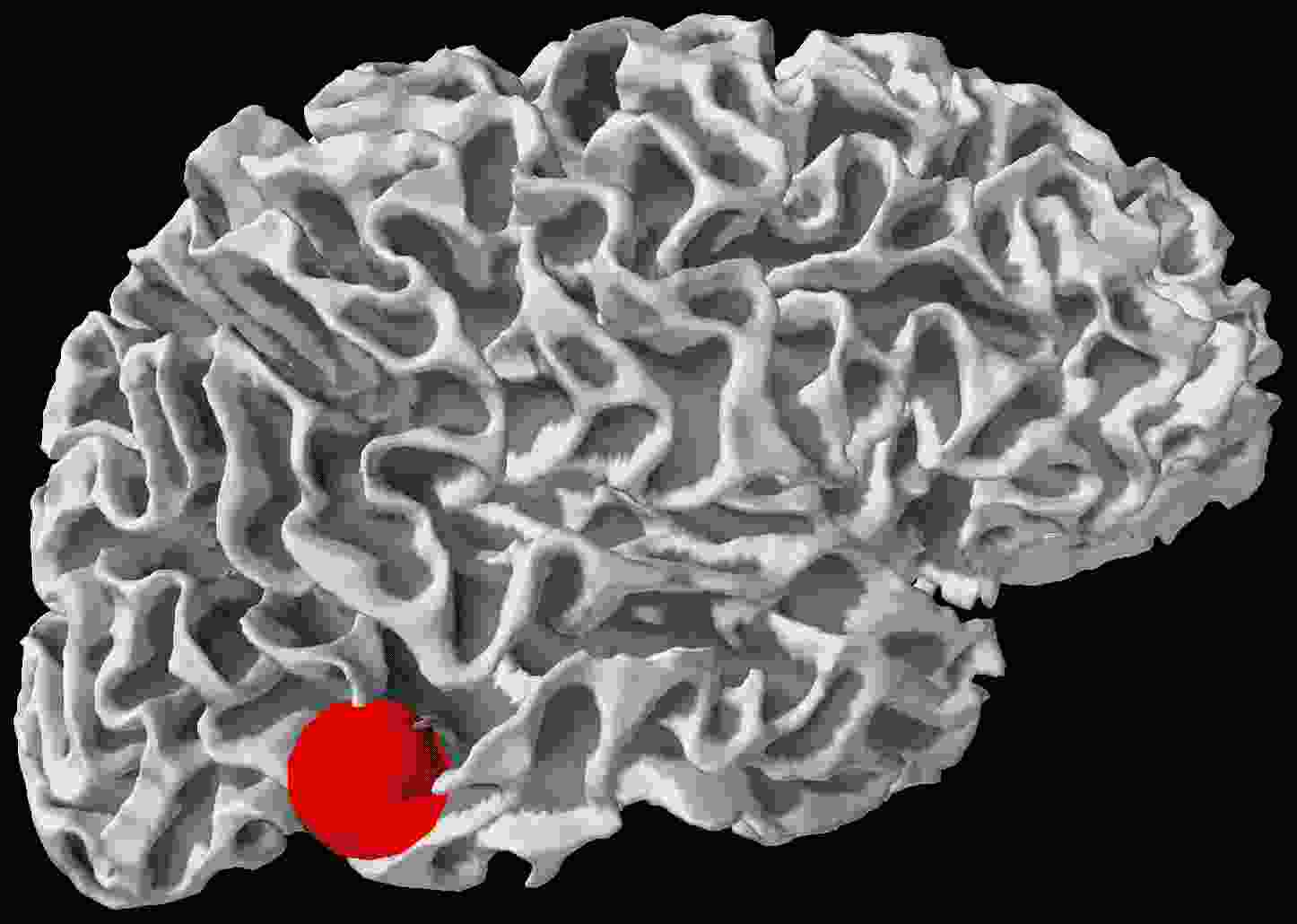}
    \caption{\mle}
    \end{subfigure}\hspace{\spacefig em}
  \begin{subfigure}{\figsize\textwidth}
    \includegraphics[width=68px,height=47px]{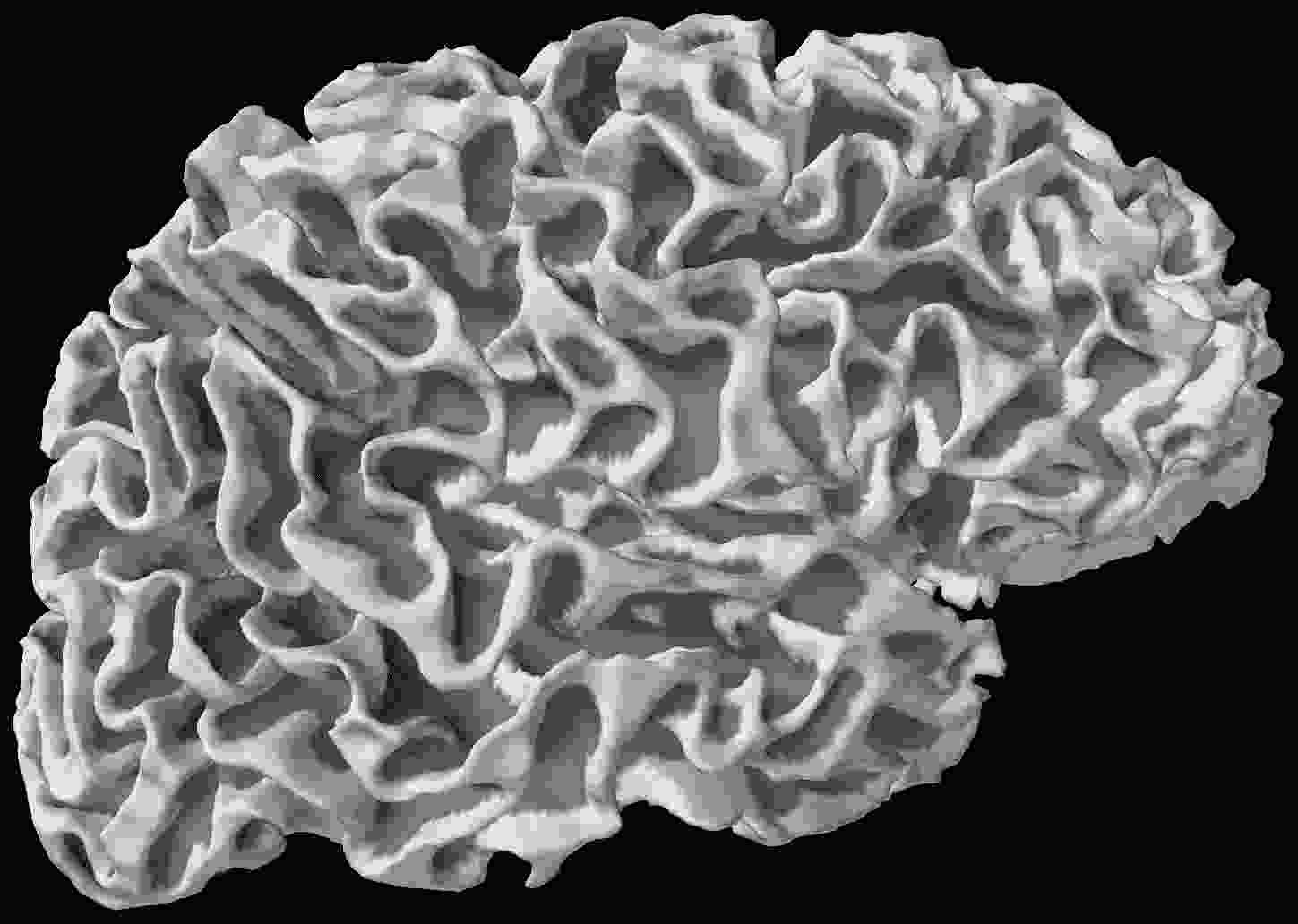}
    \caption{\mrcer}
    \end{subfigure}\hspace{\spacefig em}
  \begin{subfigure}{\figsize\textwidth}
    \includegraphics[width=68px,height=47px]{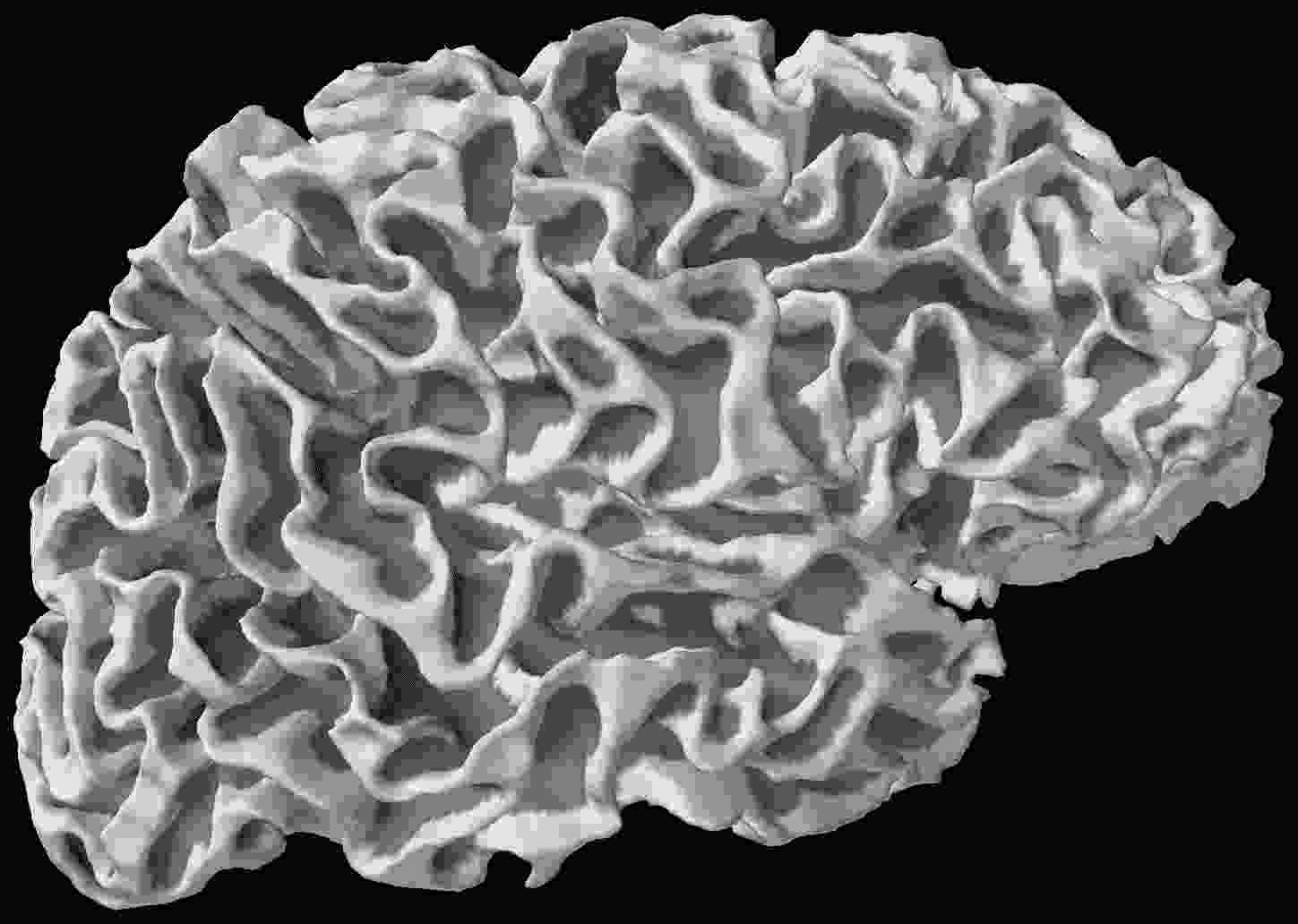}
    \caption{\mtl}
  \end{subfigure}\hspace{\spacefig em}
  \caption{\textit{Real data ($n=102$, $q=7498$, $q=48$, $r=71$)} Sources found in the left hemisphere (top) and the right hemisphere (bottom) after left visual stimulations.}
  \label{fig:real_data_left_visu}
\end{figure}

\Cref{fig:real_data_left_visu,fig:real_data_left_visu_r_31} show the results for each algorithm after left visual stimulations.
As one source is expected (in the right hemisphere), we vary $\lambda$ by dichotomy between $\lambda_{\max}$ (returning 0 sources) and a $\lambda_{\min}$ (returning more than 1 sources), until finding a lambda giving exactly 1 source.
When the number of repetitions is high (\Cref{fig:real_data_left_visu}) only \us and \mler do find a source in the visual cortex.
When the number of repetitions decreases, \us and \mler still find one source in the visual cortex, other algorithms fail.
This highlights this importance of taking into account the repetitions.

\begin{figure}[H]
  \centering 

  \begin{subfigure}{\figsize \textwidth}
  \includegraphics[width=68px,height=47px]{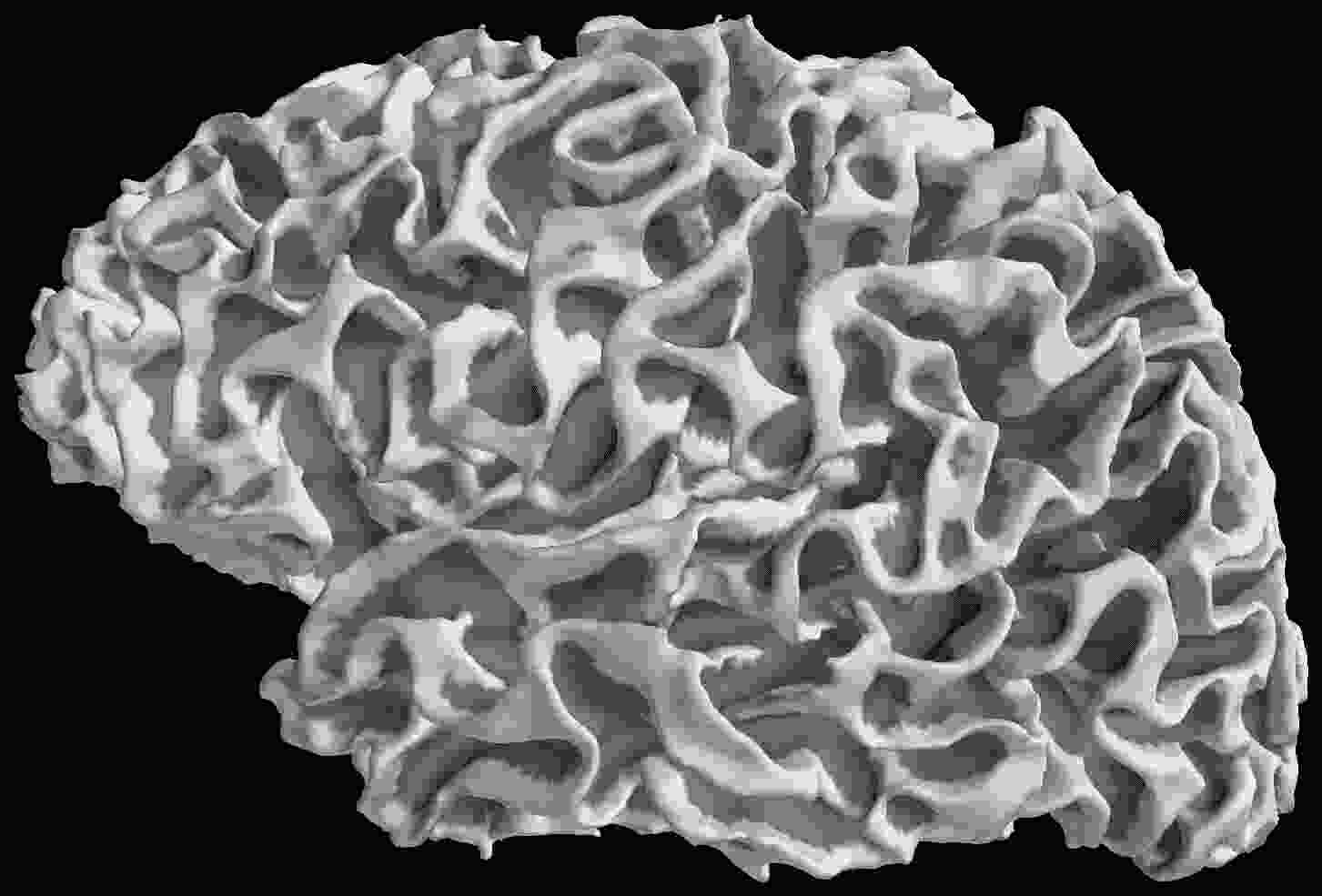}
  \end{subfigure}\hspace{\spacefig em}
  \begin{subfigure}{\figsize \textwidth}
  \includegraphics[width=68px,height=47px]{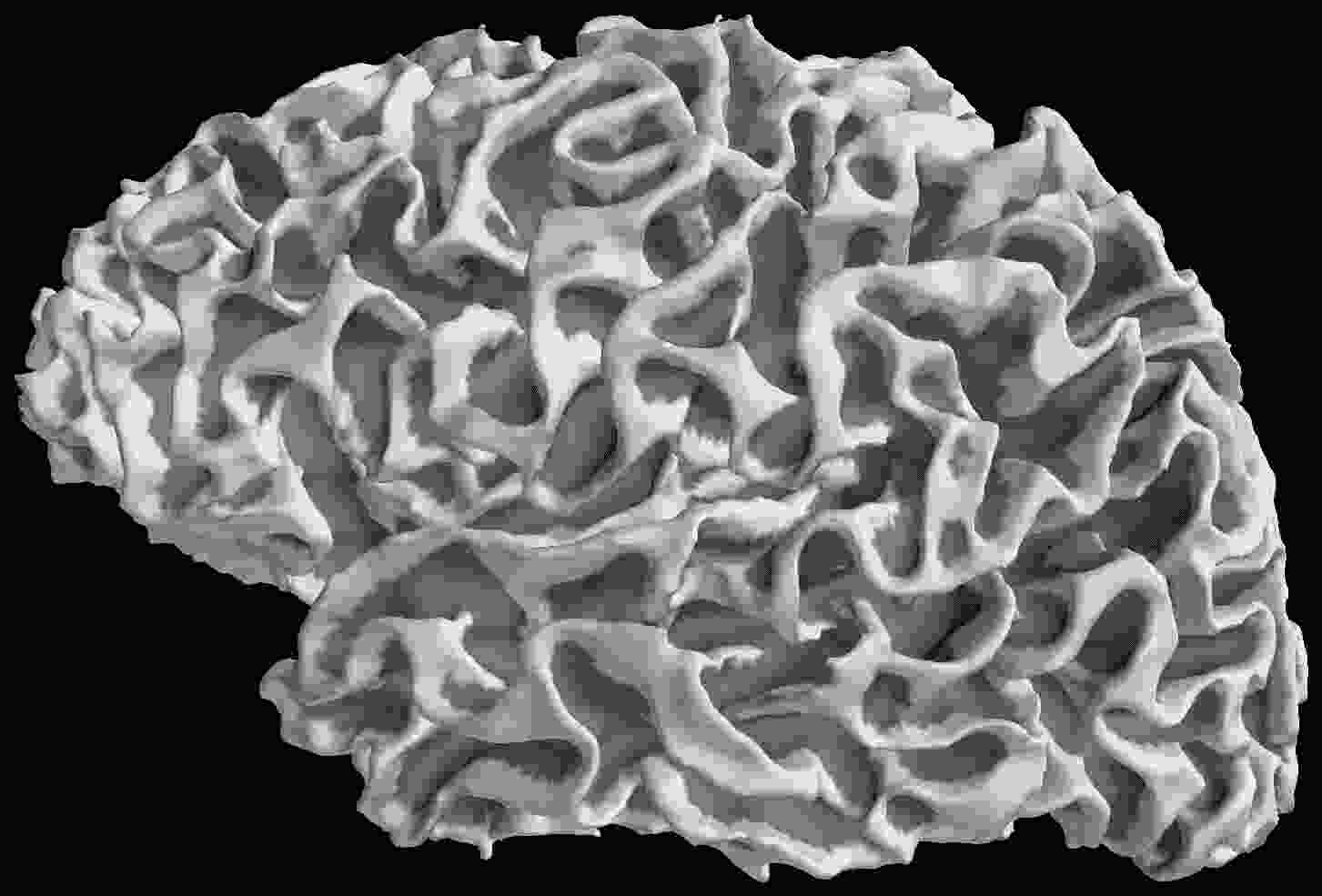}
  \end{subfigure}\hspace{\spacefig em}
  \begin{subfigure}{\figsize\textwidth}
    \includegraphics[width=68px,height=47px]{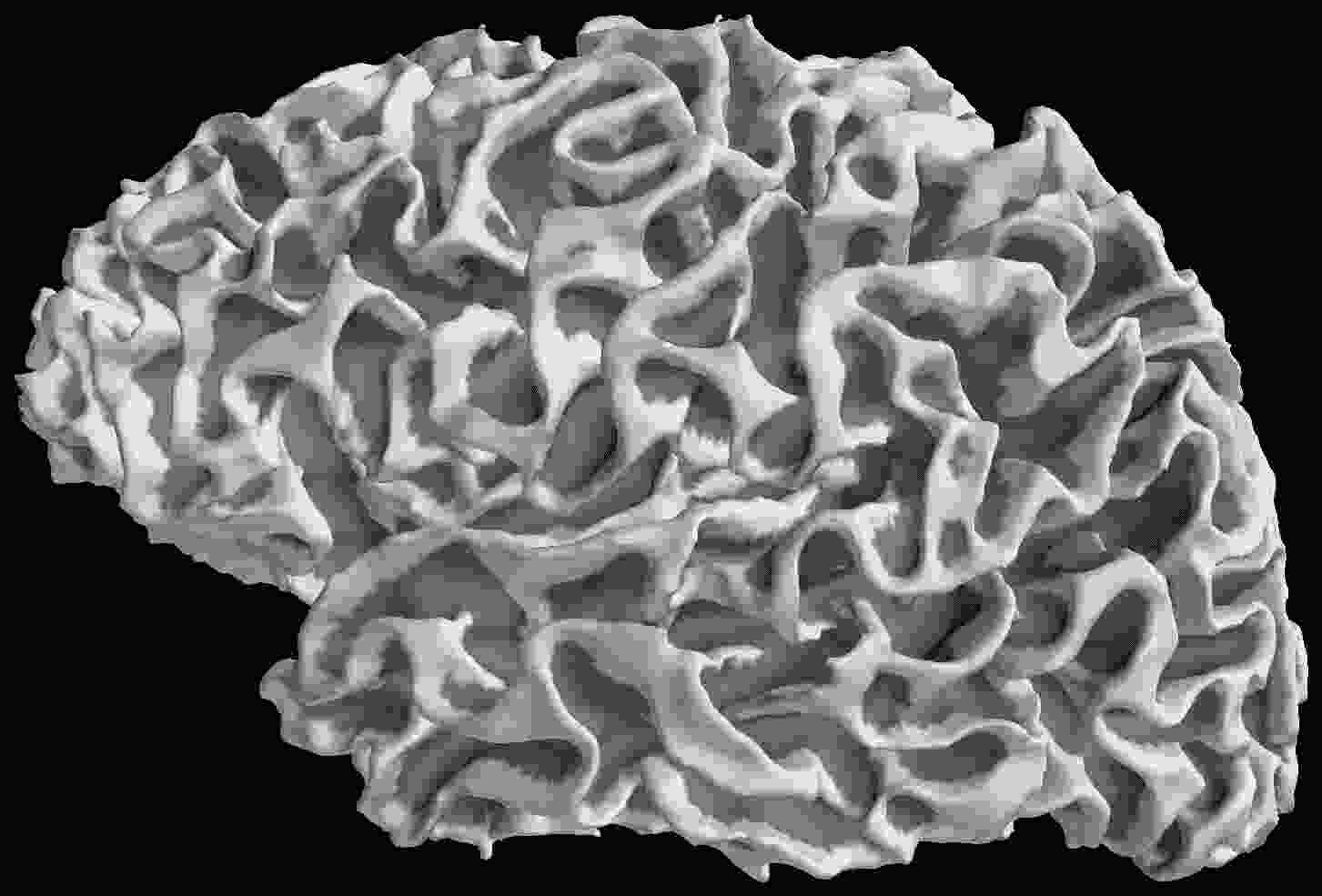}
    \end{subfigure}\hspace{\spacefig em}
  \begin{subfigure}{\figsize\textwidth}
  \includegraphics[width=68px,height=47px]{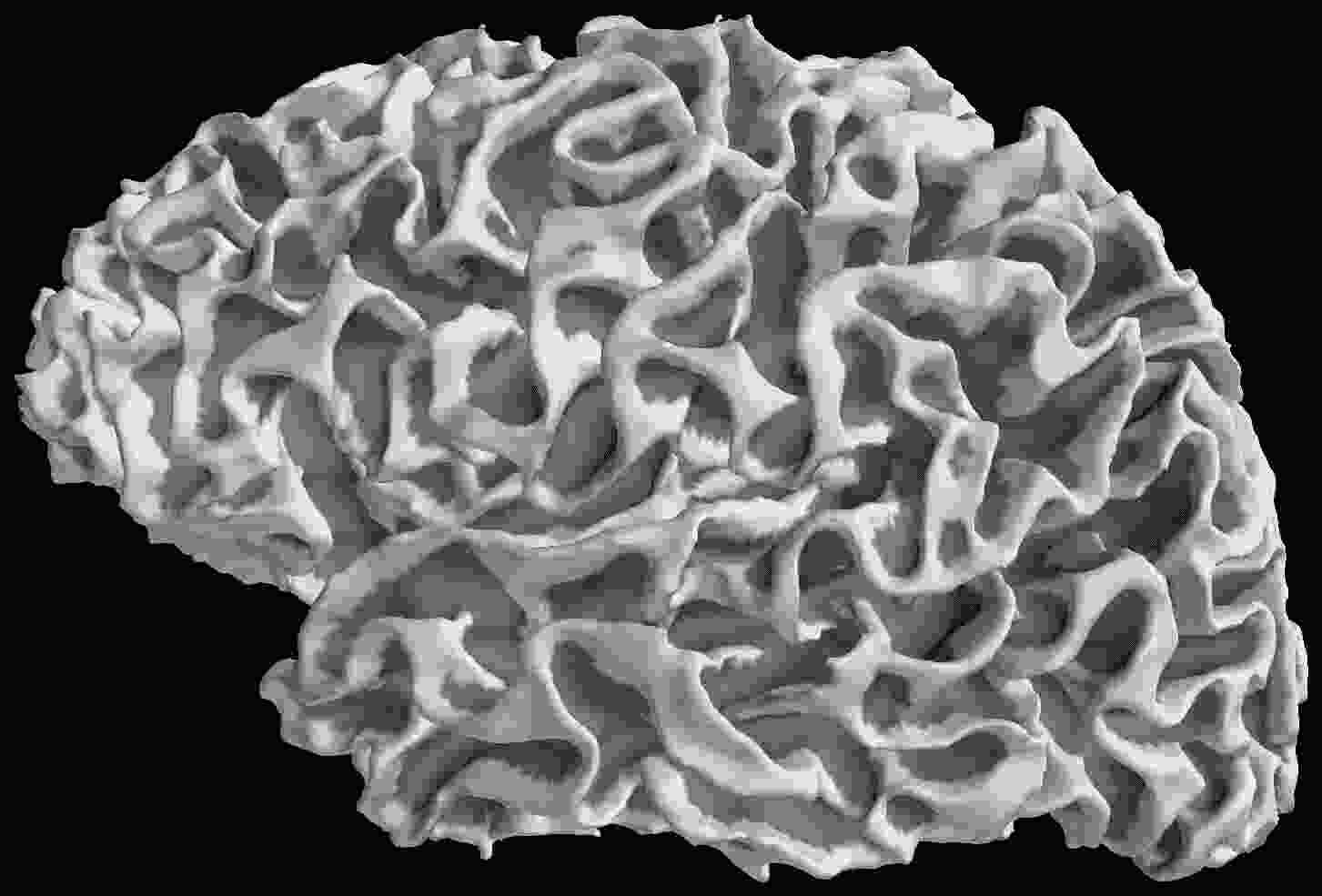}
  \end{subfigure}\hspace{\spacefig em}
  \begin{subfigure}{\figsize\textwidth}
  \includegraphics[width=68px,height=47px]{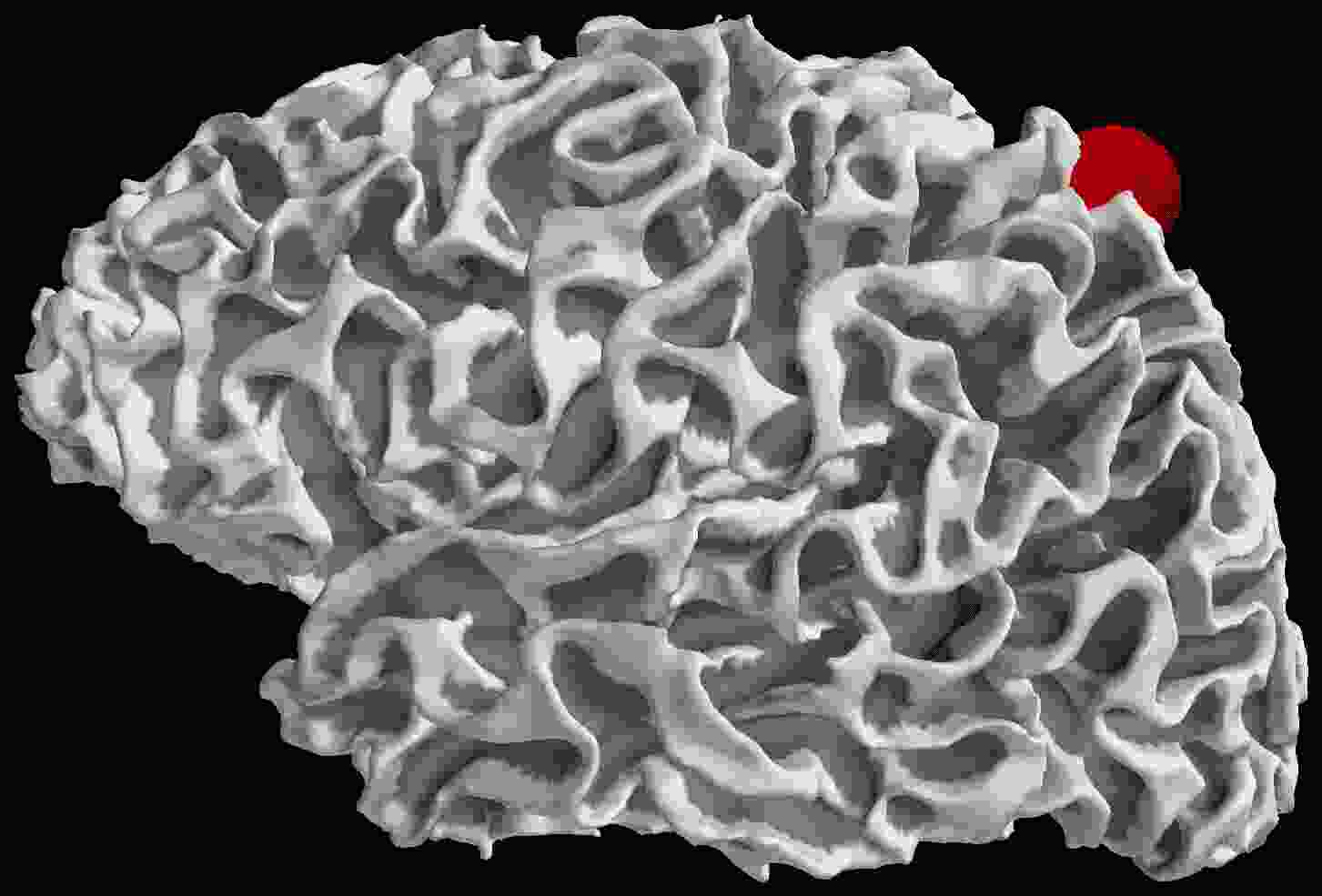}
  \end{subfigure}\hspace{\spacefig em}
  \begin{subfigure}{\figsize \textwidth}
    \includegraphics[width=68px,height=47px]{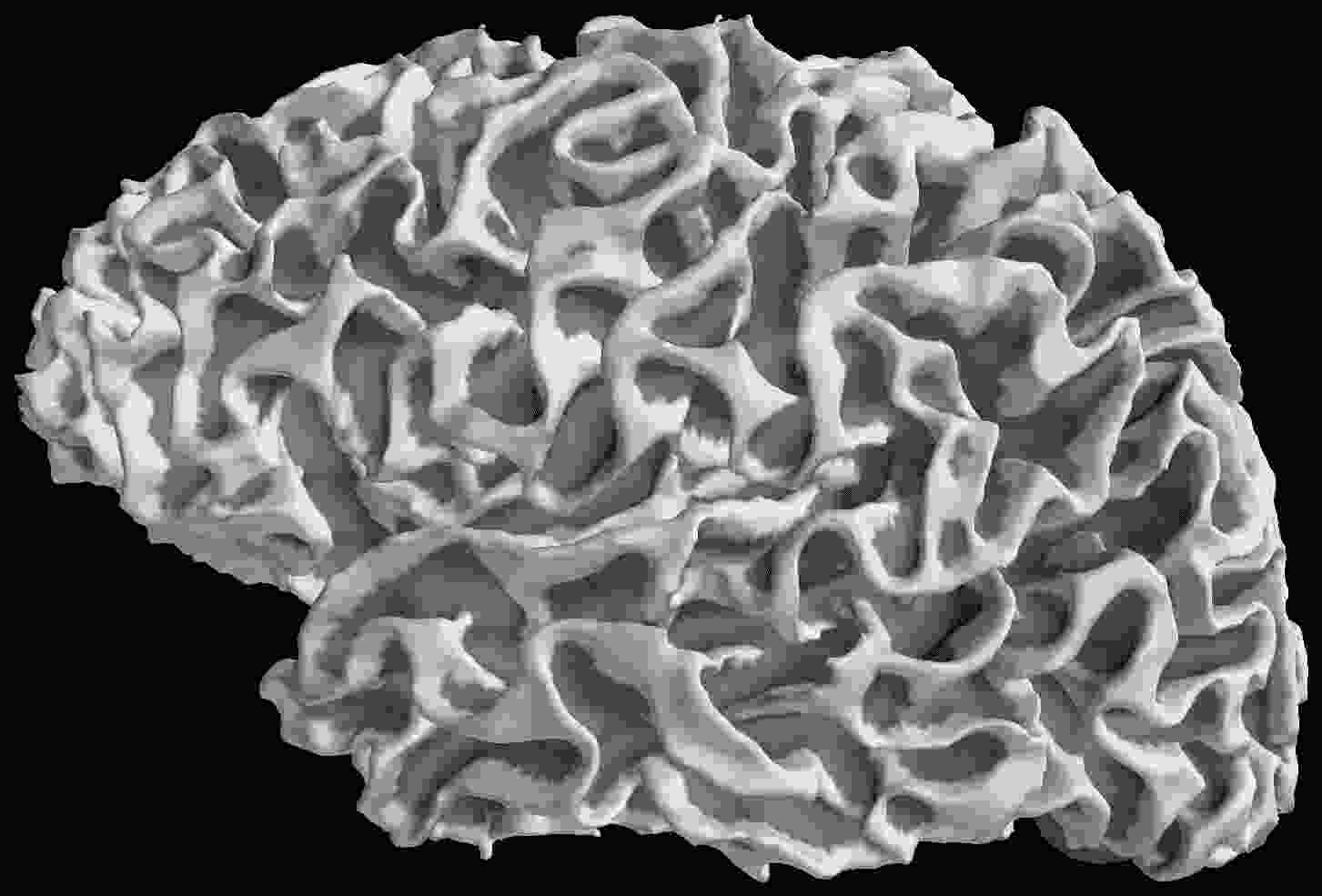}
    \end{subfigure}\hspace{\spacefig em}

  \vspace{-0.1em}
  \begin{subfigure}{\figsize\textwidth}
    \includegraphics[width=68px,height=47px]{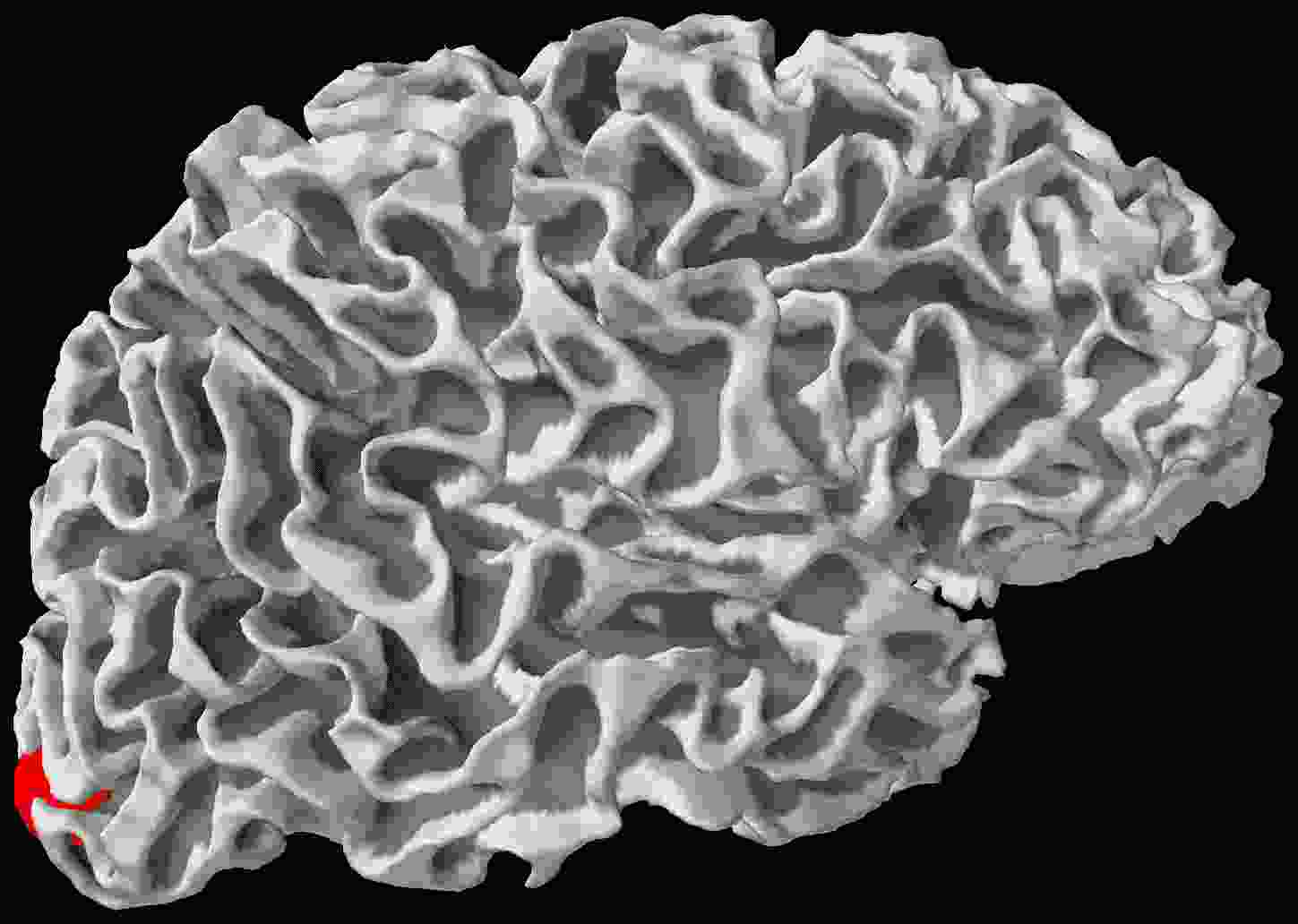}
    \caption{\us}
  \end{subfigure}\hspace{\spacefig em}
  \begin{subfigure}{\figsize\textwidth}
    \includegraphics[width=68px,height=47px]{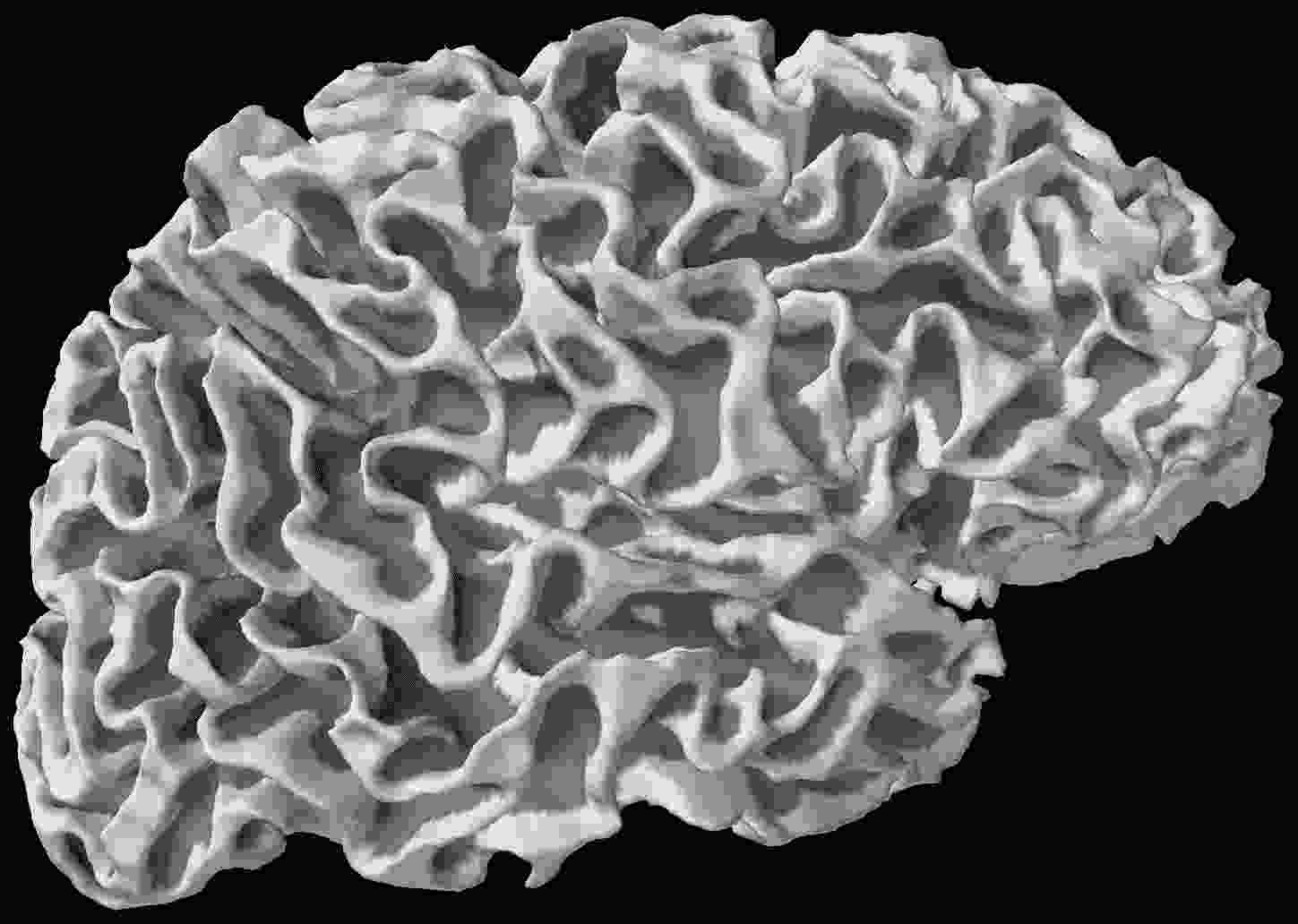}
    \caption{\sgcl}
  \end{subfigure}\hspace{\spacefig em}
  \begin{subfigure}{\figsize\textwidth}
    \includegraphics[width=68px,height=47px]{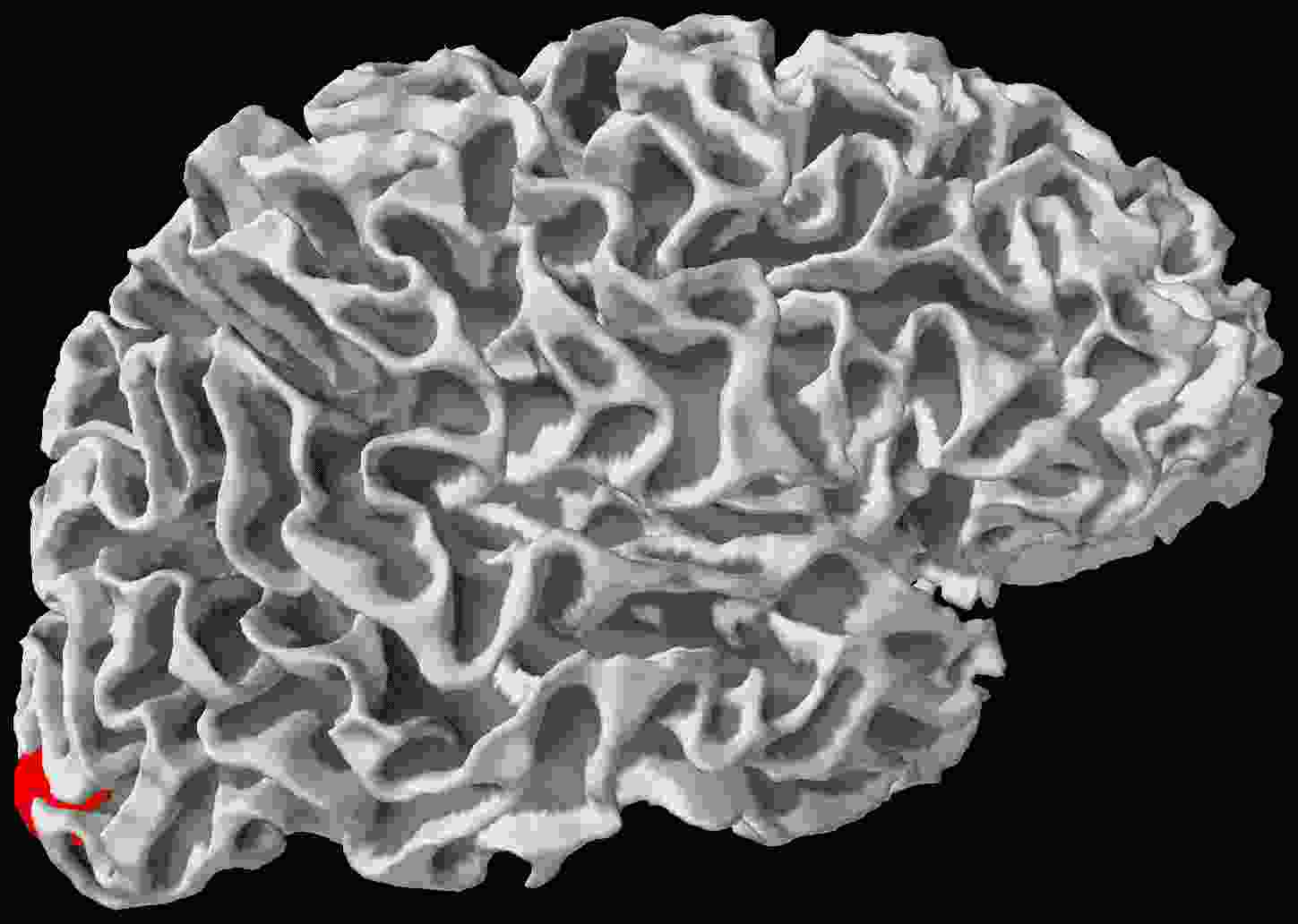}
    \caption{\mler}
    \end{subfigure}\hspace{\spacefig em}
  \begin{subfigure}{\figsize\textwidth}
    \includegraphics[width=68px,height=47px]{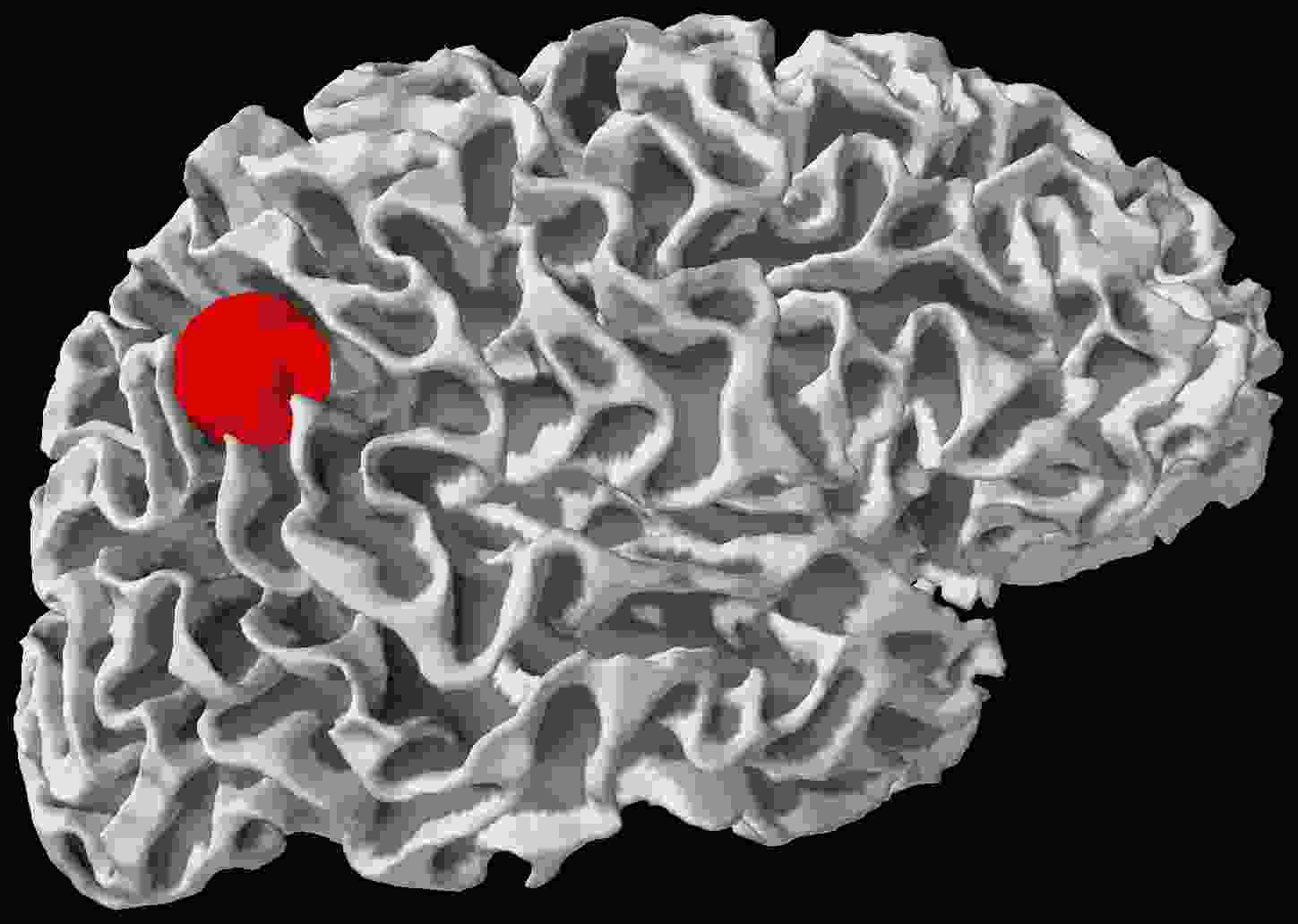}
    \caption{\mle}
    \end{subfigure}\hspace{\spacefig em}
  \begin{subfigure}{\figsize\textwidth}
    \includegraphics[width=68px,height=47px]{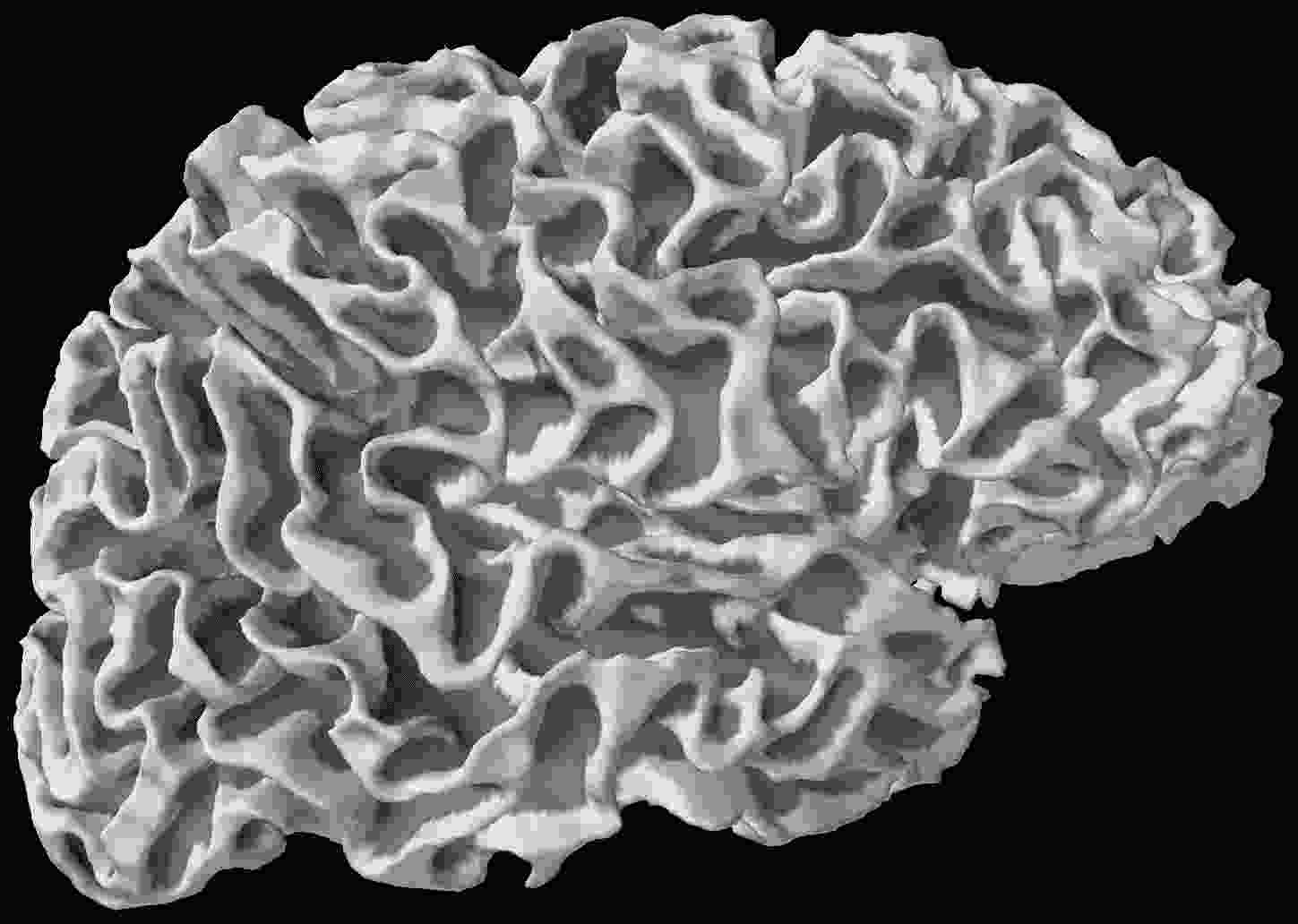}
    \caption{\mrcer}
    \end{subfigure}\hspace{\spacefig em}
  \begin{subfigure}{\figsize\textwidth}
    \includegraphics[width=68px,height=47px]{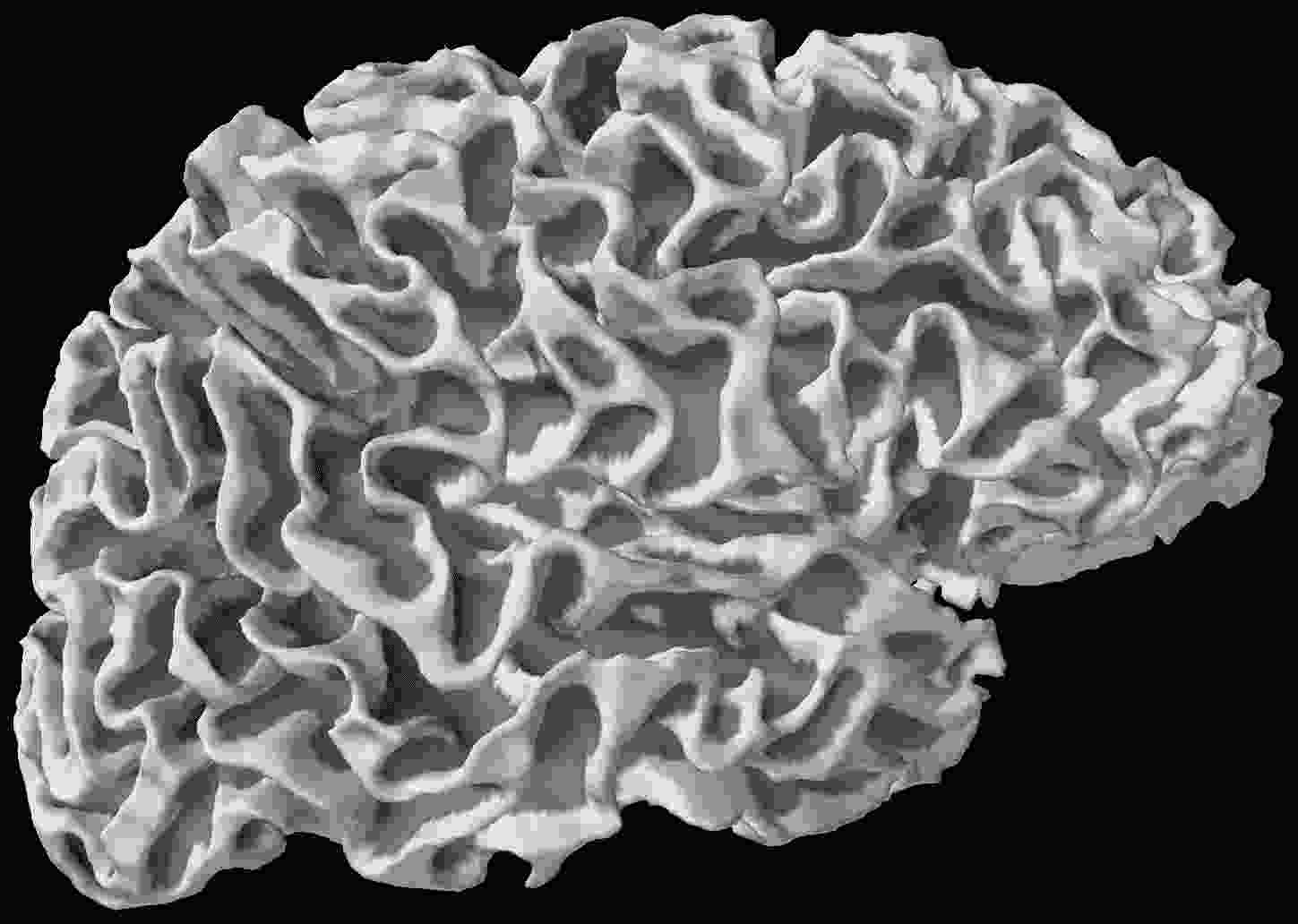}
    \caption{\mtl}
  \end{subfigure}\hspace{\spacefig em}
  \caption{\textit{Real data ($n=102$, $q=7498$, $q=48$, $r=36$)} Sources found in the left hemisphere (top) and the right hemisphere (bottom)
  after left visual stimulations.}
  \label{fig:real_data_left_visu_r_31}
\end{figure}

\subsection{Supplementary experiments on real data: right visual stimulations}
\label{app_sub:real_right_visual}

\begin{figure}[H]
  \centering 

  \begin{subfigure}{\figsize \textwidth}
  \includegraphics[width=68px,height=47px]{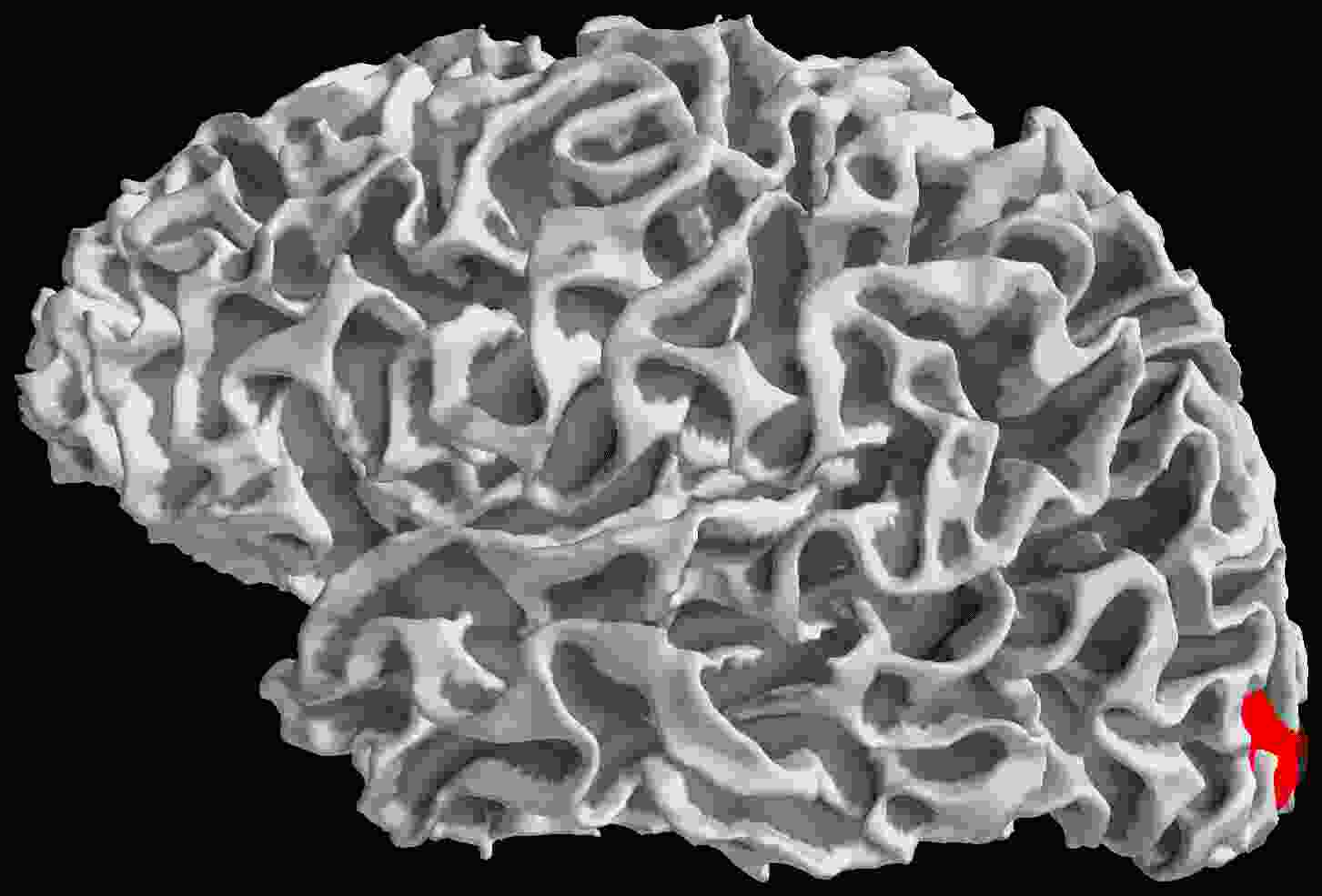}
  \end{subfigure}\hspace{\spacefig em}
  \begin{subfigure}{\figsize \textwidth}
  \includegraphics[width=68px,height=47px]{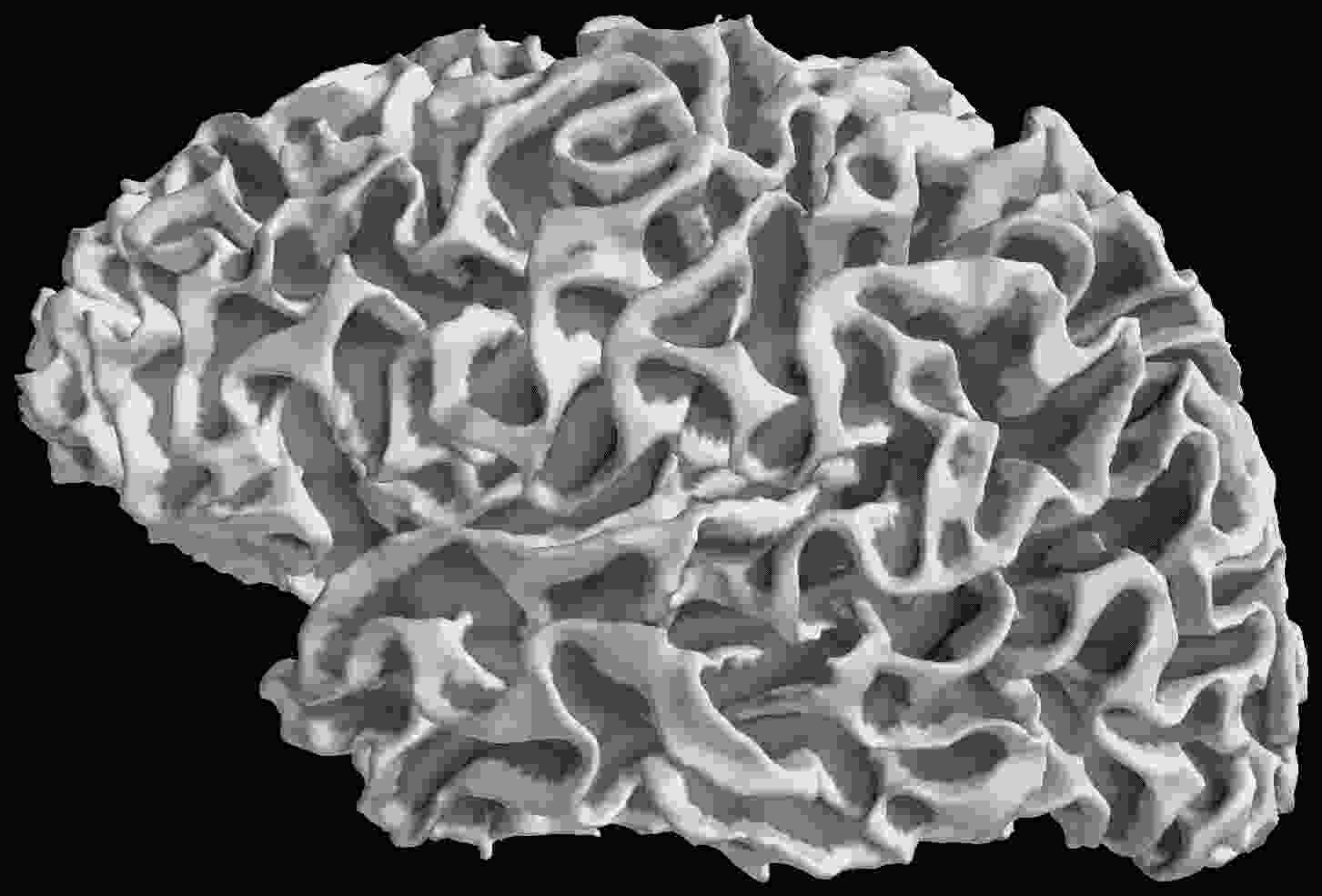}
  \end{subfigure}\hspace{\spacefig em}
  \begin{subfigure}{\figsize\textwidth}
    \includegraphics[width=68px,height=47px]{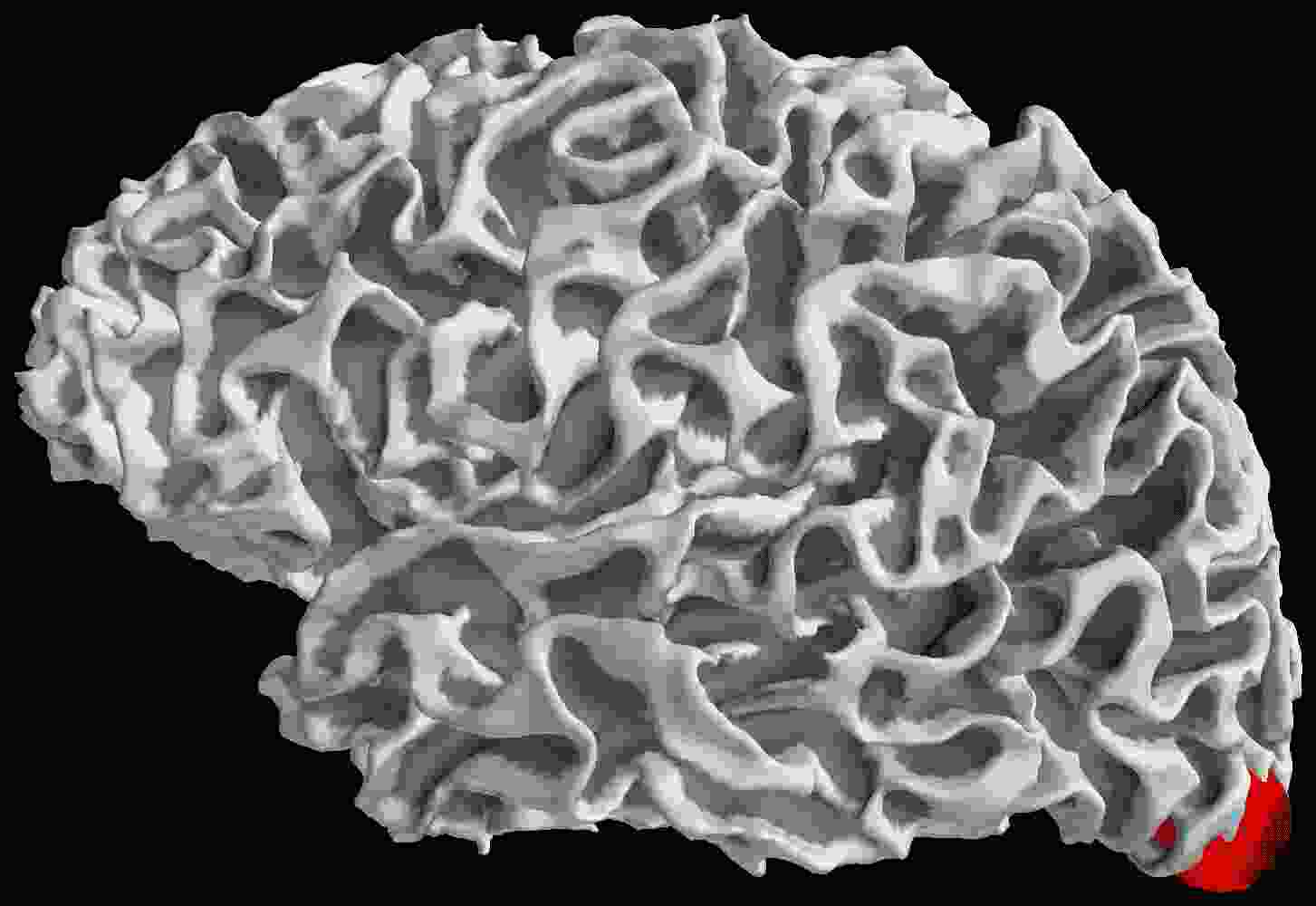}
    \end{subfigure}\hspace{\spacefig em}
  \begin{subfigure}{\figsize\textwidth}
  \includegraphics[width=68px,height=47px]{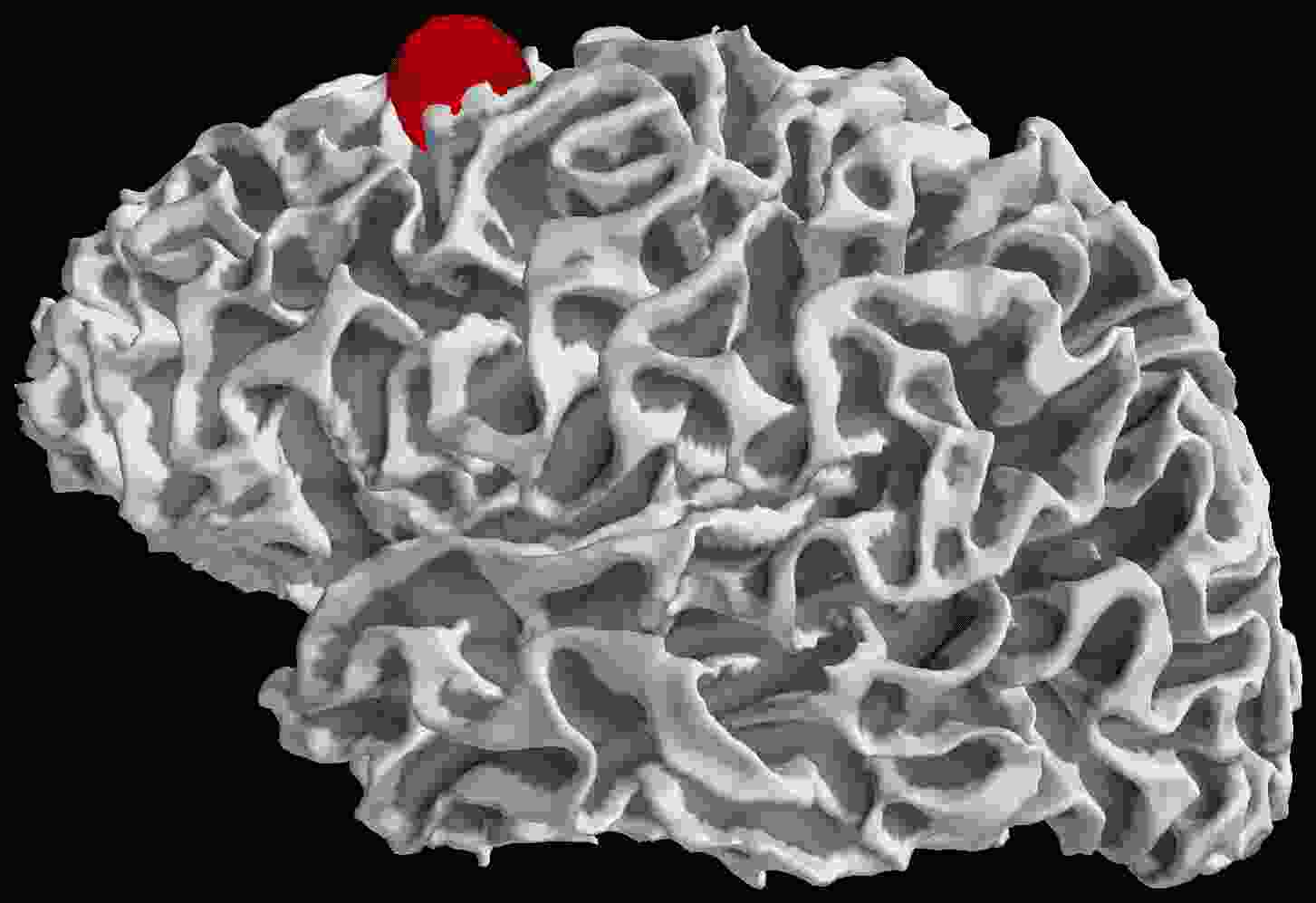}
  \end{subfigure}\hspace{\spacefig em}
  \begin{subfigure}{\figsize\textwidth}
  \includegraphics[width=68px,height=47px]{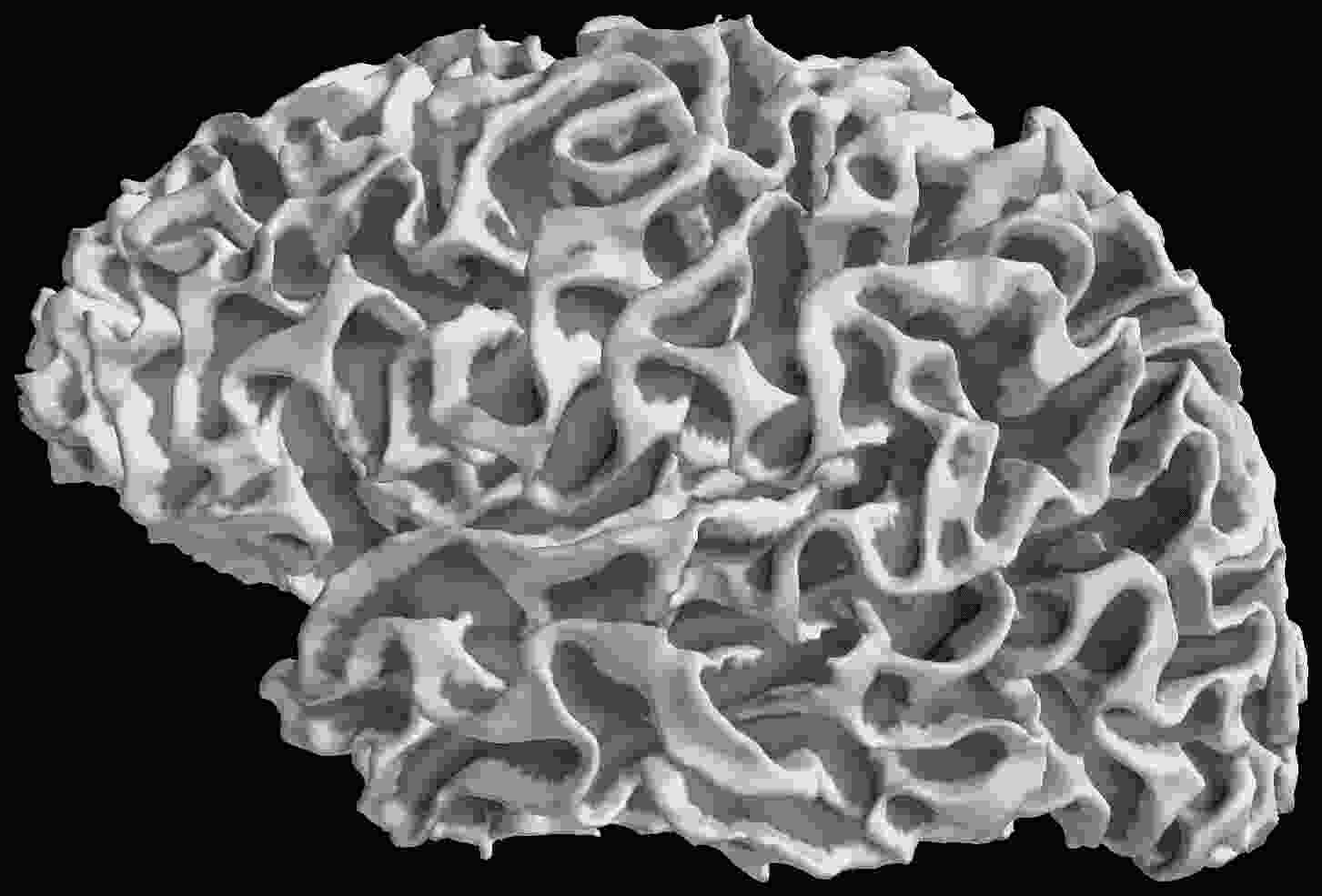}
  \end{subfigure}\hspace{\spacefig em}
  \begin{subfigure}{\figsize \textwidth}
    \includegraphics[width=68px,height=47px]{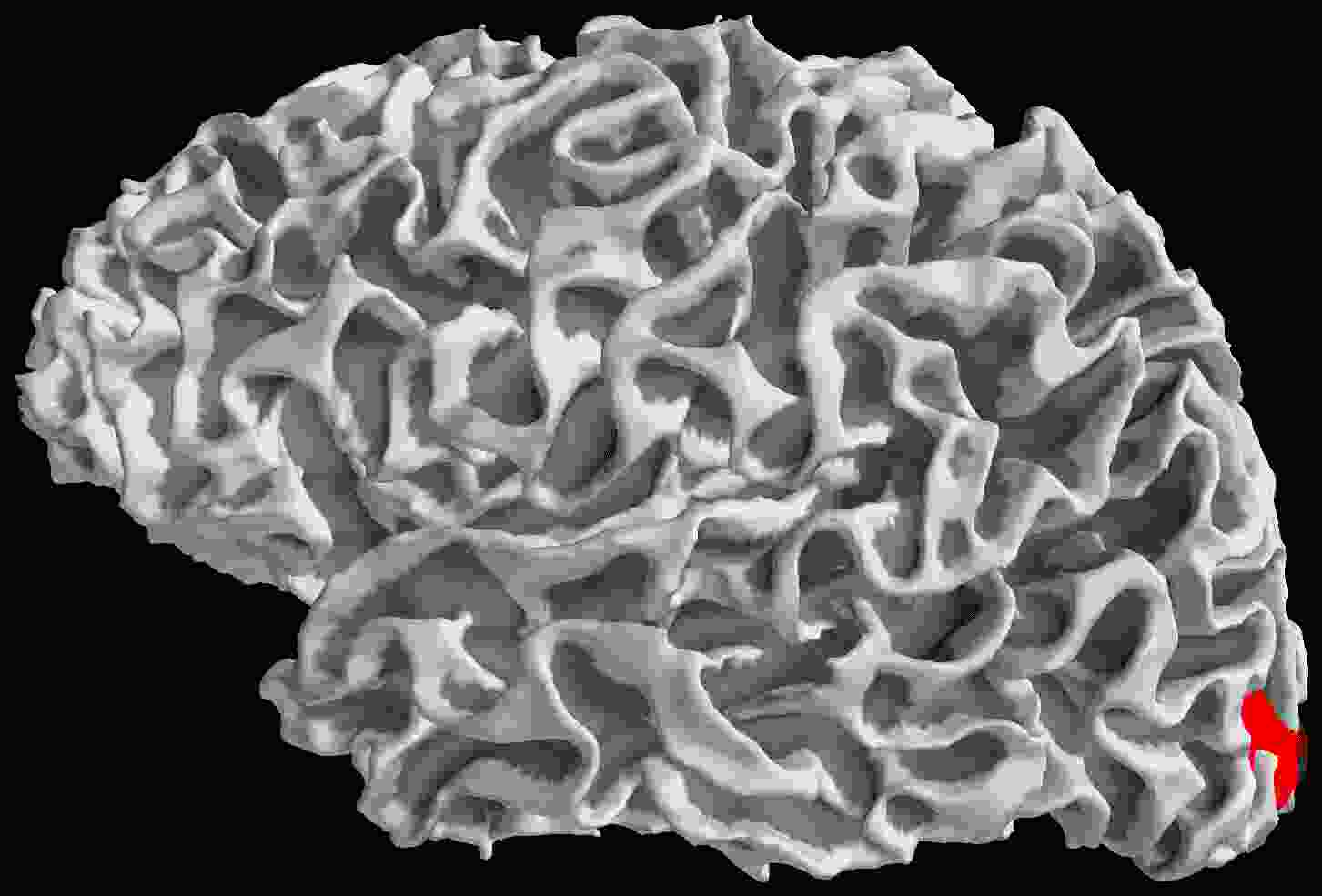}
    \end{subfigure}\hspace{\spacefig em}

  \vspace{-0.1em}
  \begin{subfigure}{\figsize\textwidth}
    \includegraphics[width=68px,height=47px]{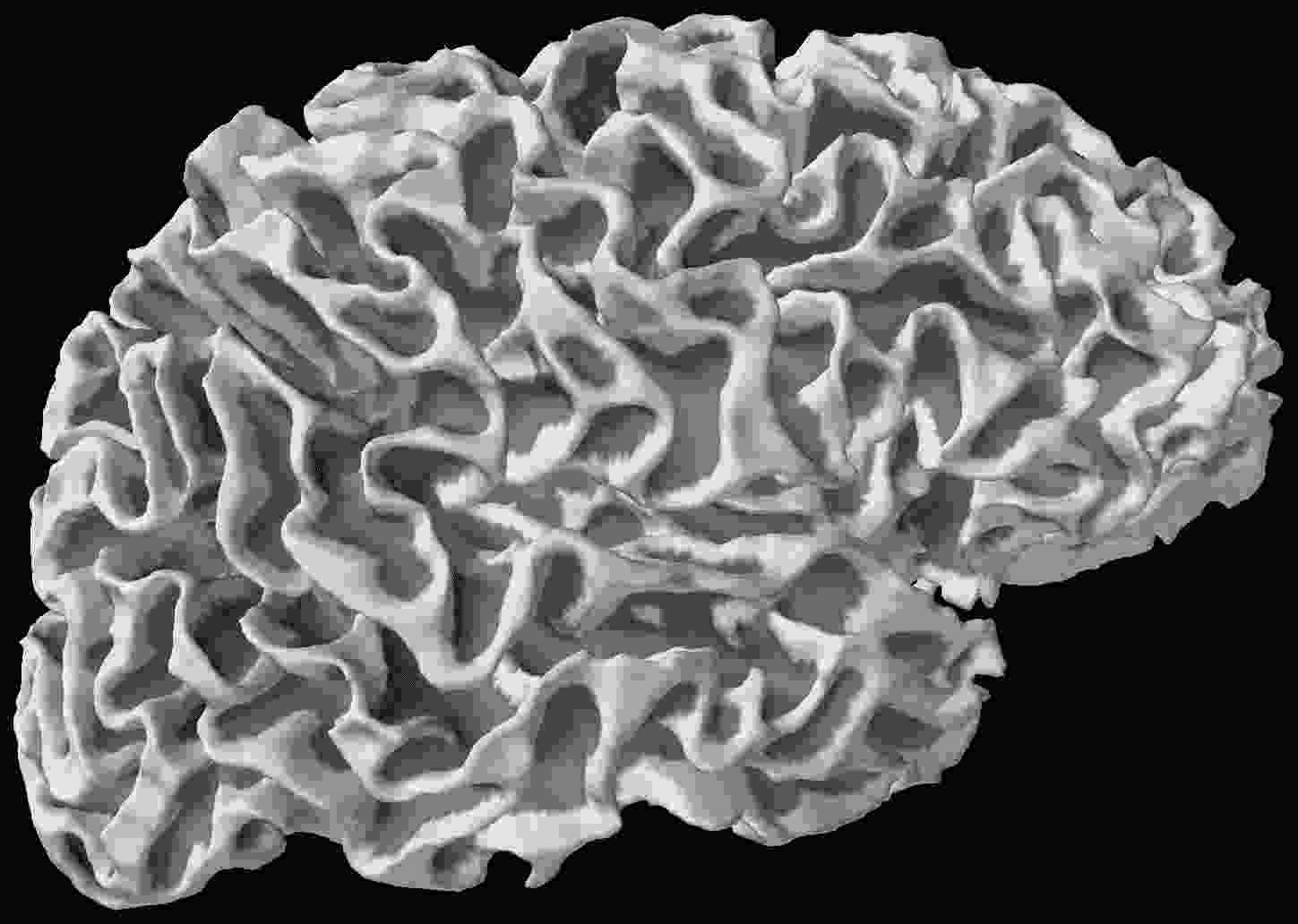}
    \caption{\us}
  \end{subfigure}\hspace{\spacefig em}
  \begin{subfigure}{\figsize\textwidth}
    \includegraphics[width=68px,height=47px]{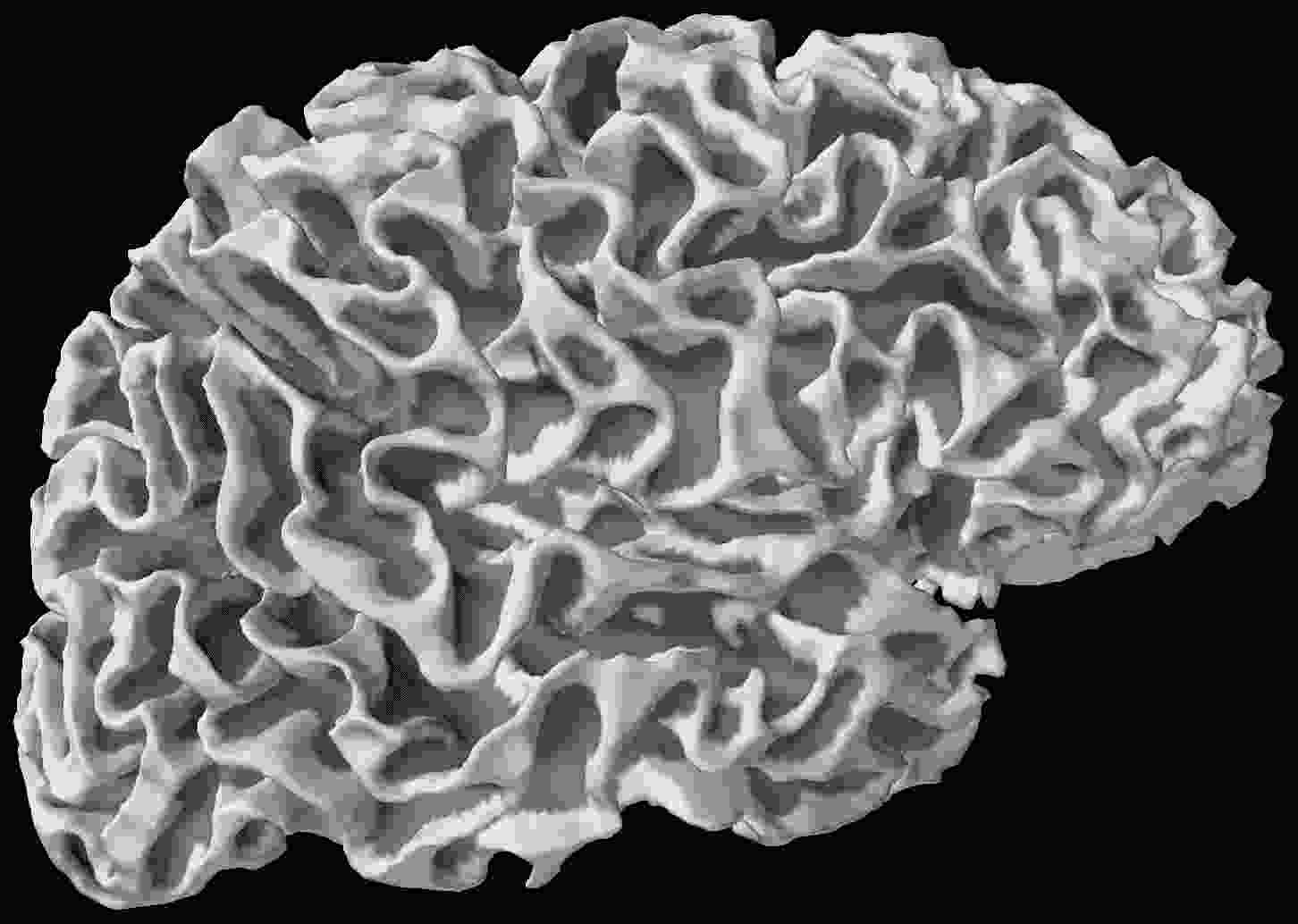}
    \caption{\sgcl}
  \end{subfigure}\hspace{\spacefig em}
  \begin{subfigure}{\figsize\textwidth}
    \includegraphics[width=68px,height=47px]{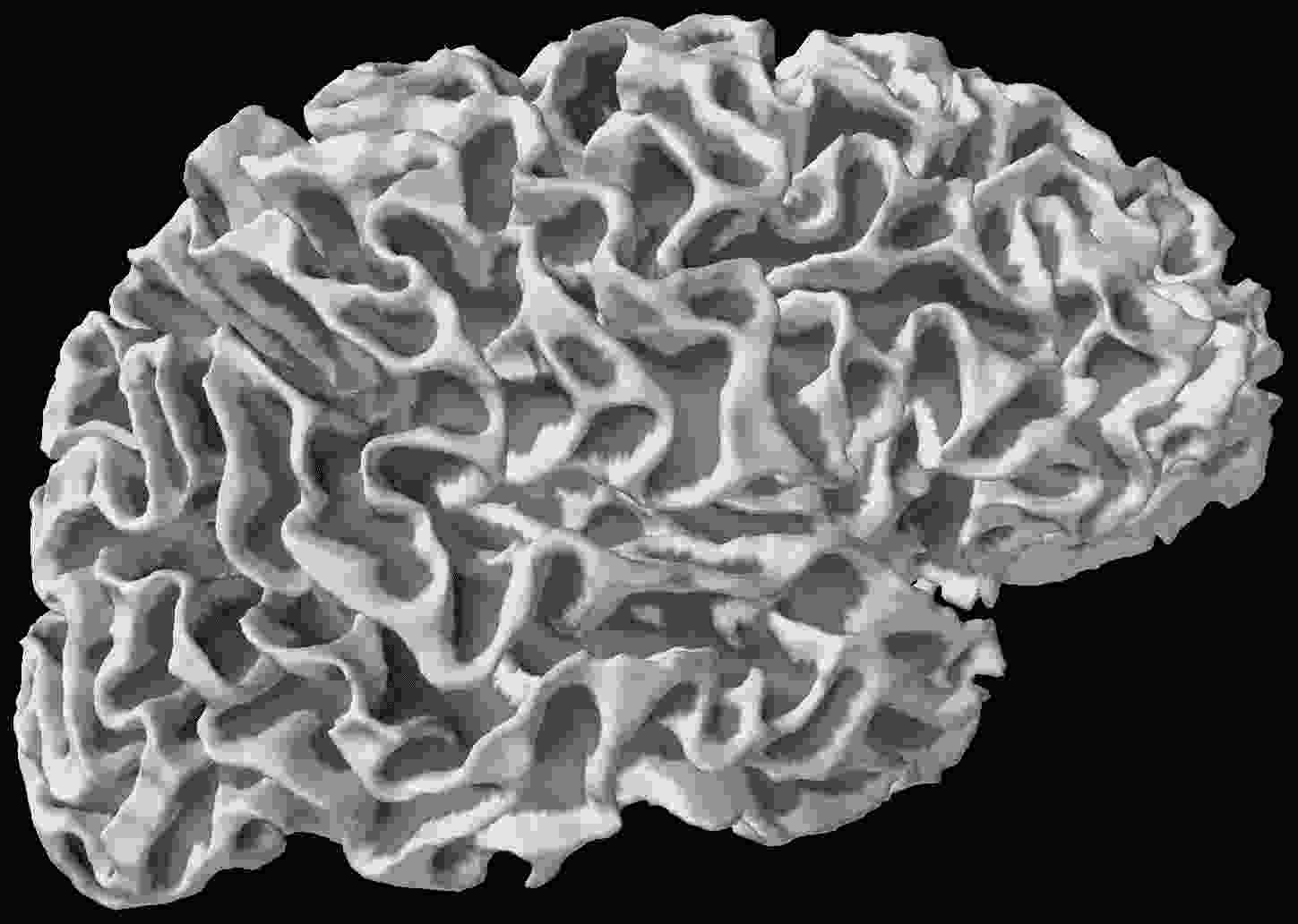}
    \caption{\mler}
    \end{subfigure}\hspace{\spacefig em}
  \begin{subfigure}{\figsize\textwidth}
    \includegraphics[width=68px,height=47px]{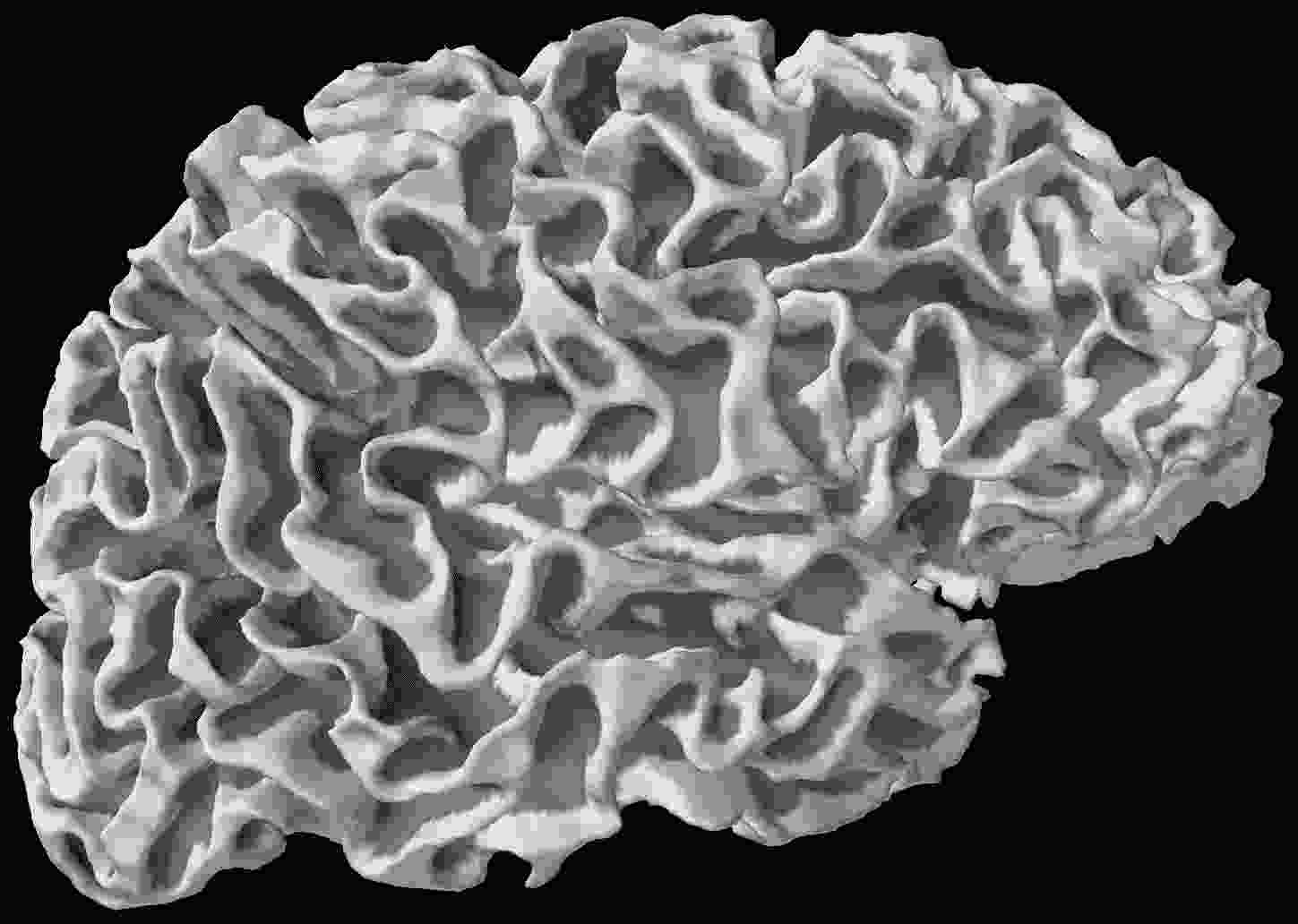}
    \caption{\mle}
    \end{subfigure}\hspace{\spacefig em}
  \begin{subfigure}{\figsize\textwidth}
    \includegraphics[width=68px,height=47px]{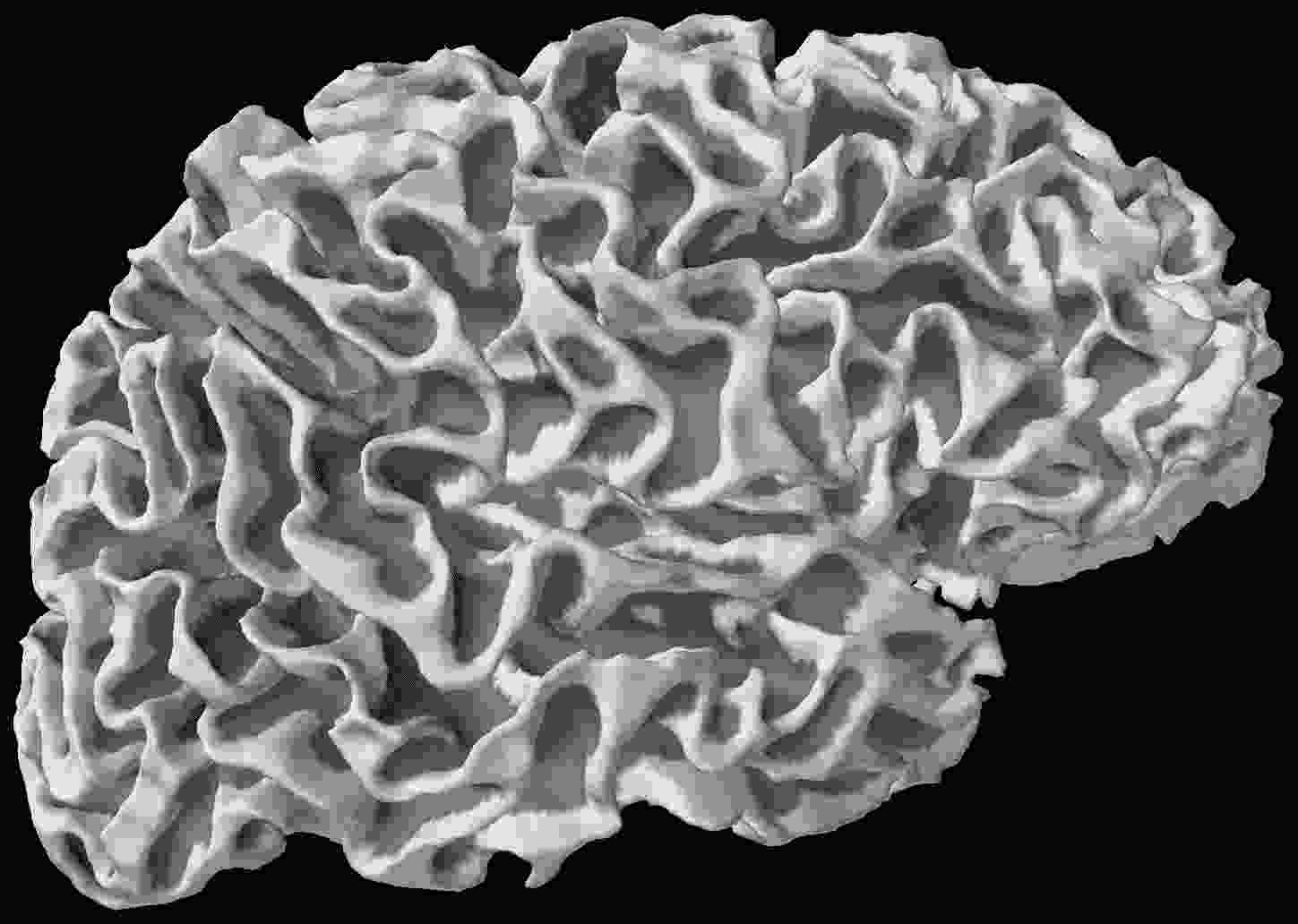}
    \caption{\mrcer}
    \end{subfigure}\hspace{\spacefig em}
  \begin{subfigure}{\figsize\textwidth}
    \includegraphics[width=68px,height=47px]{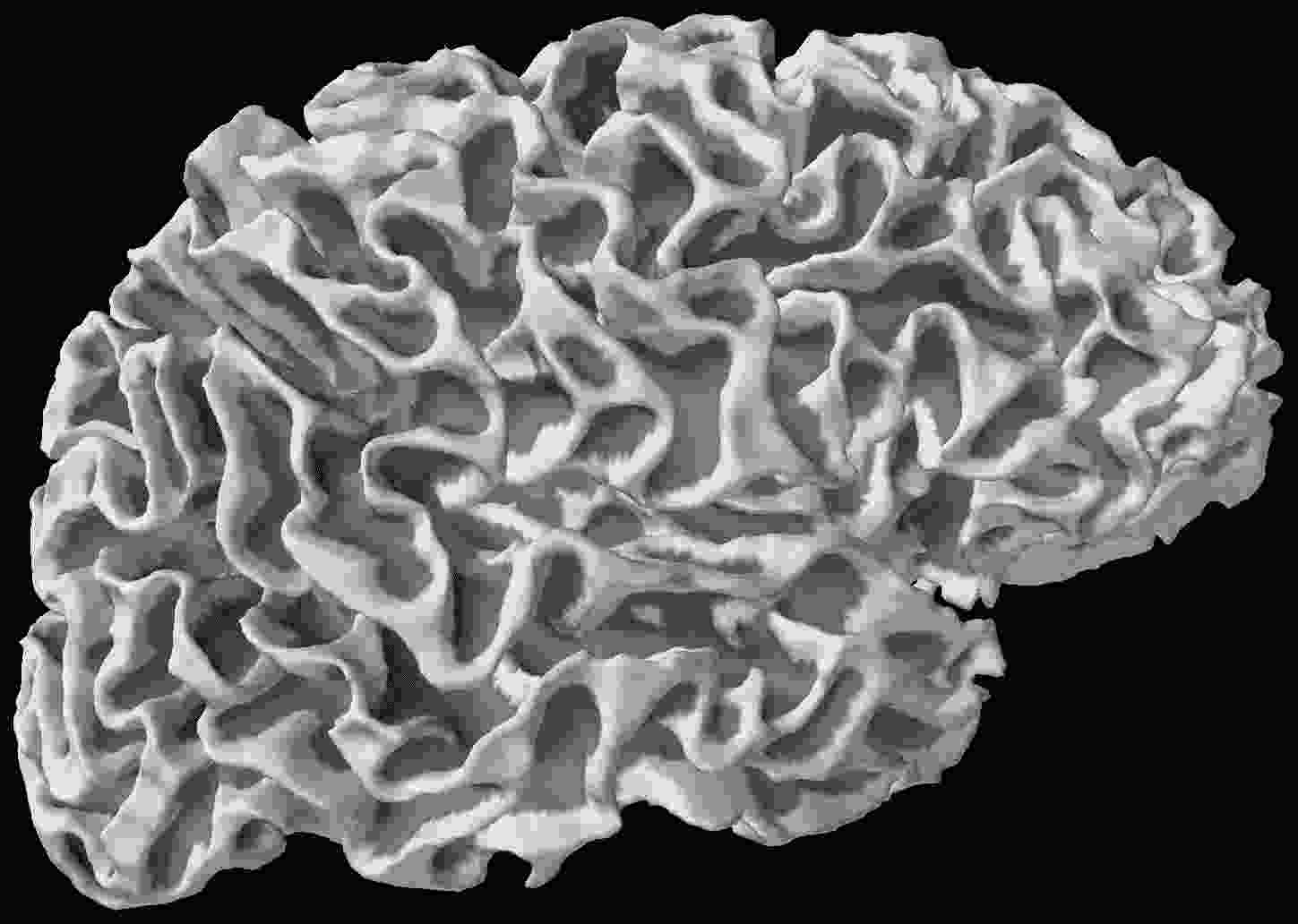}
    \caption{\mtl}
  \end{subfigure}\hspace{\spacefig em}
  \caption{\textit{Real data ($n=102$, $q=7498$, $q=48$, $r=61$)} Sources found in the left hemisphere (top) and the right hemisphere (bottom) after right visual stimulations.}
  \label{fig:real_data_right_visu}
\end{figure}

\Cref{fig:real_data_right_visu,fig:real_data_right_visu_r_31} show the results for each algorithm after right visual stimulations.
As one source is expected (in the left hemisphere), we vary $\lambda$ by dichotomy between $\lambda_{\max}$ (returning 0 sources) and a $\lambda_{\min}$ (returning more than 1 sources), until finding a lambda giving exactly 1 source.
When the number of repetitions is high (\Cref{fig:real_data_right_visu}) only \us, \mler and \mtl do find a source in the visual cortex.
When the number of repetitions decreases (\Cref{fig:real_data_right_visu_r_31}), only \us finds one source in the visual cortex, other algorithms fail.
This highlights once again the robustness of \us, even with a limited number of repetitions.

\begin{figure}[H]
  \centering 

  \begin{subfigure}{\figsize \textwidth}
  \includegraphics[width=68px,height=47px]{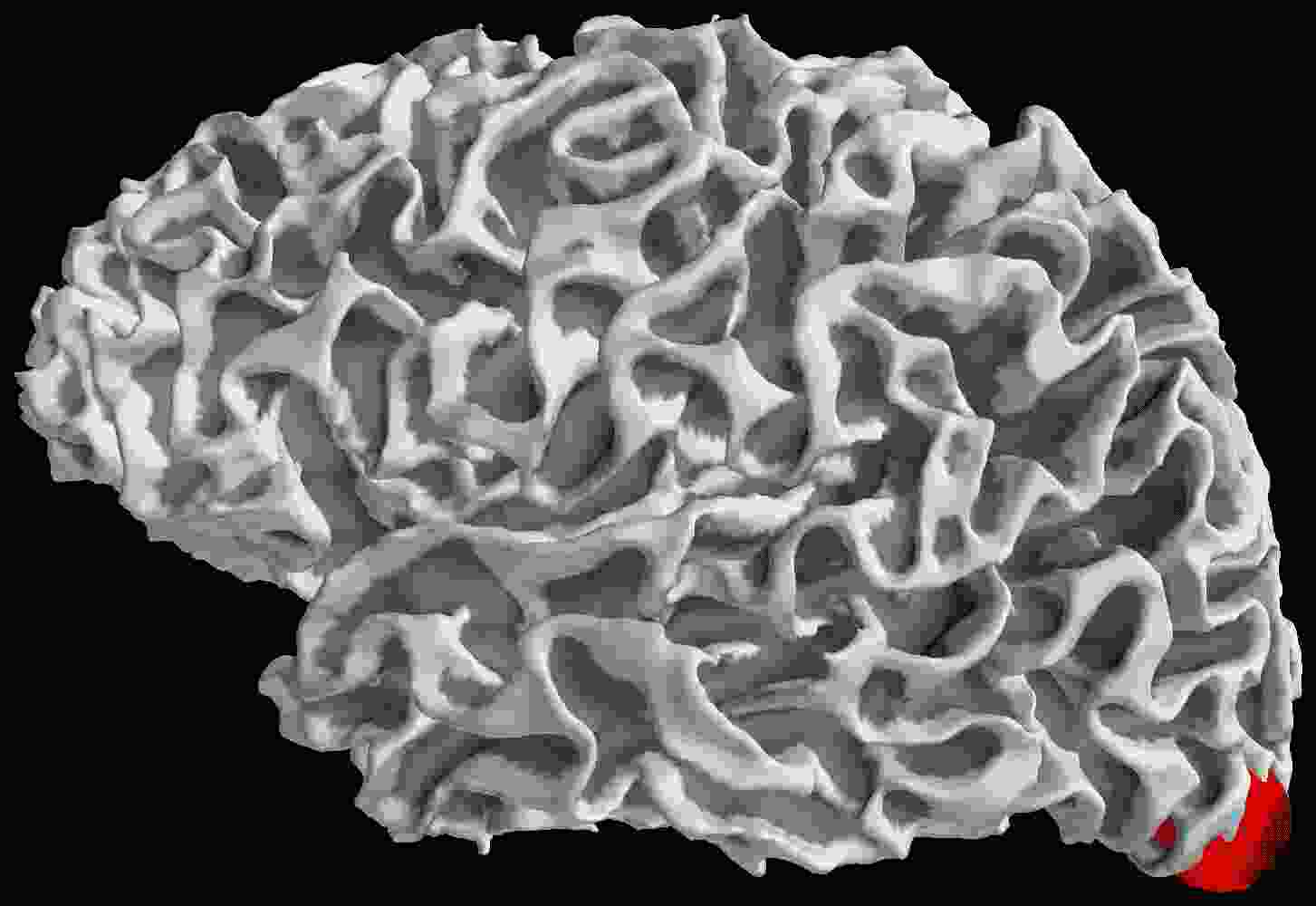}
  \end{subfigure}\hspace{\spacefig em}
  \begin{subfigure}{\figsize \textwidth}
  \includegraphics[width=68px,height=47px]{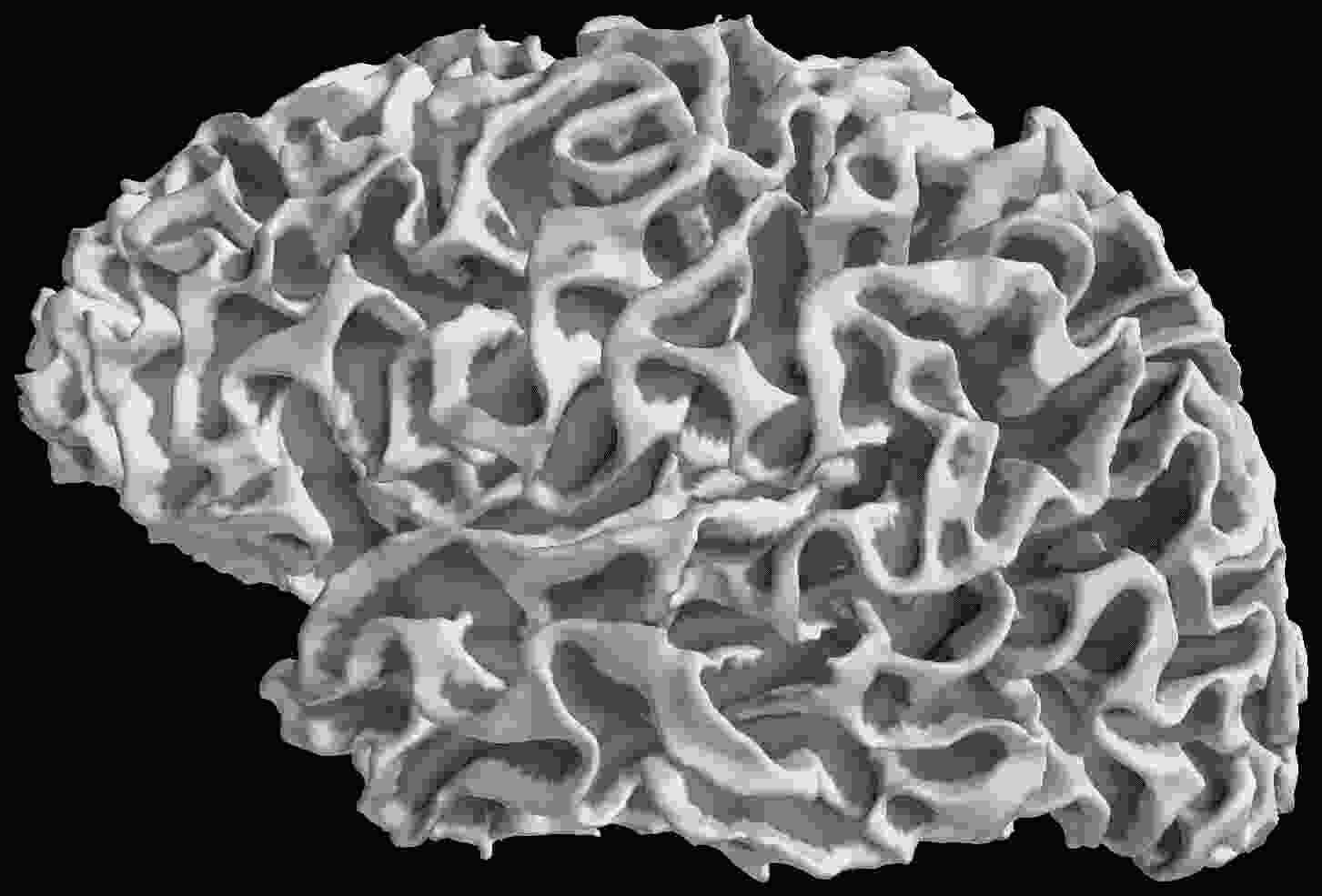}
  \end{subfigure}\hspace{\spacefig em}
  \begin{subfigure}{\figsize\textwidth}
    \includegraphics[width=68px,height=47px]{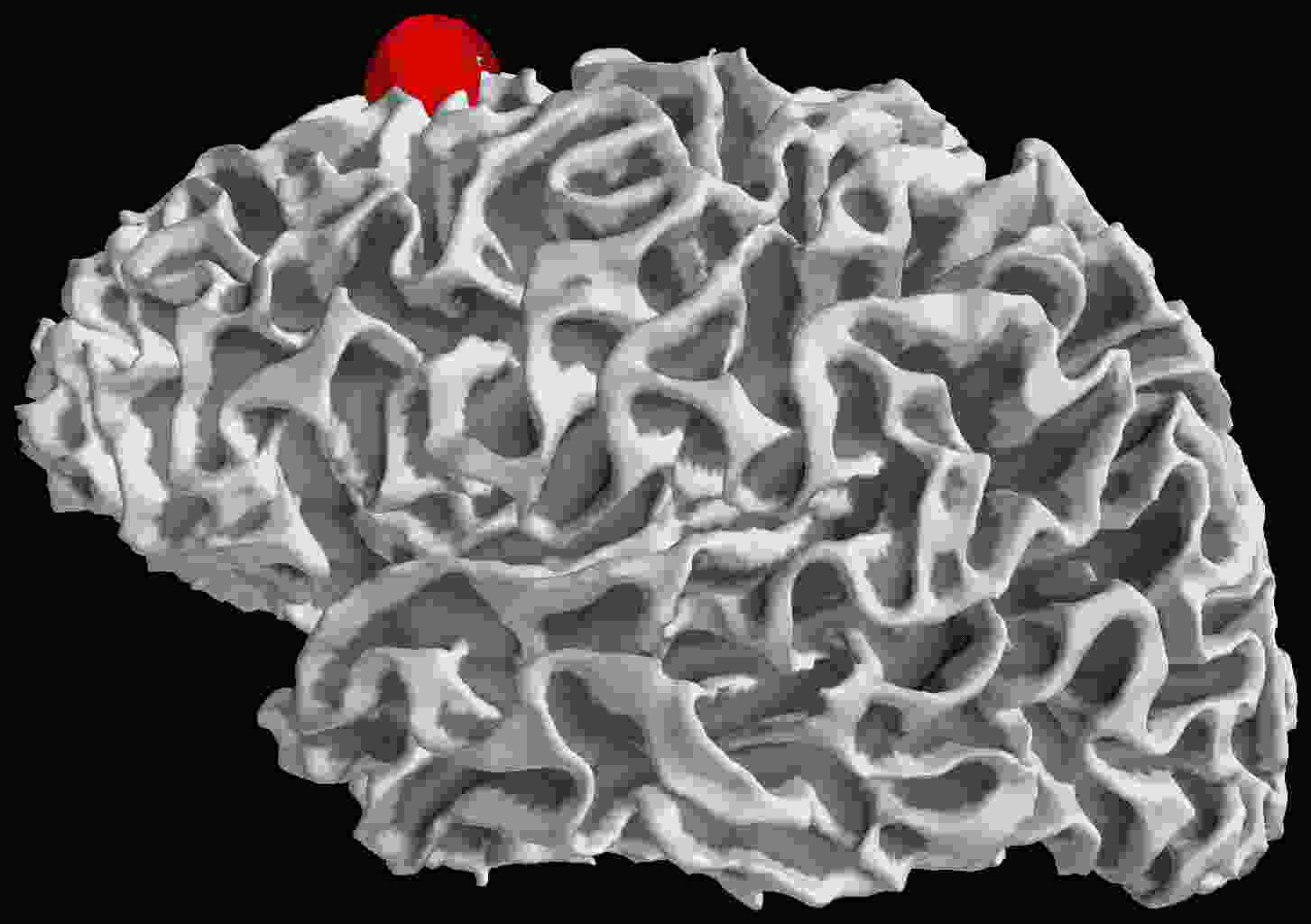}
    \end{subfigure}\hspace{\spacefig em}
  \begin{subfigure}{\figsize\textwidth}
  \includegraphics[width=68px,height=47px]{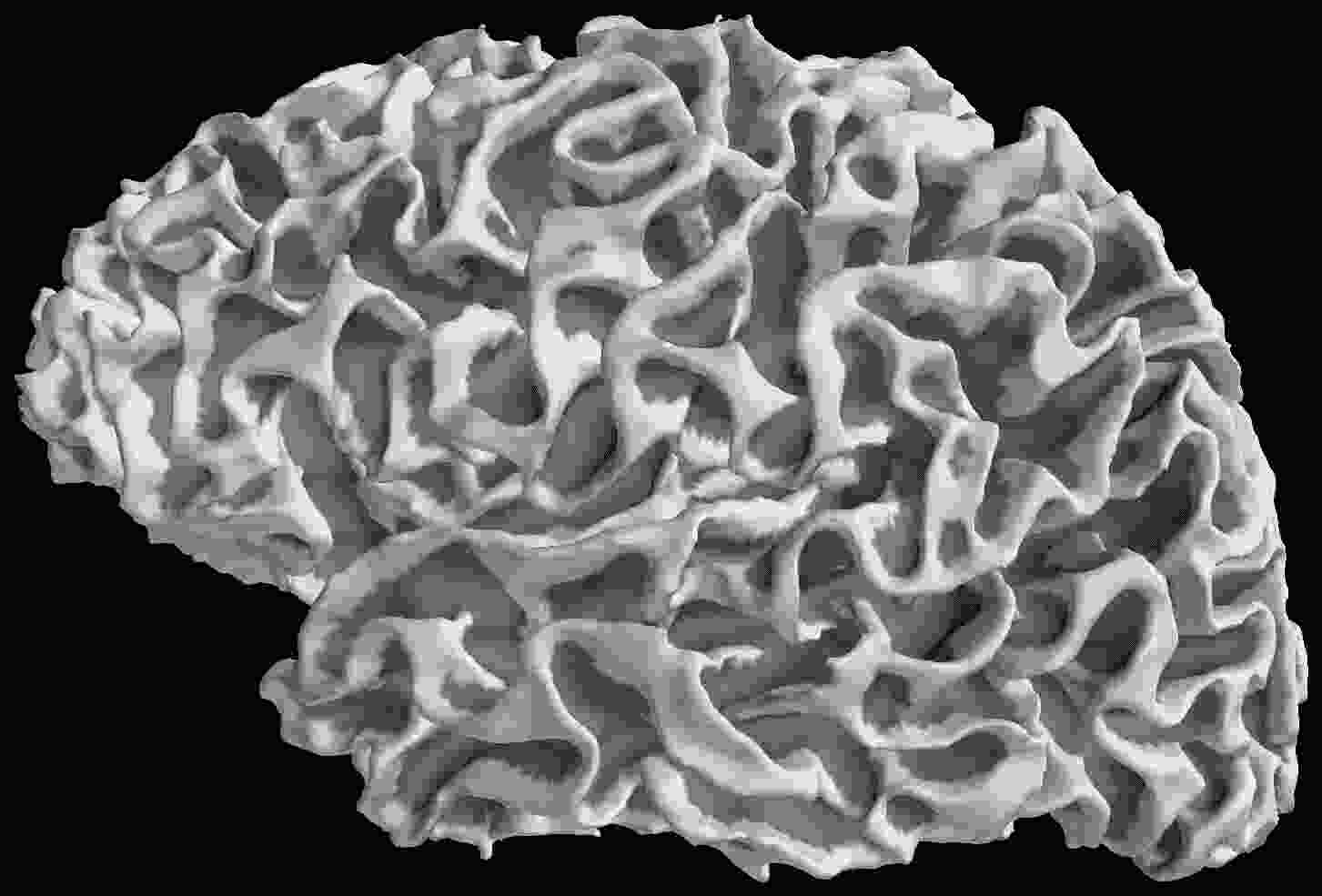}
  \end{subfigure}\hspace{\spacefig em}
  \begin{subfigure}{\figsize\textwidth}
  \includegraphics[width=68px,height=47px]{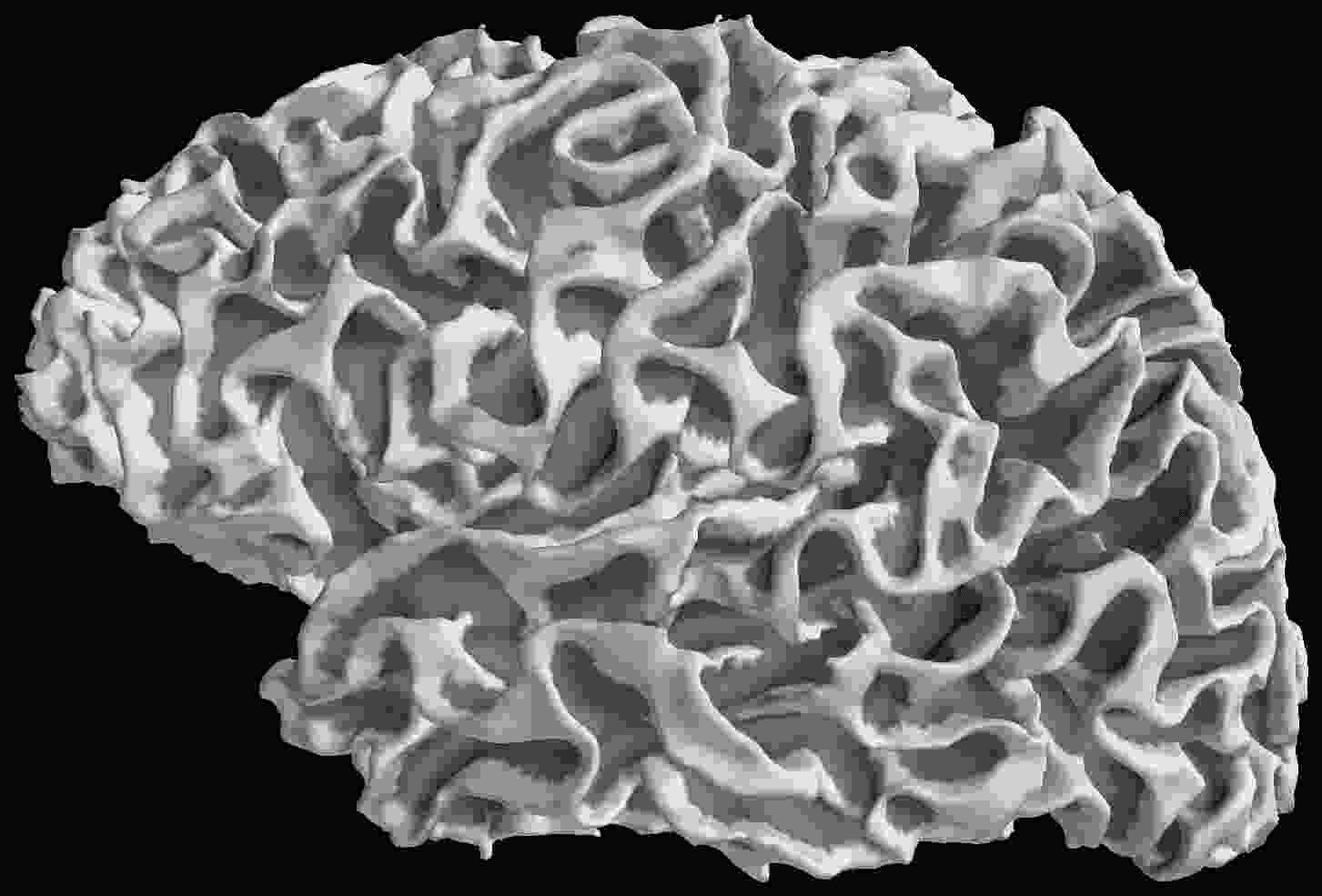}
  \end{subfigure}\hspace{\spacefig em}
  \begin{subfigure}{\figsize \textwidth}
    \includegraphics[width=68px,height=47px]{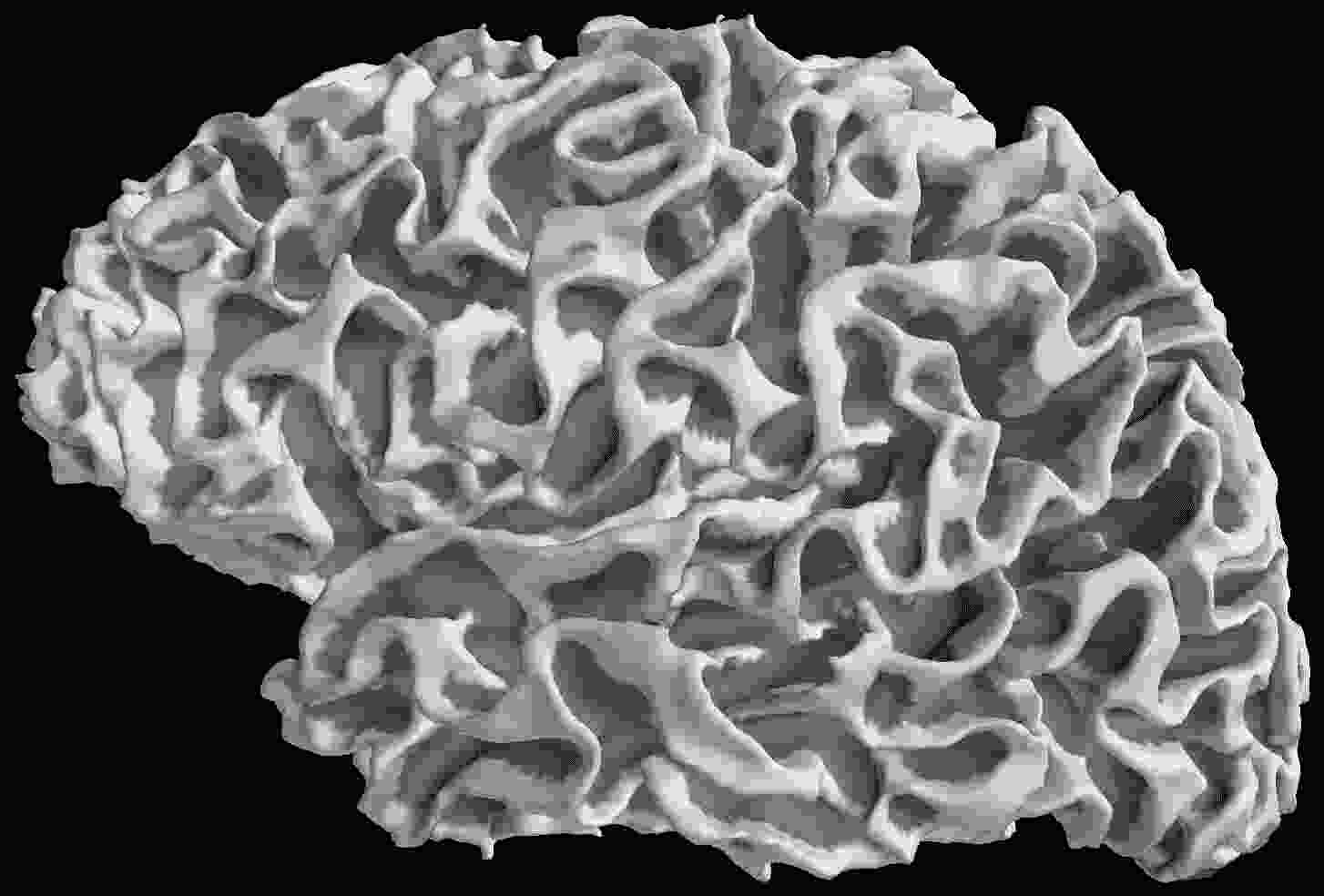}
    \end{subfigure}\hspace{\spacefig em}

  \vspace{-0.1em}
  \begin{subfigure}{\figsize\textwidth}
    \includegraphics[width=68px,height=47px]{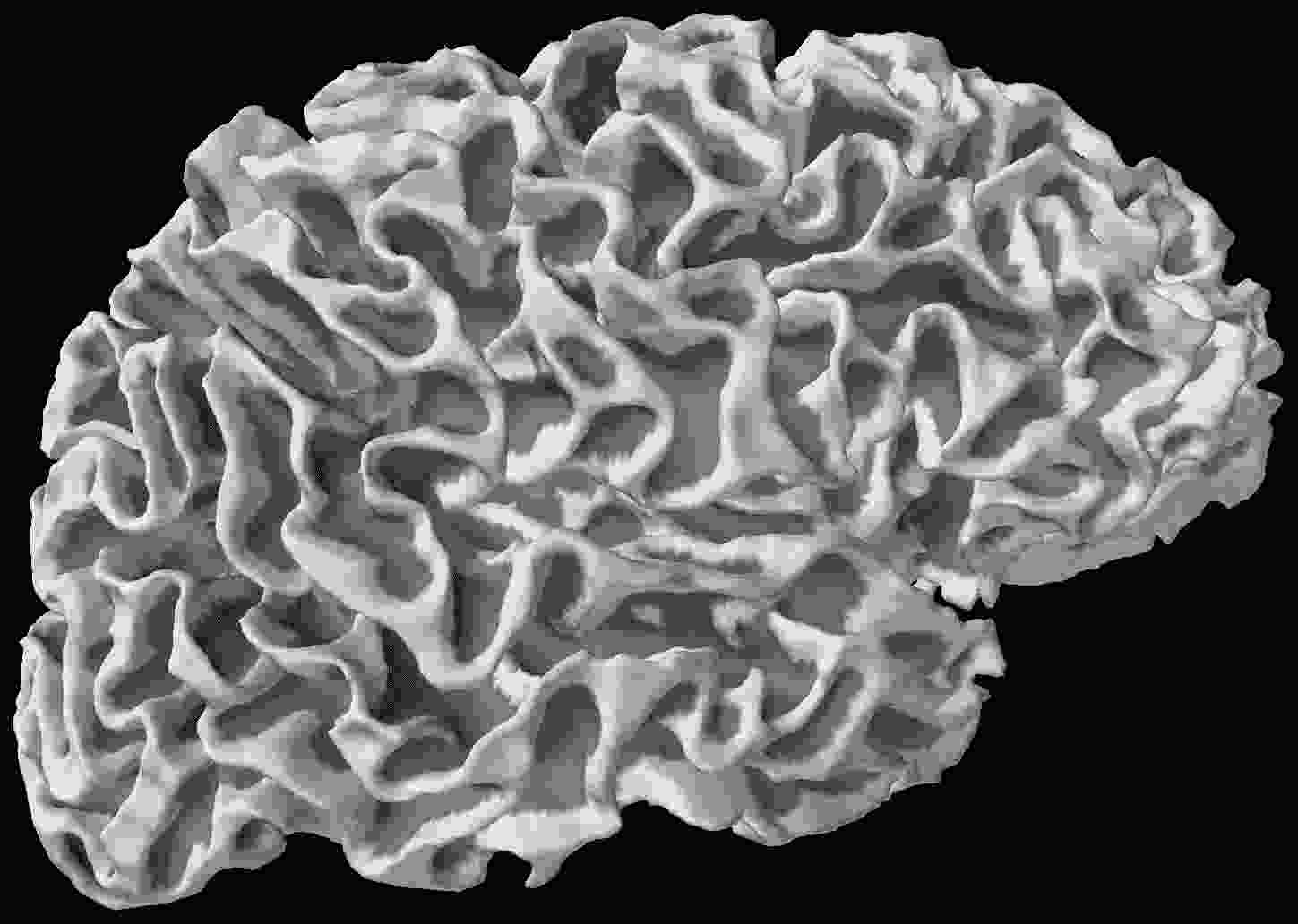}
    \caption{\us}
  \end{subfigure}\hspace{\spacefig em}
  \begin{subfigure}{\figsize\textwidth}
    \includegraphics[width=68px,height=47px]{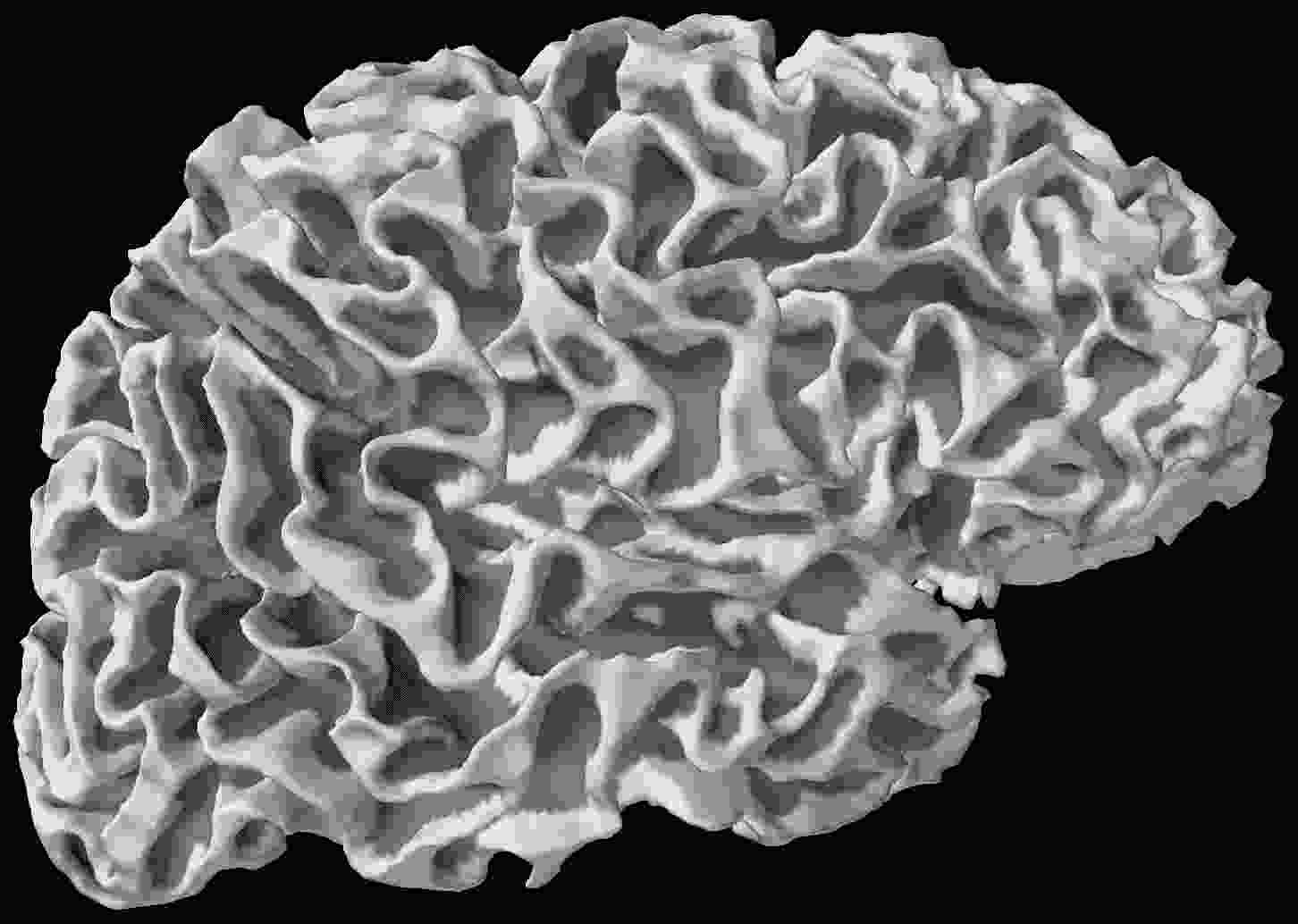}
    \caption{\sgcl}
  \end{subfigure}\hspace{\spacefig em}
  \begin{subfigure}{\figsize\textwidth}
    \includegraphics[width=68px,height=47px]{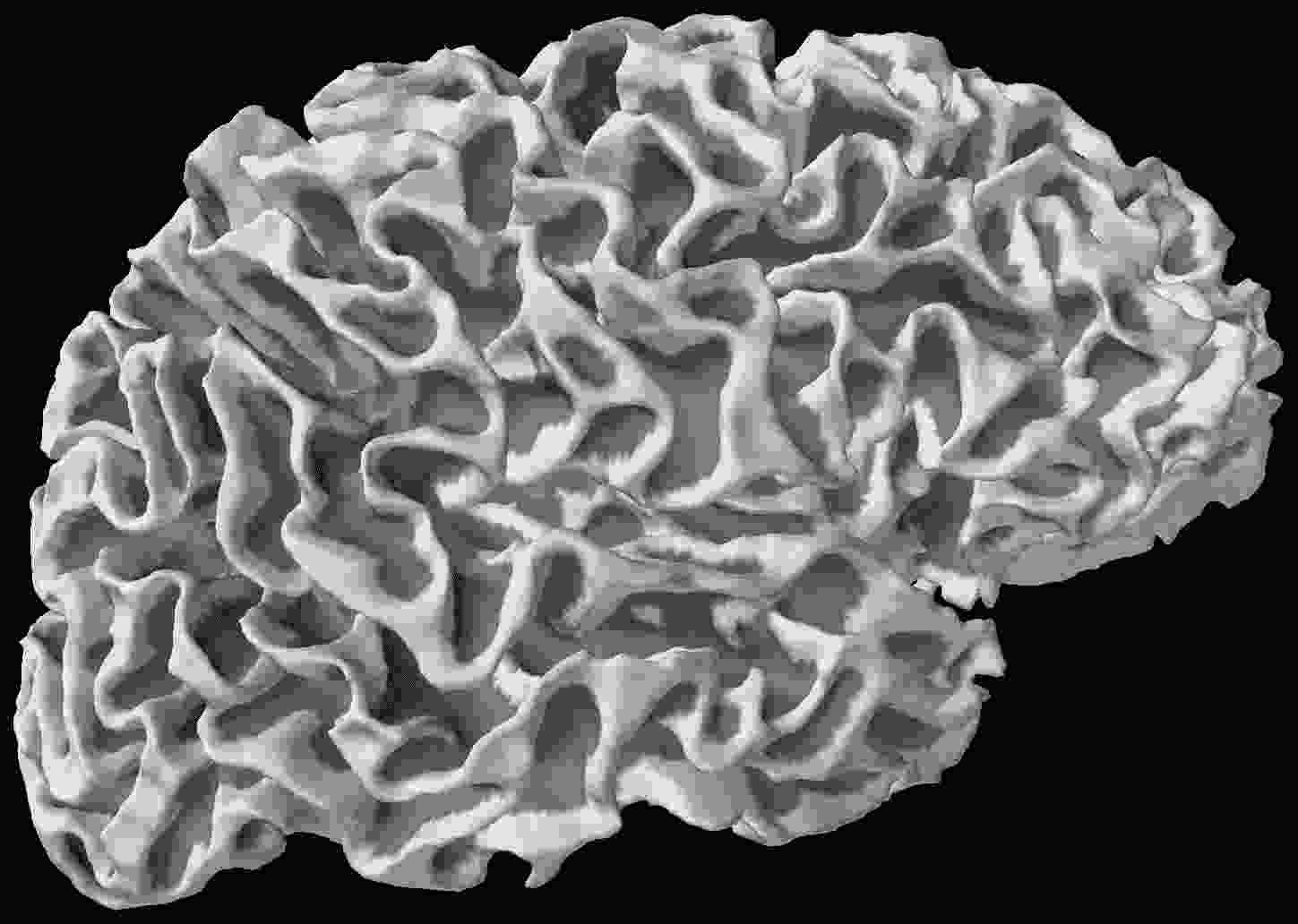}
    \caption{\mler}
    \end{subfigure}\hspace{\spacefig em}
  \begin{subfigure}{\figsize\textwidth}
    \includegraphics[width=68px,height=47px]{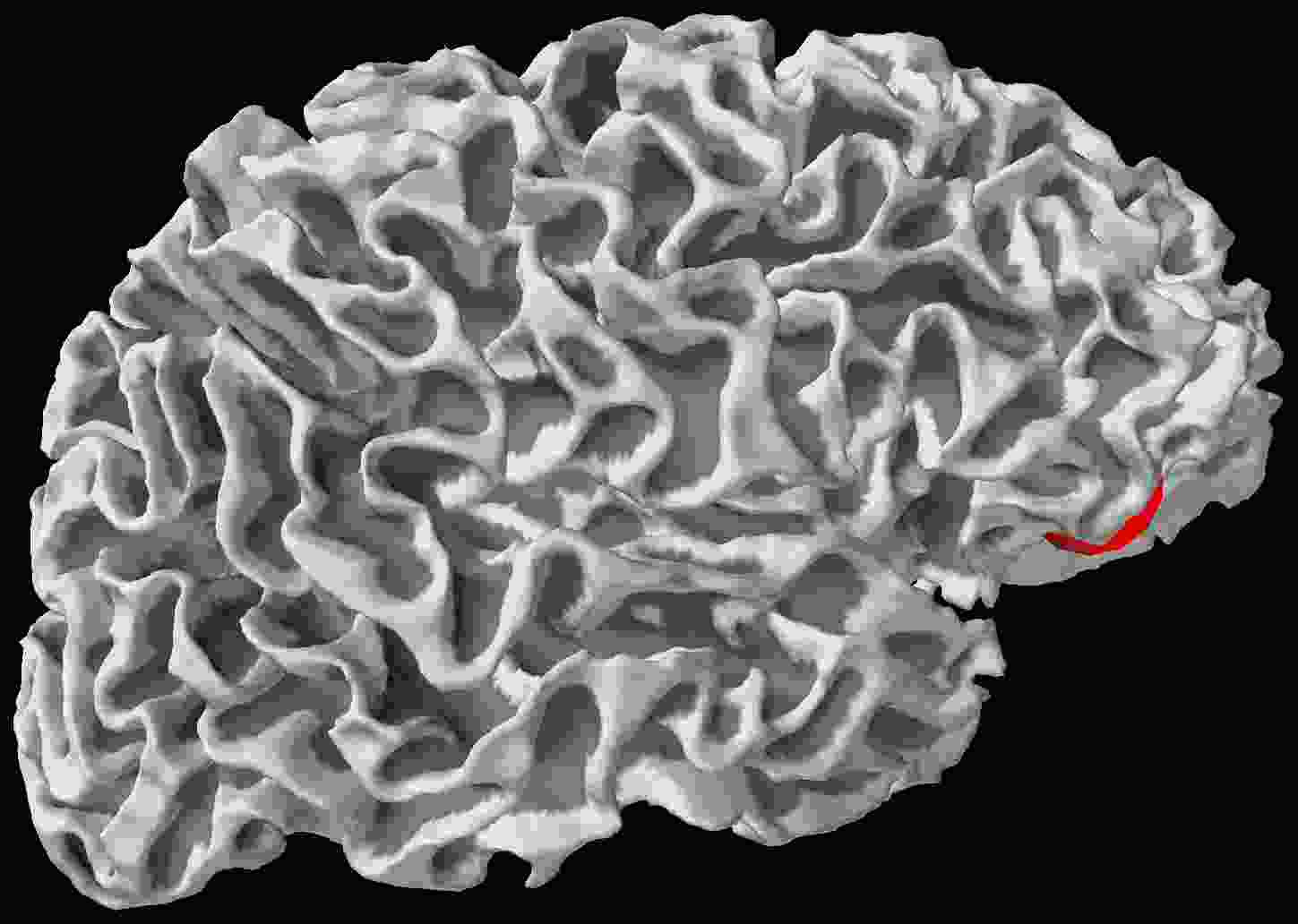}
    \caption{\mle}
    \end{subfigure}\hspace{\spacefig em}
  \begin{subfigure}{\figsize\textwidth}
    \includegraphics[width=68px,height=47px]{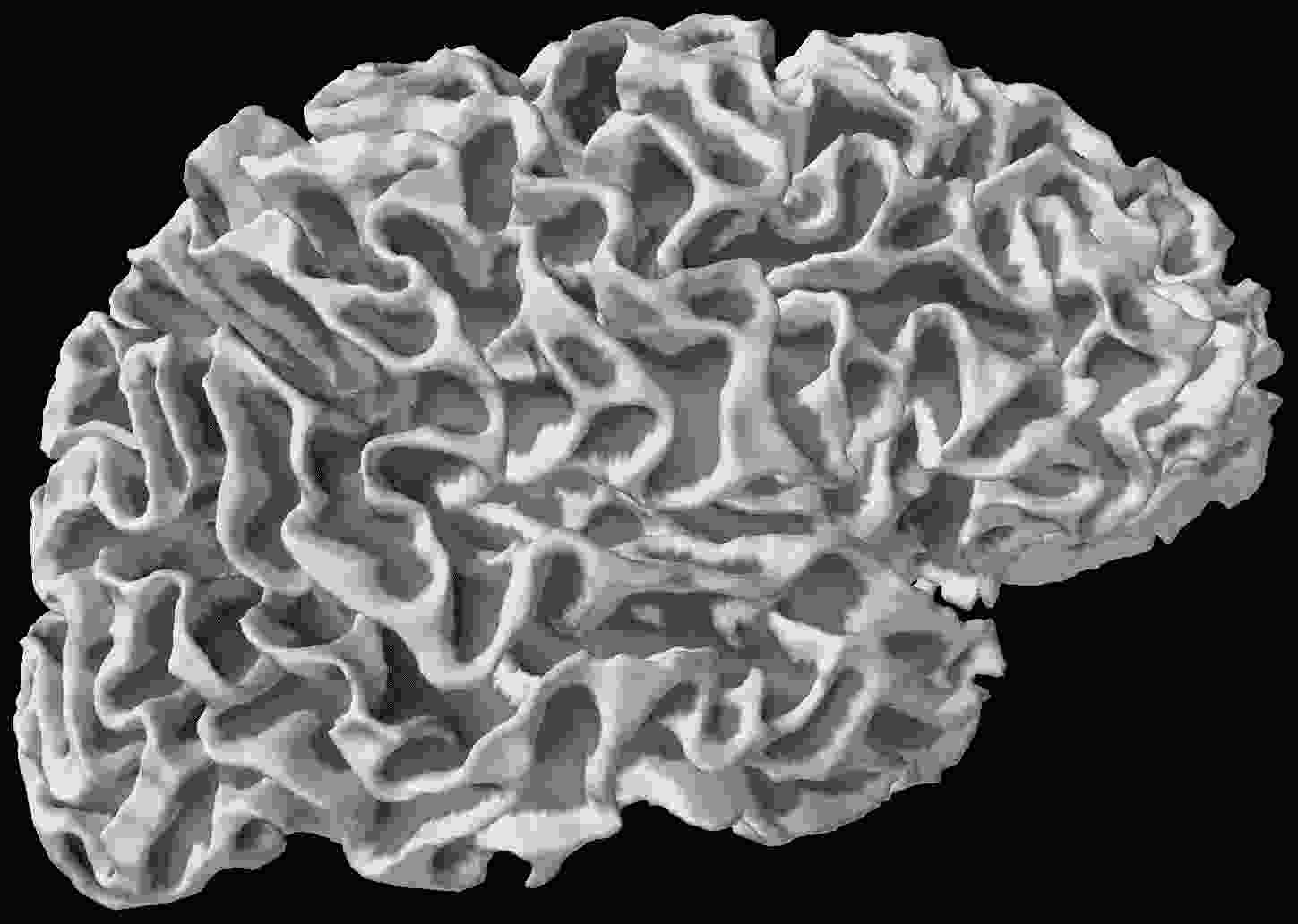}
    \caption{\mrcer}
    \end{subfigure}\hspace{\spacefig em}
  \begin{subfigure}{\figsize\textwidth}
    \includegraphics[width=68px,height=47px]{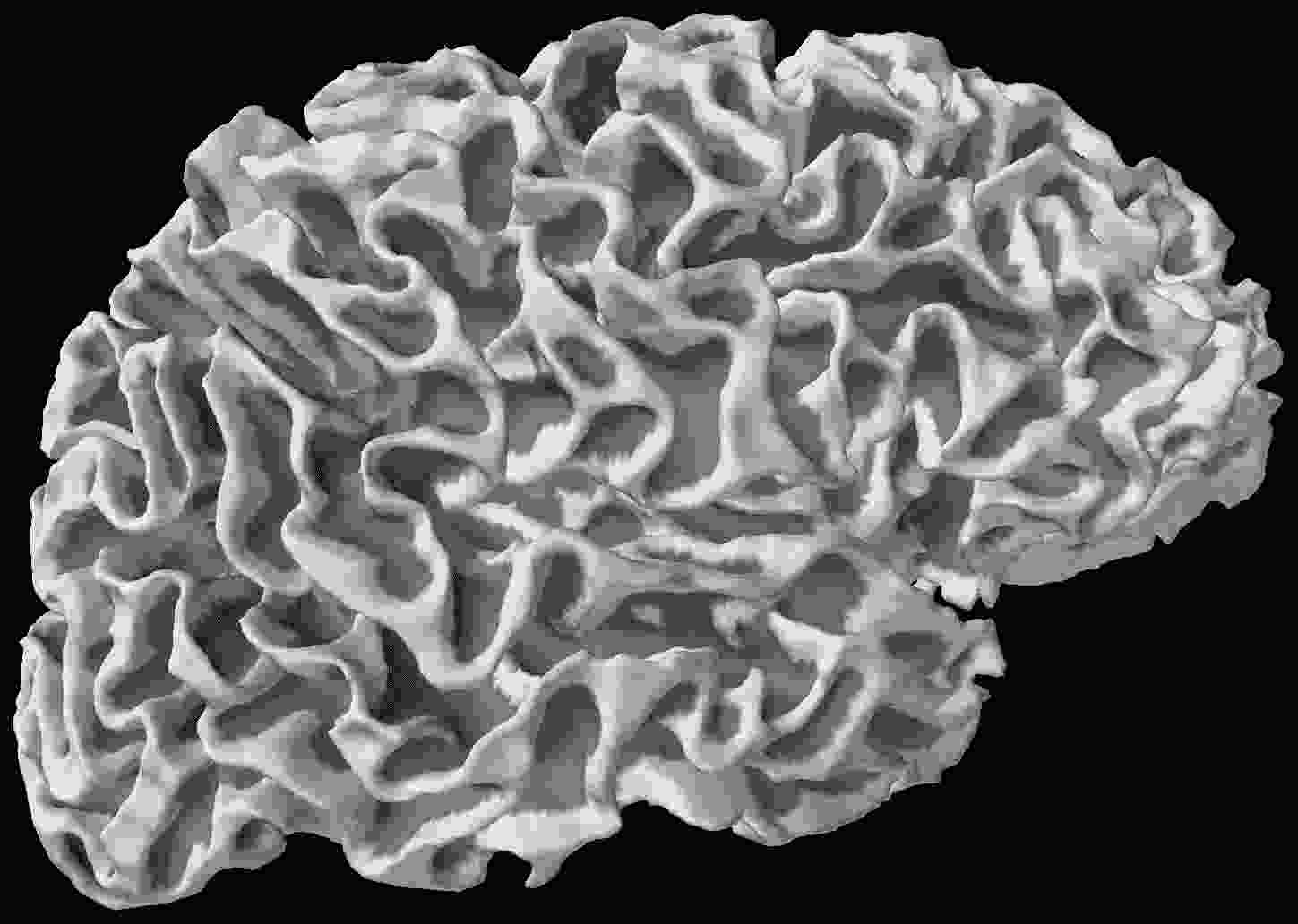}
    \caption{\mtl}
  \end{subfigure}\hspace{\spacefig em}
  \caption{\textit{Real data ($n=102$, $q=7498$, $q=48$, $r=31$)} Sources found in the left hemisphere (top) and the right hemisphere (bottom) after right visual stimulations.}
  \label{fig:real_data_right_visu_r_31}
\end{figure}

\end{document}